\documentclass{article} 
\usepackage{iclr2027_conference,times}

\usepackage{amsmath,amsfonts,bm}

\def\eqref#1{equation~\ref{#1}}

\def\1{\bm{1}}

\DeclareMathAlphabet{\mathsfit}{\encodingdefault}{\sfdefault}{m}{sl}
\SetMathAlphabet{\mathsfit}{bold}{\encodingdefault}{\sfdefault}{bx}{n}

\usepackage{hyperref}
\usepackage{url}
\hypersetup{colorlinks=true,linkcolor=blue,citecolor=blue,urlcolor=blue,breaklinks=true}

\usepackage{amsthm,placeins}
\newtheorem{fevtheorem}{Theorem}[section]
\newtheorem{fevdefinition}[fevtheorem]{Definition}
\newtheorem{fevlemma}[fevtheorem]{Lemma}
\newtheorem{fevproposition}[fevtheorem]{Proposition}
\newtheorem{fevcorollary}[fevtheorem]{Corollary}
\theoremstyle{remark}
\newtheorem{fevremark}[fevtheorem]{Remark}
\theoremstyle{plain}
\usepackage{titletoc}
\usepackage{booktabs,array,graphicx,xcolor,colortbl}

\providecolor{FEline}{HTML}{C6CBD0}
\definecolor{FlexEvoAppGreen}{HTML}{008378}
\definecolor{FlexEvoAppRed}{HTML}{AD5363}
\definecolor{FlexEvoAppLine}{HTML}{C6CBD0}
\definecolor{FlexEvoAppLabel}{HTML}{363B40}

\providecommand{\FlexEvoAppBetter}[1]{%
  \textcolor{FlexEvoAppGreen}{#1}}
\providecommand{\FlexEvoAppWorse}[1]{%
  \textcolor{FlexEvoAppRed}{#1}}

\providecommand{\FlexEvoAppHeader}[1]{%
  \shortstack[c]{#1}}

\providecommand{\FlexEvoAppTest}[2]{%
  \addlinespace[2.5pt]
  \multicolumn{#1}{@{}l}{%
    \textcolor{FlexEvoAppLabel}{\textbf{#2}}%
  }\\[1pt]}
\usepackage{xcolor,colortbl}
\providecolor{FEline}{HTML}{C6CBD0}
\definecolor{FEAline}{HTML}{C6CBD0}
\definecolor{FEAGreen}{HTML}{008378}
\colorlet{FEABest}{FEAGreen!30}   
\colorlet{FEASecond}{FEAGreen!10} 
\usepackage{booktabs,array,xcolor,caption,graphicx}
\definecolor{FEgain}{HTML}{008378}
\definecolor{FEloss}{HTML}{BE4057}
\newcommand{\FlexEvoTableEntry}[3]{#1\raisebox{1.5pt}{\hspace{0.25pt}%
  \textcolor{#2}{\fontsize{4.7}{4.7}\selectfont #3}}}
\newcommand{\FlexEvoTableGain}[2]{\FlexEvoTableEntry{#1}{FEgain}{#2}}
\newcommand{\FlexEvoTableLoss}[2]{\FlexEvoTableEntry{#1}{FEloss}{#2}}
\newcommand{\FlexEvoTestHeader}[1]{\multicolumn{1}{c}{\fontsize{6}{6.5}\selectfont #1}}
\newcommand{\FlexEvoScenario}[1]{\multicolumn{19}{@{}l}{\textit{#1}}\\[-0.3pt]}

\usepackage{booktabs,multirow,array,amssymb,amsmath}
\usepackage{colortbl,xcolor,rotating}

\usepackage{amsmath}
\usepackage{amssymb}
\usepackage{amsfonts}
\usepackage{booktabs}  
\usepackage{multirow}  
\usepackage{graphicx}  
\usepackage{multirow}
\usepackage{booktabs}
\usepackage{threeparttable}
\usepackage{siunitx}
\usepackage{algorithm}
\usepackage{algpseudocode}
\usepackage{amsmath,amssymb} 
\definecolor{TopOne}{RGB}{91,155,213}      
\definecolor{TopTwo}{RGB}{221,235,247}     
\definecolor{TopThree}{RGB}{217,217,217}   

\usepackage{mathrsfs}

\theoremstyle{remark}

\usepackage{booktabs}
\usepackage{array}
\usepackage{xurl}
\usepackage{tabularx}

\title{Robust Biomolecular Complex Design Across Protein Conformational Landscapes}

\author{
\begin{tabular}{@{}l@{}}
\textbf{Qingyuan Zeng}$^{1}$\thanks{These authors contributed equally to this work.}, \textbf{Zongqi Xu}$^{2,1}$\footnotemark[1], \textbf{Anglin Liu}$^{1}$\footnotemark[1], \textbf{Ziqi Gong}$^{1}$, \textbf{Pengxiang Cai}$^{1}$,\\
\textbf{Zixin Guan}$^{3}$, \textbf{Yunan Chen}$^{1}$, \textbf{Sen Gao}$^{1}$, \textbf{Min Zhou}$^{1}$, \textbf{Jintai Chen}$^{1,4}$\thanks{Corresponding author}\\
$^{1}$The Hong Kong University of Science and Technology (Guangzhou)\\
$^{2}$Southwest Minzu University\\
$^{3}$Guangzhou University of Chinese Medicine\\
$^{4}$The Hong Kong University of Science and Technology\\
\texttt{jintaichen@hkust-gz.edu.cn}
\end{tabular}%
}
\iclrfinalcopy 
\begin{document}

\maketitle

\begin{abstract}

Proteins populate conformational ensembles, yet structure-based biomolecular design typically optimizes candidates against a single target conformation. Consequently, a candidate that fits one state can lose favorable interactions or develop steric clashes when the target adopts another. We introduce \textsc{FlexEvo}, a model-agnostic evolutionary framework that adapts candidates once at inference time from a single target conformation to improve compatibility with alternative natural conformations unseen during adaptation, without retraining the source model or requiring a conformational ensemble. \textsc{FlexEvo} casts cross-state adaptation as geometry-constrained bi-objective optimization, balancing preservation of input-state interactions against robustness to plausible conformational perturbations. To limit the search space and reduce invalid structural edits, geometry-derived \emph{FlexBoxes} define protected anchor regions, adaptable regions for local exploration, and forbidden regions for clash avoidance. A unified all-atom representation supports topology-preserving adaptation across diverse binder categories, while Pareto selection preserves nondominated candidates across the two objectives. We evaluate \textsc{FlexEvo} across multiple generation baselines and nine representative binder categories spanning diverse molecular sizes and structural topologies. \textsc{FlexEvo} reduces the category-balanced mean relative
performance degradation from 47.8\% to 4.4\%, while adding only
1.4--3.1 minutes of adaptation per sample.
These results establish single-state inference-time adaptation
as a practical route toward robust biomolecular complex design
across protein conformational landscapes.

\end{abstract}

\section{Introduction}

Structure-based prediction and design of biomolecular complexes provide a computational route to developing molecular binders for target proteins. Recent advances in geometric learning and generative modeling have enabled the prediction and generation of three-dimensional candidate complexes conditioned on target protein structures \citep{targetdiff,decompdiff,pepglad,rfdiffusion}. However, many existing approaches rely on a single-conformation assumption at inference time, optimizing candidates against one supplied or predicted realization of the binding pocket or interaction surface \citep{flexsbdd,dynamicflow}. Proteins instead populate conformational ensembles, where transitions among apo, holo, and alternative functional states can reshape binding geometries, disrupt favorable contacts, or introduce steric clashes \citep{boehr2006dynamic,boehr2009,dynamicbind}. Consequently, a candidate optimized for one target state may lose compatibility when the same target adopts another naturally occurring conformation.

\begin{figure}
    \centering
    \includegraphics[width=1\linewidth]{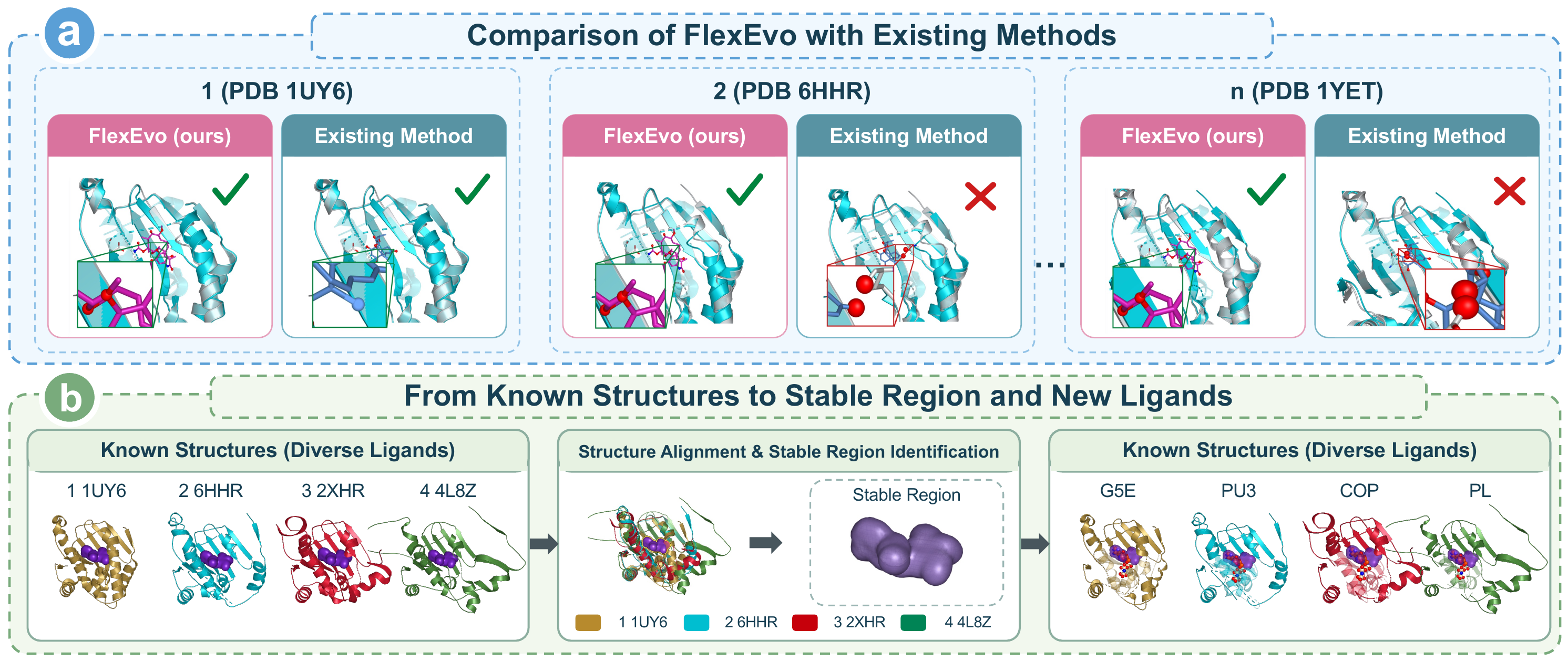}
    \caption{Motivation for cross-conformation ligand design. (a) Comparison of FlexEvo and single-state design methods, illustrating ligand compatibility across alternative target conformations. (b) Alignment of known protein–ligand complexes highlights a shared stable binding region, motivating ligand design that preserves conserved interactions while accommodating conformational variation.}
    \label{fig:1}
\end{figure}

Existing approaches can be broadly categorized into two settings. Static structure-conditioned predictors and generators treat the target as a fixed structural input, leaving conformational variation outside the candidate-design objective \citep{targetdiff,pepglad,rfdiffusion}. Flexibility-aware methods relax this assumption by incorporating receptor ensembles \citep{huang2006ensemble}, modeling conformational motions \citep{induced_fit_docking,dynamicbind}, or jointly generating receptor--binder complexes \citep{neuralplexer,flexsbdd,dynamicflow}, but this does not by itself establish that a single candidate will remain compatible with natural target conformations unseen during optimization. This distinction matters in practice because the target state encountered after design is often unknown: a candidate optimized for one state may lose favorable contacts or develop steric clashes in another. This gap raises a central question: \textit{can a candidate be adapted once from a single observed conformation and still remain compatible with multiple natural conformations of the same target, without redesigning it for each state?}

Achieving this goal requires accommodating unobserved target variation while preserving favorable interactions in the observed state. Target conformational changes range from side-chain and loop rearrangements in compact pockets to backbone and interface remodeling across extended surfaces, while binder adaptation may require coordinated changes in pose, torsions, backbone geometry, and side-chain conformations \citep{dynamicbind,flexpepdock,rosettadock4}. Moreover, the number of atomic variables grows with binder size, rapidly expanding the search space. Because bonds, torsions, and steric constraints couple these variables, unrestricted edits can produce invalid geometries \citep{buttenschoen2024posebusters}. Beyond biomolecular design, enforcing geometric constraints has also been shown to improve perceptual grounding in medical multimodal models \citep{medscout}, and related medical deep learning work has applied attention and convolutional architectures to IVIM-DWI and pharmacokinetic MRI analysis \citep{liu2022ivim,zeng2021attentionmvi,zeng2021pharmacokinetic}.

To address these challenges, we introduce \textsc{FlexEvo}, a model-agnostic evolutionary adaptation framework that refines candidates from existing structure-conditioned predictors and generators without modifying their source models. Given a single observed target conformation and an initial candidate population, \textsc{FlexEvo} formulates cross-state adaptation as bi-objective optimization over a geometry-constrained search space, balancing retention of input-state interactions against robustness to plausible conformational perturbations. Perturbation robustness is used as a proxy for compatibility with unseen natural target states. To accommodate structural diversity across binder classes, \textsc{FlexEvo} combines a unified all-atom representation of target--candidate complexes with binder-specific structural variation and repair. To limit the search space and reduce invalid structural edits, \textsc{FlexEvo} maps local packing and contact geometry of the input complex, together with protein excluded volume, to local \emph{FlexBoxes}: protected anchor regions, adaptable regions for local exploration, and forbidden regions for clash avoidance. \textsc{FlexEvo} uses structural feedback to update the spatial
constraints guiding adaptation, with candidate moves guided by
target geometry and available bond topology.
Pareto selection retains nondominated solutions across the
two objectives. Related lines of work similarly couple evolutionary search with model robustness and multi-objective trade-offs, including gradient--evolutionary multiform attacks \citep{gong2025crossmodality}, attention-shift cross-task attacks \citep{zeng2024crosstask}, and decision-based black-box attacks on image-to-text models \citep{zeng2024aaa}. Our contributions can be summarized as follows. \begin{itemize} 

\item We formulate inference-time binder adaptation from a
single target conformation as a geometry-constrained
bi-objective optimization, balancing input-state
interaction preservation with robustness to reference-derived
geometric perturbations.
Adaptation requires neither source-model retraining nor
access to alternative natural target conformations.

\item We propose \textsc{FlexEvo}, which couples an outer
evolutionary search over residue-level constraint configurations
with inner topology-preserving candidate adaptation guided
by decoded \emph{FlexBoxes}.
Structural feedback updates the configurations, while Pareto
selection retains nondominated preservation--robustness
trade-offs.

\item We conduct extensive evaluations across nine representative
binder categories spanning diverse molecular sizes, structural
topologies, and interaction geometries, using multiple generation
baselines. \textsc{FlexEvo} reduces the category-balanced mean
relative performance degradation from 47.8\% to 4.4\%, while adding only 1.4--3.1 minutes per sample.

\end{itemize}

\section{Related Work}

\subsection{Single-State Biomolecular Complex Prediction and Design}

Existing methods predict or design complexes around a single target realization. For small molecules, Pocket2Mol uses autoregressive generation, whereas TargetDiff, DiffSBDD, and DecompDiff use pocket-conditioned diffusion \citep{pocket2mol,targetdiff,diffsbdd,decompdiff}. Complementary to these single-target generators, REUSE recovers dual-target molecules by evolving the input noise of a frozen single-target diffusion model without retraining or altering denoising dynamics \citep{zeng2026reuse}. PepGLAD extends full-atom generation to peptides, while RFdiffusion and BindCraft design target-specific protein binders \citep{pepglad,rfdiffusion,bindcraft}. General-purpose predictors, including AlphaFold~3, Chai-1, Protenix, and BoltzGen, predict complexes spanning these binder classes \citep{alphafold3,chai1,protenix,boltzgen}. Despite their different architectures, these methods optimize or assess candidates against one realization of the target geometry, effectively treating it as rigid. Because target conformational variation is absent from the design objective, a candidate that fits one state may lose favorable contacts or develop steric clashes in another.

\subsection{Flexibility-Aware Biomolecular Complex Modeling}

To relax this rigid-target assumption, classical methods explicitly search receptor and binder flexibility through induced-fit docking, peptide folding and docking, or protein backbone ensembles \citep{induced_fit_docking,flexpepdock,rosettadock4}. Ensemble docking and multistate design further consider multiple supplied structures during optimization \citep{huang2006ensemble,amaro2018ensemble,sauer2020multistate}. Learning-based methods such as NeuralPLexer and DynamicBind model coupled receptor--ligand changes, while FlexSBDD, DynamicFlow, and YuelDesign incorporate protein flexibility into ligand generation \citep{neuralplexer,dynamicbind,flexsbdd,dynamicflow,yueldesign}. These methods model target flexibility, but do not ensure compatibility when the target state encountered after design is unknown. State-specific methods construct a favorable complex for one recovered conformation, leaving compatibility with other states unresolved. Ensemble-based methods broaden coverage only over conformations available beforehand; however, available structures often represent only a small fraction of the natural ensemble, and the relevant post-design state may not be known in advance. FlexEvo instead adapts one candidate from one observed conformation for compatibility with unseen natural states, extending this capability across diverse binder categories. Outside biomolecular docking, bilevel evolutionary inference of interpretable dynamics from visual observations \citep{lin2026visionlaw} and bi-level contextualization of broad biomedical knowledge into scenario-grounded propositions \citep{zeng2026scene} similarly couple upper-level search with lower-level evidence feedback.

\subsection{Evolutionary Optimization, Graph Learning, and Robustness}
\label{sec:related_evo_robust}

Beyond biomolecular modeling, multi-objective evolutionary search has been used to generate diagnostic and actionable explanations for fair graph neural networks \citep{wang2024fairgnn} and to learn multiview subgraph representations under scarce labels \citep{wang2025muse}. Gradient--evolutionary multiform optimization improves cross-modality attack transferability \citep{gong2025crossmodality}, while related robustness studies examine attention-based cross-task attacks \citep{zeng2024crosstask}, decision-based black-box attacks on image-to-text models \citep{zeng2024aaa}, black-box removal of 3DGS and diffusion watermarks \citep{zeng2026digitalink,zeng2025watermark}, dual-saliency sample selection for backdoor evaluation \citep{zhang2025dualcam}, and structured-irrelevance training for robust neural PDE solvers \citep{gong2025piif}.

\section{Methodology}
\subsection{Problem Setup and Overview}
\label{sec:setup_overview}

Given a single observed target conformation $P^0$ and an initial
pool $\mathcal{X}_0$ of binder structures positioned relative to it,
\textsc{FlexEvo} adapts candidates from an existing predictor or
generator once at inference time.
The goal is to preserve favorable interactions with $P^0$ while
improving compatibility with unseen natural conformations of the
same target.
Adaptation uses only $P^0$ and $\mathcal{X}_0$, without retraining
the source model; alternative natural conformations are reserved
for evaluation.
Let $G\in\mathcal{G}(P^0)$ denote an admissible residue-level
constraint configuration (\emph{Box genome}), and let
$\mathcal{X}_G=\mathcal{A}(\mathcal{X}_0;P^0,G)$ denote the
candidate pool obtained through constrained editing and repair.
We formulate the search over $G$ as a bi-objective surrogate problem:
\begin{equation}
\label{eq:problem_setup}
\max_{G\in\mathcal{G}(P^0)}
\quad
\bigl(
F_{\mathrm{pres}}(\mathcal{X}_G;P^0,G),
F_{\mathrm{rob}}(\mathcal{X}_G;P^0,G)
\bigr).
\end{equation}
Here, $F_{\mathrm{pres}}$ and $F_{\mathrm{rob}}$ are genome-level
scores for interaction preservation and perturbation robustness,
respectively, with geometric validity included in the latter.
Robustness is assessed using local geometric perturbations derived
from $P^0$ under $G$, as a proxy for cross-state compatibility. As illustrated in Fig.~\ref{fig:method}, each genome encodes
region labels and box scales that decode into \emph{FlexBoxes},
specifying where and how candidates can be edited.
For each genome, the inner procedure adapts the same initial
pool $\mathcal{X}_0$; candidate scores are then aggregated for
genome evaluation.
The outer search uses genome variation and structural feedback
to update constraints, while Pareto-based selection retains
configurations with different objective trade-offs.

\subsection{Evolvable FlexBox Prior}
\label{sec:flexbox}

Static packing and radial position provide structural cues to protein fluctuations \citep{jamroz2012,lin2008,shih2007}.
Studies of protein--protein recognition further show that
preorganized anchors can coexist with local side-chain
adjustment \citep{rajamani2004}. These observations, together with the need to avoid steric
clashes \citep{buttenschoen2024posebusters}, motivate distinct editing roles for anchoring, local adaptation, and avoidance. The outer search in \textsc{FlexEvo} evolves the \emph{spatial prior governing
candidate adaptation}, separating the allocation of local editing freedom
from the structures produced under that prior.
We first screen frozen source-model outputs for basic geometric defects and
redundancy to obtain $\mathcal X_0$ (Fig.~\ref{fig:method}(1)).
This pool is constructed once, without using either evolutionary objective,
and is shared across genome evaluations.

From $P^0$, we compute fixed residue descriptors, centroids
$c_r\in\mathbb R^3$, and orthonormal frames $Q_r\in\mathbb R^{3\times3}$
for the $R$ represented target residues.
A Box genome specifies the residues covered by the prior, their editing
roles, and per-residue scales:
\begin{equation}
\label{eq:genome_representation}
\begin{aligned}
G&=\bigl(\mathcal R_G,\{\ell_r\}_{r\in\mathcal R_G},Z\bigr),
& Z&\in[z_{\min},z_{\max}]^R,\\
\mathcal K&=\{\mathrm{stable},\mathrm{flexible},\mathrm{forbidden}\},
& \ell_r&\in\mathcal K,\quad r\in\mathcal R_G,
\end{aligned}
\end{equation}
where $\mathcal R_G\subseteq\{1,\ldots,R\}$ and
$0<z_{\min}<z_{\max}$.
Only residues in $\mathcal R_G$ instantiate a \emph{FlexBox}:
\begin{equation}
\label{eq:flexbox_decode}
\mathcal B_r(G)=
\left\{y\in\mathbb R^3:
\left\|Q_r^\top(y-c_r)\right\|_\infty
\leq\frac{b_{\ell_r}z_r}{2}\right\},
\qquad r\in\mathcal R_G,
\end{equation}
with role-dependent base extent $b_{\ell_r}$. The three roles implement \emph{Bind--Adapt--Avoid} guidance
(Fig.~\ref{fig:method}(2)).
Stable FlexBoxes encourage proximity to anchor regions;
flexible FlexBoxes support broader local exploration;
forbidden FlexBoxes penalize occupancy and guide
clash-relieving corrections.
These labels specify candidate-editing roles rather than
measured residue dynamics; forbidden regions mark occupancy
to avoid, not target atoms assumed to be immobile.
A geometry-derived seed and geometry-biased variants initialize
the outer population.
Evolution changes the support, roles, and scales while the
residue descriptors, centers, and frames remain fixed.
Appendix~\ref{app:encoding} specifies their construction
and initialization.

\begin{figure}
    \centering
    \includegraphics[width=1\linewidth]{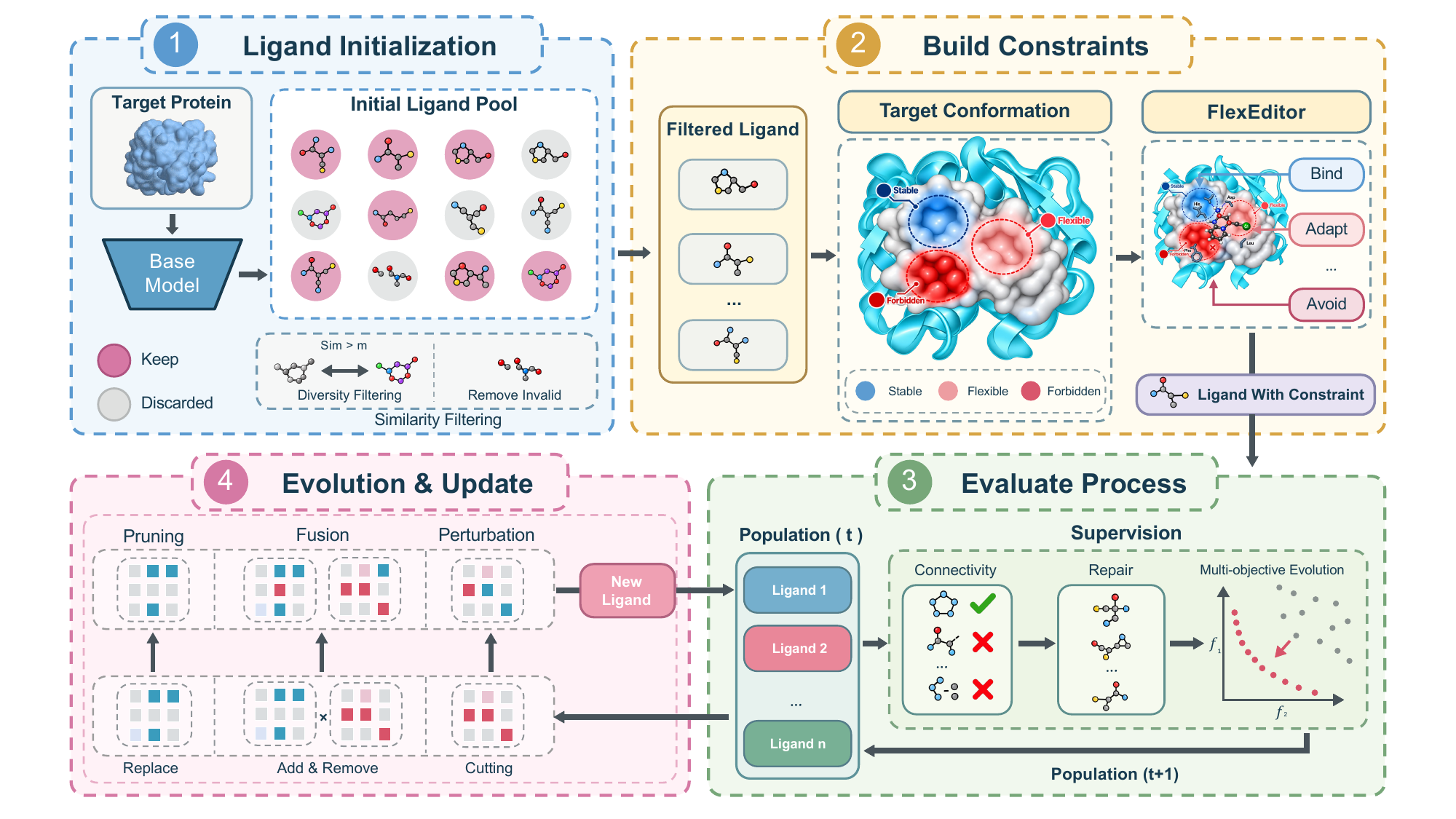}
    \caption{Overview of the FlexEvo framework. Starting from a single target conformation, FlexEvo filters an initial candidate pool and constructs geometry-derived FlexBoxes defining protected, adaptable, and forbidden regions. Iterative structural editing, repair, and Pareto-based selection balance input-state interaction preservation with robustness to conformational perturbations.}
    \label{fig:method}
\end{figure}

\subsection{Genome-Conditioned Adaptation and Single-State Robustness}
\label{sec:adapt_eval}

\paragraph{FlexBox-guided candidate adaptation.}
A genome is evaluated through the candidate population it produces.
For each $G$, \textsc{FlexEditor} starts from the shared $\mathcal X_0$ with
a common adaptation budget and reset random seed, keeping the decoded
FlexBoxes fixed.
The boxes guide the locations and types of candidate edits.
FlexEditor supports a broader coordinate-level operator family, whereas the reported molecular experiments use its graph-preserving mode, including rigid translations and rotations together with annotated bond-axis torsions. These candidate-side operations are summarized in the
lower row of Fig.~\ref{fig:method}(4). Let $\mathcal C_G^t$ be the inner population and
$\mathcal E_G^t\subseteq\mathcal C_G^t$ its reproduction subset.
For each parent $x$ and proposal index $v\in\{1,\ldots,n_v\}$, editing
$\mathcal T_G$ and geometric repair $\mathcal R$ produce
$\widehat x_{x,v}=\mathcal R(\mathcal T_G(x;\xi_{t,x,v}))$,
where $\xi$ denotes proposal randomness.
Writing $n(x)$ for the number of represented coordinates and $c(x)$ for
local neighbor support, an adaptation round is
\begin{equation}
\label{eq:inner_adaptation}
\begin{aligned}
\mathcal V_G^t
&=\left\{\widehat x_{x,v}:x\in\mathcal E_G^t,\ 1\leq v\leq n_v,
\ n(\widehat x_{x,v})>0,\ c(\widehat x_{x,v})\geq\tau_c\right\},\\
\mathcal C_G^{t+1}
&=\mathcal S_x\!\left(\mathcal C_G^t\cup\mathcal V_G^t\right),
\qquad \mathcal X_G=\mathcal C_G^T.
\end{aligned}
\end{equation}
Here $T$ is the number of inner rounds, and $\mathcal S_x$ combines
nondominated sorting, objective-space crowding, and candidate-space niching.
Repair and screening precede survivor selection, so proposals failing the
specified geometric checks cannot compensate through favorable scores.
The edit primitives and geometric screening criteria are specified in
Appendix~\ref{app:adaptation}.

\paragraph{Single-state perturbation probes.}
Good contact with $P^0$ alone does not establish tolerance to target motion.
We therefore construct a genome-conditioned probe collection
$\widetilde{\mathcal P}_G=\{\widetilde P_G^0,\ldots,\widetilde P_G^M\}$,
with $\widetilde P_G^0=P^0$, by applying local radial and tangential
perturbations to target regions selected by the prior.
Candidate coordinates remain fixed across probes: evaluation measures the
response of the \emph{same} candidate, not its performance after separate
re-optimization for each state.
The probes are synthetic geometric stress tests constructed solely from
$P^0$; alternative natural conformations are reserved for evaluation. For candidate $x$, let $q_m$ and $\chi_m$ denote contact quality and clash
fraction under $\widetilde P_G^m$. Over $m=0,\ldots,M$, define
\begin{equation}
\label{eq:perturbation_profile}
\begin{gathered}
q_{\mathrm{rob}}
=\rho\min_m q_m+(1-\rho)\bar q,
\qquad
\bar q
=\frac{1}{M+1}\sum_m q_m,
\\[2pt]
\Delta_q
=\max_m q_m-\min_m q_m,
\qquad
\chi_{\max}
=\max_m\chi_m,
\qquad
\kappa_q
=\frac{1}{M+1}\sum_m
\mathbf1_{\{q_m\geq\tau q_0\}} .
\end{gathered}
\end{equation}
Here $\rho\in[0,1]$ controls worst-case emphasis, $\tau\in(0,1)$ specifies
the retained-contact ratio, and $\kappa_q=0$ when
$q_0\leq\varepsilon_q$ for $\varepsilon_q>0$.
Together, $q_{\mathrm{rob}}$, $\Delta_q$, and $\kappa_q$ characterize
retained contact, perturbation sensitivity, and contact consistency. App.~\ref{app:flexevo-theory-v30} derives finite-probe score bounds
and conditional transfer bounds under explicit target-proximity
assumptions.

\paragraph{Bi-objective evaluation.}
We balance interaction preservation against geometric quality and
perturbation tolerance using two maximized candidate objectives:
\begin{equation}
\label{eq:candidate_objectives}
\begin{aligned}
f_{\mathrm{pres}}(x;G)
&=\alpha_a a+\alpha_0q_0+\alpha_rq_{\mathrm{rob}}
-\alpha_\chi\chi_{\max},\\
f_{\mathrm{rob}}(x;G)
&=\beta_g g+\beta_c c+\beta_k\kappa_q-\beta_d\Delta_q-\Pi(x;G).
\end{aligned}
\end{equation}
We suppress component arguments for readability.
The anchor-proximity score $a$ and reference contact $q_0$ encourage
compatibility with $P^0$, while $q_{\mathrm{rob}}$ rewards contact maintained
under perturbation.
The internal-geometry score $g$, neighbor support $c$, and consistency
$\kappa_q$ favor geometrically supported candidates with stable contact
quality. The penalty $\Pi$ combines flexible-region exposure, forbidden
occupancy, probe clashes, and accumulated edit cost.
All coefficients are fixed and nonnegative; Appendix~\ref{app:scoring}
defines the components and their values.
The search uses geometric scores rather than external docking or affinity
estimators. The cost schematic in Fig.~\ref{fig:method}(3) corresponds to
$f_1=-f_{\mathrm{pres}}$ and $f_2=-f_{\mathrm{rob}}$. To assess the prior rather than only its best candidate, we aggregate both
leading-candidate and pool-average quality. For nonempty $\mathcal X_G$,
the genome objectives in Eq.~\eqref{eq:problem_setup} are
\begin{equation}
\label{eq:genome_fitness}
\begin{aligned}
F_j(\mathcal X_G;P^0,G)
&=\lambda\max_{x\in\mathcal X_G}f_j(x;G)
+\frac{1-\lambda}{|\mathcal X_G|}\sum_{x\in\mathcal X_G}f_j(x;G)
-\eta_j\Omega(G),\\
\Omega(G)&=\frac{|\mathcal R_G|}{R},
\qquad j\in\{\mathrm{pres},\mathrm{rob}\},
\quad\eta_{\mathrm{pres}}=0,\quad\eta_{\mathrm{rob}}=\eta.
\end{aligned}
\end{equation}
Here $\lambda\in[0,1]$ balances the maximum and mean, and $\eta\geq0$
penalizes extensive constraint assignment.
The two objective-wise maxima need not correspond to the same candidate.

\begin{figure}
    \centering
    \includegraphics[width=1\linewidth]{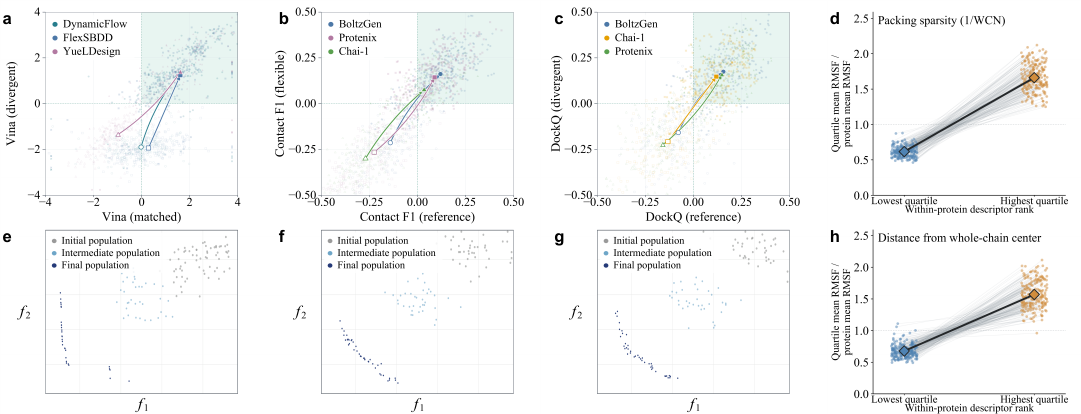}
    \caption{Cross-conformation evaluation and analysis of FlexEvo.
    (a--c) Distributions of scores derived from Vina, contact F1, and DockQ
    before and after adaptation for small molecules, peptides, and protein
    binders, respectively, under the paired target conditions shown.
    (e--g) Initial, intermediate, and final populations in the corresponding
    two-objective spaces.
    (d,h) Mean residue RMSF in the lowest and highest within-protein
    quartiles of packing sparsity ($1/\mathrm{WCN}$) and distance from the
    whole-chain $C_\alpha$ geometric center, respectively, for 200 ATLAS proteins.
    Quartile means are normalized by each protein's whole-chain mean RMSF.
    In (d,h), points denote proteins, faint lines pair quartiles, and diamonds
    with thick lines indicate equal-weight means across proteins.}
    \label{fig:main_analysis}
\end{figure}

\subsection{Feedback-Driven Genome Evolution}
\label{sec:genome_evolution}

\paragraph{Spatial genome variation.}
Each offspring applies one of the three operators in
Fig.~\ref{fig:method}(4): \emph{Pruning}, \emph{Fusion}, or
\emph{Perturbation}.
Pruning retires weakly used stable or flexible constraints, reducing the
support of the prior without removing target atoms.
Fusion inherits a spatial configuration from a donor $G_b$.
For a sampled residue $s$ and neighborhood
$\mathcal U=\{r:\|c_r-c_s\|_2\leq d_{\mathrm{cross}}\}$, it updates
\begin{equation}
\label{eq:spatial_fusion}
\begin{aligned}
\mathcal R_{\widetilde G}
&=(\mathcal R_{G_a}\setminus\mathcal U)
\cup(\mathcal R_{G_b}\cap\mathcal U),\\
(\widetilde\ell_r,\widetilde z_r)
&=\begin{cases}
(\ell_r^b,z_r^b),&r\in\mathcal R_{G_b}\cap\mathcal U,\\
(\ell_r^a,z_r^a),&r\in\mathcal R_{G_a}\setminus\mathcal U.
\end{cases}
\end{aligned}
\end{equation}
Transferring roles and scales together retains the donor's configuration
within the selected spatial block.
Perturbation changes support or roles within a local neighborhood and
resizes a separately sampled residue subset $\mathcal M_z$:
\begin{equation}
\label{eq:scale_perturbation}
\widetilde z_r=\operatorname{clip}
\bigl(z_re^{\epsilon_r},z_{\min},z_{\max}\bigr),
\qquad\epsilon_r\sim\mathcal N(0,\sigma_z^2),\quad r\in\mathcal M_z.
\end{equation}
The role changes are spatially local; scale updates need not be.
Sampling rules are given in Appendix~\ref{app:evolution}.

\paragraph{Structural feedback.}
Initialization alone does not reveal which constraints will be exercised
during adaptation.
We therefore update an offspring proposal using evidence from its primary
parent's adapted pool $\mathcal X_{G_a}$.
Let $u_r$ count uses of residue $r$'s FlexBox in surviving edit histories.
Usage and proximity define the feedback signal
\begin{equation}
\label{eq:feedback_signal}
D_r=\min_{x\in\mathcal X_{G_a},i}\|p_i(x)-c_r\|_2,
\qquad
s_r=\max\!\left\{
\frac{u_r}{\max(1,\sum_k u_k)},
\left[1-\frac{D_r}{d_c}\right]_+\right\},
\end{equation}
where $[z]_+=\max(z,0)$ and $d_c>0$ is the proximity radius.
Usage records constraints exercised during editing, while proximity can
identify relevant regions with no active box or use history.
High-signal uncovered residues are activated as flexible constraints;
low-signal stable or flexible assignments are retired under a bounded budget.
For the corresponding sets $\mathcal A_+$ and $\mathcal A_-$,
\begin{equation}
\label{eq:feedback_update}
\mathcal R_{G^+}
=(\mathcal R_{\widetilde G}\cup\mathcal A_+)\setminus\mathcal A_-,
\qquad \ell_r^+=\mathrm{flexible}\ \text{for }r\in\mathcal A_+.
\end{equation}
Pruning and feedback retirement exclude forbidden constraints.
After genome repair, the offspring is evaluated again from $\mathcal X_0$.
Thus parent candidates inform \emph{where to search}, without passing their
adapted coordinates to the offspring evaluation.

\paragraph{Diversity-aware Pareto survival.}
We apply nondominated sorting and objective-space
crowding~\citep{deb2002fast} to the parent--offspring union, following the same multi-objective survival principle used in evolutionary fairness diagnosis and robustness search~\citep{wang2024fairgnn,zeng2026digitalink}.
To discourage redundant spatial priors, genomes close to an already
selected configuration in residue-role space are deferred, then reconsidered
if slots remain. This is a survivor-diversity criterion.
For parent population $\mathcal H^g$, evaluated offspring $\mathcal Q^g$,
and survivor population size $N_G$,
\begin{equation}
\label{eq:outer_update}
\mathcal H^{g+1}=\mathcal S_G
\bigl(\mathcal H^g\cup\mathcal Q^g;N_G\bigr),
\qquad
\mathcal Y=\{(G,\mathcal X_G):G\in\mathcal H^{g_{\max}}\}.
\end{equation}
The output retains each surviving prior with its adapted candidate pool.
Algorithm~\ref{alg:flexevo} summarizes the search procedure;
Apps.~\ref{app:evolution} and~\ref{app:execution} detail
the diversity rule and execution settings.

\begin{figure}
    \centering
    \includegraphics[width=1\linewidth]{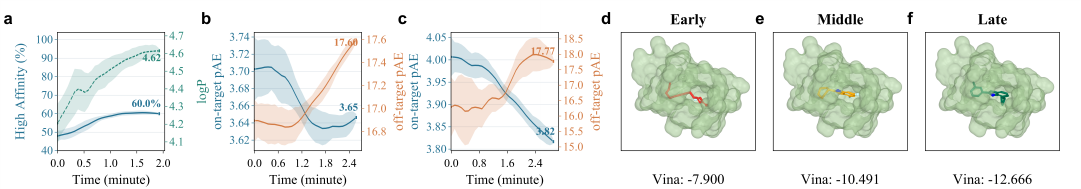}
    \caption{Metric trajectories and structural progression.
    (a) High-affinity fraction (blue, left axis) and $\log P$
    (green, right axis) over adaptation time.
    (b,c) On-target pAE (blue, left axis) and off-target pAE
    (orange, right axis).
    Time is shown in minutes.
    (d--f) Early, intermediate, and late binding poses for one ligand
    example, with the corresponding Vina scores shown below.}
    \label{fig:optimization_dynamics}
\end{figure}

\section{Experiments}
\label{sec:experiments}
\begin{table*}[!t]
\centering
\caption{Cross-conformation performance across nine binder categories.
Reference is the adaptation conformation; Matched and Divergent
are alternative natural conformations held out from adaptation.
Black entries report post-\textsc{FlexEvo} scores.
Superscripts give changes from the unadapted baseline
($\Delta=\mathrm{adapted}-\mathrm{unadapted}$);
green and red indicate improvement and deterioration according to each
metric's preferred direction.
Changes in percentage-valued metrics are expressed in percentage points.
Vina, DockQ, and QED use means; LRMSD uses medians.
Values and changes are rounded independently from the source values.
QED is reported once per method in the Reference columns and applies to all three
conformations; dashes denote repeated values omitted from display.}
\label{tab:main_results}
\begingroup
\fontsize{7.3}{8.6}\selectfont
\setlength{\tabcolsep}{0.32pt}
\renewcommand{\arraystretch}{1.03}
\setlength{\heavyrulewidth}{0.65pt}
\setlength{\lightrulewidth}{0.35pt}
\setlength{\cmidrulewidth}{0.25pt}
\setlength{\aboverulesep}{1.4pt}
\setlength{\belowrulesep}{1.4pt}
\noindent\makebox[\linewidth][l]{\textbf{A\quad De novo ligand generation}}\par\vspace{1pt}
\noindent\resizebox{\linewidth}{!}{%
\begin{tabular}{@{}w{l}{43.3777pt}@{\hspace{5pt}}w{l}{30.7244pt}w{l}{27.1744pt}w{l}{31.7399pt}@{\hspace{4pt}}w{l}{26.3031pt}w{l}{25.8958pt}w{l}{25.8958pt}@{\hspace{4pt}}w{l}{26.3031pt}w{l}{25.8958pt}w{l}{25.8958pt}@{\hspace{4pt}}w{l}{26.3031pt}w{l}{25.8958pt}w{l}{25.8958pt}@{\hspace{4pt}}w{l}{26.3031pt}w{l}{25.8958pt}w{l}{25.8958pt}@{\hspace{4pt}}w{l}{26.3031pt}w{l}{25.8958pt}w{l}{25.8958pt}@{}}
\toprule
Base model & \multicolumn{3}{c}{Vina$\downarrow$} & \multicolumn{3}{c}{High affinity (\%)$\uparrow$} & \multicolumn{3}{c}{Vinardo$\uparrow$} & \multicolumn{3}{c}{AutoDock4$\uparrow$} & \multicolumn{3}{c}{RTMScore$\uparrow$} & \multicolumn{3}{c}{QED$\uparrow$} \\[-1pt]
\cmidrule(lr){2-4}\cmidrule(lr){5-7}\cmidrule(lr){8-10}\cmidrule(lr){11-13}\cmidrule(lr){14-16}\cmidrule(lr){17-19}
 & \FlexEvoTestHeader{Reference} & \FlexEvoTestHeader{Matched} & \FlexEvoTestHeader{Divergent} & \FlexEvoTestHeader{Reference} & \FlexEvoTestHeader{Matched} & \FlexEvoTestHeader{Divergent} & \FlexEvoTestHeader{Reference} & \FlexEvoTestHeader{Matched} & \FlexEvoTestHeader{Divergent} & \FlexEvoTestHeader{Reference} & \FlexEvoTestHeader{Matched} & \FlexEvoTestHeader{Divergent} & \FlexEvoTestHeader{Reference} & \FlexEvoTestHeader{Matched} & \FlexEvoTestHeader{Divergent} & \FlexEvoTestHeader{Reference} & \FlexEvoTestHeader{Matched} & \FlexEvoTestHeader{Divergent} \\
\midrule
\FlexEvoScenario{Small molecule}
DynamicFlow & \FlexEvoTableGain{-9.18}{-2.57} & \FlexEvoTableGain{-8.36}{-1.89} & \FlexEvoTableGain{-7.80}{-2.95} & \FlexEvoTableGain{58.2}{+6.9} & \FlexEvoTableGain{56.7}{+16.8} & \FlexEvoTableGain{61.7}{+35.4} & \FlexEvoTableGain{.698}{+.282} & \FlexEvoTableGain{.537}{+.131} & \FlexEvoTableGain{.587}{+.232} & \FlexEvoTableGain{.597}{+.019} & \FlexEvoTableGain{.570}{+.141} & \FlexEvoTableGain{.672}{+.298} & \FlexEvoTableGain{.619}{+.248} & \FlexEvoTableGain{.594}{+.135} & \FlexEvoTableGain{.571}{+.129} & \FlexEvoTableGain{.594}{+.060} & -- & -- \\
FlexSBDD & \FlexEvoTableGain{-9.22}{-2.32} & \FlexEvoTableGain{-8.27}{-1.58} & \FlexEvoTableGain{-7.74}{-2.89} & \FlexEvoTableGain{55.4}{+8.3} & \FlexEvoTableGain{57.1}{+17.8} & \FlexEvoTableGain{56.5}{+17.6} & \FlexEvoTableGain{.552}{+.142} & \FlexEvoTableGain{.607}{+.218} & \FlexEvoTableGain{.603}{+.242} & \FlexEvoTableGain{.585}{+.244} & \FlexEvoTableGain{.643}{+.201} & \FlexEvoTableGain{.588}{+.182} & \FlexEvoTableGain{.532}{+.204} & \FlexEvoTableGain{.572}{+.110} & \FlexEvoTableGain{.576}{+.211} & \FlexEvoTableGain{.589}{+.215} & -- & -- \\
YuelDesign & \FlexEvoTableGain{-9.18}{-3.46} & \FlexEvoTableGain{-8.44}{-2.86} & \FlexEvoTableGain{-7.67}{-3.72} & \FlexEvoTableGain{62.5}{+13.0} & \FlexEvoTableGain{61.3}{+24.5} & \FlexEvoTableGain{59.5}{+28.3} & \FlexEvoTableGain{.584}{+.192} & \FlexEvoTableGain{.643}{+.269} & \FlexEvoTableGain{.578}{+.185} & \FlexEvoTableGain{.652}{+.220} & \FlexEvoTableGain{.616}{+.177} & \FlexEvoTableGain{.619}{+.268} & \FlexEvoTableGain{.656}{+.188} & \FlexEvoTableGain{.584}{+.171} & \FlexEvoTableGain{.587}{+.198} & \FlexEvoTableGain{.612}{+.191} & -- & -- \\
\addlinespace[1pt]\arrayrulecolor{FEline}\cmidrule[0.25pt](lr){1-19}\arrayrulecolor{black}
\FlexEvoScenario{Nonpeptidic macrocycle}
DynamicFlow & \FlexEvoTableGain{-8.32}{-2.40} & \FlexEvoTableGain{-7.37}{-1.63} & \FlexEvoTableGain{-7.05}{-2.70} & \FlexEvoTableGain{51.6}{+5.1} & \FlexEvoTableGain{49.9}{+14.6} & \FlexEvoTableGain{55.0}{+31.8} & \FlexEvoTableGain{.633}{+.264} & \FlexEvoTableGain{.486}{+.125} & \FlexEvoTableGain{.529}{+.207} & \FlexEvoTableGain{.539}{+.020} & \FlexEvoTableGain{.508}{+.121} & \FlexEvoTableGain{.601}{+.265} & \FlexEvoTableGain{.559}{+.230} & \FlexEvoTableGain{.538}{+.132} & \FlexEvoTableGain{.516}{+.117} & \FlexEvoTableGain{.616}{+.133} & -- & -- \\
FlexSBDD & \FlexEvoTableGain{-8.24}{-2.00} & \FlexEvoTableGain{-7.50}{-1.55} & \FlexEvoTableGain{-6.83}{-2.42} & \FlexEvoTableGain{50.0}{+8.1} & \FlexEvoTableGain{50.3}{+15.7} & \FlexEvoTableGain{51.1}{+16.1} & \FlexEvoTableGain{.487}{+.120} & \FlexEvoTableGain{.534}{+.182} & \FlexEvoTableGain{.540}{+.223} & \FlexEvoTableGain{.531}{+.229} & \FlexEvoTableGain{.584}{+.182} & \FlexEvoTableGain{.523}{+.164} & \FlexEvoTableGain{.475}{+.179} & \FlexEvoTableGain{.504}{+.090} & \FlexEvoTableGain{.513}{+.191} & \FlexEvoTableGain{.535}{+.204} & -- & -- \\
YuelDesign & \FlexEvoTableGain{-8.23}{-3.07} & \FlexEvoTableGain{-7.46}{-2.47} & \FlexEvoTableGain{-6.90}{-3.36} & \FlexEvoTableGain{56.3}{+12.2} & \FlexEvoTableGain{54.9}{+22.3} & \FlexEvoTableGain{53.6}{+25.7} & \FlexEvoTableGain{.520}{+.173} & \FlexEvoTableGain{.582}{+.248} & \FlexEvoTableGain{.524}{+.176} & \FlexEvoTableGain{.590}{+.199} & \FlexEvoTableGain{.549}{+.158} & \FlexEvoTableGain{.562}{+.252} & \FlexEvoTableGain{.581}{+.161} & \FlexEvoTableGain{.528}{+.153} & \FlexEvoTableGain{.517}{+.173} & \FlexEvoTableGain{.538}{+.167} & -- & -- \\
\bottomrule
\end{tabular}%
}\par
\par\vspace{3pt}
\noindent\makebox[\linewidth][l]{\textbf{B\quad Biomolecular complex design}}\par\vspace{1pt}
\noindent\resizebox{\linewidth}{!}{%
\begin{tabular}{@{}w{l}{43.3777pt}@{\hspace{5pt}}w{l}{26.2118pt}w{l}{25.8046pt}w{l}{25.8046pt}@{\hspace{4pt}}w{l}{26.2118pt}w{l}{25.8046pt}w{l}{25.8046pt}@{\hspace{4pt}}w{l}{26.2118pt}w{l}{24.7189pt}w{l}{25.7978pt}@{\hspace{4pt}}w{l}{26.2118pt}w{l}{27.0832pt}w{l}{27.0832pt}@{\hspace{4pt}}w{l}{29.3689pt}w{l}{30.6332pt}w{l}{30.6332pt}@{\hspace{4pt}}w{l}{26.2118pt}w{l}{24.7189pt}w{l}{25.7978pt}@{}}
\toprule
Base model & \multicolumn{3}{c}{DockQ$\uparrow$} & \multicolumn{3}{c}{$<2$\AA\ (\%)$\uparrow$} & \multicolumn{3}{c}{LRMSD (\AA)$\downarrow$} & \multicolumn{3}{c}{MM/GBSA$\downarrow$} & \multicolumn{3}{c}{HADDOCK$\downarrow$} & \multicolumn{3}{c}{On-target pAE$\downarrow$} \\[-1pt]
\cmidrule(lr){2-4}\cmidrule(lr){5-7}\cmidrule(lr){8-10}\cmidrule(lr){11-13}\cmidrule(lr){14-16}\cmidrule(lr){17-19}
 & \FlexEvoTestHeader{Reference} & \FlexEvoTestHeader{Matched} & \FlexEvoTestHeader{Divergent} & \FlexEvoTestHeader{Reference} & \FlexEvoTestHeader{Matched} & \FlexEvoTestHeader{Divergent} & \FlexEvoTestHeader{Reference} & \FlexEvoTestHeader{Matched} & \FlexEvoTestHeader{Divergent} & \FlexEvoTestHeader{Reference} & \FlexEvoTestHeader{Matched} & \FlexEvoTestHeader{Divergent} & \FlexEvoTestHeader{Reference} & \FlexEvoTestHeader{Matched} & \FlexEvoTestHeader{Divergent} & \FlexEvoTestHeader{Reference} & \FlexEvoTestHeader{Matched} & \FlexEvoTestHeader{Divergent} \\
\midrule
\FlexEvoScenario{Linear peptide}
BoltzGen & \FlexEvoTableGain{.874}{+.217} & \FlexEvoTableGain{.840}{+.451} & \FlexEvoTableGain{.850}{+.566} & \FlexEvoTableLoss{40.2}{-2.0} & \FlexEvoTableGain{40.2}{+13.2} & \FlexEvoTableGain{36.9}{+18.5} & \FlexEvoTableGain{5.32}{-.16} & \FlexEvoTableGain{4.82}{-2.77} & \FlexEvoTableGain{5.56}{-5.51} & \FlexEvoTableGain{-44.6}{-4.2} & \FlexEvoTableGain{-44.7}{-11.3} & \FlexEvoTableGain{-45.7}{-26.6} & \FlexEvoTableLoss{-126.6}{+.5} & \FlexEvoTableGain{-130.8}{-33.5} & \FlexEvoTableGain{-122.3}{-53.2} & \FlexEvoTableGain{3.56}{-.12} & \FlexEvoTableGain{3.49}{-2.25} & \FlexEvoTableGain{3.11}{-3.71} \\
Chai-1 & \FlexEvoTableGain{.772}{+.128} & \FlexEvoTableGain{.832}{+.429} & \FlexEvoTableGain{.822}{+.533} & \FlexEvoTableGain{39.7}{+.0} & \FlexEvoTableGain{38.6}{+11.0} & \FlexEvoTableGain{35.8}{+17.1} & \FlexEvoTableGain{4.99}{-.86} & \FlexEvoTableGain{4.60}{-4.09} & \FlexEvoTableGain{5.99}{-5.59} & \FlexEvoTableGain{-43.9}{-3.3} & \FlexEvoTableGain{-42.4}{-9.4} & \FlexEvoTableGain{-40.2}{-19.3} & \FlexEvoTableGain{-124.9}{-2.0} & \FlexEvoTableGain{-122.7}{-11.4} & \FlexEvoTableGain{-122.0}{-44.1} & \FlexEvoTableLoss{4.35}{+.22} & \FlexEvoTableGain{3.71}{-2.82} & \FlexEvoTableGain{4.03}{-2.26} \\
Protenix & \FlexEvoTableGain{.820}{+.104} & \FlexEvoTableGain{.863}{+.394} & \FlexEvoTableGain{.836}{+.484} & \FlexEvoTableLoss{39.7}{-.3} & \FlexEvoTableGain{38.8}{+11.0} & \FlexEvoTableGain{36.4}{+17.4} & \FlexEvoTableLoss{4.59}{+.31} & \FlexEvoTableGain{4.73}{-3.80} & \FlexEvoTableGain{4.25}{-8.03} & \FlexEvoTableGain{-46.3}{-1.3} & \FlexEvoTableGain{-43.2}{-8.7} & \FlexEvoTableGain{-40.8}{-22.3} & \FlexEvoTableGain{-123.3}{-.1} & \FlexEvoTableGain{-123.4}{-13.4} & \FlexEvoTableGain{-128.3}{-51.8} & \FlexEvoTableLoss{3.86}{+.64} & \FlexEvoTableGain{3.22}{-1.01} & \FlexEvoTableGain{3.08}{-2.15} \\
\addlinespace[1pt]\arrayrulecolor{FEline}\cmidrule[0.25pt](lr){1-19}\arrayrulecolor{black}
\FlexEvoScenario{Cyclic peptide}
BoltzGen & \FlexEvoTableGain{.902}{+.219} & \FlexEvoTableGain{.880}{+.488} & \FlexEvoTableGain{.873}{+.578} & \FlexEvoTableLoss{42.2}{-1.0} & \FlexEvoTableGain{41.2}{+13.6} & \FlexEvoTableGain{39.1}{+19.4} & \FlexEvoTableGain{5.64}{-.04} & \FlexEvoTableGain{4.94}{-2.82} & \FlexEvoTableGain{5.84}{-5.27} & \FlexEvoTableGain{-46.4}{-3.6} & \FlexEvoTableGain{-45.1}{-10.2} & \FlexEvoTableGain{-47.0}{-26.8} & \FlexEvoTableGain{-132.6}{+.0} & \FlexEvoTableGain{-133.3}{-32.5} & \FlexEvoTableGain{-128.2}{-55.4} & \FlexEvoTableGain{3.79}{-.06} & \FlexEvoTableGain{3.51}{-2.59} & \FlexEvoTableGain{3.29}{-3.68} \\
Chai-1 & \FlexEvoTableGain{.861}{+.005} & \FlexEvoTableGain{.867}{+.411} & \FlexEvoTableGain{.845}{+.588} & \FlexEvoTableLoss{41.0}{-.3} & \FlexEvoTableGain{40.0}{+5.1} & \FlexEvoTableGain{37.0}{+19.3} & \FlexEvoTableGain{4.74}{-.90} & \FlexEvoTableGain{5.05}{-3.82} & \FlexEvoTableGain{6.10}{-6.34} & \FlexEvoTableLoss{-44.7}{+1.6} & \FlexEvoTableGain{-44.9}{-7.2} & \FlexEvoTableGain{-41.0}{-18.2} & \FlexEvoTableLoss{-126.4}{+3.5} & \FlexEvoTableGain{-122.9}{-14.8} & \FlexEvoTableGain{-128.0}{-29.2} & \FlexEvoTableLoss{4.19}{+.34} & \FlexEvoTableGain{4.05}{-3.08} & \FlexEvoTableGain{3.98}{-2.71} \\
Protenix & \FlexEvoTableGain{.862}{+.105} & \FlexEvoTableGain{.877}{+.403} & \FlexEvoTableGain{.854}{+.500} & \FlexEvoTableGain{42.1}{+1.0} & \FlexEvoTableGain{40.1}{+11.5} & \FlexEvoTableGain{37.9}{+17.8} & \FlexEvoTableLoss{4.71}{+.26} & \FlexEvoTableGain{4.92}{-4.08} & \FlexEvoTableGain{4.40}{-8.03} & \FlexEvoTableGain{-48.5}{-1.1} & \FlexEvoTableGain{-46.0}{-9.3} & \FlexEvoTableGain{-41.3}{-21.5} & \FlexEvoTableGain{-131.2}{-2.1} & \FlexEvoTableGain{-130.5}{-16.4} & \FlexEvoTableGain{-130.1}{-49.6} & \FlexEvoTableLoss{4.06}{+.64} & \FlexEvoTableGain{3.41}{-1.06} & \FlexEvoTableGain{3.27}{-2.16} \\
\addlinespace[1pt]\arrayrulecolor{FEline}\cmidrule[0.25pt](lr){1-19}\arrayrulecolor{black}
\FlexEvoScenario{Non-antibody protein}
BoltzGen & \FlexEvoTableGain{.862}{+.219} & \FlexEvoTableGain{.850}{+.457} & \FlexEvoTableGain{.833}{+.536} & \FlexEvoTableGain{42.7}{+4.2} & \FlexEvoTableGain{39.7}{+13.9} & \FlexEvoTableGain{39.0}{+21.4} & \FlexEvoTableGain{4.22}{-1.20} & \FlexEvoTableGain{4.86}{-3.32} & \FlexEvoTableGain{5.90}{-5.53} & \FlexEvoTableGain{-40.7}{-.7} & \FlexEvoTableGain{-42.7}{-13.6} & \FlexEvoTableGain{-38.5}{-21.5} & \FlexEvoTableGain{-119.6}{-1.8} & \FlexEvoTableGain{-112.1}{-26.7} & \FlexEvoTableGain{-117.5}{-56.1} & \FlexEvoTableGain{3.87}{-.19} & \FlexEvoTableGain{3.17}{-2.15} & \FlexEvoTableGain{2.95}{-4.42} \\
Chai-1 & \FlexEvoTableGain{.833}{+.102} & \FlexEvoTableGain{.807}{+.335} & \FlexEvoTableGain{.815}{+.463} & \FlexEvoTableGain{44.8}{+3.5} & \FlexEvoTableGain{40.8}{+18.4} & \FlexEvoTableGain{39.8}{+24.9} & \FlexEvoTableGain{3.86}{-.82} & \FlexEvoTableGain{5.11}{-4.11} & \FlexEvoTableGain{4.97}{-7.87} & \FlexEvoTableGain{-39.4}{-1.0} & \FlexEvoTableGain{-43.5}{-12.7} & \FlexEvoTableGain{-40.3}{-21.7} & \FlexEvoTableLoss{-111.1}{+8.3} & \FlexEvoTableGain{-117.2}{-19.0} & \FlexEvoTableGain{-122.0}{-44.4} & \FlexEvoTableLoss{4.24}{+.68} & \FlexEvoTableGain{3.55}{-.38} & \FlexEvoTableGain{3.38}{-1.47} \\
Protenix & \FlexEvoTableGain{.818}{+.152} & \FlexEvoTableGain{.836}{+.455} & \FlexEvoTableGain{.817}{+.477} & \FlexEvoTableGain{45.3}{+1.1} & \FlexEvoTableGain{41.8}{+19.3} & \FlexEvoTableGain{40.2}{+23.3} & \FlexEvoTableGain{3.79}{-.03} & \FlexEvoTableGain{5.02}{-4.45} & \FlexEvoTableGain{4.59}{-7.92} & \FlexEvoTableGain{-41.3}{-.8} & \FlexEvoTableGain{-41.2}{-6.2} & \FlexEvoTableGain{-39.5}{-22.8} & \FlexEvoTableLoss{-113.9}{+5.3} & \FlexEvoTableGain{-111.7}{-16.7} & \FlexEvoTableGain{-114.8}{-45.3} & \FlexEvoTableGain{3.76}{-.43} & \FlexEvoTableGain{3.61}{-1.11} & \FlexEvoTableGain{3.83}{-3.05} \\
\addlinespace[1pt]\arrayrulecolor{FEline}\cmidrule[0.25pt](lr){1-19}\arrayrulecolor{black}
\FlexEvoScenario{Full-length antibody}
BoltzGen & \FlexEvoTableGain{.800}{+.190} & \FlexEvoTableGain{.800}{+.432} & \FlexEvoTableGain{.772}{+.496} & \FlexEvoTableGain{39.9}{+3.6} & \FlexEvoTableGain{37.7}{+13.5} & \FlexEvoTableGain{37.0}{+20.6} & \FlexEvoTableGain{3.98}{-1.17} & \FlexEvoTableGain{4.58}{-3.18} & \FlexEvoTableGain{5.54}{-5.07} & \FlexEvoTableGain{-37.7}{-.3} & \FlexEvoTableGain{-39.9}{-12.3} & \FlexEvoTableGain{-35.7}{-19.7} & \FlexEvoTableGain{-113.6}{-4.0} & \FlexEvoTableGain{-104.9}{-25.0} & \FlexEvoTableGain{-109.9}{-52.4} & \FlexEvoTableGain{3.58}{-.20} & \FlexEvoTableGain{3.00}{-2.02} & \FlexEvoTableGain{2.78}{-4.03} \\
Chai-1 & \FlexEvoTableGain{.791}{+.110} & \FlexEvoTableGain{.755}{+.311} & \FlexEvoTableGain{.764}{+.433} & \FlexEvoTableGain{42.1}{+2.9} & \FlexEvoTableGain{38.6}{+17.6} & \FlexEvoTableGain{37.5}{+23.5} & \FlexEvoTableGain{3.60}{-.73} & \FlexEvoTableGain{4.82}{-3.89} & \FlexEvoTableGain{4.71}{-7.44} & \FlexEvoTableGain{-37.0}{-1.6} & \FlexEvoTableGain{-40.5}{-11.5} & \FlexEvoTableGain{-38.1}{-20.7} & \FlexEvoTableLoss{-103.7}{+8.2} & \FlexEvoTableGain{-108.5}{-15.3} & \FlexEvoTableGain{-115.1}{-43.3} & \FlexEvoTableLoss{3.97}{+.62} & \FlexEvoTableGain{3.34}{-.34} & \FlexEvoTableGain{3.18}{-1.34} \\
Protenix & \FlexEvoTableGain{.762}{+.144} & \FlexEvoTableGain{.792}{+.435} & \FlexEvoTableGain{.770}{+.456} & \FlexEvoTableGain{42.9}{+2.0} & \FlexEvoTableGain{38.8}{+17.6} & \FlexEvoTableGain{37.6}{+21.7} & \FlexEvoTableGain{3.53}{-.01} & \FlexEvoTableGain{4.75}{-4.06} & \FlexEvoTableGain{4.35}{-7.49} & \FlexEvoTableGain{-39.1}{-1.5} & \FlexEvoTableGain{-38.3}{-5.3} & \FlexEvoTableGain{-37.3}{-21.8} & \FlexEvoTableLoss{-106.2}{+4.1} & \FlexEvoTableGain{-104.4}{-16.2} & \FlexEvoTableGain{-108.3}{-42.9} & \FlexEvoTableGain{3.52}{-.38} & \FlexEvoTableGain{3.37}{-1.11} & \FlexEvoTableGain{3.59}{-2.86} \\
\addlinespace[1pt]\arrayrulecolor{FEline}\cmidrule[0.25pt](lr){1-19}\arrayrulecolor{black}
\FlexEvoScenario{Antibody fragment and nanobody}
BoltzGen & \FlexEvoTableGain{.824}{+.207} & \FlexEvoTableGain{.822}{+.442} & \FlexEvoTableGain{.810}{+.525} & \FlexEvoTableGain{41.5}{+4.3} & \FlexEvoTableGain{38.0}{+13.1} & \FlexEvoTableGain{37.1}{+20.4} & \FlexEvoTableGain{4.05}{-1.19} & \FlexEvoTableGain{4.69}{-3.26} & \FlexEvoTableGain{5.65}{-5.42} & \FlexEvoTableGain{-39.2}{-1.0} & \FlexEvoTableGain{-41.0}{-12.8} & \FlexEvoTableGain{-37.4}{-21.1} & \FlexEvoTableGain{-114.8}{-2.8} & \FlexEvoTableGain{-107.0}{-25.5} & \FlexEvoTableGain{-111.8}{-52.9} & \FlexEvoTableGain{3.71}{-.22} & \FlexEvoTableGain{3.01}{-2.09} & \FlexEvoTableGain{2.85}{-4.24} \\
Chai-1 & \FlexEvoTableGain{.802}{+.105} & \FlexEvoTableGain{.770}{+.321} & \FlexEvoTableGain{.794}{+.460} & \FlexEvoTableGain{43.1}{+3.3} & \FlexEvoTableGain{39.3}{+17.6} & \FlexEvoTableGain{37.9}{+23.8} & \FlexEvoTableGain{3.72}{-.76} & \FlexEvoTableGain{4.93}{-3.98} & \FlexEvoTableGain{4.77}{-7.64} & \FlexEvoTableGain{-37.6}{-.7} & \FlexEvoTableGain{-41.8}{-12.1} & \FlexEvoTableGain{-38.9}{-21.2} & \FlexEvoTableLoss{-105.7}{+7.7} & \FlexEvoTableGain{-112.9}{-19.0} & \FlexEvoTableGain{-117.2}{-42.5} & \FlexEvoTableLoss{4.06}{+.64} & \FlexEvoTableGain{3.42}{-.32} & \FlexEvoTableGain{3.24}{-1.39} \\
Protenix & \FlexEvoTableGain{.783}{+.149} & \FlexEvoTableGain{.814}{+.450} & \FlexEvoTableGain{.796}{+.469} & \FlexEvoTableGain{43.6}{+1.3} & \FlexEvoTableGain{39.9}{+18.1} & \FlexEvoTableGain{38.5}{+22.0} & \FlexEvoTableGain{3.60}{-.07} & \FlexEvoTableGain{4.88}{-4.13} & \FlexEvoTableGain{4.46}{-7.51} & \FlexEvoTableGain{-40.2}{-1.3} & \FlexEvoTableGain{-39.4}{-5.5} & \FlexEvoTableGain{-38.4}{-22.4} & \FlexEvoTableLoss{-108.4}{+7.2} & \FlexEvoTableGain{-108.4}{-17.0} & \FlexEvoTableGain{-110.9}{-44.4} & \FlexEvoTableGain{3.63}{-.45} & \FlexEvoTableGain{3.46}{-1.12} & \FlexEvoTableGain{3.72}{-2.90} \\
\addlinespace[1pt]\arrayrulecolor{FEline}\cmidrule[0.25pt](lr){1-19}\arrayrulecolor{black}
\FlexEvoScenario{Oligonucleotide}
BoltzGen & \FlexEvoTableGain{.792}{+.211} & \FlexEvoTableGain{.772}{+.414} & \FlexEvoTableGain{.768}{+.497} & \FlexEvoTableGain{38.9}{+3.4} & \FlexEvoTableGain{36.0}{+12.4} & \FlexEvoTableGain{35.5}{+19.7} & \FlexEvoTableGain{3.87}{-1.02} & \FlexEvoTableGain{4.47}{-2.93} & \FlexEvoTableGain{5.34}{-4.95} & \FlexEvoTableGain{-37.2}{-.3} & \FlexEvoTableGain{-38.7}{-12.0} & \FlexEvoTableGain{-35.4}{-19.6} & \FlexEvoTableGain{-108.9}{-.7} & \FlexEvoTableGain{-100.9}{-22.3} & \FlexEvoTableGain{-108.3}{-51.9} & \FlexEvoTableGain{3.50}{-.16} & \FlexEvoTableGain{2.92}{-1.90} & \FlexEvoTableGain{2.68}{-4.03} \\
Chai-1 & \FlexEvoTableGain{.757}{+.083} & \FlexEvoTableGain{.745}{+.316} & \FlexEvoTableGain{.737}{+.413} & \FlexEvoTableGain{40.5}{+3.1} & \FlexEvoTableGain{37.7}{+17.4} & \FlexEvoTableGain{36.7}{+23.0} & \FlexEvoTableGain{3.51}{-.80} & \FlexEvoTableGain{4.61}{-3.85} & \FlexEvoTableGain{4.47}{-7.24} & \FlexEvoTableGain{-36.2}{-1.2} & \FlexEvoTableGain{-40.0}{-11.9} & \FlexEvoTableGain{-36.7}{-19.5} & \FlexEvoTableLoss{-100.5}{+7.4} & \FlexEvoTableGain{-106.5}{-17.9} & \FlexEvoTableGain{-110.0}{-38.8} & \FlexEvoTableLoss{3.88}{+.68} & \FlexEvoTableGain{3.26}{-.31} & \FlexEvoTableGain{3.07}{-1.32} \\
Protenix & \FlexEvoTableGain{.742}{+.128} & \FlexEvoTableGain{.759}{+.407} & \FlexEvoTableGain{.755}{+.449} & \FlexEvoTableGain{40.8}{+.9} & \FlexEvoTableGain{38.5}{+18.2} & \FlexEvoTableGain{36.3}{+20.9} & \FlexEvoTableGain{3.45}{-.02} & \FlexEvoTableGain{4.56}{-4.10} & \FlexEvoTableGain{4.15}{-7.20} & \FlexEvoTableGain{-38.0}{-.8} & \FlexEvoTableGain{-38.0}{-6.0} & \FlexEvoTableGain{-36.0}{-21.0} & \FlexEvoTableLoss{-105.3}{+3.1} & \FlexEvoTableGain{-101.6}{-14.3} & \FlexEvoTableGain{-103.9}{-41.2} & \FlexEvoTableGain{3.46}{-.32} & \FlexEvoTableGain{3.31}{-1.02} & \FlexEvoTableGain{3.48}{-2.86} \\
\addlinespace[1pt]\arrayrulecolor{FEline}\cmidrule[0.25pt](lr){1-19}\arrayrulecolor{black}
\FlexEvoScenario{mRNA / long nucleic acid}
BoltzGen & \FlexEvoTableGain{.862}{+.218} & \FlexEvoTableGain{.849}{+.457} & \FlexEvoTableGain{.833}{+.536} & \FlexEvoTableGain{45.3}{+4.0} & \FlexEvoTableGain{42.1}{+14.8} & \FlexEvoTableGain{41.3}{+22.5} & \FlexEvoTableGain{4.44}{-1.35} & \FlexEvoTableGain{5.19}{-3.52} & \FlexEvoTableGain{6.22}{-5.77} & \FlexEvoTableGain{-43.7}{-1.2} & \FlexEvoTableGain{-45.0}{-14.0} & \FlexEvoTableGain{-41.1}{-22.9} & \FlexEvoTableGain{-127.0}{-.9} & \FlexEvoTableGain{-118.4}{-27.5} & \FlexEvoTableGain{-125.9}{-60.3} & \FlexEvoTableGain{4.14}{-.16} & \FlexEvoTableGain{3.33}{-2.31} & \FlexEvoTableGain{3.14}{-4.61} \\
Chai-1 & \FlexEvoTableGain{.833}{+.102} & \FlexEvoTableGain{.808}{+.336} & \FlexEvoTableGain{.814}{+.463} & \FlexEvoTableGain{47.9}{+3.8} & \FlexEvoTableGain{43.3}{+19.7} & \FlexEvoTableGain{42.3}{+26.6} & \FlexEvoTableGain{4.09}{-.88} & \FlexEvoTableGain{5.38}{-4.38} & \FlexEvoTableGain{5.24}{-8.45} & \FlexEvoTableGain{-41.9}{-1.6} & \FlexEvoTableGain{-45.7}{-12.7} & \FlexEvoTableGain{-43.1}{-23.3} & \FlexEvoTableLoss{-117.2}{+8.8} & \FlexEvoTableGain{-125.8}{-20.3} & \FlexEvoTableGain{-129.9}{-48.4} & \FlexEvoTableLoss{4.54}{+.75} & \FlexEvoTableGain{3.74}{-.45} & \FlexEvoTableGain{3.61}{-1.48} \\
Protenix & \FlexEvoTableGain{.818}{+.153} & \FlexEvoTableGain{.836}{+.455} & \FlexEvoTableGain{.817}{+.477} & \FlexEvoTableGain{47.7}{+.3} & \FlexEvoTableGain{44.5}{+20.5} & \FlexEvoTableGain{42.9}{+25.0} & \FlexEvoTableGain{4.05}{-.02} & \FlexEvoTableGain{5.38}{-4.59} & \FlexEvoTableGain{4.83}{-8.33} & \FlexEvoTableGain{-44.0}{-.6} & \FlexEvoTableGain{-44.2}{-7.4} & \FlexEvoTableGain{-42.0}{-24.3} & \FlexEvoTableLoss{-119.7}{+7.6} & \FlexEvoTableGain{-117.5}{-17.7} & \FlexEvoTableGain{-123.0}{-49.4} & \FlexEvoTableGain{3.97}{-.50} & \FlexEvoTableGain{3.80}{-1.25} & \FlexEvoTableGain{4.08}{-3.23} \\
\bottomrule
\end{tabular}%
}\par
\endgroup
\end{table*}
\subsection{Experimental Setup}
\label{sec:experimental_setup}

We construct a cross-conformation benchmark covering nine binder
categories: small molecules, nonpeptidic macrocycles, linear peptides,
cyclic peptides, non-antibody proteins, full-length antibodies,
antibody fragments and nanobodies, oligonucleotides, and long
nucleic acids including mRNA.
Data sources, category sizes, and construction procedures are
described in App.~\ref{app:experimental-details}.
Initial candidates come from DynamicFlow, FlexSBDD, and YuelDesign
for small molecules and nonpeptidic macrocycles, and from
BoltzGen, Chai-1, and Protenix for the remaining seven binder categories.
We compare each source model's outputs before and after
\textsc{FlexEvo} adaptation. Evaluation uses three natural target conformations at the same
binding site: \textit{Reference}, \textit{Matched}, and
\textit{Divergent}.
In this benchmark, Matched preserves the Reference regulatory
state, whereas Divergent involves a propeptide-associated state
change with local structural rearrangements.
Grouping follows biological evidence rather than RMSD thresholds.
Adaptation uses only Reference and the initial candidate pool,
with source models frozen.
The same selected candidate is evaluated across all three states
without further adaptation. We report Vina scores and high-affinity fractions for ligand
generation, and DockQ, ligand RMSD (LRMSD), and the fraction
with LRMSD below $2$~\AA{} for complexes with native reference
structures.
Contact F1, interaction scores, molecular properties, and model
confidence provide complementary assessments
(App.~\ref{app:evaluation-metrics}).
These metrics are distinct from the geometric surrogate
objectives used during search.

\subsection{Main Results}
\label{sec:main_results}

\paragraph{Cross-conformation performance.}
Table~\ref{tab:main_results} shows improvements in Vina for ligand
generation and DockQ for complex design under both Matched and
Divergent conditions across the evaluated source--category combinations.
For DynamicFlow small molecules, mean Vina under Divergent improves
from $-4.85$ to $-7.80$~kcal/mol, while the high-affinity fraction
increases from $26.4\%$ to $61.7\%$.
For Protenix linear peptides, mean DockQ increases from $0.3519$ to
$0.8363$ under Divergent, compared with $0.7154$ to $0.8198$ under
Reference.
Across all 27 source--category combinations, the primary-metric gain is
larger under Divergent than under Reference, indicating that the benefit
is most pronounced on the biologically divergent held-out state.
Importantly, \textsc{FlexEvo} optimizes only Reference-derived geometric
surrogates rather than Vina, DockQ, or other external evaluators; these
held-out improvements therefore provide out-of-objective evidence for
cross-conformation transfer.
Figure~\ref{fig:main_analysis}(a--c) further shows the joint score
distributions under paired target conditions, complementing the
condition-wise summaries in Table~\ref{tab:main_results}.

\paragraph{Input-state quality and metric trade-offs.}
Reference-state Vina and DockQ also improve in
Table~\ref{tab:main_results}, so gains on alternative conformations
coexist with improved input-state scores.
The trend is less uniform across complementary metrics.
For example, BoltzGen linear peptides have higher Reference-state
DockQ, but the fraction with LRMSD below $2$~\AA{} decreases from
$42.2\%$ to $40.2\%$.
Thus, the dominant cross-state gains do not imply uniform improvement
on every secondary metric, highlighting the complementary information
captured by the reported structural measures.

\paragraph{Geometric basis of initialization.}
We examine the geometric cues used for FlexBox initialization using
200 ATLAS proteins~\citep{atlas2024}, averaging residue RMSF over
three MD simulations per protein.
Figure~\ref{fig:main_analysis}(d,h) compares the lowest and highest
within-protein quartiles of packing sparsity (inverse $C_\alpha$
weighted contact number) and distance from the whole-chain
$C_\alpha$ geometric center.
Each quartile's mean RMSF is normalized by the corresponding
protein's whole-chain mean.
For both descriptors, the highest quartile has greater normalized
RMSF on average.
These associations motivate the initial allocation of editing freedom;
the residue roles are subsequently refined by search.

\paragraph{Multi-objective search behavior.}
Figure~\ref{fig:main_analysis}(e--g) shows initial, intermediate, and
final populations in the plotted objective spaces.
Later populations shift toward lower plotted values of $f_1$ and
$f_2$, while retaining candidates with different objective trade-offs.
This spread is consistent with Pareto-based survival preserving
alternative preservation--robustness trade-offs rather than collapsing
the search to a single operating point.
These snapshots characterize internal search behavior under
genome-dependent perturbation probes; compatibility with held-out
natural conformations is assessed separately in
Table~\ref{tab:main_results}.

\paragraph{Adaptation dynamics and runtime.}
Figure~\ref{fig:optimization_dynamics}(a--c) plots evaluation metrics
against elapsed adaptation time, linking the attained scores to the
computation time of the displayed runs.
In panel (a), the high-affinity fraction reaches $60.0\%$ by
approximately two minutes and changes little near the end of the run,
while $\log P$ continues to increase.
In panels (b,c), on-target pAE decreases overall and off-target pAE
increases overall over the displayed intervals of approximately
$2$--$3$ minutes, with local reversals in some traces.

\paragraph{Structural progression.}
Figure~\ref{fig:optimization_dynamics}(d--f) provides a structural view
of adaptation in one ligand-design example. 
The early, intermediate, and late candidates have Vina scores of
$-7.900$, $-10.491$, and $-12.666$~kcal/mol, respectively.
The snapshots show changes in candidate geometry within the binding
pocket alongside progressively more favorable docking scores.

\subsection{Further Analyses}
\label{sec:further_analysis}

Among six optimization strategies, \textsc{FlexEvo} ranks first
on the reported HV-AUC and B-pocket top-10 scores
(App.~\ref{app:optimizer-ablation}).
Across nine categories, the full method achieves better
Divergent-state Vina or DockQ than each of the three
component-removal variants
(Table~\ref{tab:flexevo-ablation-multimetric}).
Runtime comparisons report an additional 1.4--3.1 minutes
of adaptation per sample
(App.~\ref{app:runtime-profiles}).
Detailed results cover all three target states across
27 source--category combinations
(App.~\ref{app:modality-results}).
Joint-state plots distinguish simultaneous threshold
attainment from attainment in only one target state
(App.~\ref{app:joint-comparisons}). Reference-library comparisons quantify candidate similarity
to sampled public molecules across nine categories
(App.~\ref{app:similarity-analysis}).
On/off-target diagnostics report the fraction meeting
joint pAE criteria within the evaluated off-target panel
(App.~\ref{app:interaction-diagnostics}).
Multimetric profiles compare docking, structural, and
molecular-property scores across methods and states
(App.~\ref{app:radar-profiles}).
Structure-pool analyses characterize structural composition
and biological annotation coverage
(App.~\ref{app:pool-context}).
Public reference structures and property profiles
illustrate binder architectures
(App.~\ref{app:reference-illustrations}).

\section{Conclusion}
\label{sec:conclusion}

We introduced \textsc{FlexEvo}, a model-agnostic framework that adapts binder candidates once at inference time using a single observed target conformation, without retraining the source model.
An outer evolutionary search optimizes residue-level constraint configurations, whose decoded \emph{FlexBoxes} guide structural editing and repair.
The bi-objective search balances input-state interaction preservation with robustness to local geometric perturbations, using perturbation
robustness as a proxy for cross-state compatibility.
Experiments across nine binder categories and multiple candidate sources show improved Vina and DockQ scores on natural target conformations unseen during adaptation, while maintaining input-state
performance on these metrics. These findings support constrained adaptation of existing model
outputs to improve compatibility across target conformations.

\subsection*{AI use statement}
We used generative AI tools for manuscript writing, language editing, organization, and feedback on the manuscript.

\subsection*{Ethics statement}
This computational study does not establish biological safety
or clinical efficacy; biomedical use requires further validation.

\subsection*{Reproducibility Statement}
Implementation details, default settings, and evaluation procedures
are documented in Apps.~\ref{app:flexevo}
and~\ref{app:experimental-details}.

\bibliography{iclr2027_conference}

@inproceedings{pocket2mol,
  author = {Peng, Xingang and Luo, Shitong and Guan, Jiaqi and Xie, Qi and Peng, Jian and Ma, Jianzhu},
  title = {{Pocket2Mol}: Efficient Molecular Sampling Based on 3D Protein Pockets},
  booktitle = {Proceedings of the 39th International Conference on Machine Learning},
  series = {Proceedings of Machine Learning Research},
  volume = {162},
  pages = {17644--17655},
  publisher = {PMLR},
  year = {2022},
  url = {https://proceedings.mlr.press/v162/peng22b.html}
}

@inproceedings{targetdiff,
  author = {Guan, Jiaqi and Qian, Wesley Wei and Peng, Xingang and Su, Yufeng and Peng, Jian and Ma, Jianzhu},
  title = {3D Equivariant Diffusion for Target-Aware Molecule Generation and Affinity Prediction},
  booktitle = {The Eleventh International Conference on Learning Representations},
  publisher = {OpenReview.net},
  year = {2023},
  url = {https://openreview.net/forum?id=kJqXEPXMsE0}
}

@article{diffsbdd,
  author = {Schneuing, Arne and Harris, Charles and Du, Yuanqi and Didi, Kieran and Jamasb, Arian and Igashov, Ilia and Du, Weitao and Gomes, Carla and Blundell, Tom L. and Lio, Pietro and Welling, Max and Bronstein, Michael and Correia, Bruno},
  title = {Structure-Based Drug Design with Equivariant Diffusion Models},
  journal = {Nature Computational Science},
  volume = {4},
  number = {12},
  pages = {899--909},
  year = {2024},
  doi = {10.1038/s43588-024-00737-x}
}

@inproceedings{decompdiff,
  author = {Guan, Jiaqi and Zhou, Xiangxin and Yang, Yuwei and Bao, Yu and Peng, Jian and Ma, Jianzhu and Liu, Qiang and Wang, Liang and Gu, Quanquan},
  title = {{DecompDiff}: Diffusion Models with Decomposed Priors for Structure-Based Drug Design},
  booktitle = {Proceedings of the 40th International Conference on Machine Learning},
  series = {Proceedings of Machine Learning Research},
  volume = {202},
  pages = {11827--11846},
  publisher = {PMLR},
  year = {2023},
  url = {https://proceedings.mlr.press/v202/guan23a.html}
}

@inproceedings{pepglad,
  author = {Kong, Xiangzhe and Jia, Yinjun and Huang, Wenbing and Liu, Yang},
  title = {Full-Atom Peptide Design with Geometric Latent Diffusion},
  booktitle = {Advances in Neural Information Processing Systems},
  volume = {37},
  pages = {74808--74839},
  publisher = {Curran Associates, Inc.},
  year = {2024},
  doi = {10.52202/079017-2379}
}

@article{rfdiffusion,
  author = {Watson, Joseph L. and Juergens, David and Bennett, Nathaniel R. and Trippe, Brian L. and Yim, Jason and Eisenach, Helen E. and Ahern, Woody and Borst, Andrew J. and Ragotte, Robert J. and Milles, Lukas F. and Wicky, Basile I. M. and Hanikel, Nikita and Pellock, Samuel J. and Courbet, Alexis and Sheffler, William and Wang, Jue and Venkatesh, Preetham and Sappington, Isaac and V{\'a}zquez Torres, Susana and Lauko, Anna and De Bortoli, Valentin and Mathieu, Emile and Ovchinnikov, Sergey and Barzilay, Regina and Jaakkola, Tommi S. and DiMaio, Frank and Baek, Minkyung and Baker, David},
  title = {De Novo Design of Protein Structure and Function with {RFdiffusion}},
  journal = {Nature},
  volume = {620},
  number = {7976},
  pages = {1089--1100},
  year = {2023},
  doi = {10.1038/s41586-023-06415-8}
}

@article{bindcraft,
  author = {Pacesa, Martin and Nickel, Lennart and Schellhaas, Christian and Schmidt, Joseph and Pyatova, Ekaterina and Kissling, Lucas and Barendse, Patrick and Choudhury, Jagrity and Kapoor, Srajan and Alcaraz-Serna, Ana and Cho, Yehlin and Ghamary, Kourosh H. and Vinu{\'e}, Laura and Yachnin, Brahm J. and Wollacott, Andrew M. and Buckley, Stephen and Westphal, Adrie H. and Lindhoud, Simon and Georgeon, Sandrine and Goverde, Casper A. and Hatzopoulos, Georgios N. and G{\"o}nczy, Pierre and Muller, Yannick D. and Schwank, Gerald and Swarts, Daan C. and Vecchio, Alex J. and Schneider, Bernard L. and Ovchinnikov, Sergey and Correia, Bruno E.},
  title = {One-Shot Design of Functional Protein Binders with {BindCraft}},
  journal = {Nature},
  volume = {646},
  number = {8084},
  pages = {483--492},
  year = {2025},
  doi = {10.1038/s41586-025-09429-6}
}

@article{alphafold3,
  author = {Abramson, Josh and Adler, Jonas and Dunger, Jack and Evans, Richard and Green, Tim and Pritzel, Alexander and Ronneberger, Olaf and Willmore, Lindsay and Ballard, Andrew J. and Bambrick, Joshua and Bodenstein, Sebastian W. and Evans, David A. and Hung, Chia-Chun and O'Neill, Michael and Reiman, David and Tunyasuvunakool, Kathryn and Wu, Zachary and {\v{Z}}emgulyt{\.{e}}, Akvil{\.{e}} and Arvaniti, Eirini and Beattie, Charles and Bertolli, Ottavia and Bridgland, Alex and Cherepanov, Alexey and Congreve, Miles and Cowen-Rivers, Alexander I. and Cowie, Andrew and Figurnov, Michael and Fuchs, Fabian B. and Gladman, Hannah and Jain, Rishub and Khan, Yousuf A. and Low, Caroline M. R. and Perlin, Kuba and Potapenko, Anna and Savy, Pascal and Singh, Sukhdeep and Stecula, Adrian and Thillaisundaram, Ashok and Tong, Catherine and Yakneen, Sergei and Zhong, Ellen D. and Zielinski, Michal and {\v{Z}}{\'i}dek, Augustin and Bapst, Victor and Kohli, Pushmeet and Jaderberg, Max and Hassabis, Demis and Jumper, John M.},
  title = {Accurate Structure Prediction of Biomolecular Interactions with {AlphaFold 3}},
  journal = {Nature},
  volume = {630},
  number = {8016},
  pages = {493--500},
  year = {2024},
  doi = {10.1038/s41586-024-07487-w}
}

@article{chai1,
  author = {{Chai Discovery} and Boitreaud, Jacques and Dent, Jack and McPartlon, Matthew and Meier, Joshua and Reis, Vinicius and Rogozhnikov, Alex and Wu, Kevin},
  title = {{Chai-1}: Decoding the Molecular Interactions of Life},
  journal = {bioRxiv},
  year = {2024},
  doi = {10.1101/2024.10.10.615955}
}

@article{protenix,
  author = {{ByteDance AML AI4Science Team} and Chen, Xinshi and Zhang, Yuxuan and Lu, Chan and Ma, Wenzhi and Guan, Jiaqi and Gong, Chengyue and Yang, Jincai and Zhang, Hanyu and Zhang, Ke and Wu, Shenghao and Zhou, Kuangqi and Yang, Yanping and Liu, Zhenyu and Wang, Lan and Shi, Bo and Shi, Shaochen and Xiao, Wenzhi},
  title = {{Protenix}: Advancing Structure Prediction Through a Comprehensive {AlphaFold3} Reproduction},
  journal = {bioRxiv},
  year = {2025},
  doi = {10.1101/2025.01.08.631967}
}

@article{induced_fit_docking,
  author = {Sherman, Woody and Day, Tyler and Jacobson, Matthew P. and Friesner, Richard A. and Farid, Ramy},
  title = {Novel Procedure for Modeling Ligand/Receptor Induced Fit Effects},
  journal = {Journal of Medicinal Chemistry},
  volume = {49},
  number = {2},
  pages = {534--553},
  year = {2006},
  doi = {10.1021/jm050540c}
}

@article{flexpepdock,
  author = {Raveh, Barak and London, Nir and Zimmerman, Lior and Schueler-Furman, Ora},
  title = {Rosetta {FlexPepDock} Ab-Initio: Simultaneous Folding, Docking and Refinement of Peptides onto Their Receptors},
  journal = {PLOS ONE},
  volume = {6},
  number = {4},
  pages = {e18934},
  year = {2011},
  doi = {10.1371/journal.pone.0018934}
}

@article{rosettadock4,
  author = {Marze, Nicholas A. and Roy Burman, Shourya S. and Sheffler, William and Gray, Jeffrey J.},
  title = {Efficient Flexible Backbone Protein--Protein Docking for Challenging Targets},
  journal = {Bioinformatics},
  volume = {34},
  number = {20},
  pages = {3461--3469},
  year = {2018},
  doi = {10.1093/bioinformatics/bty355}
}

@article{huang2006ensemble,
  author = {Huang, Sheng-You and Zou, Xiaoqin},
  title = {Ensemble Docking of Multiple Protein Structures: Considering Protein Structural Variations in Molecular Docking},
  journal = {Proteins: Structure, Function, and Bioinformatics},
  volume = {66},
  number = {2},
  pages = {399--421},
  year = {2007},
  doi = {10.1002/prot.21214}
}

@article{amaro2018ensemble,
  author = {Amaro, Rommie E. and Baudry, Jerome and Chodera, John and Demir, {\"O}zlem and McCammon, J. Andrew and Miao, Yinglong and Smith, Jeremy C.},
  title = {Ensemble Docking in Drug Discovery},
  journal = {Biophysical Journal},
  volume = {114},
  number = {10},
  pages = {2271--2278},
  year = {2018},
  doi = {10.1016/j.bpj.2018.02.038}
}

@article{sauer2020multistate,
  author = {Sauer, Marion F. and Sevy, Alexander M. and Crowe, James E. and Meiler, Jens},
  title = {Multi-State Design of Flexible Proteins Predicts Sequences Optimal for Conformational Change},
  journal = {PLOS Computational Biology},
  volume = {16},
  number = {2},
  pages = {e1007339},
  year = {2020},
  doi = {10.1371/journal.pcbi.1007339}
}

@article{neuralplexer,
  author = {Qiao, Zhuoran and Nie, Weili and Vahdat, Arash and Miller, Thomas F. and Anandkumar, Animashree},
  title = {State-Specific Protein--Ligand Complex Structure Prediction with a Multiscale Deep Generative Model},
  journal = {Nature Machine Intelligence},
  volume = {6},
  number = {2},
  pages = {195--208},
  year = {2024},
  doi = {10.1038/s42256-024-00792-z}
}

@article{dynamicbind,
  author = {Lu, Wei and Zhang, Jixian and Huang, Weifeng and Zhang, Ziqiao and Jia, Xiangyu and Wang, Zhenyu and Shi, Leilei and Li, Chengtao and Wolynes, Peter G. and Zheng, Shuangjia},
  title = {{DynamicBind}: Predicting Ligand-Specific Protein--Ligand Complex Structure with a Deep Equivariant Generative Model},
  journal = {Nature Communications},
  volume = {15},
  number = {1},
  pages = {1071},
  year = {2024},
  doi = {10.1038/s41467-024-45461-2}
}

@inproceedings{flexsbdd,
  author = {Zhang, Zaixi and Wang, Mengdi and Liu, Qi},
  title = {{FlexSBDD}: Structure-Based Drug Design with Flexible Protein Modeling},
  booktitle = {Advances in Neural Information Processing Systems},
  volume = {37},
  pages = {53918--53944},
  publisher = {Curran Associates, Inc.},
  year = {2024},
  doi = {10.52202/079017-1707}
}

@inproceedings{dynamicflow,
  author = {Zhou, Xiangxin and Xiao, Yi and Lin, Haowei and He, Xinheng and Guan, Jiaqi and Wang, Yang and Liu, Qiang and Zhou, Feng and Wang, Liang and Ma, Jianzhu},
  title = {Integrating Protein Dynamics into Structure-Based Drug Design via Full-Atom Stochastic Flows},
  booktitle = {The Thirteenth International Conference on Learning Representations},
  publisher = {OpenReview.net},
  year = {2025},
  url = {https://openreview.net/forum?id=9qS3HzSDNv}
}

@article{yueldesign,
  author = {Wang, Jian and Zhang, Dong Yan and Budakoti, Shreshty and Dokholyan, Nikolay V.},
  title = {A Diffusion-Based Framework for Designing Molecules in Flexible Protein Pockets},
  journal = {bioRxiv},
  year = {2025},
  doi = {10.1101/2025.05.27.656443}
}

@article{boehr2006dynamic,
  author = {Boehr, David D. and McElheny, Dan and Dyson, H. Jane and Wright, Peter E.},
  title = {The Dynamic Energy Landscape of Dihydrofolate Reductase Catalysis},
  journal = {Science},
  volume = {313},
  number = {5793},
  pages = {1638--1642},
  year = {2006},
  doi = {10.1126/science.1130258}
}

@article{buttenschoen2024posebusters,
  author = {Buttenschoen, Martin and Morris, Garrett M. and Deane, Charlotte M.},
  title = {{PoseBusters}: {AI}-based docking methods fail to generate physically valid poses or generalise to novel sequences},
  journal = {Chemical Science},
  volume = {15},
  number = {9},
  pages = {3130--3139},
  year = {2024},
  doi = {10.1039/D3SC04185A}
}

@article{boehr2009,
  author = {Boehr, David D. and Nussinov, Ruth and Wright, Peter E.},
  title = {The Role of Dynamic Conformational Ensembles in Biomolecular Recognition},
  journal = {Nature Chemical Biology},
  volume = {5},
  number = {11},
  pages = {789--796},
  year = {2009},
  doi = {10.1038/nchembio.232}
}

@article{jamroz2012,
  author = {Jamroz, Michal and Kolinski, Andrzej and Kihara, Daisuke},
  title = {Structural Features that Predict Real-Value Fluctuations of Globular Proteins},
  journal = {Proteins: Structure, Function, and Bioinformatics},
  volume = {80},
  number = {5},
  pages = {1425--1435},
  year = {2012},
  doi = {10.1002/prot.24040}
}

@article{lin2008,
  author = {Lin, Chih-Peng and Huang, Shao-Wei and Lai, Yan-Long and Yen, Shih-Chung and Shih, Chien-Hua and Lu, Chih-Hao and Huang, Cuen-Chao and Hwang, Jenn-Kang},
  title = {Deriving Protein Dynamical Properties from Weighted Protein Contact Number},
  journal = {Proteins},
  volume = {72},
  number = {3},
  pages = {929--935},
  year = {2008},
  doi = {10.1002/prot.21983}
}

@article{shih2007,
  author = {Shih, Chien-Hua and Huang, Shao-Wei and Yen, Shih-Chung and Lai, Yan-Long and Yu, Sung-Huan and Hwang, Jenn-Kang},
  title = {A Simple Way to Compute Protein Dynamics without a Mechanical Model},
  journal = {Proteins: Structure, Function, and Bioinformatics},
  volume = {68},
  number = {1},
  pages = {34--38},
  year = {2007},
  doi = {10.1002/prot.21430}
}

@article{rajamani2004,
  author = {Rajamani, Deepa and Thiel, Spencer and Vajda, Sandor and Camacho, Carlos J.},
  title = {Anchor Residues in Protein-Protein Interactions},
  journal = {Proceedings of the National Academy of Sciences},
  volume = {101},
  number = {31},
  pages = {11287--11292},
  year = {2004},
  doi = {10.1073/pnas.0401942101}
}

@article{bogan1998,
  author = {Bogan, Andrew A. and Thorn, Kurt S.},
  title = {Anatomy of Hot Spots in Protein Interfaces},
  journal = {Journal of Molecular Biology},
  volume = {280},
  number = {1},
  pages = {1--9},
  year = {1998},
  doi = {10.1006/jmbi.1998.1843}
}

@article{atlas2024,
  author = {Vander Meersche, Yann and Cretin, Gabriel and Gheeraert, Aria and Gelly, Jean-Christophe and Galochkina, Tatiana},
  title = {{ATLAS}: Protein Flexibility Description from Atomistic Molecular Dynamics Simulations},
  journal = {Nucleic Acids Research},
  volume = {52},
  number = {D1},
  pages = {D384--D392},
  year = {2024},
  doi = {10.1093/nar/gkad1084}
}

@article{deb2002fast,
  author = {Deb, K. and Pratap, A. and Agarwal, S. and Meyarivan, T.},
  title = {A Fast and Elitist Multiobjective Genetic Algorithm: {NSGA-II}},
  journal = {IEEE Transactions on Evolutionary Computation},
  volume = {6},
  number = {2},
  pages = {182--197},
  year = {2002},
  doi = {10.1109/4235.996017}
}

@article{boltzgen,
  author = {Stark, Hannes and Faltings, Felix and Choi, MinGyu and Xie, Yuxin and Hur, Eunsu and O'Donnell, Timothy and Bushuiev, Anton and U{\c{c}}ar, Talip and Passaro, Saro and Mao, Weian and Reveiz, Mateo and Bushuiev, Roman and Portnoi, Tally and Pluskal, Tom{\'a}{\v{s}} and Sivic, Josef and Kreis, Karsten and Vahdat, Arash and Ray, Shamayeeta and Goldstein, Jonathan T. and Savinov, Andrew and Hambalek, Jacob A. and Gupta, Anshika and Taquiri-Diaz, Diego A. and Zhang, Yaotian and Snyder, Samuel J. and Hatstat, A. Katherine and Arada, Angelika and Kim, Nam Hyeong and Fan, Haoyu and Tackie-Yarboi, Ethel and Boselli, Dylan and Schnaider, Lee and Liu, Chang C. and Li, Gene-Wei and Hnisz, Denes and Sabatini, David M. and DeGrado, William F. and Wohlwend, Jeremy and Corso, Gabriele and Barzilay, Regina and Jaakkola, Tommi},
  title = {{BoltzGen}: Toward Universal Binder Design},
  journal = {bioRxiv},
  year = {2025},
  doi = {10.1101/2025.11.20.689494}
}

@article{medscout,
  author = {Liu, Anglin and Chen, Ruichao and Lu, Yi and Xu, Hongxia and Chen, Jintai},
  title = {{Med-Scout}: Curing {MLLMs}' Geometric Blindness in Medical Perception via Geometry-Aware {RL} Post-Training},
  journal = {arXiv preprint arXiv:2601.23220},
  year = {2026},
  url = {https://arxiv.org/abs/2601.23220}
}

@article{liu2022ivim,
  author = {Liu, Baoer and Zeng, Qingyuan and Huang, Jianbin and Zhang, Jing and Zheng, Zeyu and Liao, Yuting and Deng, Kan and Zhou, Wu and Xu, Yikai},
  title = {{IVIM} using convolutional neural networks predicts microvascular invasion in {HCC}},
  journal = {European Radiology},
  volume = {32},
  number = {10},
  pages = {7185--7195},
  year = {2022},
  doi = {10.1007/s00330-022-08927-9}
}

@inproceedings{wang2024fairgnn,
  author = {Wang, Zhenzhong and Zeng, Qingyuan and Lin, Wanyu and Jiang, Min and Tan, Kay Chen},
  title = {Generating Diagnostic and Actionable Explanations for Fair Graph Neural Networks},
  booktitle = {Proceedings of the AAAI Conference on Artificial Intelligence},
  volume = {38},
  pages = {21690--21698},
  year = {2024},
  doi = {10.1609/aaai.v38i19.30168}
}

@article{gong2025crossmodality,
  author = {Gong, Yunpeng and Zeng, Qingyuan and Xu, Dejun and Wang, Zhenzhong and Jiang, Min},
  title = {Cross-Modality Attack Boosted by Gradient-Evolutionary Multiform Optimization},
  journal = {arXiv preprint arXiv:2409.17977},
  year = {2024},
  url = {https://arxiv.org/abs/2409.17977}
}

@inproceedings{zeng2024crosstask,
  author = {Zeng, Qingyuan and Gong, Yunpeng and Jiang, Min},
  title = {Cross-Task Attack: A Self-Supervision Generative Framework Based on Attention Shift},
  booktitle = {International Joint Conference on Neural Networks (IJCNN)},
  pages = {1--8},
  year = {2024},
  publisher = {IEEE},
  doi = {10.1109/IJCNN60899.2024.10650118}
}

@article{zeng2021attentionmvi,
  author = {Zeng, Qingyuan and Liu, Baoer and Xu, Yikai and Zhou, Wu},
  title = {An attention-based deep learning model for predicting microvascular invasion of hepatocellular carcinoma using an intra-voxel incoherent motion model of diffusion-weighted magnetic resonance imaging},
  journal = {Physics in Medicine \& Biology},
  volume = {66},
  number = {18},
  pages = {185019},
  year = {2021},
  doi = {10.1088/1361-6560/ac22db}
}

@article{wang2025muse,
  author = {Wang, Zhenzhong and Zeng, Qingyuan and Lin, Wanyu and Jiang, Min and Tan, Kay Chen},
  title = {Multiview Subgraph Neural Networks: Self-Supervised Learning With Scarce Labeled Data},
  journal = {IEEE Transactions on Neural Networks and Learning Systems},
  volume = {36},
  number = {6},
  pages = {11548--11561},
  year = {2025},
  doi = {10.1109/TNNLS.2024.3443074}
}

@inproceedings{zeng2024aaa,
  author = {Zeng, Qingyuan and Wang, Zhenzhong and Cheung, Yiu-ming and Jiang, Min},
  title = {Ask, Attend, Attack: An Effective Decision-Based Black-Box Targeted Attack for Image-to-Text Models},
  booktitle = {Advances in Neural Information Processing Systems},
  volume = {37},
  pages = {105819--105847},
  year = {2024}
}

@inproceedings{zeng2021pharmacokinetic,
  author = {Zeng, Qingyuan and Zhou, Wu},
  title = {An Attention Based Deep Learning Model for Direct Estimation of Pharmacokinetic Maps from {DCE-MRI} Images},
  booktitle = {2021 IEEE International Conference on Bioinformatics and Biomedicine (BIBM)},
  pages = {2368--2375},
  year = {2021},
  publisher = {IEEE},
  doi = {10.1109/BIBM52615.2021.9669582}
}

@inproceedings{lin2026visionlaw,
  author = {Lin, Jiajing and Jiang, Shu and Zeng, Qingyuan and Wang, Zhenzhong and Jiang, Min},
  title = {{VisionLaw}: Inferring Interpretable Intrinsic Dynamics from Visual Observations via Bilevel Optimization},
  booktitle = {The Fourteenth International Conference on Learning Representations},
  year = {2026},
  url = {https://openreview.net/forum?id=eWoUcwEtLt}
}

@inproceedings{zeng2026digitalink,
  author = {Zeng, Qingyuan and Jiang, Shu and Lin, Jiajing and Wang, Zhenzhong and Tan, Kay Chen and Jiang, Min},
  title = {Fading the Digital Ink: A Universal Black-Box Attack Framework for {3DGS} Watermarking Systems},
  booktitle = {Proceedings of the AAAI Conference on Artificial Intelligence},
  volume = {40},
  pages = {38084--38092},
  year = {2026}
}

@inproceedings{gong2025piif,
  author = {Gong, Yunpeng and Zeng, Qingyuan and Liu, Chenchen and Guan, Jian and Lin, Jiajing and Zhang, Chuangliang and Hou, Yongjie and Huang, Yipeng and Wang, Zhenzhong and Li, Xiaobo and Feng, Xiaodan and Jiang, Min},
  title = {Inducing Implicit Focus through Structured Irrelevance for Robust Neural {PDE} Solvers},
  booktitle = {2025 International Conference on Machine Intelligence and Nature-Inspired Computing (MIND)},
  year = {2025}
}

@article{zeng2026scene,
  author = {Zeng, Qingyuan and Chen, Ziyang and Cai, Pengxiang and Guan, Zixin and Liu, Anglin and Qin, Lang and Lai, Xinyao and Chen, Jintai},
  title = {Can Broad Biomedical Knowledge be Contextualized into Scenario-Grounded Propositions?},
  journal = {arXiv preprint arXiv:2605.27082},
  year = {2026},
  url = {https://arxiv.org/abs/2605.27082}
}

@article{zeng2026reuse,
  author = {Zeng, Qingyuan and Cai, Pengxiang and Guan, Zixin and Chen, Ziyang and Liu, Anglin and Lai, Xinyao and Chen, Jintai},
  title = {Don't Retrain, Just Reuse: Recovering Dual-Target Molecules from Single-Target Diffusion Models},
  journal = {arXiv preprint arXiv:2605.25681},
  year = {2026},
  url = {https://arxiv.org/abs/2605.25681}
}

@inproceedings{zeng2025watermark,
  author = {Zeng, Qingyuan and Gong, Yunpeng and Zhang, Chuangliang and Jiang, Shu and Wang, Zhenzhong and Jiang, Min},
  title = {Black-Box Watermark Removal Using Diffusion Models and Self-Attention Mechanisms},
  booktitle = {2025 International Conference on Machine Intelligence and Nature-Inspired Computing (MIND)},
  year = {2025}
}

@inproceedings{zhang2025dualcam,
  author = {Zhang, Chuangliang and Zeng, Qingyuan and Gong, Yunpeng and Wang, Zhenzhong and Lin, Jiajing and Jiang, Min},
  title = {{DualCAM++}: Dual-Saliency Guided Sample Selection for Plug-and-Play Backdoor Enhancement},
  booktitle = {2025 International Conference on Machine Intelligence and Nature-Inspired Computing (MIND)},
  year = {2025}
}
\bibliographystyle{iclr2027_conference}

\clearpage
\appendix

\startcontents[appendix]
\section*{Appendix Contents}
\printcontents[appendix]{}{1}{\setcounter{tocdepth}{2}}

\clearpage

\section*{Appendix Overview}

This supplement separates implementation detail, experimental protocol,
supplementary evidence, and mathematical analysis so that each claim can be
traced to the information that supports it. Appendix~\ref{app:flexevo} expands the
coordinate-level \textsc{FlexEvo} core, including target encoding, candidate
screening, constrained editing, perturbation scoring, genome evolution, and
the reference configuration. Appendix~\ref{app:experimental-details} defines the evaluation design,
metric conventions, comparison rules, and information required to reproduce
the experiments. Appendix~\ref{app:supplementary-results} reports the supplementary results and analyses,
including modality-specific tables, cross-conformation comparisons,
ablations, contextual data summaries, similarity analyses, and computational
cost. Appendix~\ref{app:flexevo-theory-v30} gives deterministic properties of the implemented geometric
scores and states the assumptions needed to transfer finite-probe bounds to
an unseen target structure.

The empirical and mathematical parts address different questions.
\textsc{FlexEvo} adapts candidates using a single Reference conformation
and evaluates the frozen representatives on held-out Matched and
Divergent natural conformations. Reference-derived geometric probes
guide the search toward contact preservation and perturbation tolerance.
The natural-state benchmark quantifies cross-conformation performance,
and the theoretical analysis establishes finite-probe score bounds and
transfer enclosures under explicit target-proximity conditions.

\paragraph{Reproducibility at a glance.}
The reference implementation, default settings, and evaluation boundaries
are documented separately from the numerical summaries. The complete core
schedule contains 72 genome evaluations, all starting from the common
screened pool. Appendix~\ref{app:reproducibility} specifies the experimental
records and aggregation rules for reproducible comparisons. 

\begin{table*}[!t]
\centering
\caption{Small-molecule drugs. Black: base model. For directional metrics with FlexEvo, green: improved or tied; red: worse. Adapted logP is descriptive and shown in black; arrows indicate preferred directions where defined. QED, SA, Lipinski, logP, and PAINS summaries are presented under Reference. Binding scores are reported separately for Reference, Matched, and Divergent. Diversity summarizes the frozen candidate pool and is reported once under Reference; dashes indicate that it is not repeated for Matched or Divergent.}
\label{tab:flexevo-app-small-molecule}
\begingroup
\fontsize{9.0}{10.5}\selectfont
\setlength{\tabcolsep}{2.0pt}
\renewcommand{\arraystretch}{1.02}
\setlength{\aboverulesep}{1.2pt}
\setlength{\belowrulesep}{1.2pt}
\setlength{\heavyrulewidth}{0.65pt}
\setlength{\lightrulewidth}{0.35pt}
\setlength{\cmidrulewidth}{0.25pt}
\noindent\makebox[\linewidth][l]{\textbf{A. Binding and independent scores}}\par\vspace{2pt}
\noindent\resizebox{\linewidth}{!}{%
\begin{tabular*}{520pt}{@{\extracolsep{\fill}}lllllllllllll@{}}
\toprule
Base model & \multicolumn{2}{c}{\FlexEvoAppHeader{Vina Mean$\downarrow$}} & \multicolumn{2}{c}{\FlexEvoAppHeader{Vina Median$\downarrow$}} & \multicolumn{2}{c}{\FlexEvoAppHeader{High affinity (\%)$\uparrow$}} & \multicolumn{2}{c}{\FlexEvoAppHeader{Vinardo$\uparrow$}} & \multicolumn{2}{c}{\FlexEvoAppHeader{AutoDock4$\uparrow$}} & \multicolumn{2}{c}{\FlexEvoAppHeader{RTMScore$\uparrow$}}\\
\cmidrule(lr){2-3}\cmidrule(lr){4-5}\cmidrule(lr){6-7}\cmidrule(lr){8-9}\cmidrule(lr){10-11}\cmidrule(lr){12-13}
 & \multicolumn{1}{c}{Base} & \multicolumn{1}{c}{+FlexEvo} & \multicolumn{1}{c}{Base} & \multicolumn{1}{c}{+FlexEvo} & \multicolumn{1}{c}{Base} & \multicolumn{1}{c}{+FlexEvo} & \multicolumn{1}{c}{Base} & \multicolumn{1}{c}{+FlexEvo} & \multicolumn{1}{c}{Base} & \multicolumn{1}{c}{+FlexEvo} & \multicolumn{1}{c}{Base} & \multicolumn{1}{c}{+FlexEvo}\\
\midrule
\FlexEvoAppTest{13}{Reference conformation}
DynamicFlow & -6.6062 & \FlexEvoAppBetter{-9.1775} & -7.0317 & \FlexEvoAppBetter{-8.7901} & 51.2524 & \FlexEvoAppBetter{58.1689} & 0.4161 & \FlexEvoAppBetter{0.698} & 0.5782 & \FlexEvoAppBetter{0.5971} & 0.3716 & \FlexEvoAppBetter{0.6193}\\
FlexSBDD & -6.906 & \FlexEvoAppBetter{-9.2245} & -6.7984 & \FlexEvoAppBetter{-8.8448} & 47.0322 & \FlexEvoAppBetter{55.3719} & 0.4102 & \FlexEvoAppBetter{0.5524} & 0.3411 & \FlexEvoAppBetter{0.5847} & 0.3287 & \FlexEvoAppBetter{0.5322}\\
YuelDesign & -5.7174 & \FlexEvoAppBetter{-9.1754} & -5.4763 & \FlexEvoAppBetter{-8.8673} & 49.5198 & \FlexEvoAppBetter{62.5176} & 0.392 & \FlexEvoAppBetter{0.584} & 0.4317 & \FlexEvoAppBetter{0.6519} & 0.4674 & \FlexEvoAppBetter{0.6556}\\
\arrayrulecolor{FlexEvoAppLine}\cmidrule[0.25pt]{1-13}\arrayrulecolor{black}
\FlexEvoAppTest{13}{Matched conformation}
DynamicFlow & -6.4728 & \FlexEvoAppBetter{-8.3648} & -6.573 & \FlexEvoAppBetter{-8.3392} & 39.9146 & \FlexEvoAppBetter{56.6844} & 0.4056 & \FlexEvoAppBetter{0.5368} & 0.4285 & \FlexEvoAppBetter{0.5696} & 0.4595 & \FlexEvoAppBetter{0.594}\\
FlexSBDD & -6.6881 & \FlexEvoAppBetter{-8.2667} & -6.8355 & \FlexEvoAppBetter{-8.3107} & 39.2887 & \FlexEvoAppBetter{57.1034} & 0.3887 & \FlexEvoAppBetter{0.6069} & 0.4422 & \FlexEvoAppBetter{0.6427} & 0.4622 & \FlexEvoAppBetter{0.5723}\\
YuelDesign & -5.58 & \FlexEvoAppBetter{-8.4389} & -5.4642 & \FlexEvoAppBetter{-8.3575} & 36.7877 & \FlexEvoAppBetter{61.2628} & 0.3737 & \FlexEvoAppBetter{0.6431} & 0.4396 & \FlexEvoAppBetter{0.6163} & 0.4138 & \FlexEvoAppBetter{0.5844}\\
\arrayrulecolor{FlexEvoAppLine}\cmidrule[0.25pt]{1-13}\arrayrulecolor{black}
\FlexEvoAppTest{13}{Divergent conformation}
DynamicFlow & -4.8519 & \FlexEvoAppBetter{-7.7991} & -4.8465 & \FlexEvoAppBetter{-8.1933} & 26.3805 & \FlexEvoAppBetter{61.7445} & 0.3548 & \FlexEvoAppBetter{0.5868} & 0.3745 & \FlexEvoAppBetter{0.6723} & 0.4422 & \FlexEvoAppBetter{0.5711}\\
FlexSBDD & -4.8552 & \FlexEvoAppBetter{-7.7438} & -4.7988 & \FlexEvoAppBetter{-8.0889} & 38.8701 & \FlexEvoAppBetter{56.5146} & 0.3615 & \FlexEvoAppBetter{0.6033} & 0.4053 & \FlexEvoAppBetter{0.5875} & 0.365 & \FlexEvoAppBetter{0.5755}\\
YuelDesign & -3.9456 & \FlexEvoAppBetter{-7.6653} & -5.2145 & \FlexEvoAppBetter{-8.2671} & 31.2482 & \FlexEvoAppBetter{59.5304} & 0.3929 & \FlexEvoAppBetter{0.5776} & 0.3511 & \FlexEvoAppBetter{0.6188} & 0.3889 & \FlexEvoAppBetter{0.5866}\\
\bottomrule
\end{tabular*}%
}\par
\vspace{7pt}
\noindent\makebox[\linewidth][l]{\textbf{B. Drug-likeness and synthesizability}}\par\vspace{2pt}
\noindent\resizebox{\linewidth}{!}{%
\begin{tabular*}{520pt}{@{\extracolsep{\fill}}lllllllllllll@{}}
\toprule
Base model & \multicolumn{2}{c}{\FlexEvoAppHeader{QED Mean$\uparrow$}} & \multicolumn{2}{c}{\FlexEvoAppHeader{QED Median$\uparrow$}} & \multicolumn{2}{c}{\FlexEvoAppHeader{SA Mean$\uparrow$}} & \multicolumn{2}{c}{\FlexEvoAppHeader{SA Median$\uparrow$}} & \multicolumn{2}{c}{\FlexEvoAppHeader{Lipinski Mean$\uparrow$}} & \multicolumn{2}{c}{\FlexEvoAppHeader{Lipinski Median$\uparrow$}}\\
\cmidrule(lr){2-3}\cmidrule(lr){4-5}\cmidrule(lr){6-7}\cmidrule(lr){8-9}\cmidrule(lr){10-11}\cmidrule(lr){12-13}
 & \multicolumn{1}{c}{Base} & \multicolumn{1}{c}{+FlexEvo} & \multicolumn{1}{c}{Base} & \multicolumn{1}{c}{+FlexEvo} & \multicolumn{1}{c}{Base} & \multicolumn{1}{c}{+FlexEvo} & \multicolumn{1}{c}{Base} & \multicolumn{1}{c}{+FlexEvo} & \multicolumn{1}{c}{Base} & \multicolumn{1}{c}{+FlexEvo} & \multicolumn{1}{c}{Base} & \multicolumn{1}{c}{+FlexEvo}\\
\midrule
\FlexEvoAppTest{13}{Reference conformation}
DynamicFlow & 0.534 & \FlexEvoAppBetter{0.5935} & 0.5416 & \FlexEvoAppBetter{0.6024} & 0.6247 & \FlexEvoAppBetter{0.7097} & 0.6352 & \FlexEvoAppBetter{0.706} & 4.386 & \FlexEvoAppBetter{4.4522} & 3.1388 & \FlexEvoAppBetter{5.2364}\\
FlexSBDD & 0.3744 & \FlexEvoAppBetter{0.589} & 0.3856 & \FlexEvoAppBetter{0.5796} & 0.565 & \FlexEvoAppBetter{0.6888} & 0.5832 & \FlexEvoAppBetter{0.6883} & 1.7628 & \FlexEvoAppBetter{3.9447} & 2.2153 & \FlexEvoAppBetter{4.3738}\\
YuelDesign & 0.4212 & \FlexEvoAppBetter{0.6118} & 0.4257 & \FlexEvoAppBetter{0.5915} & 0.817 & \FlexEvoAppWorse{0.803} & 0.8227 & \FlexEvoAppBetter{0.8342} & 3.2247 & \FlexEvoAppBetter{4.8155} & 3.7201 & \FlexEvoAppBetter{4.1107}\\
\arrayrulecolor{FlexEvoAppLine}\cmidrule[0.25pt]{1-13}\arrayrulecolor{black}
\FlexEvoAppTest{13}{Matched conformation}
DynamicFlow & -- & -- & -- & -- & -- & -- & -- & -- & -- & -- & -- & --\\
FlexSBDD & -- & -- & -- & -- & -- & -- & -- & -- & -- & -- & -- & --\\
YuelDesign & -- & -- & -- & -- & -- & -- & -- & -- & -- & -- & -- & --\\
\arrayrulecolor{FlexEvoAppLine}\cmidrule[0.25pt]{1-13}\arrayrulecolor{black}
\FlexEvoAppTest{13}{Divergent conformation}
DynamicFlow & -- & -- & -- & -- & -- & -- & -- & -- & -- & -- & -- & --\\
FlexSBDD & -- & -- & -- & -- & -- & -- & -- & -- & -- & -- & -- & --\\
YuelDesign & -- & -- & -- & -- & -- & -- & -- & -- & -- & -- & -- & --\\
\bottomrule
\end{tabular*}%
}\par
\vspace{7pt}
\noindent\makebox[\linewidth][l]{\textbf{C. Physicochemical properties and diversity}}\par\vspace{2pt}
\noindent\resizebox{\linewidth}{!}{%
\begin{tabular*}{520pt}{@{\extracolsep{\fill}}lllllllllll@{}}
\toprule
Base model & \multicolumn{2}{c}{\FlexEvoAppHeader{logP Mean}} & \multicolumn{2}{c}{\FlexEvoAppHeader{logP Median}} & \multicolumn{2}{c}{\FlexEvoAppHeader{PAINS pass (\%)$\uparrow$}} & \multicolumn{2}{c}{\FlexEvoAppHeader{Diversity Mean$\uparrow$}} & \multicolumn{2}{c}{\FlexEvoAppHeader{Diversity Median$\uparrow$}}\\
\cmidrule(lr){2-3}\cmidrule(lr){4-5}\cmidrule(lr){6-7}\cmidrule(lr){8-9}\cmidrule(lr){10-11}
 & \multicolumn{1}{c}{Base} & \multicolumn{1}{c}{+FlexEvo} & \multicolumn{1}{c}{Base} & \multicolumn{1}{c}{+FlexEvo} & \multicolumn{1}{c}{Base} & \multicolumn{1}{c}{+FlexEvo} & \multicolumn{1}{c}{Base} & \multicolumn{1}{c}{+FlexEvo} & \multicolumn{1}{c}{Base} & \multicolumn{1}{c}{+FlexEvo}\\
\midrule
\FlexEvoAppTest{11}{Reference conformation}
DynamicFlow & 3.2525 & 3.1108 & 4.6554 & 4.2811 & 90.2154 & \FlexEvoAppBetter{94.822} & 0.6719 & \FlexEvoAppBetter{0.6984} & 0.6711 & \FlexEvoAppBetter{0.6965}\\
FlexSBDD & 3.6414 & 4.0985 & 5.0487 & 5.3611 & 87.3202 & \FlexEvoAppBetter{89.3711} & 0.6677 & \FlexEvoAppBetter{0.7094} & 0.6668 & \FlexEvoAppBetter{0.7118}\\
YuelDesign & 4.3262 & 4.7377 & 3.9432 & 4.7952 & 99.1444 & \FlexEvoAppWorse{95.669} & 0.7213 & \FlexEvoAppWorse{0.7193} & 0.7076 & \FlexEvoAppBetter{0.7126}\\
\arrayrulecolor{FlexEvoAppLine}\cmidrule[0.25pt]{1-11}\arrayrulecolor{black}
\FlexEvoAppTest{11}{Matched conformation}
DynamicFlow & -- & -- & -- & -- & -- & -- & -- & -- & -- & --\\
FlexSBDD & -- & -- & -- & -- & -- & -- & -- & -- & -- & --\\
YuelDesign & -- & -- & -- & -- & -- & -- & -- & -- & -- & --\\
\arrayrulecolor{FlexEvoAppLine}\cmidrule[0.25pt]{1-11}\arrayrulecolor{black}
\FlexEvoAppTest{11}{Divergent conformation}
DynamicFlow & -- & -- & -- & -- & -- & -- & -- & -- & -- & --\\
FlexSBDD & -- & -- & -- & -- & -- & -- & -- & -- & -- & --\\
YuelDesign & -- & -- & -- & -- & -- & -- & -- & -- & -- & --\\
\bottomrule
\end{tabular*}%
}\par
\endgroup
\end{table*}

\begin{table*}[!t]
\centering
\caption{Nonpeptidic macrocyclic drugs. Black: base model. For directional metrics with FlexEvo, green: improved or tied; red: worse. Adapted logP is descriptive and shown in black; arrows indicate preferred directions where defined. QED, SA, Lipinski, logP, and PAINS summaries are presented under Reference. Binding scores are reported separately for Reference, Matched, and Divergent. Diversity summarizes the frozen candidate pool and is reported once under Reference; dashes indicate that it is not repeated for Matched or Divergent.}
\label{tab:flexevo-app-macrocycle}
\begingroup
\fontsize{9.0}{10.5}\selectfont
\setlength{\tabcolsep}{2.0pt}
\renewcommand{\arraystretch}{1.02}
\setlength{\aboverulesep}{1.2pt}
\setlength{\belowrulesep}{1.2pt}
\setlength{\heavyrulewidth}{0.65pt}
\setlength{\lightrulewidth}{0.35pt}
\setlength{\cmidrulewidth}{0.25pt}
\noindent\makebox[\linewidth][l]{\textbf{A. Binding and independent scores}}\par\vspace{2pt}
\noindent\resizebox{\linewidth}{!}{%
\begin{tabular*}{520pt}{@{\extracolsep{\fill}}lllllllllllll@{}}
\toprule
Base model & \multicolumn{2}{c}{\FlexEvoAppHeader{Vina Mean$\downarrow$}} & \multicolumn{2}{c}{\FlexEvoAppHeader{Vina Median$\downarrow$}} & \multicolumn{2}{c}{\FlexEvoAppHeader{High affinity (\%)$\uparrow$}} & \multicolumn{2}{c}{\FlexEvoAppHeader{Vinardo$\uparrow$}} & \multicolumn{2}{c}{\FlexEvoAppHeader{AutoDock4$\uparrow$}} & \multicolumn{2}{c}{\FlexEvoAppHeader{RTMScore$\uparrow$}}\\
\cmidrule(lr){2-3}\cmidrule(lr){4-5}\cmidrule(lr){6-7}\cmidrule(lr){8-9}\cmidrule(lr){10-11}\cmidrule(lr){12-13}
 & \multicolumn{1}{c}{Base} & \multicolumn{1}{c}{+FlexEvo} & \multicolumn{1}{c}{Base} & \multicolumn{1}{c}{+FlexEvo} & \multicolumn{1}{c}{Base} & \multicolumn{1}{c}{+FlexEvo} & \multicolumn{1}{c}{Base} & \multicolumn{1}{c}{+FlexEvo} & \multicolumn{1}{c}{Base} & \multicolumn{1}{c}{+FlexEvo} & \multicolumn{1}{c}{Base} & \multicolumn{1}{c}{+FlexEvo}\\
\midrule
\FlexEvoAppTest{13}{Reference conformation}
DynamicFlow & -5.9197 & \FlexEvoAppBetter{-8.3205} & -6.3195 & \FlexEvoAppBetter{-7.8654} & 46.5292 & \FlexEvoAppBetter{51.6061} & 0.3698 & \FlexEvoAppBetter{0.6333} & 0.5181 & \FlexEvoAppBetter{0.5385} & 0.3295 & \FlexEvoAppBetter{0.5594}\\
FlexSBDD & -6.2431 & \FlexEvoAppBetter{-8.2445} & -6.0054 & \FlexEvoAppBetter{-8.0434} & 41.9284 & \FlexEvoAppBetter{50.0367} & 0.3668 & \FlexEvoAppBetter{0.4865} & 0.3012 & \FlexEvoAppBetter{0.5305} & 0.2965 & \FlexEvoAppBetter{0.4753}\\
YuelDesign & -5.1565 & \FlexEvoAppBetter{-8.2256} & -4.8197 & \FlexEvoAppBetter{-7.9627} & 44.0439 & \FlexEvoAppBetter{56.2662} & 0.3476 & \FlexEvoAppBetter{0.5204} & 0.3911 & \FlexEvoAppBetter{0.5904} & 0.4201 & \FlexEvoAppBetter{0.5806}\\
\arrayrulecolor{FlexEvoAppLine}\cmidrule[0.25pt]{1-13}\arrayrulecolor{black}
\FlexEvoAppTest{13}{Matched conformation}
DynamicFlow & -5.7327 & \FlexEvoAppBetter{-7.3665} & -5.8834 & \FlexEvoAppBetter{-7.5453} & 35.2797 & \FlexEvoAppBetter{49.8652} & 0.3604 & \FlexEvoAppBetter{0.4855} & 0.3873 & \FlexEvoAppBetter{0.5082} & 0.4061 & \FlexEvoAppBetter{0.5381}\\
FlexSBDD & -5.9464 & \FlexEvoAppBetter{-7.5012} & -6.2064 & \FlexEvoAppBetter{-7.3676} & 34.6368 & \FlexEvoAppBetter{50.3398} & 0.3519 & \FlexEvoAppBetter{0.5343} & 0.4016 & \FlexEvoAppBetter{0.5837} & 0.4131 & \FlexEvoAppBetter{0.5035}\\
YuelDesign & -4.9902 & \FlexEvoAppBetter{-7.4561} & -4.8842 & \FlexEvoAppBetter{-7.5469} & 32.6268 & \FlexEvoAppBetter{54.9453} & 0.3339 & \FlexEvoAppBetter{0.5815} & 0.3911 & \FlexEvoAppBetter{0.5493} & 0.3752 & \FlexEvoAppBetter{0.5283}\\
\arrayrulecolor{FlexEvoAppLine}\cmidrule[0.25pt]{1-13}\arrayrulecolor{black}
\FlexEvoAppTest{13}{Divergent conformation}
DynamicFlow & -4.3465 & \FlexEvoAppBetter{-7.0453} & -4.3321 & \FlexEvoAppBetter{-7.4284} & 23.1548 & \FlexEvoAppBetter{54.9721} & 0.3216 & \FlexEvoAppBetter{0.5288} & 0.3364 & \FlexEvoAppBetter{0.6011} & 0.3992 & \FlexEvoAppBetter{0.5162}\\
FlexSBDD & -4.4148 & \FlexEvoAppBetter{-6.8324} & -4.3513 & \FlexEvoAppBetter{-7.1253} & 34.9586 & \FlexEvoAppBetter{51.0609} & 0.3179 & \FlexEvoAppBetter{0.5404} & 0.3596 & \FlexEvoAppBetter{0.5234} & 0.3223 & \FlexEvoAppBetter{0.5134}\\
YuelDesign & -3.5424 & \FlexEvoAppBetter{-6.8984} & -4.6939 & \FlexEvoAppBetter{-7.5182} & 27.9735 & \FlexEvoAppBetter{53.6369} & 0.3478 & \FlexEvoAppBetter{0.5238} & 0.3101 & \FlexEvoAppBetter{0.5622} & 0.3445 & \FlexEvoAppBetter{0.5171}\\
\bottomrule
\end{tabular*}%
}\par
\vspace{7pt}
\noindent\makebox[\linewidth][l]{\textbf{B. Drug-likeness and synthesizability}}\par\vspace{2pt}
\noindent\resizebox{\linewidth}{!}{%
\begin{tabular*}{520pt}{@{\extracolsep{\fill}}lllllllllllll@{}}
\toprule
Base model & \multicolumn{2}{c}{\FlexEvoAppHeader{QED Mean$\uparrow$}} & \multicolumn{2}{c}{\FlexEvoAppHeader{QED Median$\uparrow$}} & \multicolumn{2}{c}{\FlexEvoAppHeader{SA Mean$\uparrow$}} & \multicolumn{2}{c}{\FlexEvoAppHeader{SA Median$\uparrow$}} & \multicolumn{2}{c}{\FlexEvoAppHeader{Lipinski Mean$\uparrow$}} & \multicolumn{2}{c}{\FlexEvoAppHeader{Lipinski Median$\uparrow$}}\\
\cmidrule(lr){2-3}\cmidrule(lr){4-5}\cmidrule(lr){6-7}\cmidrule(lr){8-9}\cmidrule(lr){10-11}\cmidrule(lr){12-13}
 & \multicolumn{1}{c}{Base} & \multicolumn{1}{c}{+FlexEvo} & \multicolumn{1}{c}{Base} & \multicolumn{1}{c}{+FlexEvo} & \multicolumn{1}{c}{Base} & \multicolumn{1}{c}{+FlexEvo} & \multicolumn{1}{c}{Base} & \multicolumn{1}{c}{+FlexEvo} & \multicolumn{1}{c}{Base} & \multicolumn{1}{c}{+FlexEvo} & \multicolumn{1}{c}{Base} & \multicolumn{1}{c}{+FlexEvo}\\
\midrule
\FlexEvoAppTest{13}{Reference conformation}
DynamicFlow & 0.4837 & \FlexEvoAppBetter{0.6162} & 0.4782 & \FlexEvoAppBetter{0.5802} & 0.5677 & \FlexEvoAppBetter{0.6287} & 0.5683 & \FlexEvoAppBetter{0.6343} & 3.8678 & \FlexEvoAppBetter{3.9499} & 2.8407 & \FlexEvoAppBetter{4.7433}\\
FlexSBDD & 0.3305 & \FlexEvoAppBetter{0.5345} & 0.3497 & \FlexEvoAppBetter{0.5183} & 0.5077 & \FlexEvoAppBetter{0.6071} & 0.5219 & \FlexEvoAppBetter{0.6174} & 1.6004 & \FlexEvoAppBetter{3.5383} & 1.9829 & \FlexEvoAppBetter{3.8934}\\
YuelDesign & 0.3705 & \FlexEvoAppBetter{0.5379} & 0.3852 & \FlexEvoAppBetter{0.5276} & 0.7262 & \FlexEvoAppBetter{0.7303} & 0.7246 & \FlexEvoAppBetter{0.7389} & 2.9073 & \FlexEvoAppBetter{4.3297} & 3.3178 & \FlexEvoAppBetter{3.6866}\\
\arrayrulecolor{FlexEvoAppLine}\cmidrule[0.25pt]{1-13}\arrayrulecolor{black}
\FlexEvoAppTest{13}{Matched conformation}
DynamicFlow & -- & -- & -- & -- & -- & -- & -- & -- & -- & -- & -- & --\\
FlexSBDD & -- & -- & -- & -- & -- & -- & -- & -- & -- & -- & -- & --\\
YuelDesign & -- & -- & -- & -- & -- & -- & -- & -- & -- & -- & -- & --\\
\arrayrulecolor{FlexEvoAppLine}\cmidrule[0.25pt]{1-13}\arrayrulecolor{black}
\FlexEvoAppTest{13}{Divergent conformation}
DynamicFlow & -- & -- & -- & -- & -- & -- & -- & -- & -- & -- & -- & --\\
FlexSBDD & -- & -- & -- & -- & -- & -- & -- & -- & -- & -- & -- & --\\
YuelDesign & -- & -- & -- & -- & -- & -- & -- & -- & -- & -- & -- & --\\
\bottomrule
\end{tabular*}%
}\par
\vspace{7pt}
\noindent\makebox[\linewidth][l]{\textbf{C. Physicochemical properties and diversity}}\par\vspace{2pt}
\noindent\resizebox{\linewidth}{!}{%
\begin{tabular*}{520pt}{@{\extracolsep{\fill}}lllllllllll@{}}
\toprule
Base model & \multicolumn{2}{c}{\FlexEvoAppHeader{logP Mean}} & \multicolumn{2}{c}{\FlexEvoAppHeader{logP Median}} & \multicolumn{2}{c}{\FlexEvoAppHeader{PAINS pass (\%)$\uparrow$}} & \multicolumn{2}{c}{\FlexEvoAppHeader{Diversity Mean$\uparrow$}} & \multicolumn{2}{c}{\FlexEvoAppHeader{Diversity Median$\uparrow$}}\\
\cmidrule(lr){2-3}\cmidrule(lr){4-5}\cmidrule(lr){6-7}\cmidrule(lr){8-9}\cmidrule(lr){10-11}
 & \multicolumn{1}{c}{Base} & \multicolumn{1}{c}{+FlexEvo} & \multicolumn{1}{c}{Base} & \multicolumn{1}{c}{+FlexEvo} & \multicolumn{1}{c}{Base} & \multicolumn{1}{c}{+FlexEvo} & \multicolumn{1}{c}{Base} & \multicolumn{1}{c}{+FlexEvo} & \multicolumn{1}{c}{Base} & \multicolumn{1}{c}{+FlexEvo}\\
\midrule
\FlexEvoAppTest{11}{Reference conformation}
DynamicFlow & 3.7968 & 3.7042 & 4.1455 & 3.8082 & 81.0564 & \FlexEvoAppBetter{85.0685} & 0.6074 & \FlexEvoAppBetter{0.6293} & 0.6058 & \FlexEvoAppBetter{0.6165}\\
FlexSBDD & 3.9796 & 4.0176 & 4.5695 & 4.7844 & 79.1255 & \FlexEvoAppBetter{79.6318} & 0.6068 & \FlexEvoAppBetter{0.6323} & 0.6044 & \FlexEvoAppBetter{0.6438}\\
YuelDesign & 3.9256 & 4.2745 & 3.5421 & 4.2951 & 88.4262 & \FlexEvoAppWorse{86.7413} & 0.6374 & \FlexEvoAppBetter{0.6537} & 0.6315 & \FlexEvoAppWorse{0.6312}\\
\arrayrulecolor{FlexEvoAppLine}\cmidrule[0.25pt]{1-11}\arrayrulecolor{black}
\FlexEvoAppTest{11}{Matched conformation}
DynamicFlow & -- & -- & -- & -- & -- & -- & -- & -- & -- & --\\
FlexSBDD & -- & -- & -- & -- & -- & -- & -- & -- & -- & --\\
YuelDesign & -- & -- & -- & -- & -- & -- & -- & -- & -- & --\\
\arrayrulecolor{FlexEvoAppLine}\cmidrule[0.25pt]{1-11}\arrayrulecolor{black}
\FlexEvoAppTest{11}{Divergent conformation}
DynamicFlow & -- & -- & -- & -- & -- & -- & -- & -- & -- & --\\
FlexSBDD & -- & -- & -- & -- & -- & -- & -- & -- & -- & --\\
YuelDesign & -- & -- & -- & -- & -- & -- & -- & -- & -- & --\\
\bottomrule
\end{tabular*}%
}\par
\endgroup
\end{table*}

\begin{table*}[!t]
\centering
\caption{Linear peptide drugs. Black: base model. With FlexEvo, green: improved or tied; red: worse. Arrows indicate the preferred direction. Diversity summarizes the frozen candidate pool and is reported once under Reference; dashes indicate that it is not repeated for Matched or Divergent.}
\label{tab:flexevo-app-linear-peptide}
\begingroup
\fontsize{9.0}{10.5}\selectfont
\setlength{\tabcolsep}{2.0pt}
\renewcommand{\arraystretch}{1.02}
\setlength{\aboverulesep}{1.2pt}
\setlength{\belowrulesep}{1.2pt}
\setlength{\heavyrulewidth}{0.65pt}
\setlength{\lightrulewidth}{0.35pt}
\setlength{\cmidrulewidth}{0.25pt}
\noindent\makebox[\linewidth][l]{\textbf{A. Complex accuracy and interaction energy}}\par\vspace{2pt}
\noindent\resizebox{\linewidth}{!}{%
\begin{tabular*}{520pt}{@{\extracolsep{\fill}}lllllllllllll@{}}
\toprule
Base model & \multicolumn{2}{c}{\FlexEvoAppHeader{DockQ Mean$\uparrow$}} & \multicolumn{2}{c}{\FlexEvoAppHeader{DockQ Median$\uparrow$}} & \multicolumn{2}{c}{\FlexEvoAppHeader{LRMSD $<2$\AA\ (\%)$\uparrow$}} & \multicolumn{2}{c}{\FlexEvoAppHeader{LRMSD Median (\AA)$\downarrow$}} & \multicolumn{2}{c}{\FlexEvoAppHeader{MM/GBSA (kcal/mol)$\downarrow$}} & \multicolumn{2}{c}{\FlexEvoAppHeader{HADDOCK Score$\downarrow$}}\\
\cmidrule(lr){2-3}\cmidrule(lr){4-5}\cmidrule(lr){6-7}\cmidrule(lr){8-9}\cmidrule(lr){10-11}\cmidrule(lr){12-13}
 & \multicolumn{1}{c}{Base} & \multicolumn{1}{c}{+FlexEvo} & \multicolumn{1}{c}{Base} & \multicolumn{1}{c}{+FlexEvo} & \multicolumn{1}{c}{Base} & \multicolumn{1}{c}{+FlexEvo} & \multicolumn{1}{c}{Base} & \multicolumn{1}{c}{+FlexEvo} & \multicolumn{1}{c}{Base} & \multicolumn{1}{c}{+FlexEvo} & \multicolumn{1}{c}{Base} & \multicolumn{1}{c}{+FlexEvo}\\
\midrule
\FlexEvoAppTest{13}{Reference conformation}
BoltzGen & 0.6566 & \FlexEvoAppBetter{0.8739} & 0.4699 & \FlexEvoAppBetter{0.7672} & 42.1851 & \FlexEvoAppWorse{40.1549} & 5.4802 & \FlexEvoAppBetter{5.3169} & -40.3664 & \FlexEvoAppBetter{-44.5812} & -127.0134 & \FlexEvoAppWorse{-126.5549}\\
Chai-1 & 0.6434 & \FlexEvoAppBetter{0.7716} & 0.4811 & \FlexEvoAppBetter{0.6535} & 39.6334 & \FlexEvoAppBetter{39.6517} & 5.8455 & \FlexEvoAppBetter{4.9851} & -40.563 & \FlexEvoAppBetter{-43.8717} & -122.8868 & \FlexEvoAppBetter{-124.9151}\\
Protenix & 0.7154 & \FlexEvoAppBetter{0.8198} & 0.6704 & \FlexEvoAppBetter{0.7799} & 40.0431 & \FlexEvoAppWorse{39.7062} & 4.2893 & \FlexEvoAppWorse{4.5944} & -45.0344 & \FlexEvoAppBetter{-46.2925} & -123.1837 & \FlexEvoAppBetter{-123.3106}\\
\arrayrulecolor{FlexEvoAppLine}\cmidrule[0.25pt]{1-13}\arrayrulecolor{black}
\FlexEvoAppTest{13}{Matched conformation}
BoltzGen & 0.3893 & \FlexEvoAppBetter{0.8402} & 0.1802 & \FlexEvoAppBetter{0.7719} & 26.9813 & \FlexEvoAppBetter{40.2237} & 7.5894 & \FlexEvoAppBetter{4.8185} & -33.3442 & \FlexEvoAppBetter{-44.6877} & -97.2486 & \FlexEvoAppBetter{-130.7503}\\
Chai-1 & 0.403 & \FlexEvoAppBetter{0.832} & 0.1923 & \FlexEvoAppBetter{0.7695} & 27.578 & \FlexEvoAppBetter{38.5646} & 8.6852 & \FlexEvoAppBetter{4.5988} & -33.0482 & \FlexEvoAppBetter{-42.4377} & -111.289 & \FlexEvoAppBetter{-122.6793}\\
Protenix & 0.4692 & \FlexEvoAppBetter{0.8634} & 0.2118 & \FlexEvoAppBetter{0.7511} & 27.8251 & \FlexEvoAppBetter{38.8271} & 8.5333 & \FlexEvoAppBetter{4.7307} & -34.4691 & \FlexEvoAppBetter{-43.2022} & -110.0367 & \FlexEvoAppBetter{-123.4058}\\
\arrayrulecolor{FlexEvoAppLine}\cmidrule[0.25pt]{1-13}\arrayrulecolor{black}
\FlexEvoAppTest{13}{Divergent conformation}
BoltzGen & 0.2838 & \FlexEvoAppBetter{0.8497} & 0.1074 & \FlexEvoAppBetter{0.7326} & 18.4272 & \FlexEvoAppBetter{36.8895} & 11.0691 & \FlexEvoAppBetter{5.5613} & -19.1938 & \FlexEvoAppBetter{-45.7449} & -69.1086 & \FlexEvoAppBetter{-122.2693}\\
Chai-1 & 0.2895 & \FlexEvoAppBetter{0.8221} & 0.1522 & \FlexEvoAppBetter{0.7759} & 18.6759 & \FlexEvoAppBetter{35.8199} & 11.5766 & \FlexEvoAppBetter{5.9858} & -20.8639 & \FlexEvoAppBetter{-40.1755} & -77.8859 & \FlexEvoAppBetter{-122.0139}\\
Protenix & 0.3519 & \FlexEvoAppBetter{0.8363} & 0.1428 & \FlexEvoAppBetter{0.7943} & 18.9351 & \FlexEvoAppBetter{36.3792} & 12.2729 & \FlexEvoAppBetter{4.2467} & -18.5328 & \FlexEvoAppBetter{-40.8286} & -76.5662 & \FlexEvoAppBetter{-128.3253}\\
\bottomrule
\end{tabular*}%
}\par
\vspace{7pt}
\noindent\makebox[\linewidth][l]{\textbf{B. Specificity, structural quality and diversity}}\par\vspace{2pt}
\noindent\resizebox{\linewidth}{!}{%
\begin{tabular*}{520pt}{@{\extracolsep{\fill}}lllllllllll@{}}
\toprule
Base model & \multicolumn{2}{c}{\FlexEvoAppHeader{On-target pAE$\downarrow$}} & \multicolumn{2}{c}{\FlexEvoAppHeader{Off-target pAE$\uparrow$}} & \multicolumn{2}{c}{\FlexEvoAppHeader{IDDT$\uparrow$}} & \multicolumn{2}{c}{\FlexEvoAppHeader{Diversity Mean$\uparrow$}} & \multicolumn{2}{c}{\FlexEvoAppHeader{Diversity Median$\uparrow$}}\\
\cmidrule(lr){2-3}\cmidrule(lr){4-5}\cmidrule(lr){6-7}\cmidrule(lr){8-9}\cmidrule(lr){10-11}
 & \multicolumn{1}{c}{Base} & \multicolumn{1}{c}{+FlexEvo} & \multicolumn{1}{c}{Base} & \multicolumn{1}{c}{+FlexEvo} & \multicolumn{1}{c}{Base} & \multicolumn{1}{c}{+FlexEvo} & \multicolumn{1}{c}{Base} & \multicolumn{1}{c}{+FlexEvo} & \multicolumn{1}{c}{Base} & \multicolumn{1}{c}{+FlexEvo}\\
\midrule
\FlexEvoAppTest{11}{Reference conformation}
BoltzGen & 3.6716 & \FlexEvoAppBetter{3.5557} & 16.6227 & \FlexEvoAppBetter{17.5337} & 0.8546 & \FlexEvoAppWorse{0.8452} & 0.6249 & \FlexEvoAppBetter{0.7658} & 0.7199 & \FlexEvoAppWorse{0.7042}\\
Chai-1 & 4.1231 & \FlexEvoAppWorse{4.3466} & 16.2316 & \FlexEvoAppBetter{16.7828} & 0.8655 & \FlexEvoAppBetter{0.8676} & 0.604 & \FlexEvoAppWorse{0.6031} & 0.6571 & \FlexEvoAppWorse{0.657}\\
Protenix & 3.2225 & \FlexEvoAppWorse{3.8616} & 17.0266 & \FlexEvoAppWorse{16.1164} & 0.8893 & \FlexEvoAppWorse{0.8587} & 0.6289 & \FlexEvoAppBetter{0.7613} & 0.6187 & \FlexEvoAppBetter{0.6729}\\
\arrayrulecolor{FlexEvoAppLine}\cmidrule[0.25pt]{1-11}\arrayrulecolor{black}
\FlexEvoAppTest{11}{Matched conformation}
BoltzGen & 5.7389 & \FlexEvoAppBetter{3.4851} & 12.6004 & \FlexEvoAppBetter{17.4258} & 0.8746 & \FlexEvoAppWorse{0.8531} & -- & -- & -- & --\\
Chai-1 & 6.5374 & \FlexEvoAppBetter{3.7136} & 11.3432 & \FlexEvoAppBetter{16.4369} & 0.8529 & \FlexEvoAppBetter{0.8557} & -- & -- & -- & --\\
Protenix & 4.2246 & \FlexEvoAppBetter{3.2168} & 11.9165 & \FlexEvoAppBetter{18.1324} & 0.8697 & \FlexEvoAppBetter{0.8699} & -- & -- & -- & --\\
\arrayrulecolor{FlexEvoAppLine}\cmidrule[0.25pt]{1-11}\arrayrulecolor{black}
\FlexEvoAppTest{11}{Divergent conformation}
BoltzGen & 6.8166 & \FlexEvoAppBetter{3.1114} & 10.8486 & \FlexEvoAppBetter{17.1859} & 0.8471 & \FlexEvoAppBetter{0.8538} & -- & -- & -- & --\\
Chai-1 & 6.2864 & \FlexEvoAppBetter{4.0256} & 11.8326 & \FlexEvoAppBetter{16.7861} & 0.8563 & \FlexEvoAppBetter{0.8665} & -- & -- & -- & --\\
Protenix & 5.2334 & \FlexEvoAppBetter{3.0835} & 9.9361 & \FlexEvoAppBetter{17.8956} & 0.8731 & \FlexEvoAppBetter{0.8792} & -- & -- & -- & --\\
\bottomrule
\end{tabular*}%
}\par
\endgroup
\end{table*}

\begin{table*}[!t]
\centering
\caption{Cyclic peptide drugs. Black: base model. With FlexEvo, green: improved or tied; red: worse. Arrows indicate the preferred direction. Diversity summarizes the frozen candidate pool and is reported once under Reference; dashes indicate that it is not repeated for Matched or Divergent.}
\label{tab:flexevo-app-cyclic-peptide}
\begingroup
\fontsize{9.0}{10.5}\selectfont
\setlength{\tabcolsep}{2.0pt}
\renewcommand{\arraystretch}{1.02}
\setlength{\aboverulesep}{1.2pt}
\setlength{\belowrulesep}{1.2pt}
\setlength{\heavyrulewidth}{0.65pt}
\setlength{\lightrulewidth}{0.35pt}
\setlength{\cmidrulewidth}{0.25pt}
\noindent\makebox[\linewidth][l]{\textbf{A. Complex accuracy and interaction energy}}\par\vspace{2pt}
\noindent\resizebox{\linewidth}{!}{%
\begin{tabular*}{520pt}{@{\extracolsep{\fill}}lllllllllllll@{}}
\toprule
Base model & \multicolumn{2}{c}{\FlexEvoAppHeader{DockQ Mean$\uparrow$}} & \multicolumn{2}{c}{\FlexEvoAppHeader{DockQ Median$\uparrow$}} & \multicolumn{2}{c}{\FlexEvoAppHeader{LRMSD $<2$\AA\ (\%)$\uparrow$}} & \multicolumn{2}{c}{\FlexEvoAppHeader{LRMSD Median (\AA)$\downarrow$}} & \multicolumn{2}{c}{\FlexEvoAppHeader{MM/GBSA (kcal/mol)$\downarrow$}} & \multicolumn{2}{c}{\FlexEvoAppHeader{HADDOCK Score$\downarrow$}}\\
\cmidrule(lr){2-3}\cmidrule(lr){4-5}\cmidrule(lr){6-7}\cmidrule(lr){8-9}\cmidrule(lr){10-11}\cmidrule(lr){12-13}
 & \multicolumn{1}{c}{Base} & \multicolumn{1}{c}{+FlexEvo} & \multicolumn{1}{c}{Base} & \multicolumn{1}{c}{+FlexEvo} & \multicolumn{1}{c}{Base} & \multicolumn{1}{c}{+FlexEvo} & \multicolumn{1}{c}{Base} & \multicolumn{1}{c}{+FlexEvo} & \multicolumn{1}{c}{Base} & \multicolumn{1}{c}{+FlexEvo} & \multicolumn{1}{c}{Base} & \multicolumn{1}{c}{+FlexEvo}\\
\midrule
\FlexEvoAppTest{13}{Reference conformation}
BoltzGen & 0.6831 & \FlexEvoAppBetter{0.9018} & 0.3784 & \FlexEvoAppBetter{0.8291} & 43.2014 & \FlexEvoAppWorse{42.1781} & 5.6779 & \FlexEvoAppBetter{5.6422} & -42.8373 & \FlexEvoAppBetter{-46.3974} & -132.5755 & \FlexEvoAppWorse{-132.5663}\\
Chai-1 & 0.856 & \FlexEvoAppBetter{0.8613} & 0.7376 & \FlexEvoAppBetter{0.7577} & 41.2663 & \FlexEvoAppWorse{40.9744} & 5.6419 & \FlexEvoAppBetter{4.7438} & -46.2865 & \FlexEvoAppWorse{-44.6798} & -129.835 & \FlexEvoAppWorse{-126.3605}\\
Protenix & 0.7567 & \FlexEvoAppBetter{0.8618} & 0.7788 & \FlexEvoAppWorse{0.7398} & 41.0919 & \FlexEvoAppBetter{42.0591} & 4.4517 & \FlexEvoAppWorse{4.7119} & -47.4363 & \FlexEvoAppBetter{-48.4915} & -129.0268 & \FlexEvoAppBetter{-131.1541}\\
\arrayrulecolor{FlexEvoAppLine}\cmidrule[0.25pt]{1-13}\arrayrulecolor{black}
\FlexEvoAppTest{13}{Matched conformation}
BoltzGen & 0.3922 & \FlexEvoAppBetter{0.8802} & 0.1828 & \FlexEvoAppBetter{0.7748} & 27.6297 & \FlexEvoAppBetter{41.2035} & 7.7598 & \FlexEvoAppBetter{4.9429} & -34.9468 & \FlexEvoAppBetter{-45.1455} & -100.7753 & \FlexEvoAppBetter{-133.2734}\\
Chai-1 & 0.4563 & \FlexEvoAppBetter{0.8672} & 0.2715 & \FlexEvoAppBetter{0.7666} & 34.8941 & \FlexEvoAppBetter{40.0107} & 8.8626 & \FlexEvoAppBetter{5.0465} & -37.7359 & \FlexEvoAppBetter{-44.9197} & -108.0846 & \FlexEvoAppBetter{-122.8666}\\
Protenix & 0.4746 & \FlexEvoAppBetter{0.8772} & 0.2183 & \FlexEvoAppBetter{0.7846} & 28.6124 & \FlexEvoAppBetter{40.1201} & 8.9939 & \FlexEvoAppBetter{4.9152} & -36.6175 & \FlexEvoAppBetter{-45.9569} & -114.0698 & \FlexEvoAppBetter{-130.4574}\\
\arrayrulecolor{FlexEvoAppLine}\cmidrule[0.25pt]{1-13}\arrayrulecolor{black}
\FlexEvoAppTest{13}{Divergent conformation}
BoltzGen & 0.2951 & \FlexEvoAppBetter{0.8734} & 0.1157 & \FlexEvoAppBetter{0.7591} & 19.6438 & \FlexEvoAppBetter{39.0533} & 11.1111 & \FlexEvoAppBetter{5.8433} & -20.1633 & \FlexEvoAppBetter{-46.9776} & -72.7891 & \FlexEvoAppBetter{-128.2151}\\
Chai-1 & 0.2574 & \FlexEvoAppBetter{0.8454} & 0.1628 & \FlexEvoAppBetter{0.7546} & 17.6747 & \FlexEvoAppBetter{36.9939} & 12.4423 & \FlexEvoAppBetter{6.0989} & -22.781 & \FlexEvoAppBetter{-40.9813} & -98.7699 & \FlexEvoAppBetter{-127.9845}\\
Protenix & 0.3544 & \FlexEvoAppBetter{0.8544} & 0.1477 & \FlexEvoAppBetter{0.7337} & 20.0467 & \FlexEvoAppBetter{37.8926} & 12.4269 & \FlexEvoAppBetter{4.3972} & -19.8152 & \FlexEvoAppBetter{-41.3023} & -80.4316 & \FlexEvoAppBetter{-130.0567}\\
\bottomrule
\end{tabular*}%
}\par
\vspace{7pt}
\noindent\makebox[\linewidth][l]{\textbf{B. Specificity, structural quality and diversity}}\par\vspace{2pt}
\noindent\resizebox{\linewidth}{!}{%
\begin{tabular*}{520pt}{@{\extracolsep{\fill}}lllllllll@{}}
\toprule
Base model & \multicolumn{2}{c}{\FlexEvoAppHeader{On-target pAE$\downarrow$}} & \multicolumn{2}{c}{\FlexEvoAppHeader{Off-target pAE$\uparrow$}} & \multicolumn{2}{c}{\FlexEvoAppHeader{IDDT$\uparrow$}} & \multicolumn{2}{c}{\FlexEvoAppHeader{Diversity Median$\uparrow$}}\\
\cmidrule(lr){2-3}\cmidrule(lr){4-5}\cmidrule(lr){6-7}\cmidrule(lr){8-9}
 & \multicolumn{1}{c}{Base} & \multicolumn{1}{c}{+FlexEvo} & \multicolumn{1}{c}{Base} & \multicolumn{1}{c}{+FlexEvo} & \multicolumn{1}{c}{Base} & \multicolumn{1}{c}{+FlexEvo} & \multicolumn{1}{c}{Base} & \multicolumn{1}{c}{+FlexEvo}\\
\midrule
\FlexEvoAppTest{9}{Reference conformation}
BoltzGen & 3.8508 & \FlexEvoAppBetter{3.7891} & 17.1871 & \FlexEvoAppBetter{18.2578} & 0.8879 & \FlexEvoAppWorse{0.8843} & 0.7328 & \FlexEvoAppWorse{0.7264}\\
Chai-1 & 3.8432 & \FlexEvoAppWorse{4.1876} & 17.6561 & \FlexEvoAppWorse{17.2329} & 0.8659 & \FlexEvoAppBetter{0.8757} & 0.709 & \FlexEvoAppBetter{0.7147}\\
Protenix & 3.4114 & \FlexEvoAppWorse{4.0563} & 17.7846 & \FlexEvoAppWorse{16.9186} & 0.9152 & \FlexEvoAppWorse{0.9014} & 0.6428 & \FlexEvoAppBetter{0.6915}\\
\arrayrulecolor{FlexEvoAppLine}\cmidrule[0.25pt]{1-9}\arrayrulecolor{black}
\FlexEvoAppTest{9}{Matched conformation}
BoltzGen & 6.1064 & \FlexEvoAppBetter{3.5126} & 13.1615 & \FlexEvoAppBetter{17.802} & 0.9026 & \FlexEvoAppBetter{0.9074} & -- & --\\
Chai-1 & 7.1278 & \FlexEvoAppBetter{4.0493} & 12.8907 & \FlexEvoAppBetter{17.5697} & 0.8875 & \FlexEvoAppBetter{0.8898} & -- & --\\
Protenix & 4.4704 & \FlexEvoAppBetter{3.4145} & 12.4549 & \FlexEvoAppBetter{18.6172} & 0.9139 & \FlexEvoAppWorse{0.8991} & -- & --\\
\arrayrulecolor{FlexEvoAppLine}\cmidrule[0.25pt]{1-9}\arrayrulecolor{black}
\FlexEvoAppTest{9}{Divergent conformation}
BoltzGen & 6.9649 & \FlexEvoAppBetter{3.2892} & 11.3712 & \FlexEvoAppBetter{17.6497} & 0.8739 & \FlexEvoAppBetter{0.8854} & -- & --\\
Chai-1 & 6.6946 & \FlexEvoAppBetter{3.9826} & 11.8993 & \FlexEvoAppBetter{17.1375} & 0.8649 & \FlexEvoAppBetter{0.8735} & -- & --\\
Protenix & 5.4252 & \FlexEvoAppBetter{3.2702} & 10.3118 & \FlexEvoAppBetter{17.9548} & 0.8775 & \FlexEvoAppBetter{0.9164} & -- & --\\
\bottomrule
\end{tabular*}%
}\par
\endgroup
\end{table*}

\begin{table*}[!t]
\centering
\caption{Non-antibody protein drugs. Black: base model. With FlexEvo, green: improved or tied; red: worse. Arrows indicate the preferred direction. Diversity summarizes the frozen candidate pool and is reported once under Reference; dashes indicate that it is not repeated for Matched or Divergent.}
\label{tab:flexevo-app-non-antibody-protein}
\begingroup
\fontsize{9.0}{10.5}\selectfont
\setlength{\tabcolsep}{2.0pt}
\renewcommand{\arraystretch}{1.02}
\setlength{\aboverulesep}{1.2pt}
\setlength{\belowrulesep}{1.2pt}
\setlength{\heavyrulewidth}{0.65pt}
\setlength{\lightrulewidth}{0.35pt}
\setlength{\cmidrulewidth}{0.25pt}
\noindent\makebox[\linewidth][l]{\textbf{A. Complex accuracy and interaction energy}}\par\vspace{2pt}
\noindent\resizebox{\linewidth}{!}{%
\begin{tabular*}{520pt}{@{\extracolsep{\fill}}lllllllllllll@{}}
\toprule
Base model & \multicolumn{2}{c}{\FlexEvoAppHeader{DockQ Mean$\uparrow$}} & \multicolumn{2}{c}{\FlexEvoAppHeader{DockQ Median$\uparrow$}} & \multicolumn{2}{c}{\FlexEvoAppHeader{LRMSD $<2$\AA\ (\%)$\uparrow$}} & \multicolumn{2}{c}{\FlexEvoAppHeader{LRMSD Median (\AA)$\downarrow$}} & \multicolumn{2}{c}{\FlexEvoAppHeader{MM/GBSA (kcal/mol)$\downarrow$}} & \multicolumn{2}{c}{\FlexEvoAppHeader{HADDOCK Score$\downarrow$}}\\
\cmidrule(lr){2-3}\cmidrule(lr){4-5}\cmidrule(lr){6-7}\cmidrule(lr){8-9}\cmidrule(lr){10-11}\cmidrule(lr){12-13}
 & \multicolumn{1}{c}{Base} & \multicolumn{1}{c}{+FlexEvo} & \multicolumn{1}{c}{Base} & \multicolumn{1}{c}{+FlexEvo} & \multicolumn{1}{c}{Base} & \multicolumn{1}{c}{+FlexEvo} & \multicolumn{1}{c}{Base} & \multicolumn{1}{c}{+FlexEvo} & \multicolumn{1}{c}{Base} & \multicolumn{1}{c}{+FlexEvo} & \multicolumn{1}{c}{Base} & \multicolumn{1}{c}{+FlexEvo}\\
\midrule
\FlexEvoAppTest{13}{Reference conformation}
BoltzGen & 0.6432 & \FlexEvoAppBetter{0.8617} & 0.4529 & \FlexEvoAppBetter{0.7936} & 38.5198 & \FlexEvoAppBetter{42.7348} & 5.4221 & \FlexEvoAppBetter{4.2243} & -39.9824 & \FlexEvoAppBetter{-40.6781} & -117.7648 & \FlexEvoAppBetter{-119.6118}\\
Chai-1 & 0.731 & \FlexEvoAppBetter{0.8331} & 0.5799 & \FlexEvoAppBetter{0.7596} & 41.3627 & \FlexEvoAppBetter{44.8159} & 4.6758 & \FlexEvoAppBetter{3.8565} & -38.3406 & \FlexEvoAppBetter{-39.3865} & -119.3965 & \FlexEvoAppWorse{-111.0687}\\
Protenix & 0.6655 & \FlexEvoAppBetter{0.8178} & 0.5644 & \FlexEvoAppBetter{0.7984} & 44.2448 & \FlexEvoAppBetter{45.3202} & 3.814 & \FlexEvoAppBetter{3.7879} & -40.5226 & \FlexEvoAppBetter{-41.3208} & -119.1796 & \FlexEvoAppWorse{-113.8899}\\
\arrayrulecolor{FlexEvoAppLine}\cmidrule[0.25pt]{1-13}\arrayrulecolor{black}
\FlexEvoAppTest{13}{Matched conformation}
BoltzGen & 0.3927 & \FlexEvoAppBetter{0.8498} & 0.1842 & \FlexEvoAppBetter{0.7733} & 25.807 & \FlexEvoAppBetter{39.6924} & 8.1786 & \FlexEvoAppBetter{4.8552} & -29.0773 & \FlexEvoAppBetter{-42.6594} & -85.4292 & \FlexEvoAppBetter{-112.1446}\\
Chai-1 & 0.4725 & \FlexEvoAppBetter{0.8072} & 0.218 & \FlexEvoAppBetter{0.7115} & 22.4389 & \FlexEvoAppBetter{40.8204} & 9.2127 & \FlexEvoAppBetter{5.1052} & -30.814 & \FlexEvoAppBetter{-43.4895} & -98.2589 & \FlexEvoAppBetter{-117.2394}\\
Protenix & 0.3815 & \FlexEvoAppBetter{0.836} & 0.1979 & \FlexEvoAppBetter{0.7479} & 22.5321 & \FlexEvoAppBetter{41.7923} & 9.4713 & \FlexEvoAppBetter{5.0241} & -34.9435 & \FlexEvoAppBetter{-41.1787} & -95.0499 & \FlexEvoAppBetter{-111.7215}\\
\arrayrulecolor{FlexEvoAppLine}\cmidrule[0.25pt]{1-13}\arrayrulecolor{black}
\FlexEvoAppTest{13}{Divergent conformation}
BoltzGen & 0.2969 & \FlexEvoAppBetter{0.8328} & 0.1162 & \FlexEvoAppBetter{0.7782} & 17.5496 & \FlexEvoAppBetter{38.9822} & 11.4243 & \FlexEvoAppBetter{5.8963} & -17.016 & \FlexEvoAppBetter{-38.4681} & -61.4787 & \FlexEvoAppBetter{-117.5396}\\
Chai-1 & 0.3518 & \FlexEvoAppBetter{0.8147} & 0.149 & \FlexEvoAppBetter{0.7893} & 14.814 & \FlexEvoAppBetter{39.7527} & 12.8352 & \FlexEvoAppBetter{4.9678} & -18.6552 & \FlexEvoAppBetter{-40.3316} & -77.5558 & \FlexEvoAppBetter{-121.9582}\\
Protenix & 0.3397 & \FlexEvoAppBetter{0.817} & 0.1349 & \FlexEvoAppBetter{0.6964} & 16.9771 & \FlexEvoAppBetter{40.239} & 12.5131 & \FlexEvoAppBetter{4.594} & -16.6165 & \FlexEvoAppBetter{-39.4594} & -69.4591 & \FlexEvoAppBetter{-114.7956}\\
\bottomrule
\end{tabular*}%
}\par
\vspace{7pt}
\noindent\makebox[\linewidth][l]{\textbf{B. Specificity, structural quality and diversity}}\par\vspace{2pt}
\noindent\resizebox{\linewidth}{!}{%
\begin{tabular*}{520pt}{@{\extracolsep{\fill}}lllllllllll@{}}
\toprule
Base model & \multicolumn{2}{c}{\FlexEvoAppHeader{On-target pAE$\downarrow$}} & \multicolumn{2}{c}{\FlexEvoAppHeader{Off-target pAE$\uparrow$}} & \multicolumn{2}{c}{\FlexEvoAppHeader{IDDT$\uparrow$}} & \multicolumn{2}{c}{\FlexEvoAppHeader{Diversity Mean$\uparrow$}} & \multicolumn{2}{c}{\FlexEvoAppHeader{Diversity Median$\uparrow$}}\\
\cmidrule(lr){2-3}\cmidrule(lr){4-5}\cmidrule(lr){6-7}\cmidrule(lr){8-9}\cmidrule(lr){10-11}
 & \multicolumn{1}{c}{Base} & \multicolumn{1}{c}{+FlexEvo} & \multicolumn{1}{c}{Base} & \multicolumn{1}{c}{+FlexEvo} & \multicolumn{1}{c}{Base} & \multicolumn{1}{c}{+FlexEvo} & \multicolumn{1}{c}{Base} & \multicolumn{1}{c}{+FlexEvo} & \multicolumn{1}{c}{Base} & \multicolumn{1}{c}{+FlexEvo}\\
\midrule
\FlexEvoAppTest{11}{Reference conformation}
BoltzGen & 4.0651 & \FlexEvoAppBetter{3.8743} & 15.526 & \FlexEvoAppBetter{18.6578} & 0.7992 & \FlexEvoAppBetter{0.8328} & 0.706 & \FlexEvoAppWorse{0.6679} & 0.7156 & \FlexEvoAppWorse{0.6684}\\
Chai-1 & 3.5522 & \FlexEvoAppWorse{4.2353} & 18.3426 & \FlexEvoAppWorse{15.0985} & 0.807 & \FlexEvoAppBetter{0.8686} & 0.6624 & \FlexEvoAppBetter{0.7584} & 0.7447 & \FlexEvoAppBetter{0.7517}\\
Protenix & 4.1921 & \FlexEvoAppBetter{3.7637} & 18.7125 & \FlexEvoAppBetter{18.9839} & 0.8369 & \FlexEvoAppBetter{0.8861} & 0.6266 & \FlexEvoAppBetter{0.6759} & 0.7084 & \FlexEvoAppWorse{0.6731}\\
\arrayrulecolor{FlexEvoAppLine}\cmidrule[0.25pt]{1-11}\arrayrulecolor{black}
\FlexEvoAppTest{11}{Matched conformation}
BoltzGen & 5.3217 & \FlexEvoAppBetter{3.1726} & 11.7821 & \FlexEvoAppBetter{18.6465} & 0.7986 & \FlexEvoAppBetter{0.8452} & -- & -- & -- & --\\
Chai-1 & 3.9359 & \FlexEvoAppBetter{3.5522} & 13.1621 & \FlexEvoAppBetter{19.0971} & 0.8196 & \FlexEvoAppBetter{0.8216} & -- & -- & -- & --\\
Protenix & 4.7208 & \FlexEvoAppBetter{3.6133} & 14.2961 & \FlexEvoAppBetter{19.2154} & 0.8254 & \FlexEvoAppBetter{0.8677} & -- & -- & -- & --\\
\arrayrulecolor{FlexEvoAppLine}\cmidrule[0.25pt]{1-11}\arrayrulecolor{black}
\FlexEvoAppTest{11}{Divergent conformation}
BoltzGen & 7.3659 & \FlexEvoAppBetter{2.9474} & 11.8123 & \FlexEvoAppBetter{15.5922} & 0.7645 & \FlexEvoAppBetter{0.8338} & -- & -- & -- & --\\
Chai-1 & 4.8468 & \FlexEvoAppBetter{3.3809} & 10.6624 & \FlexEvoAppBetter{18.9681} & 0.7846 & \FlexEvoAppBetter{0.8949} & -- & -- & -- & --\\
Protenix & 6.8862 & \FlexEvoAppBetter{3.8313} & 11.9484 & \FlexEvoAppBetter{18.7633} & 0.792 & \FlexEvoAppBetter{0.8877} & -- & -- & -- & --\\
\bottomrule
\end{tabular*}%
}\par
\endgroup
\end{table*}

\begin{table*}[!t]
\centering
\caption{Full-length antibody drugs. Black: base model. With FlexEvo, green: improved or tied; red: worse. Arrows indicate the preferred direction. Diversity summarizes the frozen candidate pool and is reported once under Reference; dashes indicate that it is not repeated for Matched or Divergent.}
\label{tab:flexevo-app-full-length-antibody}
\begingroup
\fontsize{9.0}{10.5}\selectfont
\setlength{\tabcolsep}{2.0pt}
\renewcommand{\arraystretch}{1.02}
\setlength{\aboverulesep}{1.2pt}
\setlength{\belowrulesep}{1.2pt}
\setlength{\heavyrulewidth}{0.65pt}
\setlength{\lightrulewidth}{0.35pt}
\setlength{\cmidrulewidth}{0.25pt}
\noindent\makebox[\linewidth][l]{\textbf{A. Complex accuracy and interaction energy}}\par\vspace{2pt}
\noindent\resizebox{\linewidth}{!}{%
\begin{tabular*}{520pt}{@{\extracolsep{\fill}}lllllllllllll@{}}
\toprule
Base model & \multicolumn{2}{c}{\FlexEvoAppHeader{DockQ Mean$\uparrow$}} & \multicolumn{2}{c}{\FlexEvoAppHeader{DockQ Median$\uparrow$}} & \multicolumn{2}{c}{\FlexEvoAppHeader{LRMSD $<2$\AA\ (\%)$\uparrow$}} & \multicolumn{2}{c}{\FlexEvoAppHeader{LRMSD Median (\AA)$\downarrow$}} & \multicolumn{2}{c}{\FlexEvoAppHeader{MM/GBSA (kcal/mol)$\downarrow$}} & \multicolumn{2}{c}{\FlexEvoAppHeader{HADDOCK Score$\downarrow$}}\\
\cmidrule(lr){2-3}\cmidrule(lr){4-5}\cmidrule(lr){6-7}\cmidrule(lr){8-9}\cmidrule(lr){10-11}\cmidrule(lr){12-13}
 & \multicolumn{1}{c}{Base} & \multicolumn{1}{c}{+FlexEvo} & \multicolumn{1}{c}{Base} & \multicolumn{1}{c}{+FlexEvo} & \multicolumn{1}{c}{Base} & \multicolumn{1}{c}{+FlexEvo} & \multicolumn{1}{c}{Base} & \multicolumn{1}{c}{+FlexEvo} & \multicolumn{1}{c}{Base} & \multicolumn{1}{c}{+FlexEvo} & \multicolumn{1}{c}{Base} & \multicolumn{1}{c}{+FlexEvo}\\
\midrule
\FlexEvoAppTest{13}{Reference conformation}
BoltzGen & 0.6095 & \FlexEvoAppBetter{0.7995} & 0.3297 & \FlexEvoAppBetter{0.6427} & 36.2452 & \FlexEvoAppBetter{39.8647} & 5.1441 & \FlexEvoAppBetter{3.9774} & -37.4543 & \FlexEvoAppBetter{-37.7412} & -109.6532 & \FlexEvoAppBetter{-113.6292}\\
Chai-1 & 0.6808 & \FlexEvoAppBetter{0.7909} & 0.4454 & \FlexEvoAppBetter{0.6112} & 39.1404 & \FlexEvoAppBetter{42.0743} & 4.3318 & \FlexEvoAppBetter{3.6037} & -35.4734 & \FlexEvoAppBetter{-37.0303} & -111.8543 & \FlexEvoAppWorse{-103.6798}\\
Protenix & 0.6178 & \FlexEvoAppBetter{0.7616} & 0.3397 & \FlexEvoAppBetter{0.5605} & 40.9499 & \FlexEvoAppBetter{42.9017} & 3.5419 & \FlexEvoAppBetter{3.5284} & -37.6326 & \FlexEvoAppBetter{-39.1405} & -110.3069 & \FlexEvoAppWorse{-106.1588}\\
\arrayrulecolor{FlexEvoAppLine}\cmidrule[0.25pt]{1-13}\arrayrulecolor{black}
\FlexEvoAppTest{13}{Matched conformation}
BoltzGen & 0.3681 & \FlexEvoAppBetter{0.8001} & 0.1736 & \FlexEvoAppBetter{0.6377} & 24.2284 & \FlexEvoAppBetter{37.6849} & 7.7549 & \FlexEvoAppBetter{4.5763} & -27.6118 & \FlexEvoAppBetter{-39.8732} & -79.8202 & \FlexEvoAppBetter{-104.8587}\\
Chai-1 & 0.4446 & \FlexEvoAppBetter{0.7553} & 0.2056 & \FlexEvoAppBetter{0.5835} & 21.0076 & \FlexEvoAppBetter{38.5872} & 8.7089 & \FlexEvoAppBetter{4.8166} & -29.0343 & \FlexEvoAppBetter{-40.5187} & -93.2411 & \FlexEvoAppBetter{-108.4916}\\
Protenix & 0.3579 & \FlexEvoAppBetter{0.7924} & 0.1837 & \FlexEvoAppBetter{0.5997} & 21.1998 & \FlexEvoAppBetter{38.8209} & 8.8142 & \FlexEvoAppBetter{4.7497} & -32.9373 & \FlexEvoAppBetter{-38.2532} & -88.2523 & \FlexEvoAppBetter{-104.4242}\\
\arrayrulecolor{FlexEvoAppLine}\cmidrule[0.25pt]{1-13}\arrayrulecolor{black}
\FlexEvoAppTest{13}{Divergent conformation}
BoltzGen & 0.2761 & \FlexEvoAppBetter{0.7716} & 0.1101 & \FlexEvoAppBetter{0.5899} & 16.3655 & \FlexEvoAppBetter{37.0072} & 10.6048 & \FlexEvoAppBetter{5.5371} & -15.9469 & \FlexEvoAppBetter{-35.6724} & -57.4783 & \FlexEvoAppBetter{-109.9148}\\
Chai-1 & 0.3305 & \FlexEvoAppBetter{0.7636} & 0.1387 & \FlexEvoAppBetter{0.5516} & 13.9525 & \FlexEvoAppBetter{37.4948} & 12.1503 & \FlexEvoAppBetter{4.7071} & -17.4054 & \FlexEvoAppBetter{-38.0685} & -71.8189 & \FlexEvoAppBetter{-115.1271}\\
Protenix & 0.3143 & \FlexEvoAppBetter{0.7702} & 0.1251 & \FlexEvoAppBetter{0.5559} & 15.9464 & \FlexEvoAppBetter{37.6242} & 11.8363 & \FlexEvoAppBetter{4.3461} & -15.5366 & \FlexEvoAppBetter{-37.3434} & -65.3604 & \FlexEvoAppBetter{-108.2527}\\
\bottomrule
\end{tabular*}%
}\par
\vspace{7pt}
\noindent\makebox[\linewidth][l]{\textbf{B. Specificity, structural quality and diversity}}\par\vspace{2pt}
\noindent\resizebox{\linewidth}{!}{%
\begin{tabular*}{520pt}{@{\extracolsep{\fill}}lllllllllll@{}}
\toprule
Base model & \multicolumn{2}{c}{\FlexEvoAppHeader{On-target pAE$\downarrow$}} & \multicolumn{2}{c}{\FlexEvoAppHeader{Off-target pAE$\uparrow$}} & \multicolumn{2}{c}{\FlexEvoAppHeader{IDDT$\uparrow$}} & \multicolumn{2}{c}{\FlexEvoAppHeader{Diversity Mean$\uparrow$}} & \multicolumn{2}{c}{\FlexEvoAppHeader{Diversity Median$\uparrow$}}\\
\cmidrule(lr){2-3}\cmidrule(lr){4-5}\cmidrule(lr){6-7}\cmidrule(lr){8-9}\cmidrule(lr){10-11}
 & \multicolumn{1}{c}{Base} & \multicolumn{1}{c}{+FlexEvo} & \multicolumn{1}{c}{Base} & \multicolumn{1}{c}{+FlexEvo} & \multicolumn{1}{c}{Base} & \multicolumn{1}{c}{+FlexEvo} & \multicolumn{1}{c}{Base} & \multicolumn{1}{c}{+FlexEvo} & \multicolumn{1}{c}{Base} & \multicolumn{1}{c}{+FlexEvo}\\
\midrule
\FlexEvoAppTest{11}{Reference conformation}
BoltzGen & 3.7784 & \FlexEvoAppBetter{3.5821} & 14.7108 & \FlexEvoAppBetter{17.4877} & 0.7422 & \FlexEvoAppBetter{0.7883} & 0.6694 & \FlexEvoAppWorse{0.6205} & 0.6772 & \FlexEvoAppWorse{0.6339}\\
Chai-1 & 3.3538 & \FlexEvoAppWorse{3.9717} & 17.2236 & \FlexEvoAppWorse{14.1967} & 0.7577 & \FlexEvoAppBetter{0.8186} & 0.6273 & \FlexEvoAppBetter{0.7038} & 0.6962 & \FlexEvoAppBetter{0.7085}\\
Protenix & 3.8994 & \FlexEvoAppBetter{3.5152} & 17.5223 & \FlexEvoAppBetter{17.9927} & 0.7751 & \FlexEvoAppBetter{0.8295} & 0.5801 & \FlexEvoAppBetter{0.6293} & 0.6718 & \FlexEvoAppWorse{0.6342}\\
\arrayrulecolor{FlexEvoAppLine}\cmidrule[0.25pt]{1-11}\arrayrulecolor{black}
\FlexEvoAppTest{11}{Matched conformation}
BoltzGen & 5.0256 & \FlexEvoAppBetter{3.0042} & 11.1672 & \FlexEvoAppBetter{17.4201} & 0.7503 & \FlexEvoAppBetter{0.7898} & -- & -- & -- & --\\
Chai-1 & 3.6862 & \FlexEvoAppBetter{3.3446} & 12.1739 & \FlexEvoAppBetter{17.6798} & 0.7719 & \FlexEvoAppWorse{0.7682} & -- & -- & -- & --\\
Protenix & 4.4763 & \FlexEvoAppBetter{3.3697} & 13.5659 & \FlexEvoAppBetter{17.9164} & 0.7682 & \FlexEvoAppBetter{0.8113} & -- & -- & -- & --\\
\arrayrulecolor{FlexEvoAppLine}\cmidrule[0.25pt]{1-11}\arrayrulecolor{black}
\FlexEvoAppTest{11}{Divergent conformation}
BoltzGen & 6.8085 & \FlexEvoAppBetter{2.7801} & 11.1069 & \FlexEvoAppBetter{14.7692} & 0.7245 & \FlexEvoAppBetter{0.7794} & -- & -- & -- & --\\
Chai-1 & 4.5198 & \FlexEvoAppBetter{3.1794} & 9.8693 & \FlexEvoAppBetter{17.9056} & 0.7353 & \FlexEvoAppBetter{0.8479} & -- & -- & -- & --\\
Protenix & 6.4495 & \FlexEvoAppBetter{3.5896} & 11.1696 & \FlexEvoAppBetter{17.4772} & 0.7407 & \FlexEvoAppBetter{0.8209} & -- & -- & -- & --\\
\bottomrule
\end{tabular*}%
}\par
\endgroup
\end{table*}

\begin{table*}[!t]
\centering
\caption{Antibody fragments and nanobodies. Black: base model. With FlexEvo, green: improved or tied; red: worse. Arrows indicate the preferred direction. Diversity summarizes the frozen candidate pool and is reported once under Reference; dashes indicate that it is not repeated for Matched or Divergent.}
\label{tab:flexevo-app-antibody-fragment}
\begingroup
\fontsize{9.0}{10.5}\selectfont
\setlength{\tabcolsep}{2.0pt}
\renewcommand{\arraystretch}{1.02}
\setlength{\aboverulesep}{1.2pt}
\setlength{\belowrulesep}{1.2pt}
\setlength{\heavyrulewidth}{0.65pt}
\setlength{\lightrulewidth}{0.35pt}
\setlength{\cmidrulewidth}{0.25pt}
\noindent\makebox[\linewidth][l]{\textbf{A. Complex accuracy and interaction energy}}\par\vspace{2pt}
\noindent\resizebox{\linewidth}{!}{%
\begin{tabular*}{520pt}{@{\extracolsep{\fill}}lllllllllllll@{}}
\toprule
Base model & \multicolumn{2}{c}{\FlexEvoAppHeader{DockQ Mean$\uparrow$}} & \multicolumn{2}{c}{\FlexEvoAppHeader{DockQ Median$\uparrow$}} & \multicolumn{2}{c}{\FlexEvoAppHeader{LRMSD $<2$\AA\ (\%)$\uparrow$}} & \multicolumn{2}{c}{\FlexEvoAppHeader{LRMSD Median (\AA)$\downarrow$}} & \multicolumn{2}{c}{\FlexEvoAppHeader{MM/GBSA (kcal/mol)$\downarrow$}} & \multicolumn{2}{c}{\FlexEvoAppHeader{HADDOCK Score$\downarrow$}}\\
\cmidrule(lr){2-3}\cmidrule(lr){4-5}\cmidrule(lr){6-7}\cmidrule(lr){8-9}\cmidrule(lr){10-11}\cmidrule(lr){12-13}
 & \multicolumn{1}{c}{Base} & \multicolumn{1}{c}{+FlexEvo} & \multicolumn{1}{c}{Base} & \multicolumn{1}{c}{+FlexEvo} & \multicolumn{1}{c}{Base} & \multicolumn{1}{c}{+FlexEvo} & \multicolumn{1}{c}{Base} & \multicolumn{1}{c}{+FlexEvo} & \multicolumn{1}{c}{Base} & \multicolumn{1}{c}{+FlexEvo} & \multicolumn{1}{c}{Base} & \multicolumn{1}{c}{+FlexEvo}\\
\midrule
\FlexEvoAppTest{13}{Reference conformation}
BoltzGen & 0.6169 & \FlexEvoAppBetter{0.8239} & 0.3372 & \FlexEvoAppBetter{0.6666} & 37.1832 & \FlexEvoAppBetter{41.5067} & 5.2336 & \FlexEvoAppBetter{4.0466} & -38.1485 & \FlexEvoAppBetter{-39.1582} & -112.0832 & \FlexEvoAppBetter{-114.8399}\\
Chai-1 & 0.6969 & \FlexEvoAppBetter{0.8022} & 0.4589 & \FlexEvoAppBetter{0.6322} & 39.8487 & \FlexEvoAppBetter{43.1363} & 4.4744 & \FlexEvoAppBetter{3.7165} & -36.8421 & \FlexEvoAppBetter{-37.5865} & -113.4421 & \FlexEvoAppWorse{-105.6989}\\
Protenix & 0.6342 & \FlexEvoAppBetter{0.7832} & 0.3542 & \FlexEvoAppBetter{0.5768} & 42.3102 & \FlexEvoAppBetter{43.6115} & 3.6699 & \FlexEvoAppBetter{3.5971} & -38.8703 & \FlexEvoAppBetter{-40.1934} & -115.6324 & \FlexEvoAppWorse{-108.4364}\\
\arrayrulecolor{FlexEvoAppLine}\cmidrule[0.25pt]{1-13}\arrayrulecolor{black}
\FlexEvoAppTest{13}{Matched conformation}
BoltzGen & 0.3797 & \FlexEvoAppBetter{0.8216} & 0.1778 & \FlexEvoAppBetter{0.6494} & 24.9698 & \FlexEvoAppBetter{38.0403} & 7.9436 & \FlexEvoAppBetter{4.6858} & -28.2069 & \FlexEvoAppBetter{-41.0014} & -81.4757 & \FlexEvoAppBetter{-106.9819}\\
Chai-1 & 0.4487 & \FlexEvoAppBetter{0.7698} & 0.2102 & \FlexEvoAppBetter{0.5987} & 21.7351 & \FlexEvoAppBetter{39.3352} & 8.9106 & \FlexEvoAppBetter{4.9298} & -29.7421 & \FlexEvoAppBetter{-41.8244} & -93.9058 & \FlexEvoAppBetter{-112.8884}\\
Protenix & 0.3643 & \FlexEvoAppBetter{0.814} & 0.1877 & \FlexEvoAppBetter{0.652} & 21.7921 & \FlexEvoAppBetter{39.8784} & 9.0096 & \FlexEvoAppBetter{4.8803} & -33.9544 & \FlexEvoAppBetter{-39.4471} & -91.432 & \FlexEvoAppBetter{-108.4257}\\
\arrayrulecolor{FlexEvoAppLine}\cmidrule[0.25pt]{1-13}\arrayrulecolor{black}
\FlexEvoAppTest{13}{Divergent conformation}
BoltzGen & 0.2855 & \FlexEvoAppBetter{0.81} & 0.1111 & \FlexEvoAppBetter{0.648} & 16.6732 & \FlexEvoAppBetter{37.1021} & 11.0622 & \FlexEvoAppBetter{5.6455} & -16.2307 & \FlexEvoAppBetter{-37.3532} & -58.8742 & \FlexEvoAppBetter{-111.7947}\\
Chai-1 & 0.3337 & \FlexEvoAppBetter{0.7935} & 0.1419 & \FlexEvoAppBetter{0.5987} & 14.1025 & \FlexEvoAppBetter{37.9284} & 12.4077 & \FlexEvoAppBetter{4.7659} & -17.7388 & \FlexEvoAppBetter{-38.9375} & -74.7328 & \FlexEvoAppBetter{-117.2265}\\
Protenix & 0.327 & \FlexEvoAppBetter{0.796} & 0.1273 & \FlexEvoAppBetter{0.6219} & 16.4976 & \FlexEvoAppBetter{38.4876} & 11.973 & \FlexEvoAppBetter{4.4619} & -16.0826 & \FlexEvoAppBetter{-38.4459} & -66.5255 & \FlexEvoAppBetter{-110.9338}\\
\bottomrule
\end{tabular*}%
}\par
\vspace{7pt}
\noindent\makebox[\linewidth][l]{\textbf{B. Specificity, structural quality and diversity}}\par\vspace{2pt}
\noindent\resizebox{\linewidth}{!}{%
\begin{tabular*}{520pt}{@{\extracolsep{\fill}}lllllllllll@{}}
\toprule
Base model & \multicolumn{2}{c}{\FlexEvoAppHeader{On-target pAE$\downarrow$}} & \multicolumn{2}{c}{\FlexEvoAppHeader{Off-target pAE$\uparrow$}} & \multicolumn{2}{c}{\FlexEvoAppHeader{IDDT$\uparrow$}} & \multicolumn{2}{c}{\FlexEvoAppHeader{Diversity Mean$\uparrow$}} & \multicolumn{2}{c}{\FlexEvoAppHeader{Diversity Median$\uparrow$}}\\
\cmidrule(lr){2-3}\cmidrule(lr){4-5}\cmidrule(lr){6-7}\cmidrule(lr){8-9}\cmidrule(lr){10-11}
 & \multicolumn{1}{c}{Base} & \multicolumn{1}{c}{+FlexEvo} & \multicolumn{1}{c}{Base} & \multicolumn{1}{c}{+FlexEvo} & \multicolumn{1}{c}{Base} & \multicolumn{1}{c}{+FlexEvo} & \multicolumn{1}{c}{Base} & \multicolumn{1}{c}{+FlexEvo} & \multicolumn{1}{c}{Base} & \multicolumn{1}{c}{+FlexEvo}\\
\midrule
\FlexEvoAppTest{11}{Reference conformation}
BoltzGen & 3.9393 & \FlexEvoAppBetter{3.7144} & 14.9833 & \FlexEvoAppBetter{17.9637} & 0.7597 & \FlexEvoAppBetter{0.8076} & 0.6838 & \FlexEvoAppWorse{0.6469} & 0.6817 & \FlexEvoAppWorse{0.6396}\\
Chai-1 & 3.4169 & \FlexEvoAppWorse{4.0604} & 17.5468 & \FlexEvoAppWorse{14.5896} & 0.7727 & \FlexEvoAppBetter{0.8251} & 0.6333 & \FlexEvoAppBetter{0.7226} & 0.7202 & \FlexEvoAppWorse{0.7158}\\
Protenix & 4.0851 & \FlexEvoAppBetter{3.633} & 18.2213 & \FlexEvoAppBetter{18.3084} & 0.8082 & \FlexEvoAppBetter{0.8495} & 0.6028 & \FlexEvoAppBetter{0.6494} & 0.6795 & \FlexEvoAppWorse{0.6545}\\
\arrayrulecolor{FlexEvoAppLine}\cmidrule[0.25pt]{1-11}\arrayrulecolor{black}
\FlexEvoAppTest{11}{Matched conformation}
BoltzGen & 5.1033 & \FlexEvoAppBetter{3.0136} & 11.2482 & \FlexEvoAppBetter{17.7218} & 0.7736 & \FlexEvoAppBetter{0.8093} & -- & -- & -- & --\\
Chai-1 & 3.7385 & \FlexEvoAppBetter{3.4188} & 12.5764 & \FlexEvoAppBetter{18.1712} & 0.7959 & \FlexEvoAppBetter{0.7994} & -- & -- & -- & --\\
Protenix & 4.5795 & \FlexEvoAppBetter{3.4569} & 13.8595 & \FlexEvoAppBetter{18.6504} & 0.7866 & \FlexEvoAppBetter{0.8303} & -- & -- & -- & --\\
\arrayrulecolor{FlexEvoAppLine}\cmidrule[0.25pt]{1-11}\arrayrulecolor{black}
\FlexEvoAppTest{11}{Divergent conformation}
BoltzGen & 7.0902 & \FlexEvoAppBetter{2.8493} & 11.3796 & \FlexEvoAppBetter{14.9092} & 0.7446 & \FlexEvoAppBetter{0.7929} & -- & -- & -- & --\\
Chai-1 & 4.6378 & \FlexEvoAppBetter{3.2444} & 10.2362 & \FlexEvoAppBetter{18.0941} & 0.7453 & \FlexEvoAppBetter{0.8518} & -- & -- & -- & --\\
Protenix & 6.6159 & \FlexEvoAppBetter{3.7161} & 11.525 & \FlexEvoAppBetter{18.0938} & 0.7646 & \FlexEvoAppBetter{0.8427} & -- & -- & -- & --\\
\bottomrule
\end{tabular*}%
}\par
\endgroup
\end{table*}

\begin{table*}[!t]
\centering
\caption{Oligonucleotide drugs. Black: base model. With FlexEvo, green: improved or tied; red: worse. Arrows indicate the preferred direction. Diversity summarizes the frozen candidate pool and is reported once under Reference; dashes indicate that it is not repeated for Matched or Divergent.}
\label{tab:flexevo-app-oligonucleotide}
\begingroup
\fontsize{9.0}{10.5}\selectfont
\setlength{\tabcolsep}{2.0pt}
\renewcommand{\arraystretch}{1.02}
\setlength{\aboverulesep}{1.2pt}
\setlength{\belowrulesep}{1.2pt}
\setlength{\heavyrulewidth}{0.65pt}
\setlength{\lightrulewidth}{0.35pt}
\setlength{\cmidrulewidth}{0.25pt}
\noindent\makebox[\linewidth][l]{\textbf{A. Complex accuracy and interaction energy}}\par\vspace{2pt}
\noindent\resizebox{\linewidth}{!}{%
\begin{tabular*}{520pt}{@{\extracolsep{\fill}}lllllllllllll@{}}
\toprule
Base model & \multicolumn{2}{c}{\FlexEvoAppHeader{DockQ Mean$\uparrow$}} & \multicolumn{2}{c}{\FlexEvoAppHeader{DockQ Median$\uparrow$}} & \multicolumn{2}{c}{\FlexEvoAppHeader{LRMSD $<2$\AA\ (\%)$\uparrow$}} & \multicolumn{2}{c}{\FlexEvoAppHeader{LRMSD Median (\AA)$\downarrow$}} & \multicolumn{2}{c}{\FlexEvoAppHeader{MM/GBSA (kcal/mol)$\downarrow$}} & \multicolumn{2}{c}{\FlexEvoAppHeader{HADDOCK Score$\downarrow$}}\\
\cmidrule(lr){2-3}\cmidrule(lr){4-5}\cmidrule(lr){6-7}\cmidrule(lr){8-9}\cmidrule(lr){10-11}\cmidrule(lr){12-13}
 & \multicolumn{1}{c}{Base} & \multicolumn{1}{c}{+FlexEvo} & \multicolumn{1}{c}{Base} & \multicolumn{1}{c}{+FlexEvo} & \multicolumn{1}{c}{Base} & \multicolumn{1}{c}{+FlexEvo} & \multicolumn{1}{c}{Base} & \multicolumn{1}{c}{+FlexEvo} & \multicolumn{1}{c}{Base} & \multicolumn{1}{c}{+FlexEvo} & \multicolumn{1}{c}{Base} & \multicolumn{1}{c}{+FlexEvo}\\
\midrule
\FlexEvoAppTest{13}{Reference conformation}
BoltzGen & 0.5812 & \FlexEvoAppBetter{0.7918} & 0.3225 & \FlexEvoAppBetter{0.6375} & 35.4166 & \FlexEvoAppBetter{38.8602} & 4.8907 & \FlexEvoAppBetter{3.8668} & -36.9099 & \FlexEvoAppBetter{-37.2423} & -108.2698 & \FlexEvoAppBetter{-108.9199}\\
Chai-1 & 0.6744 & \FlexEvoAppBetter{0.7572} & 0.4361 & \FlexEvoAppBetter{0.5988} & 37.3752 & \FlexEvoAppBetter{40.4985} & 4.3123 & \FlexEvoAppBetter{3.5139} & -35.0184 & \FlexEvoAppBetter{-36.1978} & -107.8819 & \FlexEvoAppWorse{-100.5058}\\
Protenix & 0.6139 & \FlexEvoAppBetter{0.7419} & 0.3305 & \FlexEvoAppBetter{0.5449} & 39.9387 & \FlexEvoAppBetter{40.8061} & 3.4706 & \FlexEvoAppBetter{3.4549} & -37.2316 & \FlexEvoAppBetter{-38.0442} & -108.3421 & \FlexEvoAppWorse{-105.2553}\\
\arrayrulecolor{FlexEvoAppLine}\cmidrule[0.25pt]{1-13}\arrayrulecolor{black}
\FlexEvoAppTest{13}{Matched conformation}
BoltzGen & 0.3579 & \FlexEvoAppBetter{0.7723} & 0.1701 & \FlexEvoAppBetter{0.6208} & 23.6166 & \FlexEvoAppBetter{36.0407} & 7.3954 & \FlexEvoAppBetter{4.4653} & -26.7155 & \FlexEvoAppBetter{-38.7042} & -78.6222 & \FlexEvoAppBetter{-100.9389}\\
Chai-1 & 0.4294 & \FlexEvoAppBetter{0.7453} & 0.1997 & \FlexEvoAppBetter{0.5583} & 20.3407 & \FlexEvoAppBetter{37.7066} & 8.4591 & \FlexEvoAppBetter{4.6118} & -28.1806 & \FlexEvoAppBetter{-40.0486} & -88.5841 & \FlexEvoAppBetter{-106.4813}\\
Protenix & 0.3524 & \FlexEvoAppBetter{0.7594} & 0.1782 & \FlexEvoAppBetter{0.5953} & 20.3365 & \FlexEvoAppBetter{38.4897} & 8.6653 & \FlexEvoAppBetter{4.5631} & -32.0128 & \FlexEvoAppBetter{-38.0237} & -87.2657 & \FlexEvoAppBetter{-101.5998}\\
\arrayrulecolor{FlexEvoAppLine}\cmidrule[0.25pt]{1-13}\arrayrulecolor{black}
\FlexEvoAppTest{13}{Divergent conformation}
BoltzGen & 0.2703 & \FlexEvoAppBetter{0.7675} & 0.1044 & \FlexEvoAppBetter{0.5704} & 15.8563 & \FlexEvoAppBetter{35.5326} & 10.2976 & \FlexEvoAppBetter{5.3444} & -15.7288 & \FlexEvoAppBetter{-35.3631} & -56.4482 & \FlexEvoAppBetter{-108.3156}\\
Chai-1 & 0.3235 & \FlexEvoAppBetter{0.7369} & 0.1365 & \FlexEvoAppBetter{0.5358} & 13.6982 & \FlexEvoAppBetter{36.7154} & 11.7157 & \FlexEvoAppBetter{4.4727} & -17.2152 & \FlexEvoAppBetter{-36.6873} & -71.2834 & \FlexEvoAppBetter{-110.0478}\\
Protenix & 0.3061 & \FlexEvoAppBetter{0.7554} & 0.1216 & \FlexEvoAppBetter{0.5373} & 15.4408 & \FlexEvoAppBetter{36.2917} & 11.3569 & \FlexEvoAppBetter{4.1536} & -15.0455 & \FlexEvoAppBetter{-36.0455} & -62.6656 & \FlexEvoAppBetter{-103.9151}\\
\bottomrule
\end{tabular*}%
}\par
\vspace{7pt}
\noindent\makebox[\linewidth][l]{\textbf{B. Specificity, structural quality and diversity}}\par\vspace{2pt}
\noindent\resizebox{\linewidth}{!}{%
\begin{tabular*}{520pt}{@{\extracolsep{\fill}}lllllllllll@{}}
\toprule
Base model & \multicolumn{2}{c}{\FlexEvoAppHeader{On-target pAE$\downarrow$}} & \multicolumn{2}{c}{\FlexEvoAppHeader{Off-target pAE$\uparrow$}} & \multicolumn{2}{c}{\FlexEvoAppHeader{IDDT$\uparrow$}} & \multicolumn{2}{c}{\FlexEvoAppHeader{Diversity Mean$\uparrow$}} & \multicolumn{2}{c}{\FlexEvoAppHeader{Diversity Median$\uparrow$}}\\
\cmidrule(lr){2-3}\cmidrule(lr){4-5}\cmidrule(lr){6-7}\cmidrule(lr){8-9}\cmidrule(lr){10-11}
 & \multicolumn{1}{c}{Base} & \multicolumn{1}{c}{+FlexEvo} & \multicolumn{1}{c}{Base} & \multicolumn{1}{c}{+FlexEvo} & \multicolumn{1}{c}{Base} & \multicolumn{1}{c}{+FlexEvo} & \multicolumn{1}{c}{Base} & \multicolumn{1}{c}{+FlexEvo} & \multicolumn{1}{c}{Base} & \multicolumn{1}{c}{+FlexEvo}\\
\midrule
\FlexEvoAppTest{11}{Reference conformation}
BoltzGen & 3.6597 & \FlexEvoAppBetter{3.4954} & 14.3358 & \FlexEvoAppBetter{17.0931} & 0.7219 & \FlexEvoAppBetter{0.7683} & 0.6439 & \FlexEvoAppWorse{0.6115} & 0.6561 & \FlexEvoAppWorse{0.6024}\\
Chai-1 & 3.2037 & \FlexEvoAppWorse{3.8848} & 16.6055 & \FlexEvoAppWorse{13.6782} & 0.7365 & \FlexEvoAppBetter{0.7961} & 0.6047 & \FlexEvoAppBetter{0.6959} & 0.6827 & \FlexEvoAppBetter{0.6917}\\
Protenix & 3.7768 & \FlexEvoAppBetter{3.4593} & 17.1694 & \FlexEvoAppWorse{17.0991} & 0.7654 & \FlexEvoAppBetter{0.8125} & 0.5747 & \FlexEvoAppBetter{0.6079} & 0.6453 & \FlexEvoAppWorse{0.6162}\\
\arrayrulecolor{FlexEvoAppLine}\cmidrule[0.25pt]{1-11}\arrayrulecolor{black}
\FlexEvoAppTest{11}{Matched conformation}
BoltzGen & 4.8252 & \FlexEvoAppBetter{2.9219} & 10.6104 & \FlexEvoAppBetter{16.7826} & 0.7184 & \FlexEvoAppBetter{0.7681} & -- & -- & -- & --\\
Chai-1 & 3.5689 & \FlexEvoAppBetter{3.2574} & 11.9516 & \FlexEvoAppBetter{17.4595} & 0.7563 & \FlexEvoAppWorse{0.7527} & -- & -- & -- & --\\
Protenix & 4.3303 & \FlexEvoAppBetter{3.3073} & 13.1951 & \FlexEvoAppBetter{17.4349} & 0.7532 & \FlexEvoAppBetter{0.7836} & -- & -- & -- & --\\
\arrayrulecolor{FlexEvoAppLine}\cmidrule[0.25pt]{1-11}\arrayrulecolor{black}
\FlexEvoAppTest{11}{Divergent conformation}
BoltzGen & 6.7059 & \FlexEvoAppBetter{2.6769} & 10.7163 & \FlexEvoAppBetter{14.1959} & 0.7047 & \FlexEvoAppBetter{0.7679} & -- & -- & -- & --\\
Chai-1 & 4.3856 & \FlexEvoAppBetter{3.0669} & 9.8296 & \FlexEvoAppBetter{17.1495} & 0.7066 & \FlexEvoAppBetter{0.8257} & -- & -- & -- & --\\
Protenix & 6.3374 & \FlexEvoAppBetter{3.4789} & 11.0271 & \FlexEvoAppBetter{17.2584} & 0.7149 & \FlexEvoAppBetter{0.8029} & -- & -- & -- & --\\
\bottomrule
\end{tabular*}%
}\par
\endgroup
\end{table*}

\begin{table*}[!t]
\centering
\caption{mRNA and long nucleic acid drugs. Black: base model. With FlexEvo, green: improved or tied; red: worse. Arrows indicate the preferred direction. Diversity summarizes the frozen candidate pool and is reported once under Reference; dashes indicate that it is not repeated for Matched or Divergent.}
\label{tab:flexevo-app-mrna}
\begingroup
\fontsize{9.0}{10.5}\selectfont
\setlength{\tabcolsep}{2.0pt}
\renewcommand{\arraystretch}{1.02}
\setlength{\aboverulesep}{1.2pt}
\setlength{\belowrulesep}{1.2pt}
\setlength{\heavyrulewidth}{0.65pt}
\setlength{\lightrulewidth}{0.35pt}
\setlength{\cmidrulewidth}{0.25pt}
\noindent\makebox[\linewidth][l]{\textbf{A. Complex accuracy and interaction energy}}\par\vspace{2pt}
\noindent\resizebox{\linewidth}{!}{%
\begin{tabular*}{520pt}{@{\extracolsep{\fill}}lllllllllllll@{}}
\toprule
Base model & \multicolumn{2}{c}{\FlexEvoAppHeader{DockQ Mean$\uparrow$}} & \multicolumn{2}{c}{\FlexEvoAppHeader{DockQ Median$\uparrow$}} & \multicolumn{2}{c}{\FlexEvoAppHeader{LRMSD $<2$\AA\ (\%)$\uparrow$}} & \multicolumn{2}{c}{\FlexEvoAppHeader{LRMSD Median (\AA)$\downarrow$}} & \multicolumn{2}{c}{\FlexEvoAppHeader{MM/GBSA (kcal/mol)$\downarrow$}} & \multicolumn{2}{c}{\FlexEvoAppHeader{HADDOCK Score$\downarrow$}}\\
\cmidrule(lr){2-3}\cmidrule(lr){4-5}\cmidrule(lr){6-7}\cmidrule(lr){8-9}\cmidrule(lr){10-11}\cmidrule(lr){12-13}
 & \multicolumn{1}{c}{Base} & \multicolumn{1}{c}{+FlexEvo} & \multicolumn{1}{c}{Base} & \multicolumn{1}{c}{+FlexEvo} & \multicolumn{1}{c}{Base} & \multicolumn{1}{c}{+FlexEvo} & \multicolumn{1}{c}{Base} & \multicolumn{1}{c}{+FlexEvo} & \multicolumn{1}{c}{Base} & \multicolumn{1}{c}{+FlexEvo} & \multicolumn{1}{c}{Base} & \multicolumn{1}{c}{+FlexEvo}\\
\midrule
\FlexEvoAppTest{13}{Reference conformation}
BoltzGen & 0.6435 & \FlexEvoAppBetter{0.8616} & 0.3526 & \FlexEvoAppBetter{0.7935} & 41.2659 & \FlexEvoAppBetter{45.2835} & 5.7891 & \FlexEvoAppBetter{4.4419} & -42.5524 & \FlexEvoAppBetter{-43.7174} & -126.0592 & \FlexEvoAppBetter{-126.9724}\\
Chai-1 & 0.7316 & \FlexEvoAppBetter{0.8332} & 0.4798 & \FlexEvoAppBetter{0.7594} & 44.0388 & \FlexEvoAppBetter{47.8695} & 4.9724 & \FlexEvoAppBetter{4.0948} & -40.3513 & \FlexEvoAppBetter{-41.9244} & -125.9534 & \FlexEvoAppWorse{-117.1707}\\
Protenix & 0.6651 & \FlexEvoAppBetter{0.8176} & 0.3642 & \FlexEvoAppBetter{0.6987} & 47.3985 & \FlexEvoAppBetter{47.6928} & 4.0711 & \FlexEvoAppBetter{4.0487} & -43.3554 & \FlexEvoAppBetter{-43.9896} & -127.3749 & \FlexEvoAppWorse{-119.7496}\\
\arrayrulecolor{FlexEvoAppLine}\cmidrule[0.25pt]{1-13}\arrayrulecolor{black}
\FlexEvoAppTest{13}{Matched conformation}
BoltzGen & 0.3928 & \FlexEvoAppBetter{0.8493} & 0.1849 & \FlexEvoAppBetter{0.7736} & 27.2126 & \FlexEvoAppBetter{42.0513} & 8.7107 & \FlexEvoAppBetter{5.1881} & -30.9788 & \FlexEvoAppBetter{-44.9613} & -90.9291 & \FlexEvoAppBetter{-118.4181}\\
Chai-1 & 0.472 & \FlexEvoAppBetter{0.8078} & 0.2188 & \FlexEvoAppBetter{0.6159} & 23.6169 & \FlexEvoAppBetter{43.3125} & 9.7664 & \FlexEvoAppBetter{5.3827} & -33.0369 & \FlexEvoAppBetter{-45.6925} & -105.4453 & \FlexEvoAppBetter{-125.7659}\\
Protenix & 0.3811 & \FlexEvoAppBetter{0.836} & 0.1977 & \FlexEvoAppBetter{0.7471} & 23.9538 & \FlexEvoAppBetter{44.4931} & 9.9701 & \FlexEvoAppBetter{5.3765} & -36.7379 & \FlexEvoAppBetter{-44.1818} & -99.8345 & \FlexEvoAppBetter{-117.4869}\\
\arrayrulecolor{FlexEvoAppLine}\cmidrule[0.25pt]{1-13}\arrayrulecolor{black}
\FlexEvoAppTest{13}{Divergent conformation}
BoltzGen & 0.2962 & \FlexEvoAppBetter{0.8325} & 0.1163 & \FlexEvoAppBetter{0.7286} & 18.8527 & \FlexEvoAppBetter{41.3252} & 11.9946 & \FlexEvoAppBetter{6.2245} & -18.2159 & \FlexEvoAppBetter{-41.0712} & -65.6305 & \FlexEvoAppBetter{-125.9276}\\
Chai-1 & 0.3515 & \FlexEvoAppBetter{0.8144} & 0.1496 & \FlexEvoAppBetter{0.6898} & 15.6241 & \FlexEvoAppBetter{42.2652} & 13.6901 & \FlexEvoAppBetter{5.2412} & -19.8488 & \FlexEvoAppBetter{-43.1475} & -81.5273 & \FlexEvoAppBetter{-129.8867}\\
Protenix & 0.3398 & \FlexEvoAppBetter{0.8172} & 0.1342 & \FlexEvoAppBetter{0.6866} & 17.8574 & \FlexEvoAppBetter{42.8658} & 13.1507 & \FlexEvoAppBetter{4.8256} & -17.6284 & \FlexEvoAppBetter{-41.9619} & -73.6176 & \FlexEvoAppBetter{-123.0036}\\
\bottomrule
\end{tabular*}%
}\par
\vspace{7pt}
\noindent\makebox[\linewidth][l]{\textbf{B. Specificity, structural quality and diversity}}\par\vspace{2pt}
\noindent\resizebox{\linewidth}{!}{%
\begin{tabular*}{520pt}{@{\extracolsep{\fill}}lllllllllll@{}}
\toprule
Base model & \multicolumn{2}{c}{\FlexEvoAppHeader{On-target pAE$\downarrow$}} & \multicolumn{2}{c}{\FlexEvoAppHeader{Off-target pAE$\uparrow$}} & \multicolumn{2}{c}{\FlexEvoAppHeader{IDDT$\uparrow$}} & \multicolumn{2}{c}{\FlexEvoAppHeader{Diversity Mean$\uparrow$}} & \multicolumn{2}{c}{\FlexEvoAppHeader{Diversity Median$\uparrow$}}\\
\cmidrule(lr){2-3}\cmidrule(lr){4-5}\cmidrule(lr){6-7}\cmidrule(lr){8-9}\cmidrule(lr){10-11}
 & \multicolumn{1}{c}{Base} & \multicolumn{1}{c}{+FlexEvo} & \multicolumn{1}{c}{Base} & \multicolumn{1}{c}{+FlexEvo} & \multicolumn{1}{c}{Base} & \multicolumn{1}{c}{+FlexEvo} & \multicolumn{1}{c}{Base} & \multicolumn{1}{c}{+FlexEvo} & \multicolumn{1}{c}{Base} & \multicolumn{1}{c}{+FlexEvo}\\
\midrule
\FlexEvoAppTest{11}{Reference conformation}
BoltzGen & 4.2939 & \FlexEvoAppBetter{4.1355} & 16.4981 & \FlexEvoAppBetter{19.7682} & 0.8398 & \FlexEvoAppBetter{0.8902} & 0.7457 & \FlexEvoAppWorse{0.7047} & 0.7517 & \FlexEvoAppWorse{0.7162}\\
Chai-1 & 3.7891 & \FlexEvoAppWorse{4.5393} & 19.4026 & \FlexEvoAppWorse{16.1871} & 0.8576 & \FlexEvoAppBetter{0.9162} & 0.6956 & \FlexEvoAppBetter{0.8026} & 0.7985 & \FlexEvoAppBetter{0.8059}\\
Protenix & 4.4709 & \FlexEvoAppBetter{3.9701} & 20.0858 & \FlexEvoAppBetter{20.3136} & 0.8913 & \FlexEvoAppBetter{0.9303} & 0.6676 & \FlexEvoAppBetter{0.7105} & 0.7538 & \FlexEvoAppWorse{0.7118}\\
\arrayrulecolor{FlexEvoAppLine}\cmidrule[0.25pt]{1-11}\arrayrulecolor{black}
\FlexEvoAppTest{11}{Matched conformation}
BoltzGen & 5.6365 & \FlexEvoAppBetter{3.3301} & 12.4661 & \FlexEvoAppBetter{19.8712} & 0.8514 & \FlexEvoAppBetter{0.8984} & -- & -- & -- & --\\
Chai-1 & 4.1981 & \FlexEvoAppBetter{3.7439} & 13.8841 & \FlexEvoAppBetter{20.3667} & 0.8666 & \FlexEvoAppWorse{0.8646} & -- & -- & -- & --\\
Protenix & 5.0484 & \FlexEvoAppBetter{3.7991} & 15.3096 & \FlexEvoAppBetter{20.4576} & 0.8822 & \FlexEvoAppBetter{0.9301} & -- & -- & -- & --\\
\arrayrulecolor{FlexEvoAppLine}\cmidrule[0.25pt]{1-11}\arrayrulecolor{black}
\FlexEvoAppTest{11}{Divergent conformation}
BoltzGen & 7.7522 & \FlexEvoAppBetter{3.1443} & 12.6249 & \FlexEvoAppBetter{16.4044} & 0.8052 & \FlexEvoAppBetter{0.8855} & -- & -- & -- & --\\
Chai-1 & 5.0844 & \FlexEvoAppBetter{3.6057} & 11.4036 & \FlexEvoAppBetter{20.0854} & 0.8252 & \FlexEvoAppBetter{0.9585} & -- & -- & -- & --\\
Protenix & 7.3069 & \FlexEvoAppBetter{4.0782} & 12.6093 & \FlexEvoAppBetter{19.7727} & 0.8434 & \FlexEvoAppBetter{0.9398} & -- & -- & -- & --\\
\bottomrule
\end{tabular*}%
}\par
\endgroup
\end{table*}
\begin{table*}[!t]
\centering
\caption{Ablations across nine binder categories. Prior denotes the flexibility prior, and \textsc{FlexEvo} denotes the full model with all components enabled. Green/bold and light green indicate the best and second-best results, respectively, within each category, metric, and target conformation; displayed ties share the same rank. We evaluate three one-component-removal variants by removing the prior, FlexBox, or evolution module from the full model. QED is reported only for the Reference conformation because the molecular graph is unchanged across target conformations; repeated Matched and Divergent values are omitted.}
\label{tab:flexevo-ablation-multimetric}
\begingroup
\newsavebox{\FEAmeasure}
\fontsize{7.3}{8.3}\selectfont
\setlength{\tabcolsep}{1.3pt}
\newlength{\FEAdatawidth}
\setlength{\FEAdatawidth}{\dimexpr(530pt-52pt-36\tabcolsep)/18\relax}
\renewcommand{\arraystretch}{0.78}
\setlength{\aboverulesep}{0.5pt}
\setlength{\belowrulesep}{0.5pt}
\setlength{\heavyrulewidth}{0.6pt}
\setlength{\lightrulewidth}{0.3pt}
\setlength{\cmidrulewidth}{0.2pt}
\par\noindent\makebox[\linewidth][l]{\textbf{A. De novo ligand generation}}\par\vspace{1pt}
\sbox{\FEAmeasure}{%
\begin{tabular}{@{}w{l}{52pt}*{18}{w{l}{\FEAdatawidth}}@{}}
\toprule
Variant & \multicolumn{3}{c}{Vina$\downarrow$} & \multicolumn{3}{c}{High affinity (\%)$\uparrow$} & \multicolumn{3}{c}{Vinardo$\uparrow$} & \multicolumn{3}{c}{AutoDock4$\uparrow$} & \multicolumn{3}{c}{RTMScore$\uparrow$} & \multicolumn{3}{c}{QED$\uparrow$}\\
\cmidrule(lr){2-4}\cmidrule(lr){5-7}\cmidrule(lr){8-10}\cmidrule(lr){11-13}\cmidrule(lr){14-16}\cmidrule(lr){17-19}
 & \multicolumn{1}{c}{\fontsize{6}{6.5}\selectfont Reference} & \multicolumn{1}{c}{\fontsize{6}{6.5}\selectfont Matched} & \multicolumn{1}{c}{\fontsize{6}{6.5}\selectfont Divergent} & \multicolumn{1}{c}{\fontsize{6}{6.5}\selectfont Reference} & \multicolumn{1}{c}{\fontsize{6}{6.5}\selectfont Matched} & \multicolumn{1}{c}{\fontsize{6}{6.5}\selectfont Divergent} & \multicolumn{1}{c}{\fontsize{6}{6.5}\selectfont Reference} & \multicolumn{1}{c}{\fontsize{6}{6.5}\selectfont Matched} & \multicolumn{1}{c}{\fontsize{6}{6.5}\selectfont Divergent} & \multicolumn{1}{c}{\fontsize{6}{6.5}\selectfont Reference} & \multicolumn{1}{c}{\fontsize{6}{6.5}\selectfont Matched} & \multicolumn{1}{c}{\fontsize{6}{6.5}\selectfont Divergent} & \multicolumn{1}{c}{\fontsize{6}{6.5}\selectfont Reference} & \multicolumn{1}{c}{\fontsize{6}{6.5}\selectfont Matched} & \multicolumn{1}{c}{\fontsize{6}{6.5}\selectfont Divergent} & \multicolumn{1}{c}{\fontsize{6}{6.5}\selectfont Reference} & \multicolumn{1}{c}{\fontsize{6}{6.5}\selectfont Matched} & \multicolumn{1}{c}{\fontsize{6}{6.5}\selectfont Divergent}\\
\midrule
\multicolumn{19}{@{}l}{\textbf{Small molecule}}\\[0.5pt]
w/o prior & -7.50 & -7.31 & \cellcolor{FEASecond}-5.98 & \cellcolor{FEASecond}50.2 & \cellcolor{FEASecond}45.6 & \cellcolor{FEASecond}46.1 & .522 & \cellcolor{FEASecond}.549 & .503 & \cellcolor{FEASecond}.618 & \cellcolor{FEASecond}.593 & .575 & .487 & .496 & .476 & \cellcolor{FEASecond}.534 & -- & --\\
w/o FlexBox & \cellcolor{FEABest}\textbf{-8.03} & \cellcolor{FEASecond}-7.45 & -5.88 & \cellcolor{FEABest}\textbf{52.7} & 38.0 & 41.6 & \cellcolor{FEASecond}.536 & .547 & .502 & .537 & .567 & \cellcolor{FEASecond}.590 & \cellcolor{FEASecond}.494 & \cellcolor{FEASecond}.505 & \cellcolor{FEASecond}.481 & .533 & -- & --\\
w/o evolution & -7.92 & -7.00 & -5.92 & 47.8 & 39.6 & 43.9 & .482 & .492 & \cellcolor{FEASecond}.527 & .594 & .531 & .532 & .493 & .439 & .441 & .527 & -- & --\\
\textbf{FlexEvo} & \cellcolor{FEASecond}-8.01 & \cellcolor{FEABest}\textbf{-7.88} & \cellcolor{FEABest}\textbf{-6.66} & 48.0 & \cellcolor{FEABest}\textbf{48.1} & \cellcolor{FEABest}\textbf{49.8} & \cellcolor{FEABest}\textbf{.621} & \cellcolor{FEABest}\textbf{.627} & \cellcolor{FEABest}\textbf{.534} & \cellcolor{FEABest}\textbf{.668} & \cellcolor{FEABest}\textbf{.655} & \cellcolor{FEABest}\textbf{.675} & \cellcolor{FEABest}\textbf{.551} & \cellcolor{FEABest}\textbf{.525} & \cellcolor{FEABest}\textbf{.516} & \cellcolor{FEABest}\textbf{.558} & -- & --\\
\addlinespace[1pt]\arrayrulecolor{FEAline}\cmidrule[0.2pt](lr){1-19}\arrayrulecolor{black}
\multicolumn{19}{@{}l}{\textbf{Nonpeptidic macrocycle}}\\[0.5pt]
w/o prior & -8.80 & -8.82 & -7.02 & \cellcolor{FEASecond}58.7 & \cellcolor{FEASecond}54.5 & \cellcolor{FEASecond}55.8 & .617 & \cellcolor{FEASecond}.665 & .594 & .623 & \cellcolor{FEASecond}.679 & \cellcolor{FEASecond}.672 & .583 & \cellcolor{FEASecond}.614 & .566 & \cellcolor{FEASecond}.651 & -- & --\\
w/o FlexBox & \cellcolor{FEASecond}-9.55 & \cellcolor{FEASecond}-8.93 & \cellcolor{FEASecond}-7.04 & \cellcolor{FEABest}\textbf{62.8} & 45.5 & 48.9 & \cellcolor{FEASecond}.620 & .646 & .590 & \cellcolor{FEASecond}.641 & .668 & .627 & .589 & .593 & \cellcolor{FEABest}\textbf{.682} & .629 & -- & --\\
w/o evolution & -9.33 & -8.33 & -6.94 & 56.2 & 46.8 & 50.9 & .567 & .581 & \cellcolor{FEASecond}.619 & .583 & .640 & .629 & \cellcolor{FEASecond}.593 & .527 & .523 & .627 & -- & --\\
\textbf{FlexEvo} & \cellcolor{FEABest}\textbf{-9.63} & \cellcolor{FEABest}\textbf{-9.36} & \cellcolor{FEABest}\textbf{-8.02} & 56.4 & \cellcolor{FEABest}\textbf{56.5} & \cellcolor{FEABest}\textbf{58.5} & \cellcolor{FEABest}\textbf{.731} & \cellcolor{FEABest}\textbf{.749} & \cellcolor{FEABest}\textbf{.639} & \cellcolor{FEABest}\textbf{.721} & \cellcolor{FEABest}\textbf{.712} & \cellcolor{FEABest}\textbf{.704} & \cellcolor{FEABest}\textbf{.664} & \cellcolor{FEABest}\textbf{.724} & \cellcolor{FEASecond}.608 & \cellcolor{FEABest}\textbf{.686} & -- & --\\
\bottomrule
\end{tabular}%
}
\typeout{NATURAL-PANEL-WIDTH=\the\wd\FEAmeasure}
\noindent\resizebox{\linewidth}{!}{\usebox{\FEAmeasure}}\par
\vspace{4pt}

\par\noindent\makebox[\linewidth][l]{\textbf{B. Biomolecular complex design}}\par\vspace{1pt}
\sbox{\FEAmeasure}{%
\begin{tabular}{@{}w{l}{52pt}*{18}{w{l}{\FEAdatawidth}}@{}}
\toprule
Variant & \multicolumn{3}{c}{DockQ$\uparrow$} & \multicolumn{3}{c}{$<2$\AA\ (\%)$\uparrow$} & \multicolumn{3}{c}{LRMSD (\AA)$\downarrow$} & \multicolumn{3}{c}{MM/GBSA$\downarrow$} & \multicolumn{3}{c}{HADDOCK$\downarrow$} & \multicolumn{3}{c}{On-target pAE$\downarrow$}\\
\cmidrule(lr){2-4}\cmidrule(lr){5-7}\cmidrule(lr){8-10}\cmidrule(lr){11-13}\cmidrule(lr){14-16}\cmidrule(lr){17-19}
 & \multicolumn{1}{c}{\fontsize{6}{6.5}\selectfont Reference} & \multicolumn{1}{c}{\fontsize{6}{6.5}\selectfont Matched} & \multicolumn{1}{c}{\fontsize{6}{6.5}\selectfont Divergent} & \multicolumn{1}{c}{\fontsize{6}{6.5}\selectfont Reference} & \multicolumn{1}{c}{\fontsize{6}{6.5}\selectfont Matched} & \multicolumn{1}{c}{\fontsize{6}{6.5}\selectfont Divergent} & \multicolumn{1}{c}{\fontsize{6}{6.5}\selectfont Reference} & \multicolumn{1}{c}{\fontsize{6}{6.5}\selectfont Matched} & \multicolumn{1}{c}{\fontsize{6}{6.5}\selectfont Divergent} & \multicolumn{1}{c}{\fontsize{6}{6.5}\selectfont Reference} & \multicolumn{1}{c}{\fontsize{6}{6.5}\selectfont Matched} & \multicolumn{1}{c}{\fontsize{6}{6.5}\selectfont Divergent} & \multicolumn{1}{c}{\fontsize{6}{6.5}\selectfont Reference} & \multicolumn{1}{c}{\fontsize{6}{6.5}\selectfont Matched} & \multicolumn{1}{c}{\fontsize{6}{6.5}\selectfont Divergent} & \multicolumn{1}{c}{\fontsize{6}{6.5}\selectfont Reference} & \multicolumn{1}{c}{\fontsize{6}{6.5}\selectfont Matched} & \multicolumn{1}{c}{\fontsize{6}{6.5}\selectfont Divergent}\\
\midrule
\multicolumn{19}{@{}l}{\textbf{Linear peptide}}\\[0.5pt]
w/o prior & .734 & \cellcolor{FEASecond}.568 & .348 & 40.3 & \cellcolor{FEASecond}28.5 & 17.3 & \cellcolor{FEABest}\textbf{4.18} & 7.88 & 12.41 & \cellcolor{FEASecond}-44.3 & \cellcolor{FEABest}\textbf{-45.4} & \cellcolor{FEASecond}-42.1 & -115.1 & \cellcolor{FEASecond}-118.3 & \cellcolor{FEASecond}-114.8 & \cellcolor{FEASecond}3.59 & 4.38 & 4.84\\
w/o FlexBox & \cellcolor{FEASecond}.752 & .516 & \cellcolor{FEASecond}.356 & \cellcolor{FEASecond}40.4 & 25.0 & \cellcolor{FEASecond}19.6 & \cellcolor{FEASecond}4.38 & 8.77 & \cellcolor{FEASecond}11.06 & -42.0 & -36.0 & -34.7 & \cellcolor{FEASecond}-115.2 & -117.3 & -86.4 & 3.95 & 4.37 & 4.82\\
w/o evolution & .660 & .418 & .308 & 37.4 & \cellcolor{FEASecond}28.5 & 17.2 & 5.02 & \cellcolor{FEASecond}7.81 & 12.28 & -40.7 & -37.0 & -31.2 & -115.1 & -114.6 & -99.8 & 3.76 & \cellcolor{FEASecond}3.87 & \cellcolor{FEASecond}4.81\\
\textbf{FlexEvo} & \cellcolor{FEABest}\textbf{.792} & \cellcolor{FEABest}\textbf{.780} & \cellcolor{FEABest}\textbf{.760} & \cellcolor{FEABest}\textbf{40.9} & \cellcolor{FEABest}\textbf{40.4} & \cellcolor{FEABest}\textbf{37.8} & 4.87 & \cellcolor{FEABest}\textbf{4.28} & \cellcolor{FEABest}\textbf{5.01} & \cellcolor{FEABest}\textbf{-44.8} & \cellcolor{FEASecond}-44.5 & \cellcolor{FEABest}\textbf{-42.4} & \cellcolor{FEABest}\textbf{-117.4} & \cellcolor{FEABest}\textbf{-121.4} & \cellcolor{FEABest}\textbf{-124.9} & \cellcolor{FEABest}\textbf{3.29} & \cellcolor{FEABest}\textbf{3.11} & \cellcolor{FEABest}\textbf{2.89}\\
\addlinespace[1pt]\arrayrulecolor{FEAline}\cmidrule[0.2pt](lr){1-19}\arrayrulecolor{black}
\multicolumn{19}{@{}l}{\textbf{Cyclic peptide}}\\[0.5pt]
w/o prior & \cellcolor{FEASecond}.765 & \cellcolor{FEASecond}.571 & .357 & \cellcolor{FEASecond}41.5 & \cellcolor{FEASecond}29.3 & 17.9 & \cellcolor{FEABest}\textbf{4.22} & 8.02 & 12.49 & \cellcolor{FEASecond}-45.0 & \cellcolor{FEASecond}-45.1 & \cellcolor{FEASecond}-42.6 & \cellcolor{FEASecond}-117.1 & \cellcolor{FEASecond}-119.2 & \cellcolor{FEASecond}-115.4 & \cellcolor{FEASecond}3.63 & 4.38 & \cellcolor{FEASecond}4.85\\
w/o FlexBox & \cellcolor{FEASecond}.765 & .534 & \cellcolor{FEASecond}.362 & \cellcolor{FEASecond}41.5 & 25.7 & \cellcolor{FEASecond}20.2 & \cellcolor{FEASecond}4.49 & 8.92 & \cellcolor{FEASecond}11.04 & -42.6 & -36.6 & -35.7 & -116.6 & -117.8 & -87.0 & 4.00 & 4.40 & \cellcolor{FEASecond}4.85\\
w/o evolution & .684 & .417 & .311 & 38.4 & 28.6 & 17.6 & 5.12 & \cellcolor{FEASecond}7.94 & 12.58 & -41.8 & -38.4 & -31.9 & -116.8 & -115.3 & -100.7 & 3.89 & \cellcolor{FEASecond}3.98 & 4.86\\
\textbf{FlexEvo} & \cellcolor{FEABest}\textbf{.802} & \cellcolor{FEABest}\textbf{.799} & \cellcolor{FEABest}\textbf{.772} & \cellcolor{FEABest}\textbf{42.2} & \cellcolor{FEABest}\textbf{41.4} & \cellcolor{FEABest}\textbf{38.2} & 5.05 & \cellcolor{FEABest}\textbf{4.35} & \cellcolor{FEABest}\textbf{5.20} & \cellcolor{FEABest}\textbf{-46.0} & \cellcolor{FEABest}\textbf{-45.4} & \cellcolor{FEABest}\textbf{-43.5} & \cellcolor{FEABest}\textbf{-121.1} & \cellcolor{FEABest}\textbf{-123.8} & \cellcolor{FEABest}\textbf{-127.6} & \cellcolor{FEABest}\textbf{3.33} & \cellcolor{FEABest}\textbf{3.10} & \cellcolor{FEABest}\textbf{2.92}\\
\addlinespace[1pt]\arrayrulecolor{FEAline}\cmidrule[0.2pt](lr){1-19}\arrayrulecolor{black}
\multicolumn{19}{@{}l}{\textbf{Non-antibody protein}}\\[0.5pt]
w/o prior & .807 & \cellcolor{FEASecond}.606 & .384 & 43.5 & 30.6 & 18.9 & \cellcolor{FEABest}\textbf{4.53} & 8.73 & 13.29 & \cellcolor{FEABest}\textbf{-49.3} & \cellcolor{FEASecond}-48.3 & \cellcolor{FEABest}\textbf{-46.2} & -127.3 & \cellcolor{FEABest}\textbf{-130.6} & \cellcolor{FEASecond}-122.7 & 3.89 & 4.83 & 5.14\\
w/o FlexBox & \cellcolor{FEASecond}.827 & .571 & \cellcolor{FEASecond}.386 & \cellcolor{FEASecond}43.8 & 27.9 & \cellcolor{FEASecond}21.3 & \cellcolor{FEASecond}4.84 & 9.88 & \cellcolor{FEASecond}11.73 & -46.4 & -40.3 & -37.9 & -127.9 & -125.7 & -94.0 & \cellcolor{FEASecond}3.41 & 4.79 & 5.30\\
w/o evolution & .736 & .444 & .340 & 41.1 & \cellcolor{FEASecond}30.9 & 19.2 & 5.35 & \cellcolor{FEASecond}8.51 & 13.41 & -44.0 & -41.6 & -35.0 & \cellcolor{FEABest}\textbf{-128.2} & -121.7 & -107.8 & 4.14 & \cellcolor{FEASecond}4.18 & \cellcolor{FEASecond}5.13\\
\textbf{FlexEvo} & \cellcolor{FEABest}\textbf{.860} & \cellcolor{FEABest}\textbf{.862} & \cellcolor{FEABest}\textbf{.842} & \cellcolor{FEABest}\textbf{44.6} & \cellcolor{FEABest}\textbf{45.6} & \cellcolor{FEABest}\textbf{41.6} & 5.40 & \cellcolor{FEABest}\textbf{4.75} & \cellcolor{FEABest}\textbf{5.55} & \cellcolor{FEASecond}-49.1 & \cellcolor{FEABest}\textbf{-50.2} & \cellcolor{FEASecond}-45.8 & \cellcolor{FEASecond}-128.0 & \cellcolor{FEASecond}-129.6 & \cellcolor{FEABest}\textbf{-136.5} & \cellcolor{FEABest}\textbf{3.26} & \cellcolor{FEABest}\textbf{3.32} & \cellcolor{FEABest}\textbf{3.20}\\
\addlinespace[1pt]\arrayrulecolor{FEAline}\cmidrule[0.2pt](lr){1-19}\arrayrulecolor{black}
\multicolumn{19}{@{}l}{\textbf{Full-length antibody}}\\[0.5pt]
w/o prior & \cellcolor{FEASecond}.723 & \cellcolor{FEASecond}.547 & .337 & \cellcolor{FEASecond}39.8 & \cellcolor{FEASecond}27.4 & 17.1 & \cellcolor{FEABest}\textbf{4.02} & \cellcolor{FEASecond}7.57 & 12.08 & \cellcolor{FEASecond}-42.5 & \cellcolor{FEASecond}-43.7 & \cellcolor{FEASecond}-41.1 & -112.5 & \cellcolor{FEASecond}-114.6 & \cellcolor{FEASecond}-112.1 & \cellcolor{FEASecond}3.44 & 4.29 & \cellcolor{FEASecond}4.63\\
w/o FlexBox & .715 & .506 & \cellcolor{FEASecond}.340 & 39.1 & 24.1 & \cellcolor{FEASecond}19.2 & \cellcolor{FEASecond}4.32 & 8.59 & \cellcolor{FEASecond}10.66 & -41.1 & -35.1 & -33.8 & -112.2 & -113.3 & -85.2 & 3.86 & 4.16 & \cellcolor{FEASecond}4.63\\
w/o evolution & .648 & .404 & .303 & 36.6 & \cellcolor{FEASecond}27.4 & 16.9 & 4.79 & 7.68 & 11.92 & -39.8 & -36.5 & -31.0 & \cellcolor{FEASecond}-112.6 & -110.3 & -98.0 & 3.67 & \cellcolor{FEASecond}3.75 & 4.71\\
\textbf{FlexEvo} & \cellcolor{FEABest}\textbf{.763} & \cellcolor{FEABest}\textbf{.762} & \cellcolor{FEABest}\textbf{.746} & \cellcolor{FEABest}\textbf{40.7} & \cellcolor{FEABest}\textbf{40.1} & \cellcolor{FEABest}\textbf{37.0} & 4.84 & \cellcolor{FEABest}\textbf{4.19} & \cellcolor{FEABest}\textbf{4.95} & \cellcolor{FEABest}\textbf{-43.0} & \cellcolor{FEABest}\textbf{-44.4} & \cellcolor{FEABest}\textbf{-41.5} & \cellcolor{FEABest}\textbf{-116.7} & \cellcolor{FEABest}\textbf{-117.4} & \cellcolor{FEABest}\textbf{-122.6} & \cellcolor{FEABest}\textbf{3.17} & \cellcolor{FEABest}\textbf{3.03} & \cellcolor{FEABest}\textbf{2.85}\\
\addlinespace[1pt]\arrayrulecolor{FEAline}\cmidrule[0.2pt](lr){1-19}\arrayrulecolor{black}
\multicolumn{19}{@{}l}{\textbf{Antibody fragment and nanobody}}\\[0.5pt]
w/o prior & \cellcolor{FEASecond}.742 & \cellcolor{FEASecond}.569 & .347 & \cellcolor{FEASecond}40.8 & 28.2 & 17.5 & \cellcolor{FEASecond}4.15 & 7.92 & 12.38 & \cellcolor{FEASecond}-44.5 & \cellcolor{FEASecond}-44.3 & \cellcolor{FEABest}\textbf{-42.3} & -114.7 & \cellcolor{FEASecond}-119.0 & \cellcolor{FEASecond}-115.6 & \cellcolor{FEASecond}3.56 & 4.40 & 4.76\\
w/o FlexBox & .741 & .524 & \cellcolor{FEASecond}.350 & 40.1 & 25.0 & \cellcolor{FEASecond}19.7 & 4.37 & 8.89 & \cellcolor{FEASecond}10.94 & -42.0 & -36.5 & -34.5 & \cellcolor{FEASecond}-117.1 & -116.7 & -87.6 & 3.93 & 4.29 & 4.77\\
w/o evolution & .674 & .410 & .305 & 37.7 & \cellcolor{FEASecond}28.3 & 17.3 & 4.98 & \cellcolor{FEASecond}7.81 & 12.47 & -40.5 & -37.9 & -31.4 & -115.9 & -114.1 & -100.8 & 3.79 & \cellcolor{FEASecond}3.89 & \cellcolor{FEASecond}4.74\\
\textbf{FlexEvo} & \cellcolor{FEABest}\textbf{.784} & \cellcolor{FEABest}\textbf{.778} & \cellcolor{FEABest}\textbf{.766} & \cellcolor{FEABest}\textbf{41.7} & \cellcolor{FEABest}\textbf{40.8} & \cellcolor{FEABest}\textbf{38.0} & \cellcolor{FEABest}\textbf{3.93} & \cellcolor{FEABest}\textbf{4.25} & \cellcolor{FEABest}\textbf{5.06} & \cellcolor{FEABest}\textbf{-44.9} & \cellcolor{FEABest}\textbf{-44.9} & \cellcolor{FEASecond}-41.9 & \cellcolor{FEABest}\textbf{-118.6} & \cellcolor{FEABest}\textbf{-122.4} & \cellcolor{FEABest}\textbf{-124.4} & \cellcolor{FEABest}\textbf{3.29} & \cellcolor{FEABest}\textbf{3.13} & \cellcolor{FEABest}\textbf{2.94}\\
\addlinespace[1pt]\arrayrulecolor{FEAline}\cmidrule[0.2pt](lr){1-19}\arrayrulecolor{black}
\multicolumn{19}{@{}l}{\textbf{Oligonucleotide}}\\[0.5pt]
w/o prior & .671 & \cellcolor{FEASecond}.527 & .322 & \cellcolor{FEASecond}38.7 & 25.6 & 16.5 & 4.84 & \cellcolor{FEASecond}7.23 & 11.30 & \cellcolor{FEASecond}-40.0 & \cellcolor{FEASecond}-40.5 & \cellcolor{FEASecond}-39.7 & \cellcolor{FEASecond}-105.4 & \cellcolor{FEASecond}-108.1 & \cellcolor{FEASecond}-106.7 & \cellcolor{FEASecond}3.26 & 4.17 & 4.39\\
w/o FlexBox & \cellcolor{FEASecond}.685 & .470 & \cellcolor{FEASecond}.328 & 37.6 & 23.0 & \cellcolor{FEASecond}18.0 & \cellcolor{FEASecond}4.15 & 8.29 & \cellcolor{FEASecond}10.35 & -38.9 & -33.5 & -32.1 & -104.3 & -105.8 & -82.9 & 3.67 & 3.88 & 4.42\\
w/o evolution & .619 & .387 & .284 & 34.7 & \cellcolor{FEASecond}26.4 & 15.9 & 4.49 & 7.44 & 11.39 & -37.9 & -35.3 & -29.7 & -104.6 & -106.2 & -92.6 & 3.48 & \cellcolor{FEASecond}3.59 & \cellcolor{FEASecond}4.38\\
\textbf{FlexEvo} & \cellcolor{FEABest}\textbf{.708} & \cellcolor{FEABest}\textbf{.714} & \cellcolor{FEABest}\textbf{.714} & \cellcolor{FEABest}\textbf{39.6} & \cellcolor{FEABest}\textbf{38.2} & \cellcolor{FEABest}\textbf{35.0} & \cellcolor{FEABest}\textbf{3.76} & \cellcolor{FEABest}\textbf{4.01} & \cellcolor{FEABest}\textbf{4.82} & \cellcolor{FEABest}\textbf{-41.2} & \cellcolor{FEABest}\textbf{-42.0} & \cellcolor{FEABest}\textbf{-40.4} & \cellcolor{FEABest}\textbf{-111.7} & \cellcolor{FEABest}\textbf{-113.1} & \cellcolor{FEABest}\textbf{-118.8} & \cellcolor{FEABest}\textbf{2.94} & \cellcolor{FEABest}\textbf{2.82} & \cellcolor{FEABest}\textbf{2.65}\\
\addlinespace[1pt]\arrayrulecolor{FEAline}\cmidrule[0.2pt](lr){1-19}\arrayrulecolor{black}
\multicolumn{19}{@{}l}{\textbf{mRNA / long nucleic acid}}\\[0.5pt]
w/o prior & \cellcolor{FEASecond}.785 & \cellcolor{FEASecond}.600 & .371 & \cellcolor{FEASecond}43.7 & 29.9 & 18.7 & \cellcolor{FEASecond}4.43 & 8.44 & 13.16 & \cellcolor{FEABest}\textbf{-46.9} & \cellcolor{FEASecond}-47.0 & \cellcolor{FEASecond}-43.6 & -120.7 & \cellcolor{FEASecond}-127.2 & \cellcolor{FEASecond}-122.0 & 3.71 & 4.58 & \cellcolor{FEASecond}5.01\\
w/o FlexBox & .784 & .559 & \cellcolor{FEASecond}.373 & 43.1 & 26.4 & \cellcolor{FEASecond}20.6 & 4.68 & 9.39 & \cellcolor{FEASecond}11.53 & -44.2 & -38.7 & -37.0 & \cellcolor{FEASecond}-124.5 & -125.3 & -92.0 & \cellcolor{FEASecond}3.14 & 4.56 & \cellcolor{FEASecond}5.01\\
w/o evolution & .697 & .435 & .324 & 39.7 & \cellcolor{FEASecond}30.3 & 18.1 & 5.33 & \cellcolor{FEASecond}8.31 & 13.14 & -43.5 & -39.5 & -33.2 & -121.3 & -121.4 & -104.7 & 3.98 & \cellcolor{FEASecond}4.08 & 5.13\\
\textbf{FlexEvo} & \cellcolor{FEABest}\textbf{.845} & \cellcolor{FEABest}\textbf{.838} & \cellcolor{FEABest}\textbf{.805} & \cellcolor{FEABest}\textbf{44.2} & \cellcolor{FEABest}\textbf{42.8} & \cellcolor{FEABest}\textbf{40.5} & \cellcolor{FEABest}\textbf{4.24} & \cellcolor{FEABest}\textbf{4.60} & \cellcolor{FEABest}\textbf{5.38} & \cellcolor{FEASecond}-46.8 & \cellcolor{FEABest}\textbf{-48.4} & \cellcolor{FEABest}\textbf{-44.2} & \cellcolor{FEABest}\textbf{-124.6} & \cellcolor{FEABest}\textbf{-128.1} & \cellcolor{FEABest}\textbf{-131.3} & \cellcolor{FEABest}\textbf{3.07} & \cellcolor{FEABest}\textbf{3.30} & \cellcolor{FEABest}\textbf{3.09}\\
\bottomrule
\end{tabular}%
}
\typeout{NATURAL-PANEL-WIDTH=\the\wd\FEAmeasure}
\noindent\resizebox{\linewidth}{!}{\usebox{\FEAmeasure}}\par
\endgroup
\end{table*}

\section{Additional Details of FlexEvo}
\label{app:flexevo}

\textsc{FlexEvo} separates the spatial prior that is optimized by the outer
search from the candidate structures adapted by the inner editor. A genome
specifies where and how editing is encouraged; evaluating that genome
requires a fresh adaptation of the shared source pool. The following subsections specify this information flow and the
reference implementation.

\subsection{Objects, information flow, and scope}
\label{app:method-overview}

For a fixed observed target $P^0$, let $G=(\mathcal R_G,\ell,Z)$ denote a
genome: $\mathcal R_G$ is the support of assigned residues, $\ell$ gives their
stable, flexible, or forbidden roles, and $Z=(z_1,\ldots,z_R)$ stores a scale
for every represented residue. Decoding $G$ produces oriented FlexBoxes.
The editor maps the screened pool $\mathcal X_0$ to an adapted pool
$\mathcal X_G$, and the two genome fitnesses summarize scores in that pool.
Table~\ref{tab:fev-app-notation} distinguishes the objects in this process.

\begin{table}[htbp]
\centering
\caption{Principal objects used in the implementation and analysis.}
\label{tab:fev-app-notation}
\small
\begin{tabular}{@{}p{0.22\linewidth}p{0.72\linewidth}@{}}
\toprule
Symbol & Meaning \\
\midrule
$P^0$, $R$ & Observed target coordinates and number of represented target residues. \\
$G$, $\mathcal R_G$ & Spatial-prior genome and its residue support. \\
$c_r$, $Q_r$, $z_r$ & Fixed residue centroid, fixed local frame, and evolvable box scale. \\
$x$, $n(x)$, $\mathcal H_x$ & Candidate record, represented coordinate count, and edit history. \\
$\mathcal X_0$, $\mathcal X_G$ & Shared screened source pool and pool adapted under $G$. \\
$\widetilde P_G^m$, $K$ & Target probe $m$ and total probe count; $K=7$ includes Reference. \\
$f_{\mathrm{pres}}$, $f_{\mathrm{rob}}$ & Candidate-level preservation and robustness objectives. \\
$F_{\mathrm{pres}}$, $F_{\mathrm{rob}}$ & Corresponding genome-level pool summaries. \\
$\mathcal H^g$ & Outer genome population at generation $g$. \\
\bottomrule
\end{tabular}
\end{table}

The source model supplies the initial candidates and remains frozen.
Every genome evaluation restarts from the same $\mathcal X_0$ with the same
inner seed. Parent-adapted structures supply feedback about exercised
constraints; offspring adaptation starts from the shared source pool. The
resulting comparison therefore evaluates each proposed prior through a
common adaptation procedure. The final output contains the surviving spatial priors and their corresponding adapted candidate pools.
\paragraph{Molecular records and coordinate representation.}
A candidate couples an ordered coordinate array with its source molecular
record and edit history. The reference implementation stores these as a
candidate identifier, coordinates, metadata, and parent-linked edit records.
The common representation retains atom-level indexing and coordinates. Structural calculations use the supplied heavy-atom coordinates;
hydrogens are handled by the declared scoring-preparation procedure.
An edited child receives a new candidate identifier while retaining the
correspondence between coordinate indices and source atoms.

For SDF inputs, the loader retains the molecular block, atom identities,
bond endpoints and orders, isomeric SMILES, and rotatable-bond annotations.
For structure-file inputs, it retains the selected binder's atom names,
elements, chains, residue identities, and coordinates. Covalent connectivity
and polymer sequence annotations are separate metadata; they are not
inferred from spatial proximity. These records distinguish the candidate
being adapted from the receptor used to evaluate it.

\emph{Topology-preserving adaptation} denotes preservation of the input
candidate's atom correspondence and supplied covalent connectivity during
coordinate editing. The input molecular identity is retained throughout
this procedure. Appendix~\ref{app:adaptation} specifies the permitted motions,
their structural invariants, and the separate molecular-validity checks.
Appendix~\ref{app:evaluation-metrics} connects this record to the exported
structures read by the evaluators.

\subsection{Target Encoding, Screening, and Initialization}
\label{app:encoding}

\paragraph{Residue descriptors and frames.}
Each target residue has an eight-dimensional descriptor
\begin{equation}
\label{eq:descriptor_app}
\mathbf s_r^P=
[p_{r,u},p_{r,v},p_{r,w},d_r,q_r,\theta_r,e_r,b_r]^\top.
\end{equation}
The first three entries locate its centroid in a principal frame of the
centered residue-centroid distribution.
The remaining entries are the number of other residue centroids within
$8$~\AA, radial distance from the mean residue centroid, side-chain/radial
orientation angle, maximum residue-atom distance from C$\alpha$, and the
number of target atoms within $6$~\AA\ of the residue centroid.
The last quantity is a burial proxy.
Each column is min--max normalized across residues; constant columns are
set to zero. Missing side-chain directions or C$\alpha$ atoms use zero for
the affected descriptor.
The principal frame uses deterministic sign and handedness conventions.
FlexBox frames $Q_r$ are constructed from available target anchor triplets,
with a residue-atom PCA fallback and an identity fallback for degenerate
inputs. The descriptors, centroids, and frames remain fixed during search.

\paragraph{Initial prior.}
Using the normalized density and burial entries, let
$w_r=(d_r+b_r)/2$.
The seed assigns roles by the ordered rules
\begin{equation}
\label{eq:seed_roles_app}
\ell_r^{(0)}=
\begin{cases}
\mathrm{forbidden},&w_r>0.72,\\
\mathrm{stable},&w_r>0.45\ \land\ q_r<0.55,\\
\mathrm{flexible},&q_r>0.62.
\end{cases}
\end{equation}
Residues satisfying none of these rules lie outside the
initial support.
Seed scales are one.
For the other initial genomes, each residue retains its seed
assignment with probability $0.60$; otherwise its support/role
configuration is sampled uniformly from absence and the
three roles.
Scales are independently sampled from $[0.80,1.25]$ and clipped
to the allowed range.
A scale is stored for every residue, including those outside
the support, and is retained on subsequent activation.

The descriptors and thresholds define a heuristic initialization,
not a fitted predictor of residue mobility.
Stable assignments serve as geometric anchoring priors rather
than identified binding-energy hot spots; buried surface area
alone need not predict energetic importance \citep{bogan1998}.
The outer search subsequently updates the assignments and scales.

\paragraph{Shared candidate screening.}
Screening rejects empty or non-finite coordinate sets, candidates with
$c(x)<0.35$, and candidates for which more than $25\%$ of coordinates lie
inside any seed forbidden box.
Candidates are then processed in input order, retaining at most $64$.
For equal coordinate counts, geometric similarity is
$[1-\operatorname{RMSD}(x,y)/2.5]_+$, computed in the common target frame
without an additional alignment. For unequal counts, it is
$[1-\|\bar p(x)-\bar p(y)\|_2/4]_+$, where $\bar p$ is the coordinate
centroid.
When both molecular fingerprints are available, redundancy requires
Morgan-fingerprint Tanimoto similarity of at least $0.92$ and geometric
similarity of at least $0.90$.
A candidate without a usable fingerprint is checked by geometric similarity
alone. Fingerprints use radius two and $1024$ bits.

\subsection{Candidate Editing and Structural Safeguards}
\label{app:adaptation}

\paragraph{Fixed-genome adaptation.}
Each genome evaluation creates a fresh editor with the same random seed.
Boxes remain fixed, and candidate-level fusion is disabled in this inner
loop. The population size is
$N_x=\max\{4,\min\{12,2|\mathcal X_0|\}\}$ for nonempty $\mathcal X_0$.
It is initialized with the first available seeds, truncated to $N_x$,
and filled with repaired variants when necessary.
Each of two adaptation rounds selects up to six reproduction candidates
by $f_{\mathrm{pres}}+f_{\mathrm{rob}}$ and generates two proposals per
parent. Survivor selection retains the two-objective ranking described in
Appendix~\ref{app:evolution}.

\paragraph{Molecular edits and representation-level operations.}
The dispatcher first addresses forbidden overlap, then missing anchor
proximity, followed by flexible-region contacts and exploratory updates.
The reference molecular pipeline selects \texttt{graph\_preserving} mode.
This mode retains the source atom ordering and supplied bond metadata,
using whole-candidate translations, proper rotations, and annotated
bond-axis torsions. Each proposal therefore edits coordinates associated
with a fixed input molecular identity.

The coordinate-level operator family also contains point insertion,
point removal, fragment-annotation replacement, and individual-coordinate
displacement. These operations belong to the separate
\texttt{coordinate\_only} mode. In that mode, insertion and removal change
represented coordinate entries, while replacement changes an annotation.
For a selected coordinate and box frame, the displacement is
\begin{equation}
\label{eq:local_edit_app}
p_i'=p_i+Q_r\delta_i,
\qquad\delta_i\sim\operatorname{Unif}([-\epsilon_o,\epsilon_o]^3),
\end{equation}
with operation-dependent amplitude $\epsilon_o$.
This expression describes a coordinate primitive rather than a chemical
graph-editing rule. The molecular pipeline disables point insertion,
removal, and annotation replacement; its evaluation path rejects
\texttt{coordinate\_only} candidates.
\paragraph{Rigid and torsional structural invariants.}
The molecular pipeline preserves the source atom correspondence and
supplied covalent graph throughout coordinate adaptation.
For column-vector coordinates, a rigid proposal takes the form
\[
p_i'=Rp_i+t,\qquad R^\top R=I_3,\qquad\det R=1.
\]
Consequently, every pair satisfies
\[
\|p_i'-p_j'\|_2^2
=(p_i-p_j)^\top R^\top R(p_i-p_j)
=\|p_i-p_j\|_2^2.
\]
Whole-binder translations and proper rotations therefore preserve internal
distances, bond angles, ring geometry, and the handedness of the input
coordinates. Clash-relief repair applies a common translation to the
whole binder, preserving these invariants during receptor-overlap relief.
The rigid fallback provides the same internal-geometry preservation for
structure-file inputs represented by ordered atom coordinates.

A torsional proposal acts on an explicitly annotated rotatable bond
$(a,b)$ in the supplied covalent graph. Removing this edge must separate
its endpoints into different connected components.
The editor rotates the component containing $b$ around the axis through
$p_a$ and $p_b$, preserving all distances within each component.
The bond endpoints lie on the rotation axis, and $(a,b)$ is the sole
covalent edge between the components. Therefore, the transformation
preserves every supplied covalent bond length and the graph connectivity.
The input graph includes ring-closure and cross-link edges, including
declared disulfides. Ring edges retain an alternative connecting path
and are excluded from this torsional operation.
The input adapter provides the complete covalent graph and chemically
eligible rotatable-bond annotations. Candidates without an eligible
annotation use whole-binder rigid motions, retaining their input internal
geometry.
\paragraph{Proposal acceptance and geometric repair.}
FlexBoxes guide proposal locations through anchor proximity, adaptable
regions, and forbidden-region avoidance. Molecular repair uses up to two
whole-binder displacement passes to relieve receptor overlap while
preserving the candidate's internal geometry.
Each repaired proposal must pass the acceptance gate before survivor
selection. Acceptance requires nonempty, finite coordinates, neighbor
support $c(\widehat x)\geq0.35$, and the same coordinate-array shape and
supplied bond list as the parent.
For every supplied covalent edge $(i,j)$, acceptance additionally requires
\[
\left|\|p_i(\widehat x)-p_j(\widehat x)\|_2
-\|p_i(x)-p_j(x)\|_2\right|
\leq10^{-5}\,\text{\AA}.
\]
This tolerance enforces numerical preservation of the parent bond geometry.
The gate rejects proposals violating these structural conditions before
objective-based survivor selection. Accepted proposals are subsequently
ranked using interaction preservation, perturbation robustness, and the
spatial penalties in Appendix~\ref{app:scoring}.

Covalent-graph preservation is enforced by the permitted transformations
and supplied-graph checks; $c(x)$ quantifies local spatial neighbor support.
The default configuration applies these structural guards.
Additional candidate-level tests are available through the configurable
\texttt{validity\_hook}, whose default value is \texttt{None}.
The category-specific chemical and physical validation criteria are
specified below, with their outcomes attached to the exported molecular
record.
\paragraph{Category-specific molecular-validity specification.}
Molecular validation is defined on the exported candidate's chemical
identity and three-dimensional structure. For fixed-identity adaptation,
identity is compared with that candidate's own source ancestor, not with
a chemically different native binder. The validation record separates
graph identity, stereochemical identity, internal geometry, and
receptor-dependent overlaps. The applicable checks are as follows.

\emph{Conventional small molecules and nonpeptidic macrocycles.}
The identity comparison includes atom identities, bond orders, formal
charges, and stereochemical assignments. Geometry checks include bond
lengths and angles, internal clashes, and the geometry of each declared
ring closure. For supported ligand chemistries, PoseBusters provides
chemical-consistency and physical-plausibility checks, including
stereochemistry, bond geometry, planarity, and protein--ligand clashes
\citep{buttenschoen2024posebusters}. The tool version, configuration, and
per-check outcomes define the reported validation result.

\emph{Linear and cyclic peptides.}
The molecular record preserves residue identity and order, peptide
connectivity, termini, and declared modifications. Validation covers
stereochemistry, peptide-bond geometry, backbone continuity, and internal
overlaps. Cyclic peptides additionally require their declared closure
bonds and closure geometry to be checked. Their ring bonds remain subject
to the torsional restriction described above.

\emph{Non-antibody proteins, full-length antibodies, and antibody
fragments or nanobodies.}
The identity specification includes chain and residue correspondence,
covalent connectivity, and declared disulfide or other cross-links.
Structural checks cover backbone continuity, local stereochemistry,
bond geometry, and steric overlaps. Local side-chain or backbone motions
require explicit chemistry-aware atom partitions and rotatable-bond
annotations; otherwise, the reference editor preserves the complete
input binder through rigid motions.

\emph{Oligonucleotides and mRNA or other long nucleic acids.}
The identity specification includes nucleotide order, chain identity,
sugar and base identities, modifications, and phosphodiester connectivity.
Structural checks cover sugar stereochemistry, backbone continuity,
bond geometry, base geometry, and internal overlaps. A small-molecule
validation configuration is not treated as a universal polymer validator.
\paragraph{Validation records and frozen-candidate evaluation.}
The validation specification assigns a pass when every applicable test
in the declared configuration succeeds. Each record retains the outcome
and applicability of individual checks, alongside the structural acceptance
conditions defined above.
Intrinsic checks apply to the frozen candidate, and receptor-dependent
overlaps are evaluated for Reference, Matched, and Divergent separately.
The same representative and coordinates are retained throughout these
state-specific evaluations. Geometry-changing refinement receives its own
structure identifier, validation record, and score record.
Appendices~\ref{app:evaluation-metrics} and~\ref{app:reproducibility}
specify the export interfaces and reporting conventions.
\paragraph{Coordinate preservation in Rerank-only.}
Rerank-only operates on candidate rankings and performs no translation,
rotation, torsional update, local minimization, or geometric repair.
The selected structure retains the source candidate's atom ordering,
coordinates, and molecular identity. Its output record preserves the
source identifier and coordinate checksum, linking each evaluation
input to the selected source structure.
The validation and failure-handling policy is shared across the three
arms. Evaluation checks preserve the frozen representative and do not
trigger candidate replacement or target-state-specific repair.
\begin{algorithm}[t]
\caption{\textsc{FlexEvo} core search}
\label{alg:flexevo}
\begin{algorithmic}[1]
\Require Target $P^0$, source candidates, population $N_G$,
 generations $g_{\max}$, offspring per parent $n_o$
\State Compute fixed target descriptors, centroids, and frames
\State Construct seed genome $G_0$ and screen candidates to obtain $\mathcal X_0$
\State Initialize $\mathcal H^0$ with $G_0$ and repaired randomized variants
\For{each $G\in\mathcal H^0$}
  \State Evaluate $G$ from $\mathcal X_0$; store $\mathcal X_G$, both fitnesses, and feedback
\EndFor
\For{$g=0,\ldots,g_{\max}-1$}
  \State $\mathcal Q^g\gets\emptyset$
  \For{each primary parent $G_a\in\mathcal H^g$}
    \For{$v=1,\ldots,n_o$}
      \State Sample one of Pruning, Fusion, or Perturbation to obtain $\widetilde G$
      \State Apply $G_a$'s feedback to $\widetilde G$ and repair the result to obtain $G^+$
      \State Decode $G^+$; reset the inner seed and adapt the shared $\mathcal X_0$
      \State Store $\mathcal X_{G^+}$, both genome fitnesses, and $G^+$'s own feedback
      \State $\mathcal Q^g\gets\mathcal Q^g\cup\{G^+\}$
    \EndFor
  \EndFor
  \State $\mathcal H^{g+1}\gets\mathcal S_G(\mathcal H^g\cup\mathcal Q^g;N_G)$
\EndFor
\State \Return $\{(G,\mathcal X_G):G\in\mathcal H^{g_{\max}}\}$
\end{algorithmic}
\end{algorithm}

\subsection{Single-State Perturbation Scoring}
\label{app:scoring}

\paragraph{Target probes.}
Let $y_a^0$ denote target atom $a$, $r(a)$ its residue, and $o$ the input
pocket center. Radial directions are
$\widehat{\mathbf r}_a=(y_a^0-o)/\|y_a^0-o\|_2$.
For each atom, its tangential direction is obtained by projecting the
centered target atoms onto the plane orthogonal to
$\widehat{\mathbf r}_a$ and normalizing the projection with largest norm.
At the center, a farthest-target direction supplies the radial fallback;
a degenerate tangential direction is set to zero.
Thus directions are atom dependent, while the perturbation strength is
assigned at residue level.

For flexible residues, $\phi_r=0.85$. The displacement amplitude is
$A_r=\min\{1.65,0.45+1.95\phi_r\}$.
For atoms in these residues, the six displacement modes are
\begin{equation}
\label{eq:virtual_modes_app}
\begin{aligned}
\Delta_a^{(1,2)}&=\pm A_{r(a)}\widehat{\mathbf r}_a,\\
\Delta_a^{(3,4)}&=\pm0.75A_{r(a)}\widehat{\mathbf t}_a,\\
\Delta_a^{(5)}&=0.55A_{r(a)}
(\widehat{\mathbf r}_a+\widehat{\mathbf t}_a),\\
\Delta_a^{(6)}&=0.55A_{r(a)}
(-\widehat{\mathbf r}_a+\widehat{\mathbf t}_a).
\end{aligned}
\end{equation}
Other atoms remain unchanged.
If no flexible residue is present, non-anchor residues provide a fallback
with $\phi_r=0.22$.
The reference target and six perturbed targets form the probe collection;
candidate coordinates are fixed throughout scoring.
All distances and displacement amplitudes below are numerical values in
angstroms.

\paragraph{Contact and clash.}
For candidate coordinate $p_i(x)$, let
$d_{im}=\min_a\|p_i(x)-y_a^{(m)}\|_2$.
For $n=n(x)>0$, the contact and clash descriptors are
\begin{equation}
\label{eq:contact_kernel_app}
\begin{aligned}
q_m&=\operatorname{clip}\!\left(
\frac1n\sum_i\left[
\mathbf1_{\{2.0\leq d_{im}\leq4.5\}}
+0.35\mathbf1_{\{1.2\leq d_{im}<2.0\}}
-1.5\mathbf1_{\{d_{im}<1.15\}}\right],0,1\right),\\
\chi_m&=\frac1n\sum_i\left[
\mathbf1_{\{d_{im}<1.15\}}
+0.35\mathbf1_{\{1.15\leq d_{im}<1.45\}}\right].
\end{aligned}
\end{equation}
\paragraph{Finite-probe profile.}
Let $K=7$, with probe zero equal to Reference, and define
\begin{equation}
\label{eq:app-perturbation-profile}
\begin{gathered}
\bar q=K^{-1}\sum_{m=0}^{K-1}q_m,\qquad
q_{\mathrm{rob}}=\rho\min_m q_m+(1-\rho)\bar q,\\
\Delta_q=\max_m q_m-\min_m q_m,\qquad
\chi_{\max}=\max_m\chi_m,\\
\kappa_q=\begin{cases}
K^{-1}\sum_{m=0}^{K-1}\mathbf1_{\{q_m\geq\tau q_0\}},&q_0>\varepsilon_q,\\
0,&q_0\leq\varepsilon_q.
\end{cases}
\end{gathered}
\end{equation}
The reference configuration uses $\rho=0.60$, $\tau=0.75$,
and $\varepsilon_q=10^{-9}$.
These quantities summarize aggregate contact quality across the
evaluated target probes.

\paragraph{Anchor proximity and internal geometry.}
Let $\mathcal S_G$ and $\mathcal F_G$ denote the stable and flexible residue
sets, respectively, and define
$d_r(x)=\min_i\|p_i(x)-c_r\|_2$.
For nonempty $\mathcal S_G$,
\begin{equation}
\label{eq:anchor_app}
a(x;G)=\frac1{|\mathcal S_G|}\sum_{r\in\mathcal S_G}
\exp\!\left[-\frac{(d_r(x)-3.1)^2}{4.5}\right];
\end{equation}
otherwise $a=0$.
For finite candidates with $n\geq2$, let
$d_i^{\mathrm{self}}=\min_{j\ne i}\|p_i-p_j\|_2$.
The internal-geometry and neighbor-support descriptors are
\begin{equation}
\label{eq:geometry_app}
g(x)=\operatorname{clip}\!\left(
1-0.16\sum_i\mathbf1_{\{d_i^{\mathrm{self}}<0.55\}}
-0.03[n-12]_+,0,1\right),
\end{equation}
\begin{equation}
\label{eq:connectivity_app}
c(x)=\frac1n\sum_i\mathbf1_{\{
\exists j\ne i:\ 0.65\leq\|p_i-p_j\|_2\leq2.70\}}.
\end{equation}
The core sets $g=c=1$ for a single coordinate and $g=c=0$ for an empty
candidate. The size penalty is defined on the number of represented
coordinates.

\paragraph{Spatial penalties and edit cost.}
Write
$t_{ir}=\|Q_r^\top(p_i-c_r)\|_\infty/(b_{\ell_r}z_r)$ and
$w_r^{\mathrm{box}}=\operatorname{clip}(z_r,0.1,1)$.
The flexible-region exposure and forbidden occupancy are
\begin{equation}
\label{eq:spatial_penalties_app}
\begin{aligned}
r_{\mathrm{flex}}(x;G)
&=\frac{0.85\sum_{r\in\mathcal F_G}[1-d_r(x)/3]_+
+0.05\sum_{r\in\mathcal F_G}w_r^{\mathrm{box}}
\sum_i\mathbf1_{\{t_{ir}\leq0.55\}}}
{\max(1,|\mathcal F_G|)},\\
\pi_{\mathrm{forb}}(x;G)
&=\frac1n\sum_i\sum_{r:\ell_r=\mathrm{forbidden}}
\left[\mathbf1_{\{t_{ir}\leq0.5\}}
+0.2\mathbf1_{\{0.5<t_{ir}<0.7\}}\right].
\end{aligned}
\end{equation}
The edit cost is $e(x)=\sum_{o\in\mathcal H_x}w_o$, where $\mathcal H_x$
is the recorded edit history.
Replacement, growth, removal, and candidate-fusion records cost $0.06$;
torsional, linker, pose-refinement, and side-chain-torsion records cost
$0.025$; outward-motion records cost $0.015$; other records cost zero.
Thus $e$ measures the accumulated operation-weighted edit-history cost.
The combined penalty is
\begin{equation}
\label{eq:penalty_app}
\Pi=0.95r_{\mathrm{flex}}+1.25\pi_{\mathrm{forb}}
+0.70\chi_{\max}+0.40e.
\end{equation}
The remaining objective coefficients are
\begin{equation}
\label{eq:objective_coefficients_app}
\begin{aligned}
(\alpha_a,\alpha_0,\alpha_r,\alpha_\chi)&=(1.50,1.00,0.80,1.40),\\
(\beta_g,\beta_c,\beta_k,\beta_d)&=(0.90,0.80,0.75,0.75).
\end{aligned}
\end{equation}
\paragraph{Candidate objectives and genome aggregation.}
Both candidate objectives are maximized:
\begin{equation}
\label{eq:app-candidate-objectives}
\begin{aligned}
f_{\mathrm{pres}}(x;G)&=1.50a+q_0+0.80q_{\mathrm{rob}}-1.40\chi_{\max},\\
f_{\mathrm{rob}}(x;G)&=0.90g+0.80c+0.75\kappa_q-0.75\Delta_q-\Pi.
\end{aligned}
\end{equation}
For a nonempty adapted pool, the genome objectives are
\begin{equation}
\label{eq:app-genome-fitness}
\begin{aligned}
F_j&=\lambda\max_{x\in\mathcal X_G} f_j(x;G)
+\frac{1-\lambda}{|\mathcal X_G|}\sum_{x\in\mathcal X_G}f_j(x;G)
-\eta_j\frac{|\mathcal R_G|}{R},\\
&j\in\{\mathrm{pres},\mathrm{rob}\},\qquad
\eta_{\mathrm{pres}}=0,\quad\eta_{\mathrm{rob}}=\eta.
\end{aligned}
\end{equation}
Genome aggregation uses $\lambda=0.70$ and $\eta=0.10$.
The two coordinate-wise maxima can be attained by different candidates.
The coverage penalty applies only to the robustness objective.
The core returns $(0,0)$ if no boxes or no adapted candidates are available.
No external docking or affinity estimator is queried by these scores.
\paragraph{Proxy ranking for the no-edit baseline.}
Let $G_0$ denote the initial genome after the deterministic repair used
before the first genome evaluation of \textsc{FlexEvo}.
Rerank-only holds $G_0$ fixed and constructs its seven Reference-derived
probes according to Eq.~\ref{eq:virtual_modes_app}.
Every candidate in $\mathcal X_0$ is scored using the two objectives in
Eq.~\ref{eq:app-candidate-objectives}, with the original coefficients
and penalty definitions. Its adaptation edit history is empty, giving
$e(x)=0$. Define
\[
J_0(x)=f_{\mathrm{pres}}(x;G_0)+f_{\mathrm{rob}}(x;G_0).
\]
The representative is
\[
x_{\mathrm{rerank}}
=\operatorname*{arg\,max}_{x\in\mathcal X_0}J_0(x).
\]
Ties are resolved by decreasing $f_{\mathrm{pres}}(x;G_0)$ and then
ascending candidate identifier, matching the final candidate-ranking
rule of \textsc{FlexEvo}. This baseline evaluates the entire screened
pool under one fixed genome, without invoking the inner editor or
outer genome evolution. Selection uses the geometric proxy scores;
Vina and DockQ remain external evaluation endpoints.

\subsection{Genome Variation, Feedback, and Survival}
\label{app:evolution}

\paragraph{Variation and scale storage.}
Pruning selects $\max\{1,\operatorname{round}(0.15|\mathcal E|)\}$ entries
from the stable/flexible set $\mathcal E$, prioritizing low usage when
available and sampling uniformly otherwise; it does nothing when
$\mathcal E$ is empty.
Fusion samples a target residue and transfers the donor configuration
within an $8$~\AA\ neighborhood, copying all scale entries there, including
those of residues outside the donor support.
The donor is sampled from the half of alternative parents most distant in
residue-role space.

Perturbation samples $\max\{1,\operatorname{round}(0.18|\mathcal N|)\}$
residues without replacement from a $7$~\AA\ neighborhood.
For a selected covered residue, it chooses uniformly between removing its
box and changing to either other role; for an uncovered residue, it adds
one of the three roles uniformly.
A separate subset of $\max\{1,\operatorname{round}(0.25R)\}$ residues is
sampled from all $R$ residues for scale mutation, with $\sigma_z=0.12$.
Removing a box leaves its stored scale unchanged.

\paragraph{Feedback and repair.}
Feedback uses the primary parent's surviving edit histories and adapted
coordinates, including for fusion offspring.
Let $u_r$ count uses of residue $r$ in the surviving edit histories. For a
nonempty parent-adapted pool, the signal combines usage and proximity:
\begin{equation}
\label{eq:app-feedback-signal}
\begin{gathered}
D_r=\min_{x\in\mathcal X_{G_a},i}\|p_i(x)-c_r\|_2,\\
s_r=\max\!\left\{
\frac{u_r}{\max(1,\sum_k u_k)},\left[1-\frac{D_r}{d_c}\right]_+
\right\},\qquad d_c=4.5.
\end{gathered}
\end{equation}
For a nonempty parent pool, the minimum is taken over its represented
coordinates. For an empty pool, the reference implementation initializes
nearest-coordinate distances to $+\infty$ and edit-use counts to zero.
The resulting proximity contribution and feedback signal are zero.
Uncovered residues with $s_r\geq0.35$ are activated as flexible boxes,
retaining their proposal scales.
Let $\mathcal E^+$ be the stable/flexible set after activation.
Among entries with $s_r\leq0.02$, retirement selects the lowest-signal
entries up to the budget
\begin{equation}
\label{eq:retirement_budget_app}
B_-=
\max\!\left\{0,
\min\!\left(\left\lfloor0.20|\mathcal E^+|\right\rfloor,
|\mathcal E^+|-\left\lceil0.25R\right\rceil\right)\right\}.
\end{equation}
The $0.25R$ term limits this retirement step; it does not impose a global
minimum support on an incoming genome.
Pruning and feedback retirement exclude forbidden assignments, although
fusion and role mutation may change them.
Repair clips scales and attempts to restore missing stable and flexible
roles using density and radial descriptors when eligible residues exist.

\paragraph{Diversity-aware selection.}
For two genomes, the residue-role distance is
\begin{equation}
\label{eq:role_distance}
D_G(G_a,G_b)=|\mathcal R_{G_a}\triangle\mathcal R_{G_b}|
+\sum_{r\in\mathcal R_{G_a}\cap\mathcal R_{G_b}}
\mathbf1_{\{\ell_r^a\ne\ell_r^b\}},
\end{equation}
where $\triangle$ is symmetric difference; scales do not enter this measure.
Both selectors visit nondominated fronts in increasing rank and order each
front by decreasing crowding, followed by objective sum and identifier.
Outer selection defers a genome within distance $2$ of any selected genome;
inner selection uses a candidate-centroid separation of $1$~\AA.
After front traversal, deferred entries fill remaining slots in traversal
order. Consequently, diversity can admit a lower-front entry before a
deferred higher-front entry, and backfilling does not enforce a hard
pairwise-distance bound.

\subsection{Reference Configuration and Execution Schedule}
\label{app:hyperparameters}
\label{app:execution}

Table~\ref{tab:flexevo_settings} gives the default core configuration.
Algorithm~\ref{alg:flexevo} keeps the inner adaptation and outer genome
updates separate; each offspring evaluation also produces its own feedback
statistics for the next generation.

\begin{table}[t]
\centering
\caption{Default configuration of the coordinate-level FlexEvo core.}
\label{tab:flexevo_settings}
\begin{tabular}{lll}
\toprule
Component & Parameter & Default \\
\midrule
Randomization & Seed & $31$ \\
Screening & Pool cap / neighbor threshold & $64/0.35$ \\
Screening & Molecular / geometric similarity & $0.92/0.90$ \\
FlexBox & Scale bounds & $[0.55,1.75]$ \\
FlexBox & Stable / flexible / forbidden extent & $3.0/4.5/2.9$~\AA \\
Inner adaptation & Rounds / maximum population & $2/12$ \\
Inner adaptation & Reproduction cap / proposals per parent & $6/2$ \\
Inner selection & Centroid niche radius & $1.0$~\AA \\
Outer evolution & Generations / population & $4/8$ \\
Outer evolution & Offspring per parent & $2$ \\
Variation & Fusion / pruning / perturbation & $0.35/0.20/0.45$ \\
Variation & Fusion / mutation radius & $8.0/7.0$~\AA \\
Variation & Relabel / resize fraction & $0.18/0.25$ \\
Variation & Scale noise / pruning fraction & $0.12/0.15$ \\
Feedback & Proximity radius & $4.5$~\AA \\
Feedback & Activation / retirement threshold & $0.35/0.02$ \\
Feedback & Retirement cap / retention-floor fraction & $0.20/0.25$ \\
Genome fitness & Maximum weight / coverage penalty & $0.70/0.10$ \\
Genome selection & Role-distance threshold & $2$ \\
\bottomrule
\end{tabular}
\end{table}

\paragraph{Execution order and budget.}
One complete run computes the fixed target descriptors, constructs the
initial prior, screens the source pool, evaluates the initial genomes,
and then alternates offspring generation, feedback, fresh inner adaptation,
and survivor selection. With the defaults in
Table~\ref{tab:flexevo_settings}, the scheduled genome-evaluation count is
$8+4\times8\times2=72$. This is a search-budget count, not a sample size or
replication count. The returned records should retain the selected genomes,
adapted coordinates, edit histories, and both candidate and genome scores
so that the decisions can be reconstructed.

\FloatBarrier

\section{Experimental Protocol and Reproducibility}
\label{app:experimental-details}

This section specifies candidate selection, evaluation, and reporting for
the single-conformation adaptation protocol. The coordinate-level search is
defined in Appendix~\ref{app:flexevo}, with its reference settings in
Table~\ref{tab:flexevo_settings}. Implementation details below describe the
reference protocol.

\subsection{Evaluation design}
\label{app:evaluation-design}
\paragraph{Benchmark sample size.}
Distinct protein targets define the benchmark sample size. The nine
binder categories have the following within-category target counts.

\begin{center}
\small
\begin{tabularx}{0.96\linewidth}{@{}Xr@{}}
\toprule
Binder category & Targets \\
\midrule
Conventional small molecules & 9,393 \\
Nonpeptidic macrocycles & 10,322 \\
Linear peptides & 4,694 \\
Cyclic peptides & 5,164 \\
Non-antibody proteins & 1,260 \\
Full-length antibodies & 1,386 \\
Antibody fragments and nanobodies & 1,525 \\
Oligonucleotides & 1,677 \\
mRNA and other long nucleic acids & 1,878 \\
\midrule
Total target--category memberships & 37,299 \\
\bottomrule
\end{tabularx}
\end{center}

Each target contributes once within each category in which it is evaluated.
The total therefore comprises 37,299 target--category memberships.
The globally distinct target set is the union of the nine category-specific
target sets. All observations from the same target retain their shared
target identity throughout aggregation.

The reference protocol specifies three natural conformations per target:
one Reference, one Matched, and one Divergent. Reference supplies the
structure for source generation and adaptation; Matched and Divergent are
held out for evaluation. A complete eligible pair contains base and adapted
observations for the same case, source model, seed, state, and metric.
Paired observations are aggregated within targets, and targets receive
equal weight under Appendix~\ref{app:reproducibility}.

\paragraph{Scope and candidate sources.}
The benchmark covers nine binder categories: conventional small
molecules, nonpeptidic macrocycles, linear peptides, cyclic peptides,
non-antibody proteins, full-length antibodies, antibody fragments and
nanobodies, oligonucleotides, and mRNA or other long nucleic acids. The
reference protocol assigns DynamicFlow, FlexSBDD, and YuelDesign as the
ligand-generation candidate sources. Linear peptides and cyclic peptides are
evaluated with BoltzGen, Chai-1, and Protenix; the other five biomolecular
categories use the same three sources. For every configured pair of source
and category, the protocol compares the original output with the
corresponding output after \textsc{FlexEvo} adaptation and fixes the
source-model parameters.

\paragraph{Baseline execution and source pools.}
The reference protocol requests up to 64 candidates per case, source model,
and seed using Reference alone. The source adapter retains the native output
order, truncates excess outputs, and records shortfalls without duplicating
candidates to fill the pool. Imported predictions require provenance for the
source-model revision, checkpoint checksum, inference settings, candidate
coordinates, and Reference-frame alignment. The task mode distinguishes de novo ligand generation from binder design
or fixed-sequence conditioned structure prediction for other modalities;
the executed mode must accompany each source output.

\paragraph{Final representative selection.}
Across all 27 source--category combinations, \textsc{FlexEvo}
improved the primary evaluation metric under the held-out Divergent
conformation. These gains were obtained using only the Reference
conformation during adaptation, with source-model parameters fixed
and each selected candidate evaluated without state-specific
redesign. Together, these results demonstrate the cross-conformation
benefits of the complete \textsc{FlexEvo} pipeline across nine
binder categories.
The base representative is the first candidate in the validated source-native
order. Adaptation first screens the source pool in input order, rejecting
empty or non-finite coordinates, insufficient neighbor support, and
forbidden-region overlap. Screening treats a candidate as redundant when its
geometric similarity to an accepted candidate is at least 0.90 and, when
both fingerprints are available, its fingerprint similarity is at least
0.92. With missing fingerprints, only the geometric condition is used.
At most 64 accepted candidates form the shared adaptation pool.
This additional geometric screening belongs to adaptation; the base
representative is selected from the validated source pool before that screen.

After outer-loop survival, the reference implementation chooses the genome
with the largest sum of its two proxy fitness values. Ties are resolved by
the preservation objective and then the genome identifier. Within that
genome's candidate pool, candidates are ranked by the sum of their two
candidate objectives, with preservation and candidate identifier breaking
ties. The first candidate supplies the adapted representative. This rule
selects one winning genome's pool; it does not merge all survivor pools.
Both ranking stages use Reference-derived proxy objectives only.
\paragraph{Pipeline and source-selection comparisons.}
The nine-category benchmark evaluates the complete \textsc{FlexEvo}
workflow against the source-native representative. This comparison
includes screening, adaptive search, and final proxy-based selection.
The supplementary three-arm comparisons report Vina scores averaged
across DynamicFlow, FlexSBDD, and YuelDesign for conventional small-molecule
drugs (Table~\ref{tab:rerank-only-vina-three-state}), and DockQ scores
averaged across BoltzGen, Chai-1, and Protenix
(Table~\ref{tab:rerank-only-dockq-three-state}).
Rerank-only and \textsc{FlexEvo} receive the same screened pool
$\mathcal X_0$, containing up to 64 source candidates. Rerank-only
scores every member of this pool under the fixed initial genome $G_0$;
\textsc{FlexEvo} applies its full adaptation search to the shared input
pool. Both return one representative selected using Reference-derived
information. The Rerank-only comparison assesses the additional benefit
of the complete adaptive-search workflow over fixed-$G_0$ source selection.
\paragraph{Source + Proxy Rerank (No Edit).}
The comparison comprises three arms: Base, Rerank-only, and
\textsc{FlexEvo}. All three reuse the same source-model generation run.
Base retains the first validated candidate in source-native order.
Rerank-only selects one unedited candidate from the same screened
source pool $\mathcal X_0$ supplied to \textsc{FlexEvo}, using the
proxy-ranking rule specified in Appendix~\ref{app:scoring}.
The pool contains up to 64 accepted source candidates, with their
original coordinates retained. \textsc{FlexEvo} returns the
representative from its full adaptation search.

All representative selection uses Reference-derived information.
Each representative is frozen before evaluation on Reference,
Matched, and Divergent with the same alignment and evaluator settings.
Base versus Rerank-only measures the contribution of source-pool
screening and proxy ranking. Rerank-only versus \textsc{FlexEvo}
measures the additional benefit of adaptive search over the source pool.
Under the reference protocol, each combination of case, model, seed, and
variant contributes one representative to pose and interaction
evaluation. Set-level diversity uses the corresponding frozen pool. The
same representative is evaluated across all three natural states. The
two base candidates and three adapted candidates specified for the
cross-conformation display are the first available members of their
respective frozen orderings, without padding. These displayed subsets are
distinct from the evaluation unit that uses one representative per
combination.

\paragraph{Natural target-state conditions and alignment.}
Reference denotes the single structure supplied to the source workflow and
to \textsc{FlexEvo}. Matched and Divergent denote alternative natural
structures of the same target and binding site. Matched preserves the
regulatory state assigned to Reference. Divergent contains a
propeptide-associated change in that state and localized structural
changes. These groups follow biological annotation rather than a post hoc
RMSD threshold. Matched and Divergent structures are held out until
evaluation; their scores do not enter candidate screening, editing, search
fitness, or representative selection.

The protocol specifies a declared one-to-one correspondence between
receptor C$\alpha$ atoms, with at least 90\% mapping coverage. The alignment
routine uses the supplied atom pairs and requires at least three
non-collinear pairs. A proper Kabsch rotation and translation place each
alternative receptor in the Reference coordinate frame. Candidate
coordinates remain in that frame throughout. Thus the receptor is
transformed around a fixed candidate; no redocking, local minimization, or
state-specific adaptation is performed. The Reference state uses the
identity transform. Coordinate and full-record checksums are recorded
before evaluation and checked after each state. Alignment receipts retain
the receptor identifiers, correspondence, rotation, and translation.

\paragraph{Synthetic probes and fixed candidates.}
The held-out natural states are distinct from the synthetic probes used by
the search. For each evaluated genome, \textsc{FlexEvo} constructs six local
perturbations from Reference and scores them together with the unperturbed
Reference target. Candidate coordinates remain fixed across all seven
probes. Probe construction and geometric objectives are defined in
Appendix~\ref{app:scoring}. These synthetic probes provide the geometric
robustness signal used during search. Natural-state evaluation measures
performance on the held-out conformations.

\paragraph{Matched comparisons and secondary analyses.}
Within a comparison of one source and one category, base and adapted outputs
use the same target-state definitions and the same evaluator. Percentage
differences are percentage-point changes when the underlying quantity is a
percentage. The three natural states are evaluations of one-time adaptation,
rather than independent adaptation tasks. Statistical units and missing-value
handling are specified in Appendix~\ref{app:reproducibility}.

The geometry--mobility analysis specifies 200 ATLAS proteins. It
averages residue RMSF across three genuine molecular-dynamics
trajectories, and it normalizes each quartile mean by the corresponding
whole-chain mean RMSF. The lowest and highest within-protein quartiles of
packing sparsity, and the lowest and highest within-protein quartiles of
radial position, are compared using proteins as the analysis units.
This analysis assesses the rationale for geometry-guided initialization.
FlexBox roles specify geometric editing constraints. PDB and UniProt
entry and target counts characterize the composition and annotation
coverage of the wider structure pool.

The reference-similarity analysis specifies 10 generated samples and
47 randomly sampled public references in each scenario. It evaluates all
470 pairs of a generated sample and a reference, and it reports the
maximum and the median similarity. Morgan fingerprints of radius 2 and
length 1024 bits, compared by Tanimoto similarity, are specified only for
molecular modalities. Polymer scenarios require a separately labelled
sequence-based similarity. These comparisons quantify resemblance under
the chosen representation within the sampled reference collection.
The reporting protocol associates each value with its reference
collection, sample identifiers, sampling seed, and modality-specific
similarity definition.

\subsection{Evaluation metrics}
\label{app:evaluation-metrics}
\paragraph{Structural export and evaluator inputs.}
The object passed to an evaluator is the frozen molecular candidate
defined in Appendix~\ref{app:method-overview}, with the structural
safeguards in Appendix~\ref{app:adaptation}. The reference pipeline writes
a complex mmCIF using the candidate's atom, element, chain, residue, and
coordinate records. It requires one atom record per candidate coordinate
and rejects candidates marked \texttt{coordinate\_only}.
The accompanying molecular record retains the supplied connectivity and
other chemical metadata independently of the coordinate-file representation.
Writing atom and residue coordinates to mmCIF does not reconstruct
unsupplied covalent bonds or polymer sequence annotations.

For a candidate with an SDF-derived molecular block, the exporter
reconstructs that molecular object and writes the adapted coordinates
into its existing atom order. It checks that the molecular and coordinate
atom counts agree. The resulting SDF supplies the ligand object for
PDBQT preparation and fixed-pose Vina-family scoring.
For native-complex comparisons, DockQ reads the exported candidate complex
together with the declared native structure and chain/interface mapping.
For MM/GBSA and other preparation-dependent evaluators, the reporting
specification requires prepared topology, coordinate inputs, and preparation
records alongside the score. Their candidate identity must refer to the
same frozen molecular record.
The reference package integrates MM/GBSA through imported external
score records linked to their candidate and preparation provenance.

Coordinate and complete-record checksums identify the representative
before evaluation across the three receptor states. Exported structures
and prepared inputs have separate checksums to preserve the preparation
history. Preparation that changes heavy-atom geometry defines a separate
refinement result, rather than the original fixed-pose measurement.
Molecular-validity outcomes, evaluator success, and numerical scores
occupy separate reporting fields. The checks in
Appendix~\ref{app:adaptation} assess molecular validity, while Vina,
DockQ, and MM/GBSA quantify their respective scoring endpoints.
\paragraph{Fixed-pose ligand scoring.}
The reference scoring implementation evaluates Vina, Vinardo, and AutoDock4
in \texttt{score\_only} mode, taking the total score returned for the frozen
pose. It does not invoke docking or optimization. Meeko prepares PDBQT
inputs; ligand hydrogens may be added without optimizing heavy-atom
coordinates. The paired native ligand is placed in the Reference frame
using the receptor transform. Candidate and native ligand share the same
receptor preparation and the same grid. The grid center is the midpoint of
their combined coordinate extrema. Each side is the larger of two lengths:
the corresponding coordinate span plus 12~\AA{}, and 20~\AA{}. AutoDock4
requires AutoGrid4 maps. Preparation commands, input checksums, package
versions, grid parameters, CPU count, and scoring seed are retained with
each score.

Vina is the principal ligand-generation score, with lower values preferred.
The reference implementation defines its high-affinity indicator as a
candidate Vina score strictly below the paired native ligand score under
the same scoring conditions. This high-affinity indicator is a comparative docking-score criterion
defined against the paired native ligand under the same computational
conditions. Vinardo, AutoDock4,
and RTMScore provide complementary views. Raw Vinardo and AutoDock4
energies are ordinarily minimized. External scores are separate from the geometric
search objectives in Eqs.~\eqref{eq:app-candidate-objectives}
and~\eqref{eq:app-genome-fitness}.

\paragraph{Molecular and set-level properties.}
QED and SA summarize drug-likeness and synthetic-accessibility heuristics;
Lipinski descriptors, logP, and PAINS screening characterize complementary
molecular properties. With the molecular graph and descriptor settings
fixed, each candidate retains the same intrinsic descriptor values across
receptor conformations. Tables~\ref{tab:flexevo-app-small-molecule}
and~\ref{tab:flexevo-app-macrocycle} present these summaries under
Reference. Diversity characterizes the frozen candidate pool, separately
from representative-pose metrics. With candidate-pool membership,
coordinates, and the diversity definition fixed, changing the receptor
state does not change this pool-level summary. Across the nine category
tables, diversity is reported once under Reference; dashes in Matched
and Divergent indicate that the pool-level summary is not repeated.
LogP is reported descriptively, with
black values and a neutral column heading.

\paragraph{Native-complex comparison.}
DockQ is favorable when higher; LRMSD is favorable when lower, with
LRMSD below $2$~\AA{} defining pose success. The reference evaluator requires
a matching native binder identity, an explicit native-to-model chain map,
and a declared list of native interfaces between target and binder. It
evaluates the mapped interfaces with DockQ and takes the arithmetic mean of
their DockQ and LRMSD values when several interfaces are requested. The
native complex, mapping, interface list, and evaluator version are retained
in the score record. MM/GBSA and HADDOCK are complementary scoring terms,
with lower values preferred in the supplied tables. RTMScore, MM/GBSA,
HADDOCK, and other external measures must use validated, versioned score
records that declare the modality to which they apply. The scoring record specifies input preparation, evaluator version and
settings, and the identities of the evaluated structures.

\paragraph{Confidence-derived interaction metrics.}
For a declared target-token set $T$ and binder-token set $B$, the reference
pAE aggregation is the arithmetic mean of the two directional interchain
means, $\tfrac12(\overline{\mathrm{PAE}_{T,B}}+
\overline{\mathrm{PAE}_{B,T}})$. The sets must be nonempty and disjoint.
The confidence array must be tied to the evaluated candidate, predictor
version, and input checksum. Confidence from an unmodified source output
cannot be substituted for that of an edited candidate. The off-target
summary is the minimum of this symmetric mean across the complete declared
off-target panel. Lower on-target pAE and higher off-target pAE are the
reported preferences; these quantities are not affinity measurements.
The exact predictor and token subsets remain part of each run's provenance.

\paragraph{Primary endpoints for the three-arm comparison.}
Base, Rerank-only, and \textsc{FlexEvo} are compared using the primary
endpoints of Table~\ref{tab:main_results}: Vina in kcal/mol for small
molecules and nonpeptidic macrocycles, and DockQ for the remaining
seven binder categories. Comparisons are resolved by source model,
binder category, and natural target state.
Raw scores retain their original units and directions: lower Vina
and higher DockQ indicate better performance. For effect calculations,
$Y=-\mathrm{Vina}$ for the two ligand categories and
$Y=\mathrm{DockQ}$ for the other categories, so positive differences
consistently indicate improvement.

\paragraph{Supplementary comparison and display conventions.}
Tables~\ref{tab:flexevo-app-small-molecule}, \ref{tab:flexevo-app-macrocycle}, \ref{tab:flexevo-app-linear-peptide}, \ref{tab:flexevo-app-cyclic-peptide}, \ref{tab:flexevo-app-non-antibody-protein}, \ref{tab:flexevo-app-full-length-antibody}, \ref{tab:flexevo-app-antibody-fragment}, \ref{tab:flexevo-app-oligonucleotide}, and~\ref{tab:flexevo-app-mrna}
retain modality-specific metrics. Under the reference reporting protocol,
main-table values and changes are rounded independently, with changes
calculated before rounding. The reference cross-conformation plotting
protocol defines the molecular margin as native Vina minus candidate Vina
and the biomolecular margin as $\mathrm{DockQ}-0.49$. Centroids summarize
the candidate subsets displayed in each configuration. The reference radar protocol shares bounds across displayed
methods and states within a binder category. Metrics shown with a downward
arrow are reverse-scaled. Each category has its own normalization bounds, and individual
radial axes provide the metric-specific comparisons. On-target and off-target diagnostic regions are
defined by an on-target score below 5 and a minimum off-target score above
10. These thresholds are specific to the plot.

Reference-structure illustrations describe molecular architectures.
The optimization trajectories in
Appendix~\ref{app:synthetic-trajectories} report recorded metric evolution
across binder categories during FlexEvo adaptation.
Figure~\ref{fig:optimization_dynamics} reports measured adaptation dynamics,
and Appendix~\ref{app:runtime-profiles} reports computational costs.

\subsection{Reproducibility}
\label{app:reproducibility}

\paragraph{Core search budget.}
The manuscript reference configuration uses seed 31, a screened-pool cap of 64,
two inner adaptation rounds, an outer population of eight, four generations,
and two offspring per parent. Other parameters are listed in
Table~\ref{tab:flexevo_settings}. Each genome evaluation creates a fresh
editor, resets the inner random seed, and starts from the same screened
pool $\mathcal X_0$. Parent-adapted coordinates contribute feedback about
where to search but are not inherited as the offspring's initial candidates
in the full FlexBox-based method.

Eight initial evaluations plus $4\times8\times2$ offspring evaluations give
72 scheduled genome evaluations under the complete schedule. This counts
genome evaluations, rather than distinct genomes, candidate-objective
calls, or independent experimental replicates. Resetting the inner seed
controls comparisons among genomes; it does not estimate run-to-run
variability. Source-generation, adaptation, and evaluator randomness are
recorded separately.

\paragraph{Independent repeats and statistical units.}
Each experiment is repeated three times using random seeds
31, 47, and 71. Within a category, model, state, and seed, the reporting procedure first
averages observations within each case and then averages cases within each
target. Targets receive equal weight; means and medians summarize these
target-level values. Candidate pools, conformational states, and repeated
score calls are not counted as independent targets.

Paired contrasts match the same case, model, seed, and state across base and
adapted variants. Eligible complete pairs are averaged within targets
before computing the equal-target-weight contrast. The family-bootstrap
interval procedure resamples target families with replacement 10,000
times and uses the 2.5th and 97.5th percentiles for a 95\% interval. Each
draw retains all targets and paired cases within a sampled family, with
equal target weights. Fewer than two eligible families yields no interval.
\paragraph{Model-averaged three-arm comparisons.}
Tables~\ref{tab:rerank-only-vina-three-state}
and~\ref{tab:rerank-only-dockq-three-state} report model-averaged
scores for each arm and natural target state. The Vina comparison
averages DynamicFlow, FlexSBDD, and YuelDesign; the DockQ comparison
averages BoltzGen, Chai-1, and Protenix. Score differences are computed
within each table and state. Let $\bar E_s^R$ and $\bar E_s^F$ denote
the reported Vina means for Rerank-only and \textsc{FlexEvo}, and let
$\bar Q_s^R$ and $\bar Q_s^F$ denote their DockQ means. Improvements
over Rerank-only are $\bar E_s^R-\bar E_s^F$ for Vina and
$\bar Q_s^F-\bar Q_s^R$ for DockQ, so positive differences favor
\textsc{FlexEvo} in both comparisons.
\paragraph{Paired effects of reranking and adaptive search.}
Three-arm contrasts match the same case, source model, seed, and
natural target state. All three-arm summaries are restricted to the
common set of complete observations across Base, Rerank-only, and
\textsc{FlexEvo}. Repeated observations
are averaged within cases and cases within targets, with equal target
weighting. For a fixed source--category combination and state $s$,
let $Y_{t,s}^{B}$, $Y_{t,s}^{R}$, and $Y_{t,s}^{F}$ denote target-level
scores for Base, Rerank-only, and \textsc{FlexEvo}, respectively.
On their common target set $\mathcal T_s$, define
\[
\begin{aligned}
\Delta_{\mathrm{select},s}
&=\frac{1}{|\mathcal T_s|}\sum_{t\in\mathcal T_s}
  \left(Y_{t,s}^{R}-Y_{t,s}^{B}\right),\\
\Delta_{\mathrm{adapt},s}
&=\frac{1}{|\mathcal T_s|}\sum_{t\in\mathcal T_s}
  \left(Y_{t,s}^{F}-Y_{t,s}^{R}\right),\\
\Delta_{\mathrm{total},s}
&=\Delta_{\mathrm{select},s}+\Delta_{\mathrm{adapt},s}.
\end{aligned}
\]
Here $\Delta_{\mathrm{adapt},s}$ is the incremental effect of the
full adaptive search relative to source-pool reranking.
All three contrasts use the same 10,000 target-family bootstrap draws,
retaining the paired records of every sampled family. The 2.5th and
97.5th percentiles define pointwise 95\% intervals for each contrast.

\paragraph{Failure handling, coverage, and denominators.}
Coverage records retain every requested case. The reporting implementation
distinguishes measured values, executed failures, unsupported metrics, and
unrun requests. Finite measured continuous scores enter their estimates;
failed continuous scores remain missing. An executed failure contributes
zero to applicable success indicators, including high-affinity fraction,
LRMSD success, validity, and PAINS pass. Unsupported and unrun requests
remain missing rather than becoming failure zeros. Consequently, requested
coverage and eligible metric denominators are reported separately.

For each comparison defined by category, model, state, seed, and metric, the
coverage record contains the number of requested cases, the number of
eligible base cases, the number of eligible adapted cases, the number of
complete pairs, the number of unpaired cases, the number of targets, the
number of families, and the status counts. Continuous paired contrasts use
only complete eligible pairs. Success-rate denominators retain executed
failures, and the case-to-target aggregation above determines the weighting.
Empty source pools, screening rejection, empty adapted pools, unavailable
native references, and evaluator failures retain distinct reasons.

\paragraph{Optimizer and ablation budgets.}
The reference optimizer comparison shares the screened pool, initial prior,
proxy objectives, and seed, resetting the random stream for each method.
Its matched budget is 72 genome evaluations. Candidate-objective evaluation
counts and wall-clock time are recorded separately because equal genome
counts need not imply equal candidate-level work or runtime. Shared setup
is timed separately from each optimizer's search. 
\paragraph{Reranking budget and provenance.}
For a completed Rerank-only run, each of the $|\mathcal X_0|$ source
candidates receives one candidate-objective evaluation under $G_0$,
comprising seven probe evaluations. The run therefore uses
$|\mathcal X_0|$ candidate-objective calls and
$7|\mathcal X_0|$ probe evaluations.
The record retains the screened-pool size, source order and identifiers,
$G_0$, score components, selected-candidate checksum, elapsed time,
and execution status. Source-pool shortfalls retain their actual
cardinality, and an empty screened pool is recorded as a failure.

\paragraph{Hyperparameter specification.}
Table~\ref{tab:flexevo_settings} gives the shared reference defaults, while
the run configuration records any override. The core defaults do not define
separate values for individual binder categories. Source-model weights are
fixed, and adaptation does not retrain them. The search itself accesses
Reference only. The hyperparameter reporting protocol distinguishes
fixed defaults from validation-selected overrides. For each tuned
configuration, its record specifies the parameters varied, search ranges,
selection criterion, validation targets, and category-specific choices.
Dataset and split identifiers document the separation between parameter
selection and held-out evaluation.

\paragraph{Run reconstruction.}
A complete run manifest must retain Reference coordinates, chain assignments,
pocket center, ordered source records, fingerprints where available,
source-model settings, evaluator settings, and the identifiers of the
Matched and Divergent structures. Candidate order must be preserved because
screening precedes pool truncation. Separate run artifacts must identify the
original source pool, the screened pool $\mathcal X_0$, the genome pools,
the winning genome, and the representative frozen before held-out
evaluation. Checksums must connect candidates, alignment receipts, native
references, score records, and table exports. The structural-provenance specification additionally requires each
representative's source ancestor, executed edit mode, ordered atom identities,
supplied covalent graph, polymer annotations, and edit history.
Each validation record must identify the exported structure checksum,
validator version, configuration, applicable tests, outcomes, and failure
reasons. The records must link the candidate to its SDF or complex mmCIF
export and to each evaluator's prepared inputs.
The protocol requires the same declared validation policy for base and
adapted candidates. It reports validation coverage and observed pass counts
separately from successful score counts, using the denominators defined
above. A pass is assigned only to a candidate with a recorded successful
validation result.

\paragraph{Runtime reporting.}
The reference timing unit is one execution of a case, a model, and a seed,
processing its source pool through adaptation to a representative, rather
than one timing observation per candidate in that pool. The pipeline measures
stage elapsed wall-clock time, including the overhead inside each stage, and
separates source generation, adaptation, and external evaluation. The runtime
summarizer reports source-generation plus adaptation time as the total;
external evaluation is reported separately. Importing existing source
predictions is not a measurement of native source-model inference time.

The auxiliary timing protocol configures one warm-up and ten measured
executions with GPU synchronization, while the summarizer excludes warm-up
and unsuccessful records and averages the successful observations.
The timing record specifies hardware, software revision, parallelism,
batch size, numerical precision, preprocessing boundaries, and the number
of successful measured executions. Downloads and setup are recorded
separately under the reference timing protocol.

\FloatBarrier

\FloatBarrier
\section{Supplementary Results and Analyses}
\label{app:supplementary-results}

This section presents the experimental results and accompanying analyses.
The detailed tables report base and adapted results for each evaluated
model, binder category, and target conformation. Main-text values are rounded
from these same results, and changes are calculated before rounding.

\subsection{Structure-pool coverage and biological context}
\label{app:pool-context}
Figures~\ref{fig:flexevo_structure_pool_features} and
\ref{fig:flexevo_uniprot_translation_atlas} characterize the structural
and biological coverage of the source pool. The small-molecule subset
contains 19,715 PDB entries, and the UniProt analysis covers 1,774 mapped
targets. These counts describe the source pool at the structure-entry
and annotated-target levels, respectively. The category-specific benchmark
target counts appear in Appendix~\ref{app:evaluation-design};
Appendix~\ref{app:reproducibility} defines the metric-specific pairing,
coverage denominators, and family-level resampling procedure.

\begin{figure}[!htb]
    \centering
    \includegraphics[width=1\linewidth,height=0.70\textheight,keepaspectratio]{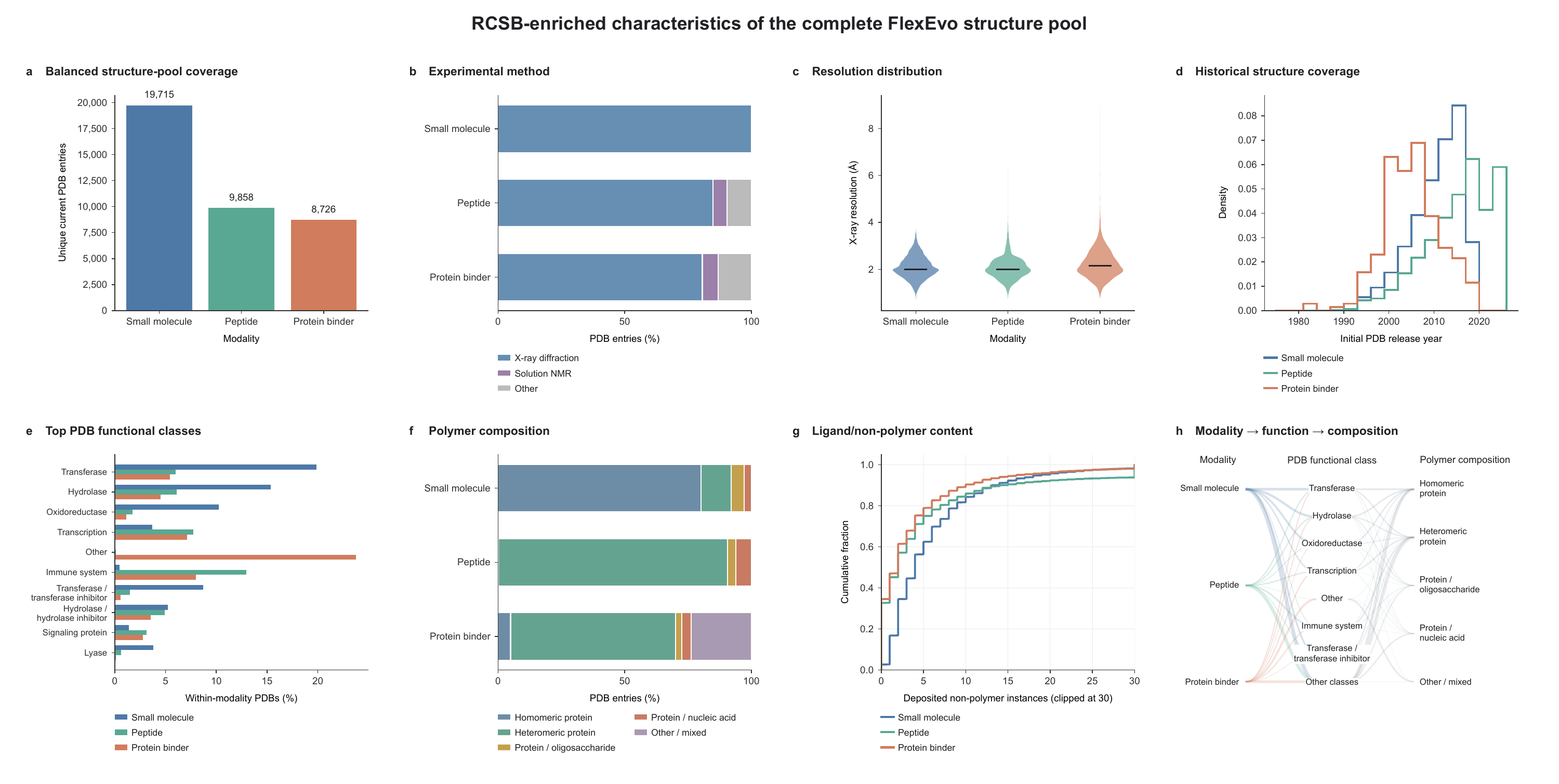}
   \caption{
\textbf{Structural and functional characteristics of the FlexEvo structure pool.}
RCSB Protein Data Bank (PDB) annotations characterize the small-molecule, peptide and protein-binder subsets.
\textbf{a,} Numbers of unique current PDB entries in each subset: 19,715 for small molecules, 9,858 for peptides and 8,726 for protein binders.
\textbf{b,} Proportions of structures determined by X-ray diffraction, solution nuclear magnetic resonance (NMR) and other experimental methods.
\textbf{c,} Distributions of reported resolution for X-ray structures.
\textbf{d,} Density distributions of initial PDB release years.
\textbf{e,} Prevalence of the displayed PDB functional classes, expressed as percentages of entries within each modality.
\textbf{f,} Polymer composition, categorized as homomeric protein, heteromeric protein, protein--oligosaccharide, protein--nucleic acid and other or mixed assemblies.
\textbf{g,} Cumulative distributions of deposited non-polymer instance counts per entry, displayed up to 30 instances.
\textbf{h,} Alluvial representation linking modality, PDB functional class and polymer composition.
Blue, teal and orange distinguish the small-molecule, peptide and protein-binder subsets, respectively, in panels a, c--e and g.
}
\label{fig:flexevo_structure_pool_features}
\end{figure}

\begin{figure}[!htb]
    \centering
    \includegraphics[width=1\linewidth,height=0.70\textheight,keepaspectratio]{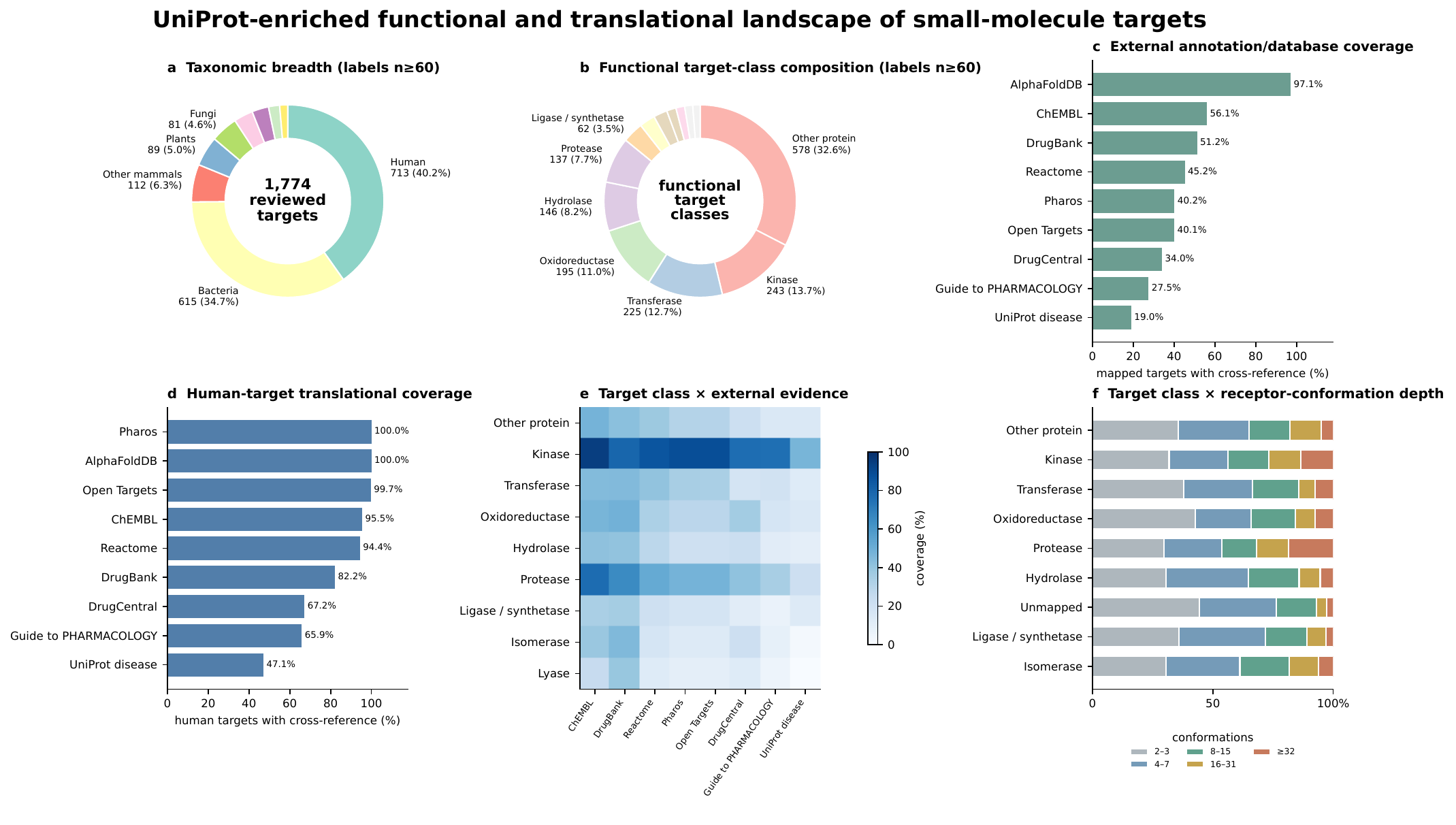}
    \caption{
\textbf{Taxonomic, functional and translational annotation landscape of small-molecule targets in FlexEvo.}
\textbf{a,} Taxonomic composition of 1,774 targets mapped to reviewed UniProt entries, including 713 human targets (40.2\%) and 615 bacterial targets (34.7\%).
\textbf{b,} Functional target-class composition. In panels a and b, categories containing at least 60 targets are labelled with their counts and percentages.
\textbf{c,} Coverage of external database cross-references and UniProt disease annotations among mapped targets, including AlphaFoldDB, ChEMBL, DrugBank, Reactome, Pharos, Open Targets, DrugCentral and Guide to PHARMACOLOGY.
\textbf{d,} Corresponding annotation coverage within the human-target subset.
\textbf{e,} Annotation coverage stratified by functional target class; darker blue indicates a higher percentage of targets with the corresponding cross-reference or annotation.
\textbf{f,} Within-class distributions of receptor-conformation counts per target, grouped into 2--3, 4--7, 8--15, 16--31 and $\geq 32$ conformations, with an additional unmapped category.
Database coverage quantifies the availability of target annotations and external cross-references.
}
\label{fig:flexevo_uniprot_translation_atlas}
\end{figure}

\subsection{Detailed results across nine binder categories}
\label{app:modality-results}
\textsc{FlexEvo} improves the Divergent-state primary endpoint across
all 27 source--category combinations in
Tables~\ref{tab:flexevo-app-small-molecule}--\ref{tab:flexevo-app-mrna}.
The primary endpoint is Vina for conventional small molecules and
nonpeptidic macrocycles, and DockQ for the remaining seven categories.
Each table reports Base and \textsc{FlexEvo} under Reference, Matched,
and Divergent conditions. Additional metrics characterize pose agreement,
molecular properties, model confidence, and diversity.
Colors encode the direction of change under each metric's stated
preference.
\paragraph{Model-averaged rerank-only controls.}
Tables~\ref{tab:rerank-only-vina-three-state}
and~\ref{tab:rerank-only-dockq-three-state} compare Base, Rerank-only,
and \textsc{FlexEvo} across Reference, Matched, and Divergent.
The conventional small-molecule Vina results are averaged across
DynamicFlow, FlexSBDD, and YuelDesign. The DockQ results are averaged
across BoltzGen, Chai-1, and Protenix. Rerank-only follows the
fixed-$G_0$, unchanged-source-pool protocol in
Appendices~\ref{app:evaluation-design} and~\ref{app:scoring}.
\textsc{FlexEvo} achieves the best model-averaged score in all three
states for both metric groups.
\begin{table}[!htbp]
\centering
\caption{Three-state rerank-only comparison for conventional
small-molecule drugs. Vina scores are reported on the common
three-arm evaluation set and aggregated across DynamicFlow,
FlexSBDD, and YuelDesign. Scores are in kcal/mol; lower values are
better. Rerank-only selects an original source candidate without
coordinate editing. Bold values indicate the best model-averaged
score in each state.}
\label{tab:rerank-only-vina-three-state}
\small
\begin{tabular}{@{}lccc@{}}
\toprule
Method & Reference $\downarrow$ & Matched $\downarrow$
& Divergent $\downarrow$ \\
\midrule
Base & $-8.260$ & $-7.440$ & $-6.540$ \\
Rerank-only & $-8.560$ & $-7.370$ & $-7.750$ \\
\textsc{FlexEvo} & $\mathbf{-9.740}$ & $\mathbf{-8.678}$
& $\mathbf{-8.698}$ \\
\bottomrule
\end{tabular}
\end{table}
\begin{table}[!htbp]
\centering
\caption{Three-state rerank-only comparison using DockQ.
DockQ scores are reported on the common three-arm evaluation set
and aggregated across BoltzGen, Chai-1, and Protenix.
Higher values are better. Rerank-only selects an original source
candidate without coordinate editing. Bold values indicate the
best model-averaged score in each state.}
\label{tab:rerank-only-dockq-three-state}
\small
\begin{tabular}{@{}lccc@{}}
\toprule
Method & Reference $\uparrow$ & Matched $\uparrow$
& Divergent $\uparrow$ \\
\midrule
Base & $0.6270$ & $0.3960$ & $0.2760$ \\
Rerank-only & $0.6630$ & $0.4213$ & $0.3372$ \\
\textsc{FlexEvo} & $\mathbf{0.8539}$ & $\mathbf{0.8614}$
& $\mathbf{0.8173}$ \\
\bottomrule
\end{tabular}
\end{table}
For conventional small molecules, \textsc{FlexEvo} improves the
model-averaged Vina score over Rerank-only by $1.180$, $1.308$,
and $0.948$~kcal/mol on Reference, Matched, and Divergent,
respectively. In the DockQ comparison, the corresponding gains
are $0.1909$, $0.4401$, and $0.4801$, with the largest gain on
Divergent. These comparisons demonstrate improved model-averaged
performance beyond source-pool reranking in both metric groups
across all three natural target states.
\paragraph{Category-balanced cross-conformation degradation.}
We quantified Reference-to-Divergent degradation using the primary metrics
in Tables~\ref{tab:flexevo-app-small-molecule}--\ref{tab:flexevo-app-mrna}.
For category $c$, source model $m$, and configuration
$v\in\{\mathrm{base},\mathrm{FlexEvo}\}$, let
$S_{cm,R}^{(v)}$ and $S_{cm,D}^{(v)}$ denote the Reference and Divergent scores.
We used negative mean Vina scores for small molecules and nonpeptidic
macrocycles, and mean DockQ for the remaining seven categories.
Thus, higher $S$ indicates better performance within each category.
We computed the relative degradation and its category-balanced mean as
\begin{align}
\label{eq:relative-state-degradation}
d_{cm}^{(v)}
&=100\%\times
\frac{S_{cm,R}^{(v)}-S_{cm,D}^{(v)}}{|S_{cm,R}^{(v)}|},\\
\label{eq:category-balanced-degradation}
\bar d^{(v)}
&=\frac{1}{9}\sum_{c=1}^{9}
\frac{1}{|\mathcal M_c|}\sum_{m\in\mathcal M_c}d_{cm}^{(v)},
\end{align}
where $\mathcal M_c$ contains all three source models reported for category $c$.
This includes Chai-1 for both linear and cyclic peptides, giving
27 source--category combinations in total.
Each configuration was normalized by its own Reference score, and all
Reference denominators were nonzero.
Negative degradation values indicate improved Divergent performance and
were retained without clipping to zero.
Ratios were calculated from the reported state-wise means before averaging
within categories and then equally across the nine categories.
This statistic summarizes category-balanced relative degradation
using the reported state-wise means.

For example, DynamicFlow small molecules had mean Vina scores of
$-6.6062$ and $-4.8519$~kcal/mol under Reference and Divergent, respectively,
giving $26.5554\%$ degradation.
After \textsc{FlexEvo}, the corresponding scores were $-9.1775$ and
$-7.7991$~kcal/mol, giving $15.0193\%$ degradation.
Across all nine categories and 27 source--category combinations,
the complete \textsc{FlexEvo} selection-and-adaptation workflow reduced
category-balanced mean relative degradation from $47.8376\%$ to
$4.3854\%$, a reduction of $43.4522$ percentage points.
Rounded to one decimal place, these values are $47.8\%$ and $4.4\%$.
This benchmark summary quantifies cross-conformation score retention
for the complete workflow relative to the source-native representative.
Tables~\ref{tab:rerank-only-vina-three-state}
and~\ref{tab:rerank-only-dockq-three-state} complement this benchmark
summary with model-averaged three-arm comparisons, quantifying
the additional Vina and DockQ improvements over source-pool reranking.

\FloatBarrier
\subsection{Joint cross-conformation comparisons}
\label{app:joint-comparisons}
Figures~\ref{fig:cross_conformation_small_molecules_peptides}--\ref{fig:cross_conformation_fragments_nucleic_acids}
place two target conditions on separate axes. Their joint margins distinguish
candidates meeting both plotted thresholds from those meeting only one.
Centroid shifts describe the displayed subsets of two Base and three
\textsc{FlexEvo} candidates, respectively.

\begin{figure}[!htbp]
    \centering
    \includegraphics[width=1\linewidth,height=0.60\textheight,keepaspectratio]{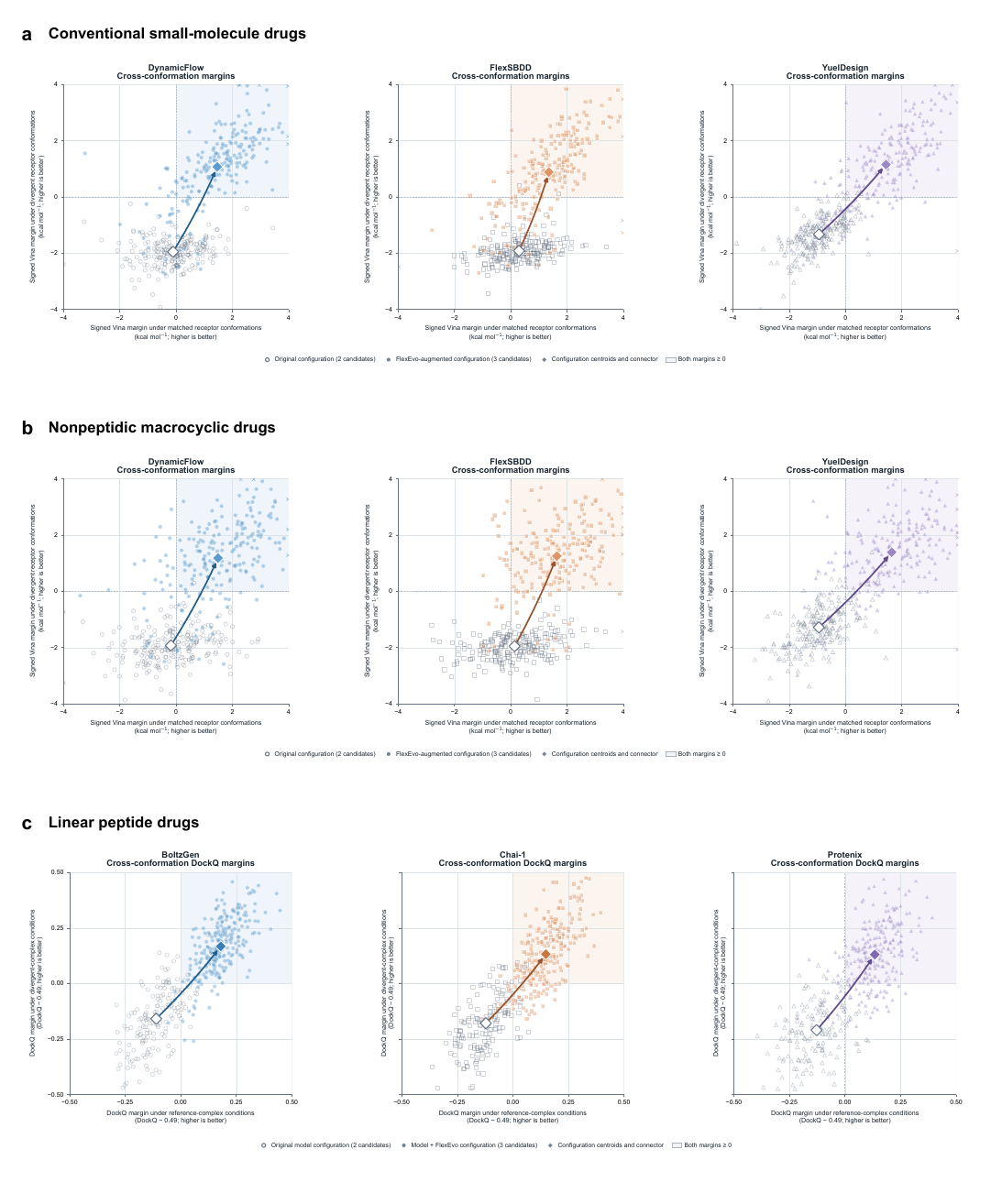}
    \caption{\textbf{Cross-conformation evaluation of conventional small molecules, nonpeptidic macrocycles and linear peptides.}
    \textbf{a}, Conventional small-molecule drugs.
    \textbf{b}, Nonpeptidic macrocyclic drugs.
    \textbf{c}, Linear peptide drugs.
    In \textbf{a} and \textbf{b}, columns show DynamicFlow, FlexSBDD and YuelDesign from left to right. The horizontal and vertical axes show signed Vina margins under matched and divergent receptor conformations, respectively (kcal~mol$^{-1}$; higher is better).
    In \textbf{c}, columns show BoltzGen, Chai-1 and Protenix from left to right. The horizontal and vertical axes show DockQ margins under reference-complex and divergent-complex conditions, respectively, defined as $\mathrm{DockQ}-0.49$.
    Grey open symbols and coloured points denote the original configuration (two candidates) and the FlexEvo-augmented configuration (three candidates), respectively.
    Diamond markers indicate configuration centroids, with connecting lines showing the displacement between configurations.
    Dashed lines mark zero margins, and the shaded upper-right quadrant indicates non-negative margins under both conditions.}
    \label{fig:cross_conformation_small_molecules_peptides}
\end{figure}

\begin{figure}[!htbp]
    \centering
    \includegraphics[width=1\linewidth,height=0.70\textheight,keepaspectratio]{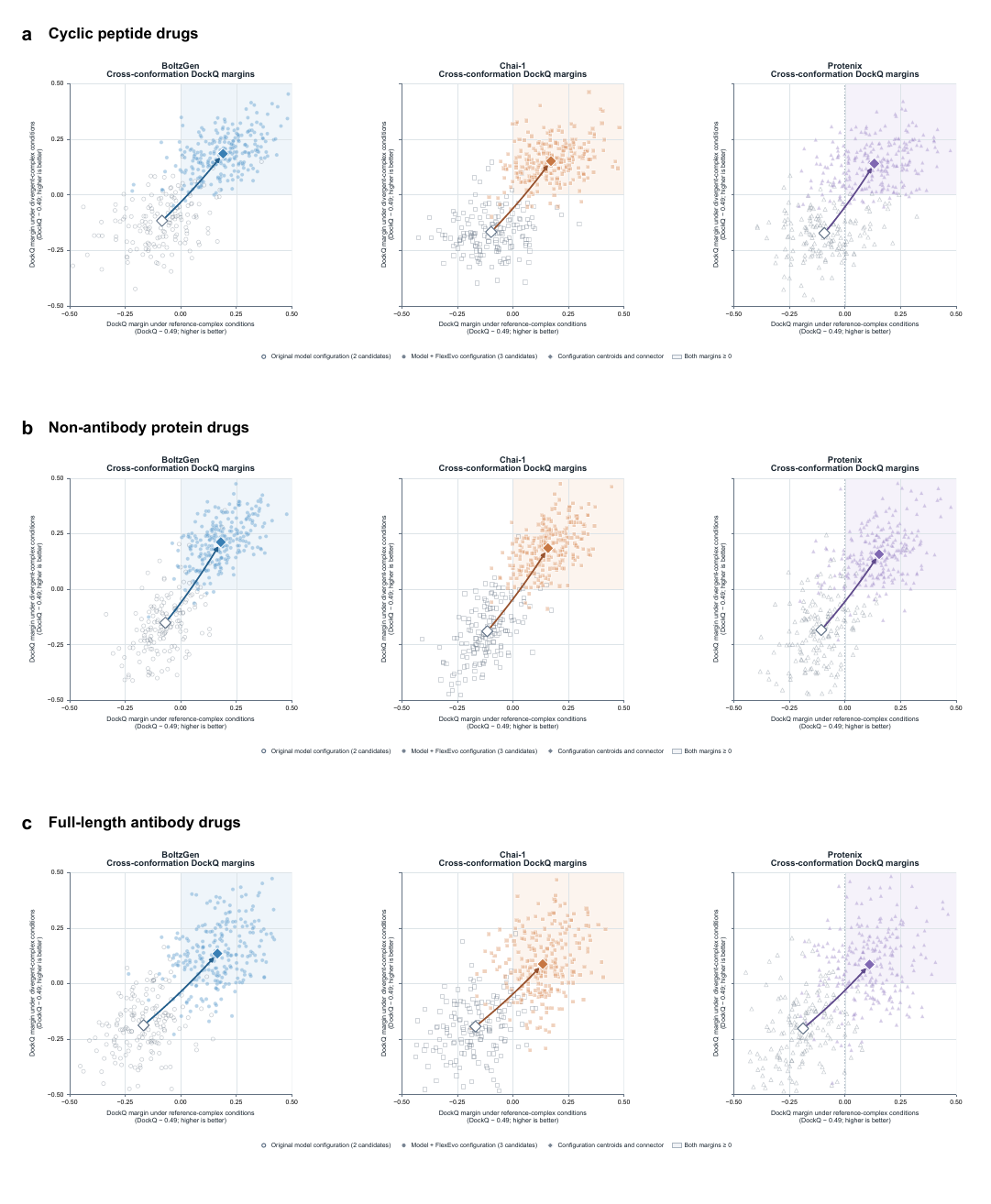}
    \caption{\textbf{Cross-conformation evaluation of cyclic peptides, non-antibody proteins and full-length antibodies.}
    \textbf{a}, Cyclic peptide drugs.
    \textbf{b}, Non-antibody protein drugs.
    \textbf{c}, Full-length antibody drugs.
    Columns show BoltzGen, Chai-1 and Protenix from left to right in each panel.
    The horizontal and vertical axes show DockQ margins under reference-complex and divergent-complex conditions, respectively.
    Margins are defined as $\mathrm{DockQ}-0.49$, with higher values indicating better structural agreement with the reference complex under the corresponding condition.
    Grey open symbols and coloured points denote the original model configuration (two candidates) and the model supplemented with FlexEvo (three candidates), respectively.
    Diamond markers indicate configuration centroids, with connecting lines showing the displacement between configurations.
    Dashed lines mark zero margins, corresponding to $\mathrm{DockQ}=0.49$.
    The shaded upper-right quadrant indicates $\mathrm{DockQ}\geq0.49$ under both conditions.}
    \label{fig:cross_conformation_peptides_proteins_antibodies}
\end{figure}

\begin{figure}[!htbp]
    \centering
    \includegraphics[width=1\linewidth,height=0.70\textheight,keepaspectratio]{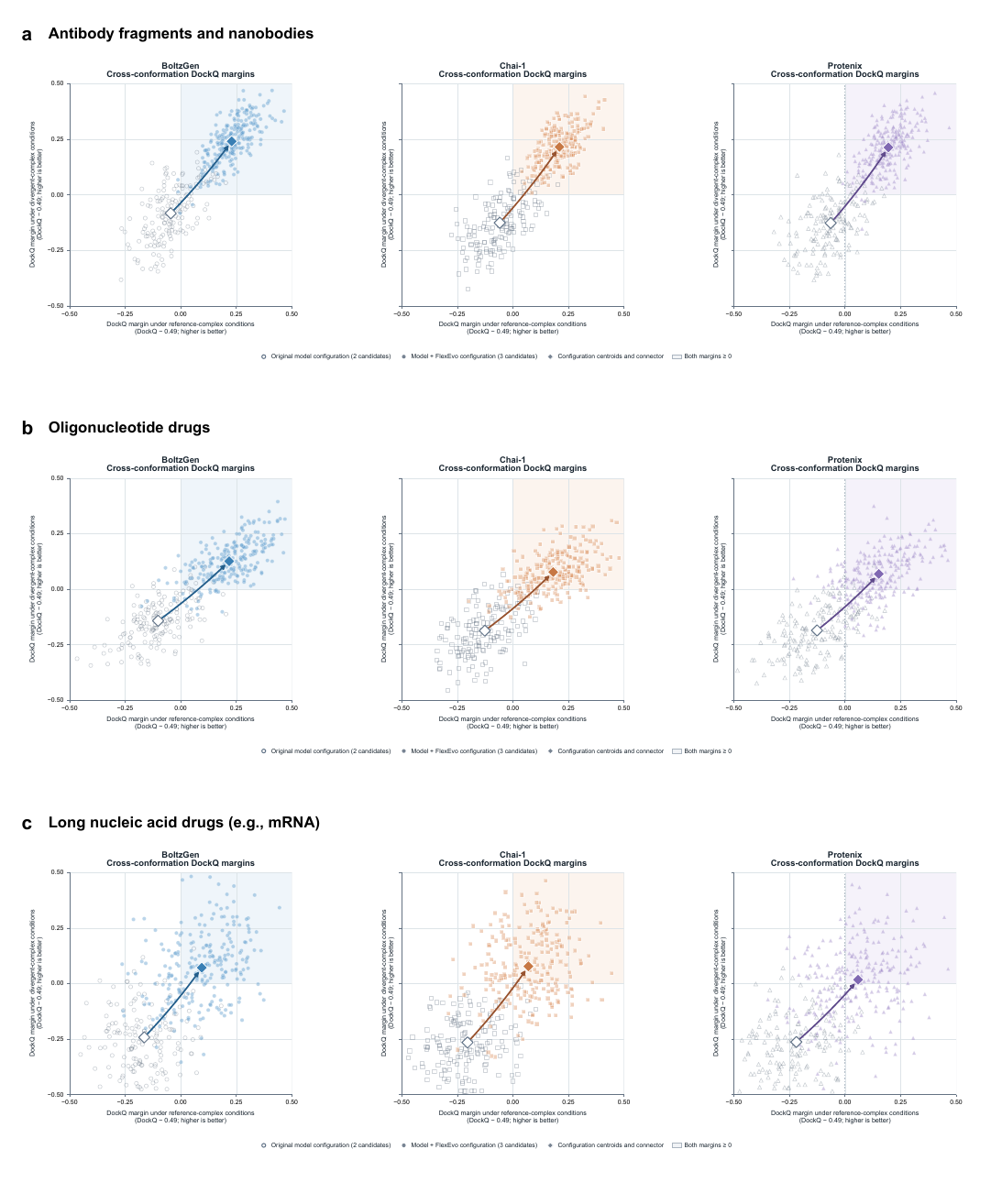}
    \caption{\textbf{Cross-conformation evaluation of antibody fragments, nanobodies and nucleic acid drugs.}
    \textbf{a}, Antibody fragments and nanobodies.
    \textbf{b}, Oligonucleotide drugs.
    \textbf{c}, Long nucleic acid drugs, including mRNA.
    Columns show BoltzGen, Chai-1 and Protenix from left to right in each panel.
    The horizontal and vertical axes show DockQ margins under reference-complex and divergent-complex conditions, respectively.
    Margins are defined as $\mathrm{DockQ}-0.49$, with higher values indicating better structural agreement with the reference complex under the corresponding condition.
    Grey open symbols and coloured points denote the original model configuration (two candidates) and the model supplemented with FlexEvo (three candidates), respectively.
    Diamond markers indicate configuration centroids, with connecting lines showing the displacement between configurations.
    Dashed lines mark zero margins, corresponding to $\mathrm{DockQ}=0.49$.
    The shaded upper-right quadrant indicates $\mathrm{DockQ}\geq0.49$ under both conditions.}
    \label{fig:cross_conformation_fragments_nucleic_acids}
\end{figure}
\FloatBarrier
\subsection{Alternative optimizers and component ablations}
\label{app:optimizer-ablation}
Figure~\ref{fig:optimization_method_comparison} compares six optimization
strategies under the displayed HV-AUC and B-pocket summaries. The two
criteria should be read separately because their scales differ. The ranking summarizes the reported scores. The budget-matching protocol
specifies initialization, objective-evaluation count, stopping rule,
hypervolume reference point, and repetition procedure for each optimizer.
\paragraph{Comparison with Best-of-$N$.}
Table~\ref{tab:bestofn-proxy-comparison} summarizes the Best-of-$N$
and \textsc{FlexEvo} values displayed in
Figure~\ref{fig:optimization_method_comparison}.
The reported HV-AUC is $8.482$ for \textsc{FlexEvo} and $0.959$
for Best-of-$N$, corresponding to an approximately $8.84$-fold ratio.
The corresponding B-pocket top-10 task scores are $-31.409$ and
$-32.497$, giving a $1.088$-point advantage for \textsc{FlexEvo}.

\begin{table}[!htbp]
\centering
\caption{Best-of-$N$ comparison using the values displayed in
Figure~\ref{fig:optimization_method_comparison}.
Higher values are preferred for both metrics.}
\label{tab:bestofn-proxy-comparison}
\small
\begin{tabular}{lrr}
\toprule
Method & HV-AUC $\uparrow$ & B-pocket top-10 $\uparrow$ \\
\midrule
Best-of-$N$ & 0.959 & $-32.497$ \\
\textsc{FlexEvo} & 8.482 & $-31.409$ \\
\bottomrule
\end{tabular}
\end{table}
Table~\ref{tab:flexevo-ablation-multimetric} displays the full method and
three component-removal variants. Component-removal comparisons show that the full method achieves the best Divergent-state primary metric across all nine binder categories. The table includes trade-offs: for example, the
small-molecule Reference Vina is slightly more favorable without FlexBox
($-8.03$ versus $-8.01$), whereas the full method has the more favorable
Divergent Vina ($-6.66$ versus $-5.88$). 

\begin{figure}[!htbp]
    \centering
    \includegraphics[width=1\linewidth,height=0.70\textheight,keepaspectratio]{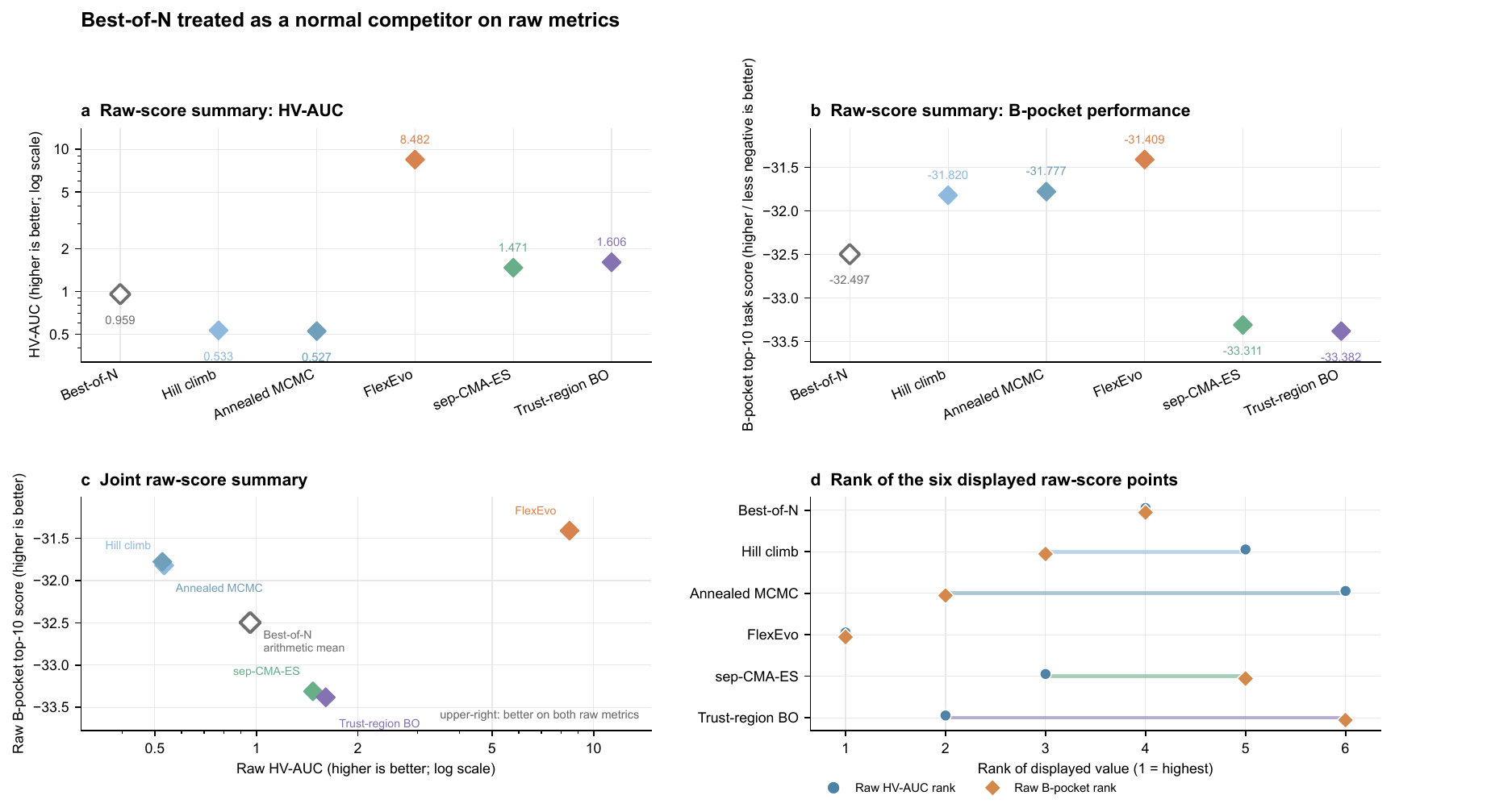}
\caption{
\textbf{Comparison of FlexEvo and alternative optimization strategies using raw performance metrics.}
Six methods are compared: Best-of-$N$, hill climbing, annealed Markov chain Monte Carlo (MCMC), FlexEvo, separable covariance matrix adaptation evolution strategy (sep-CMA-ES) and trust-region Bayesian optimization (BO).
\textbf{a,} Hypervolume area under the curve (HV-AUC), displayed on a logarithmic scale; higher values indicate better performance.
\textbf{b,} B-pocket top-10 task scores, for which higher (less negative) values indicate better performance.
\textbf{c,} Joint comparison of the two metrics, with the upper-right region indicating better performance on both. The open grey diamond denotes the Best-of-$N$ arithmetic-mean summary.
\textbf{d,} Rankings of the six displayed values for HV-AUC (blue circles) and B-pocket score (orange diamonds), with rank 1 indicating the highest value. Connecting lines link the two metric-specific ranks for each method.
FlexEvo achieves the highest displayed value on both metrics, with an HV-AUC of 8.482 and a B-pocket top-10 score of $-31.409$.
}
\label{fig:optimization_method_comparison}
\end{figure}

\subsection{Runtime and state-wise summary profiles}
\label{app:runtime-profiles}
Figure~\ref{fig:flexevo-compute-cost} separates source generation from
the additional adaptation cost under a common timing boundary.
Across the nine workflows, \textsc{FlexEvo} adds 1.393--3.141~min per
sample, yielding total generation-plus-adaptation times of
6.861--12.966~min per sample.
The reference search schedule contains 72 genome evaluations, with
execution and timing definitions given in Appendix~\ref{app:reproducibility}.

\paragraph{Time-resolved adaptation and structural progression.}
Figure~\ref{fig:optimization_dynamics} links the displayed adaptation
trajectories to elapsed computation time.
In panel~(a), the Vina-defined high-affinity fraction reaches $60.0\%$
at approximately two minutes and remains near that level late in the run.
Panels~(b,c) show an overall decrease in on-target pAE and an increase in
off-target pAE over the displayed $2$--$3$~min intervals.
In the ligand example in panels~(d--f), Vina improves from $-7.900$
through $-10.491$ to $-12.666$~kcal/mol across the early, intermediate,
and late structures. These observations demonstrate progressive docking-score
improvement in the illustrated structural sequence and favorable
interaction-confidence trends in the displayed runs.

\paragraph{State-wise summary profiles.}

Figures~\ref{fig:small_molecule_receptor_performance}--\ref{fig:protein_binder_receptor_performance}
summarize supplied percentage profiles across Reference, Matched, and
Divergent states. The displayed ``performance'' percentage is a
figure-level summary, distinct from DockQ, high-affinity fraction, and
LRMSD success. Metric-specific comparisons use their own definitions,
selection rules, and denominators.

\begin{figure}[!htbp]
    \centering
    \includegraphics[width=1\linewidth,height=0.70\textheight,keepaspectratio]{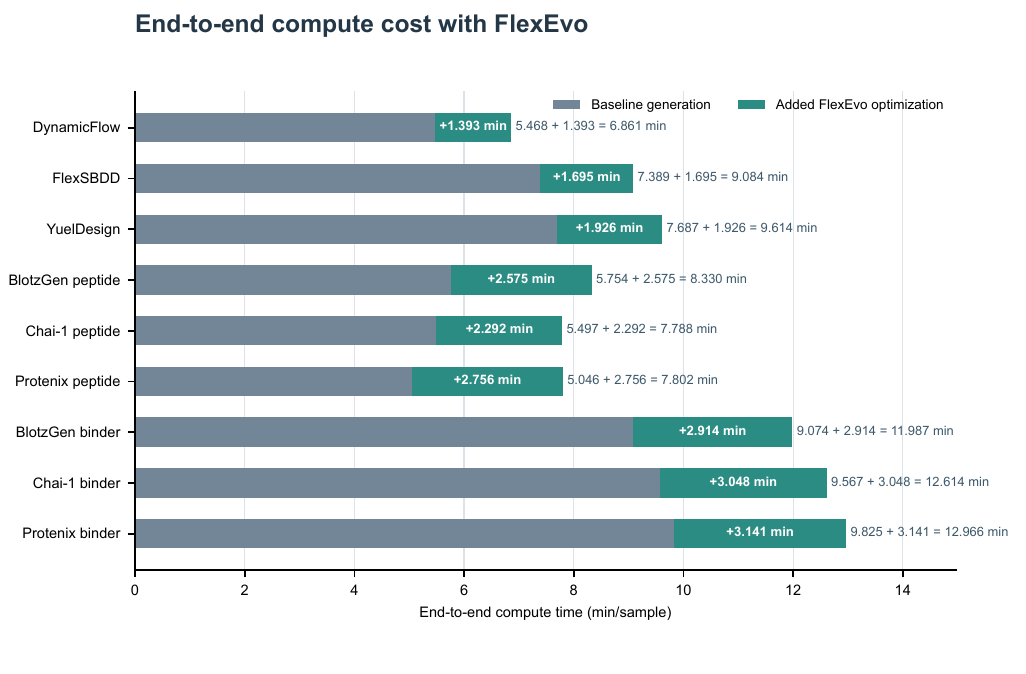}
    \caption{End-to-end computational cost of incorporating FlexEvo into generation workflows. Per-sample compute time is shown for DynamicFlow, FlexSBDD, YuelDesign, and the peptide and binder workflows of BoltzGen, Chai-1 and Protenix. Stacked bars separate baseline generation time (grey) from the additional time required for FlexEvo optimization (teal). Labels within the teal segments indicate the added optimization time, and annotations beside each bar report baseline, optimization and total times. Across the nine workflows, FlexEvo adds 1.393--3.141 min per sample, resulting in total compute times of 6.861--12.966 min per sample.}
    \label{fig:flexevo-compute-cost}
\end{figure}

\begin{figure}[!htbp]
    \centering
    \includegraphics[width=\linewidth,height=0.70\textheight,keepaspectratio]{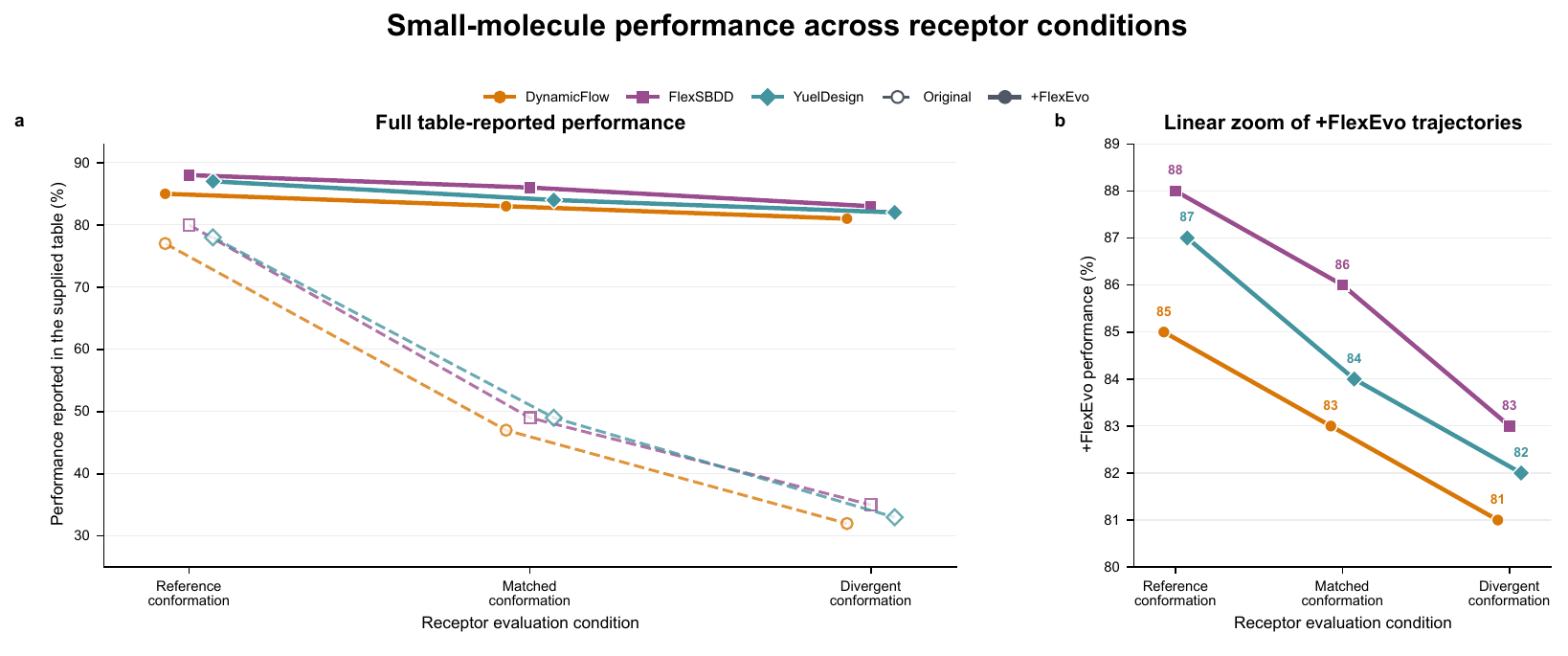}
    \caption{
        \textbf{Small-molecule performance across receptor conformations with and without FlexEvo.}
        \textbf{a,} Comparison of DynamicFlow (orange circles), FlexSBDD (purple squares) and YuelDesign (teal diamonds) under reference, matched and divergent receptor conformations. Dashed lines with open markers indicate the original methods; solid lines with filled markers indicate their FlexEvo-optimized counterparts.
        \textbf{b,} Expanded view of the FlexEvo trajectories on a linear scale, with labels indicating performance values (\%). FlexEvo-optimized DynamicFlow, FlexSBDD and YuelDesign retain performance of 81\%, 83\% and 82\%, respectively, under the divergent condition, with decreases of 4--5 percentage points relative to the reference condition.
    }
    \label{fig:small_molecule_receptor_performance}
\end{figure}

\begin{figure}[!htbp]
    \centering
    \includegraphics[width=\linewidth,height=0.70\textheight,keepaspectratio]{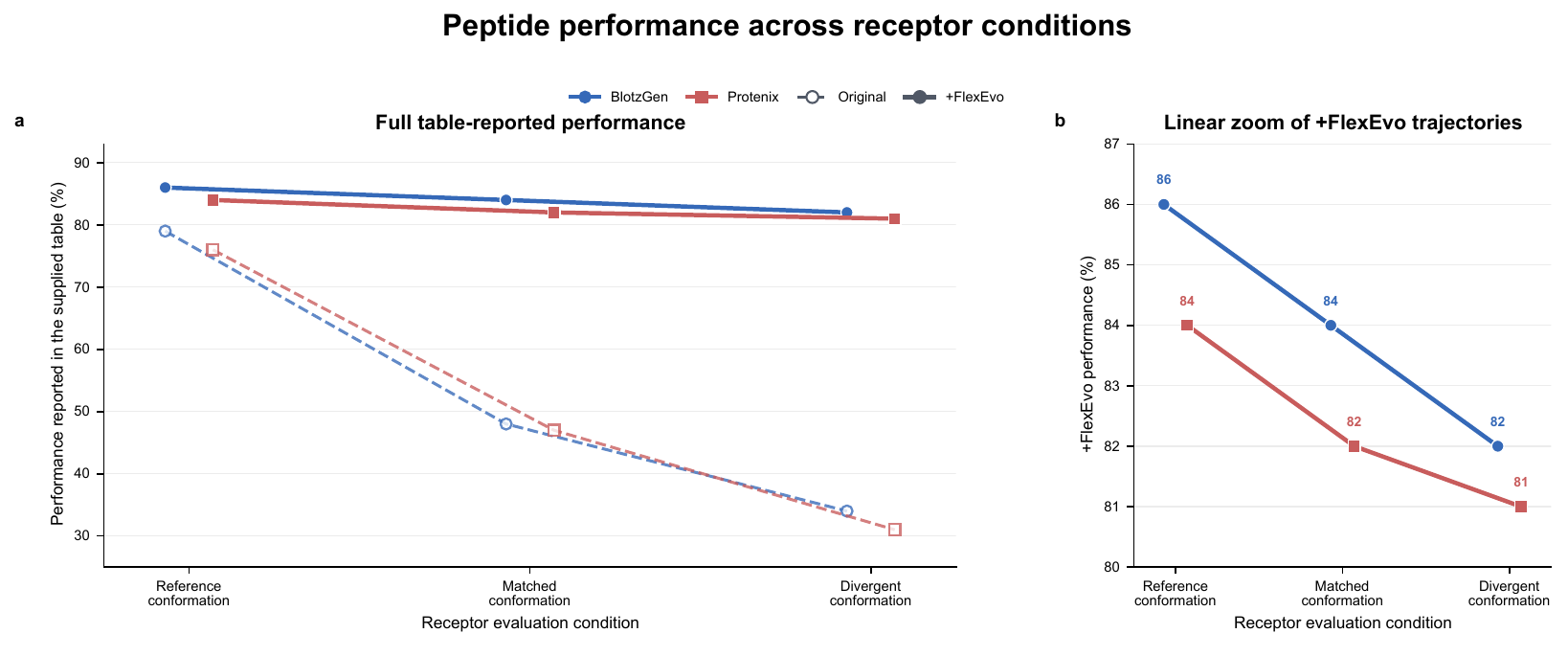}
    \caption{
        \textbf{Peptide performance across receptor conformations with and without FlexEvo.}
        \textbf{a,} Comparison of BoltzGen (blue circles) and Protenix (red squares) under reference, matched and divergent receptor conformations. Dashed lines with open markers indicate the original methods; solid lines with filled markers indicate their FlexEvo-optimized counterparts.
        \textbf{b,} Expanded view of the FlexEvo trajectories on a linear scale, with labels indicating performance values (\%). BoltzGen achieves 86\%, 84\% and 82\%, and Protenix achieves 84\%, 82\% and 81\%, respectively, across the three conditions. FlexEvo reduces the decline in performance associated with changes in receptor conformation.
    }
    \label{fig:peptide_receptor_performance}
\end{figure}

\begin{figure}[!htbp]
    \centering
    \includegraphics[width=\linewidth,height=0.70\textheight,keepaspectratio]{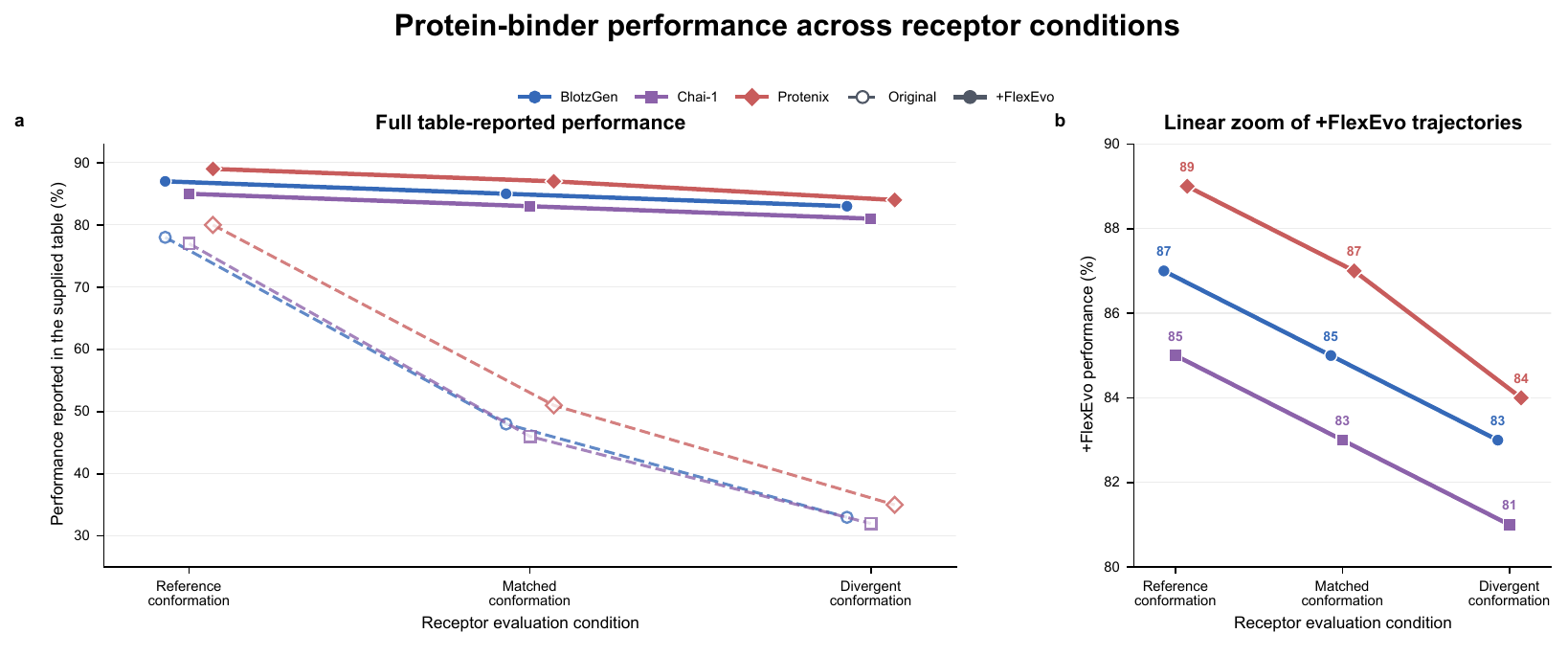}
    \caption{
        \textbf{Protein-binder performance across receptor conformations with and without FlexEvo.}
        \textbf{a,} Comparison of BoltzGen (blue circles), Chai-1 (purple squares) and Protenix (red diamonds) under reference, matched and divergent receptor conformations. Dashed lines with open markers indicate the original methods; solid lines with filled markers indicate their FlexEvo-optimized counterparts.
        \textbf{b,} Expanded view of the FlexEvo trajectories on a linear scale, with labels indicating performance values (\%). Performance decreases from 87\% to 83\% for BoltzGen, from 85\% to 81\% for Chai-1 and from 89\% to 84\% for Protenix between the reference and divergent conditions, showing smaller declines than those observed for the original methods.
    }
    \label{fig:protein_binder_receptor_performance}
\end{figure}
\FloatBarrier
\subsection{Similarity to a sampled reference library}
\label{app:similarity-analysis}
Table~\ref{tab:flexevo_reference_similarity} and
Figures~\ref{fig:similarity_group_1}--\ref{fig:similarity_group_3}
characterize generated molecules relative to a sampled reference library.
Each scenario comprises 470 pairwise comparisons between 10 generated
and 47 reference entries, with each entry participating in multiple
comparisons. Low similarity indicates limited resemblance within the
sampled library under the molecular representation used.

\begin{table}[htbp]
    \centering
    \caption{
    Reference-library similarity across nine binder categories.
    For each category, 10 evaluated candidates are compared with 47 randomly
    sampled public reference entries. Maximum and median similarity are reported
    over the resulting 470 pairwise comparisons using the modality-specific
    similarity definition.
    }
    \label{tab:flexevo_reference_similarity}
    \small
    \begin{tabular}{p{0.60\linewidth}cc}
        \hline
        Evaluation scenario & Maximum $T_c$ & Median $T_c$ \\
        \hline
        Antibody fragments and nanobodies & 0.222 & 0.052 \\
        Conventional small-molecule drugs & 0.159 & 0.052 \\
        Cyclic peptide drugs & 0.262 & 0.055 \\
        Full-length antibody drugs & 0.245 & 0.059 \\
        Linear peptide drugs & 0.226 & 0.064 \\
        Long nucleic acid drugs, including mRNA & 0.254 & 0.071 \\
        Non-antibody protein drugs & 0.270 & 0.052 \\
        Nonpeptidic macrocyclic drugs & 0.145 & 0.044 \\
        Oligonucleotide drugs & 0.314 & 0.039 \\
        \hline
    \end{tabular}
    \par\smallskip
    \begin{minipage}{\linewidth}
        \footnotesize
        $T_c$, Tanimoto coefficient. Values summarize fingerprint
        similarity to the sampled reference molecules under the
        specified molecular representation.
    \end{minipage}
\end{table}

\begin{figure}[!htbp]
    \centering
    \includegraphics[width=0.75\linewidth,height=0.70\textheight,keepaspectratio]{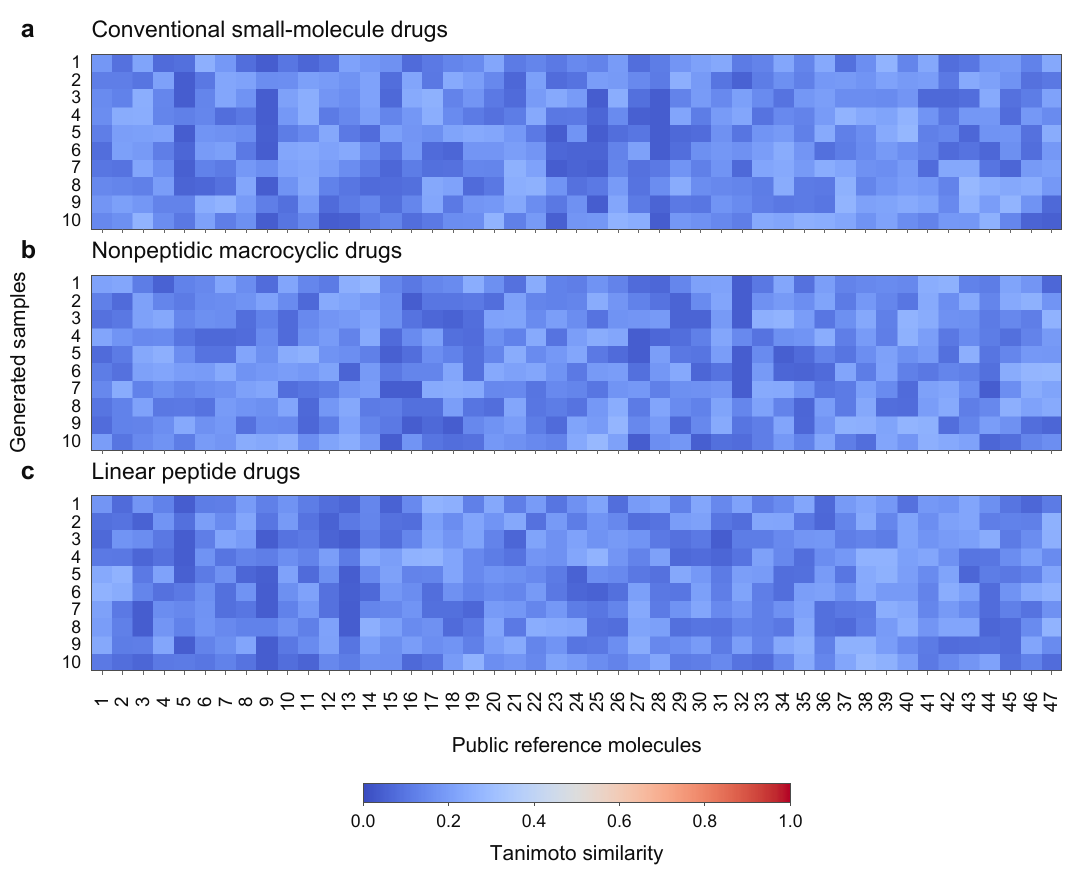}
    \caption{Tanimoto similarity between generated samples and public reference molecules for three drug modalities. \textbf{a}, Conventional small-molecule drugs. \textbf{b}, Nonpeptidic macrocyclic drugs. \textbf{c}, Linear peptide drugs. Each heatmap compares 10 generated samples (rows) with 47 public reference molecules (columns). Each cell represents the Tanimoto similarity of the corresponding pair, with colours indicating values from 0 to 1. Higher values indicate greater similarity, and all panels share the same colour scale.}
    \label{fig:similarity_group_1}
\end{figure}

\begin{figure}[!htbp]
    \centering
    \includegraphics[width=0.75\linewidth,height=0.70\textheight,keepaspectratio]{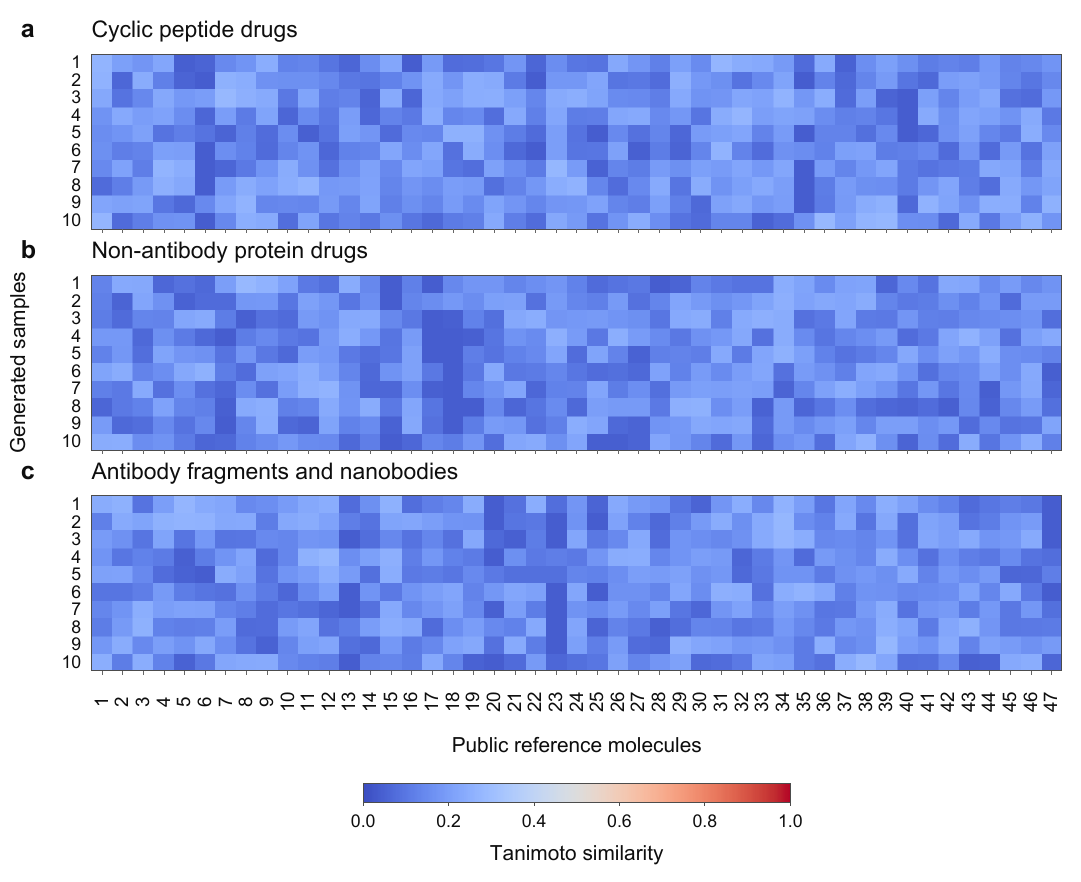}
    \caption{Tanimoto similarity between generated samples and public reference molecules for cyclic peptides and protein-based drug modalities. \textbf{a}, Cyclic peptide drugs. \textbf{b}, Non-antibody protein drugs. \textbf{c}, Antibody fragments and nanobodies. Each heatmap compares 10 generated samples (rows) with 47 public reference molecules (columns). Each cell represents the Tanimoto similarity of the corresponding pair, with colours indicating values from 0 to 1. Higher values indicate greater similarity, and all panels share the same colour scale.}
    \label{fig:similarity_group_2}
\end{figure}

\begin{figure}[!htbp]
    \centering
    \includegraphics[width=0.75\linewidth,height=0.70\textheight,keepaspectratio]{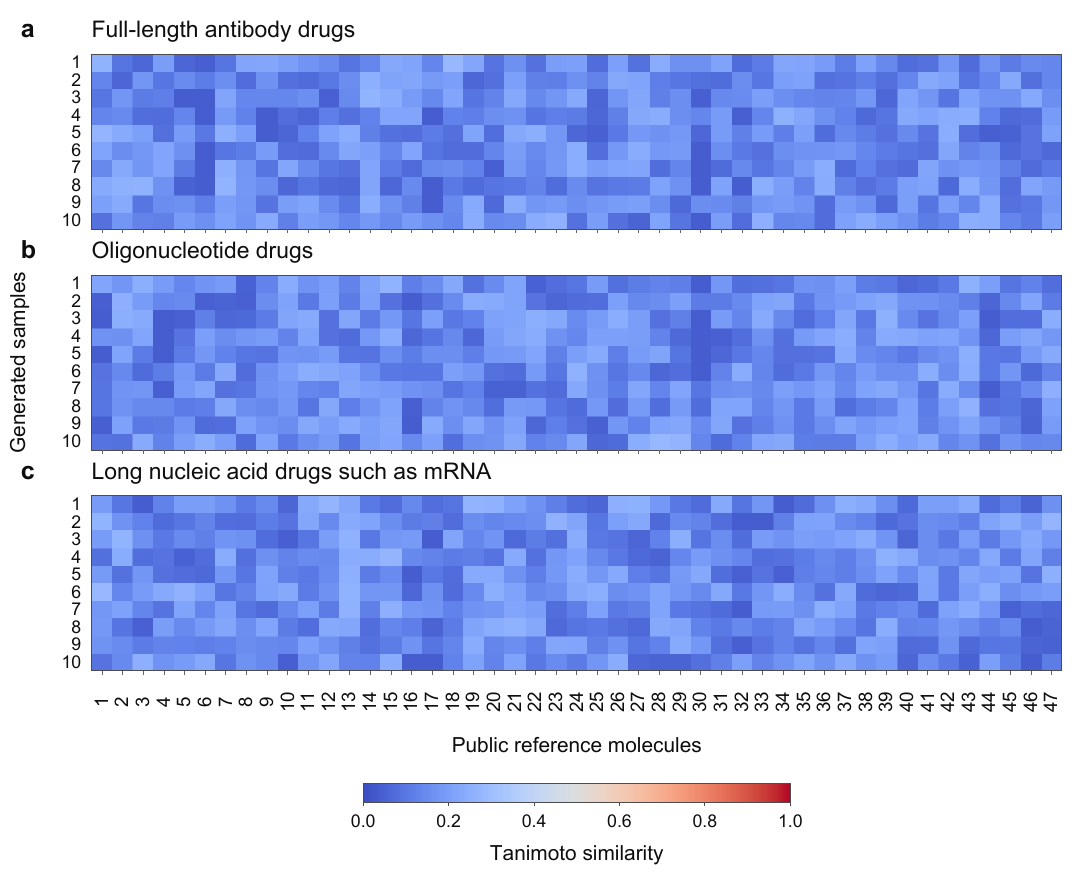}
    \caption{Tanimoto similarity between generated samples and public reference molecules for antibody and nucleic acid drug modalities. \textbf{a}, Full-length antibody drugs. \textbf{b}, Oligonucleotide drugs. \textbf{c}, Long nucleic acid drugs, such as mRNA. Each heatmap compares 10 generated samples (rows) with 47 public reference molecules (columns). Each cell represents the Tanimoto similarity of the corresponding pair, with colours indicating values from 0 to 1. Higher values indicate greater similarity, and all panels share the same colour scale.}
    \label{fig:similarity_group_3}
\end{figure}
\FloatBarrier
\subsection{Structural examples and physicochemical context}
\label{app:reference-illustrations}
Figures~\ref{fig:small_molecule_macrocycle_linear_peptide}--\ref{fig:antibody_fragments_oligonucleotides_long_rna}
illustrate the architectures of the nine binder categories using publicly
deposited reference structures and explicitly identified RNA schematics.
Figure~\ref{fig:physicochemical_profiles} summarizes physicochemical
descriptors for 116 reference structures arranged in two illustrative
sets of 58. Guideline lines provide context for interpreting the
descriptor profiles.

\begin{figure}[!htbp]
    \centering
    \includegraphics[width=1\linewidth,height=0.70\textheight,keepaspectratio]{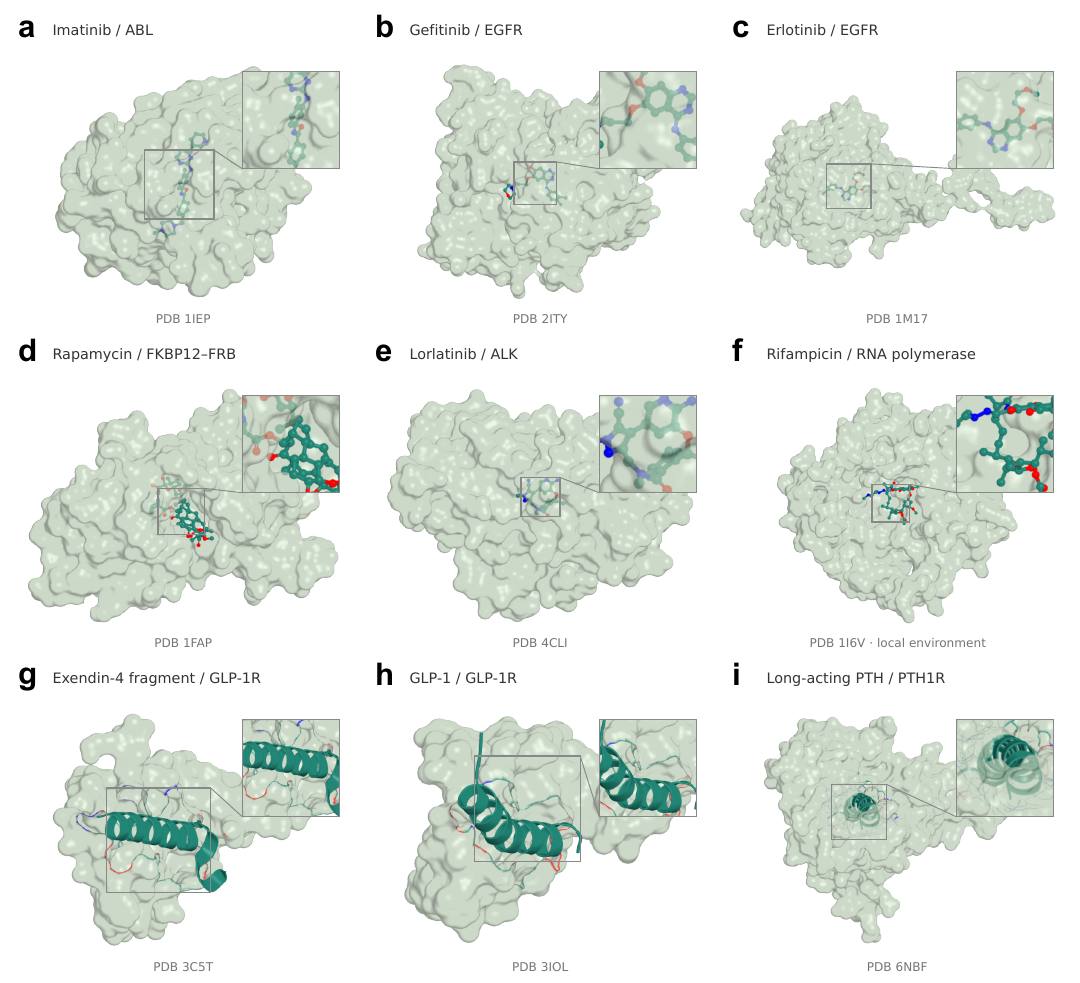}
    \caption{
    \textbf{Representative structures of conventional small molecules, nonpeptidic macrocycles and linear peptides.}
    \textbf{a--c}, Conventional small-molecule drugs: imatinib bound to ABL (PDB: 1IEP), gefitinib bound to EGFR (2ITY), and erlotinib bound to EGFR (1M17).
    \textbf{d--f}, Nonpeptidic macrocyclic drugs: rapamycin in the FKBP12--FRB complex (1FAP), lorlatinib bound to ALK (4CLI), and rifampicin bound to RNA polymerase (1I6V).
    Panel \textbf{f} shows the local protein environment surrounding rifampicin.
    \textbf{g--i}, Linear peptide examples: an exendin-4 fragment bound to the GLP-1 receptor extracellular domain (3C5T), GLP-1 bound to the same domain (3IOL), and a long-acting parathyroid hormone analogue bound to PTH1R (6NBF).
    Target proteins are shown as pale-green surfaces, with bound molecules represented as sticks or cartoons.
    Grey boxes identify regions reproduced in the enlarged insets.
    Structures are rendered from deposited coordinates and independently scaled.
    }
    \label{fig:small_molecule_macrocycle_linear_peptide}
\end{figure}

\begin{figure}[!htbp]
    \centering
    \includegraphics[width=1\linewidth,height=0.70\textheight,keepaspectratio]{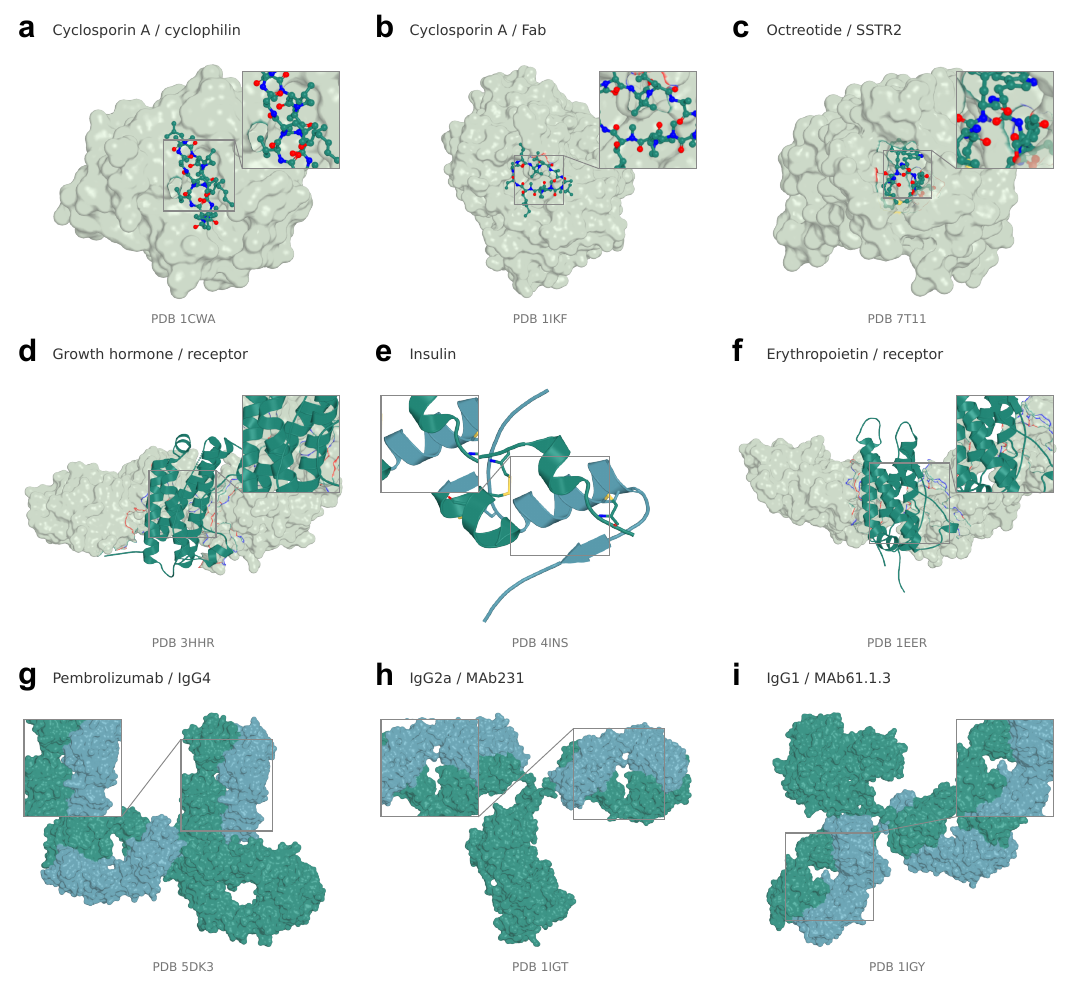}
    \caption{
    \textbf{Representative structures of cyclic peptides, non-antibody proteins and full-length antibodies.}
    \textbf{a--c}, Cyclic peptide examples: cyclosporin A bound to cyclophilin A (PDB: 1CWA), cyclosporin A bound to an antibody Fab fragment (1IKF), and octreotide bound to somatostatin receptor 2 (7T11).
    \textbf{d--f}, Non-antibody protein examples: human growth hormone bound to its receptor (3HHR), porcine insulin (4INS), and human erythropoietin bound to its receptor (1EER).
    \textbf{g--i}, Full-length antibody examples: pembrolizumab IgG4 (5DK3), IgG2a MAb231 (1IGT), and IgG1 MAb61.1.3 (1IGY).
    The latter two antibodies illustrate intact immunoglobulin architectures.
    Where present, target proteins are shown as pale-green surfaces.
    Protein ligands are shown as cartoons, whereas intact antibodies are shown as surfaces with heavy and light chains distinguished by colour.
    Grey boxes identify regions reproduced in the enlarged insets.
    Structures are rendered from deposited coordinates and independently scaled.
    }
    \label{fig:cyclic_peptide_protein_full_length_antibody}
\end{figure}

\begin{figure}[!htbp]
    \centering
    \includegraphics[width=1\linewidth,height=0.70\textheight,keepaspectratio]{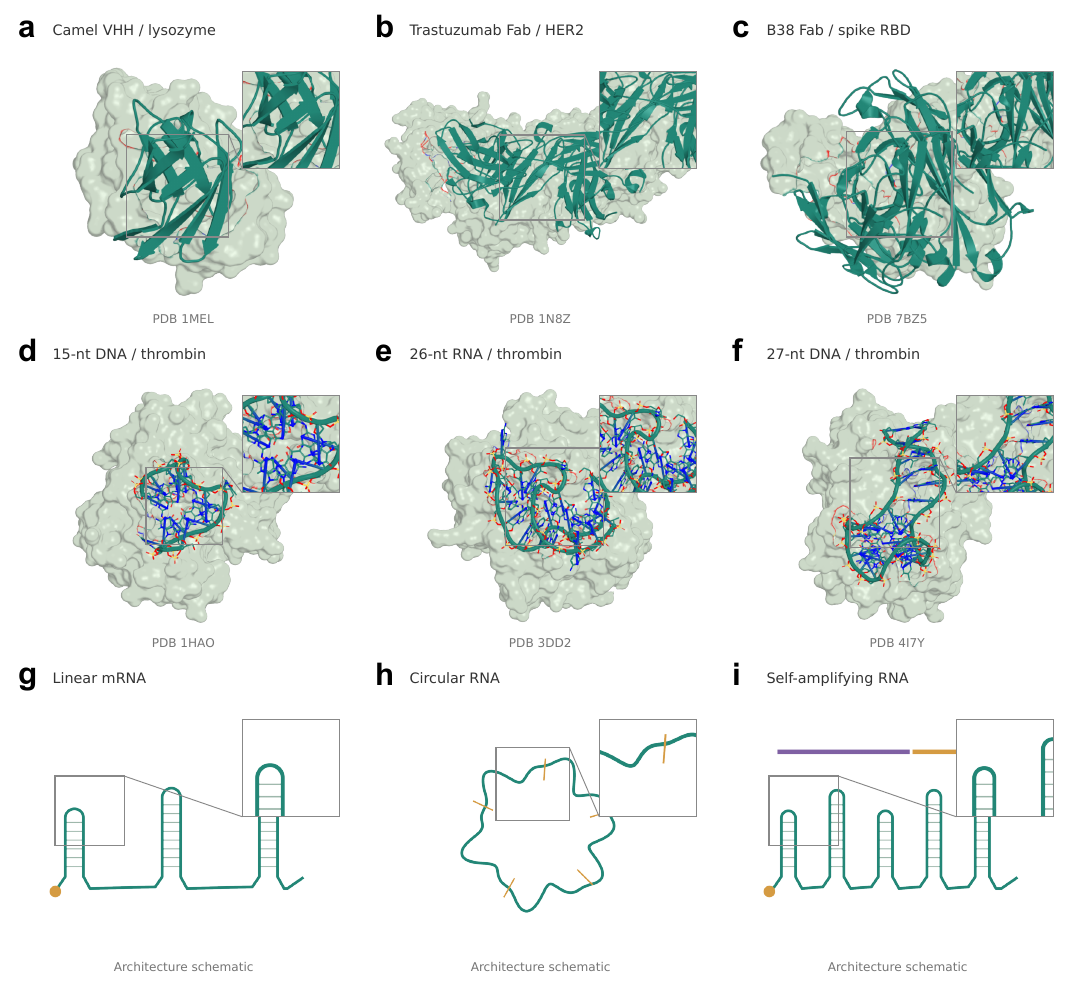}
    \caption{
    \textbf{Representative antibody-fragment and oligonucleotide structures, and long-RNA architectures.}
    \textbf{a--c}, Antibody fragments and nanobodies: a camel single-domain antibody bound to lysozyme (PDB: 1MEL), trastuzumab Fab bound to HER2 (1N8Z), and B38 Fab bound to the SARS-CoV-2 spike receptor-binding domain (7BZ5).
    \textbf{d--f}, Oligonucleotide examples: a 15-nucleotide DNA aptamer (1HAO), a 26-nucleotide RNA aptamer (3DD2), and a 27-nucleotide DNA aptamer (4I7Y), each bound to thrombin.
    These aptamers serve as structural representatives of the oligonucleotide modality.
    Target proteins are shown as pale-green surfaces, with antibody fragments and nucleic acids represented as cartoons or sticks.
    \textbf{g--i}, Schematic architectures of linear mRNA, circular RNA and self-amplifying RNA, respectively.
    RNA schematics illustrate molecular organization without specifying sequence, chain length or experimentally determined folding.
    Grey boxes identify regions reproduced in the enlarged insets.
    Molecular views are independently scaled.
    }
    \label{fig:antibody_fragments_oligonucleotides_long_rna}
\end{figure}

\begin{figure}[!htb]
    \centering
    \includegraphics[width=1\linewidth,height=0.70\textheight,keepaspectratio]{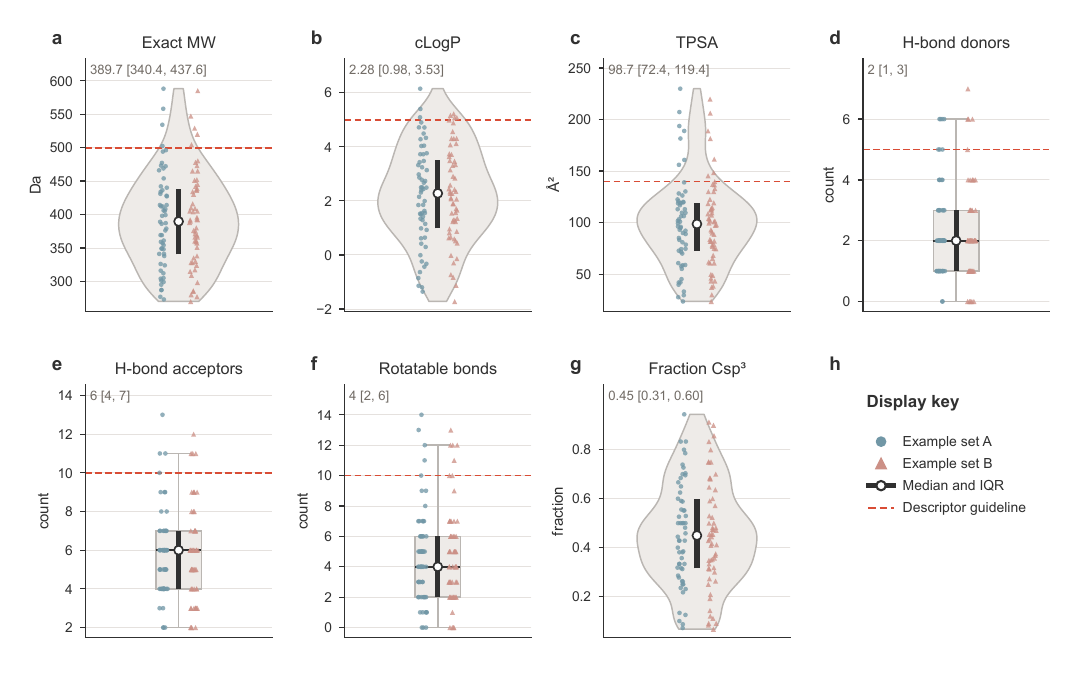}
    \caption{
\textbf{Physicochemical profiles of selected reference molecules.}
The figure summarizes seven molecular descriptors for 116 reference structures selected for illustration and divided into two example sets ($n=58$ each), shown as blue circles and terracotta triangles.
\textbf{a--g,} Exact molecular mass (Da), calculated octanol--water partition coefficient (cLogP), topological polar surface area (TPSA; \AA$^2$), hydrogen-bond donor count, hydrogen-bond acceptor count, rotatable bond count and the fraction of carbon atoms with $sp^3$ hybridization, respectively.
Grey violins show pooled distributions of continuous descriptors; boxes show the first-to-third-quartile intervals for count descriptors, with whiskers extending to the most extreme observations within 1.5 interquartile ranges (IQRs) of the box boundaries.
Open black circles and thick black intervals indicate pooled medians and IQRs. Numerical annotations report median [Q1, Q3], and each coloured marker represents one molecule.
Red dashed lines mark commonly used physicochemical guideline upper bounds: molecular mass of 500 Da, cLogP of 5, TPSA of 140 \AA$^2$, 5 hydrogen-bond donors, 10 hydrogen-bond acceptors and 10 rotatable bonds.
For all six descriptors with reference lines, the pooled median and upper quartile lie below the corresponding guideline, although individual molecules exceed these values.
The median fraction of $sp^3$-hybridized carbon atoms is 0.45, with an IQR of 0.31--0.60.
\textbf{h,} Key to the graphical encodings.
}
\label{fig:physicochemical_profiles}
\end{figure}
\FloatBarrier
\subsection{Multimetric profiles}
\label{app:radar-profiles}
Figures~\ref{fig:radar_macrocycle_oligonucleotide_antibody}--\ref{fig:radar_nucleic_acids_small_molecules_proteins}
visualize metric trade-offs using common normalization bounds across
the 18 method--state combinations within each modality. Each radial axis
supports comparison within its category-specific bounds. Interpretation
follows the individual metric axes; the logP axis displays hydrophobicity
according to the stated plotting convention.

\begin{figure}[!htbp]
    \centering
    \includegraphics[width=\linewidth,height=0.70\textheight,keepaspectratio]{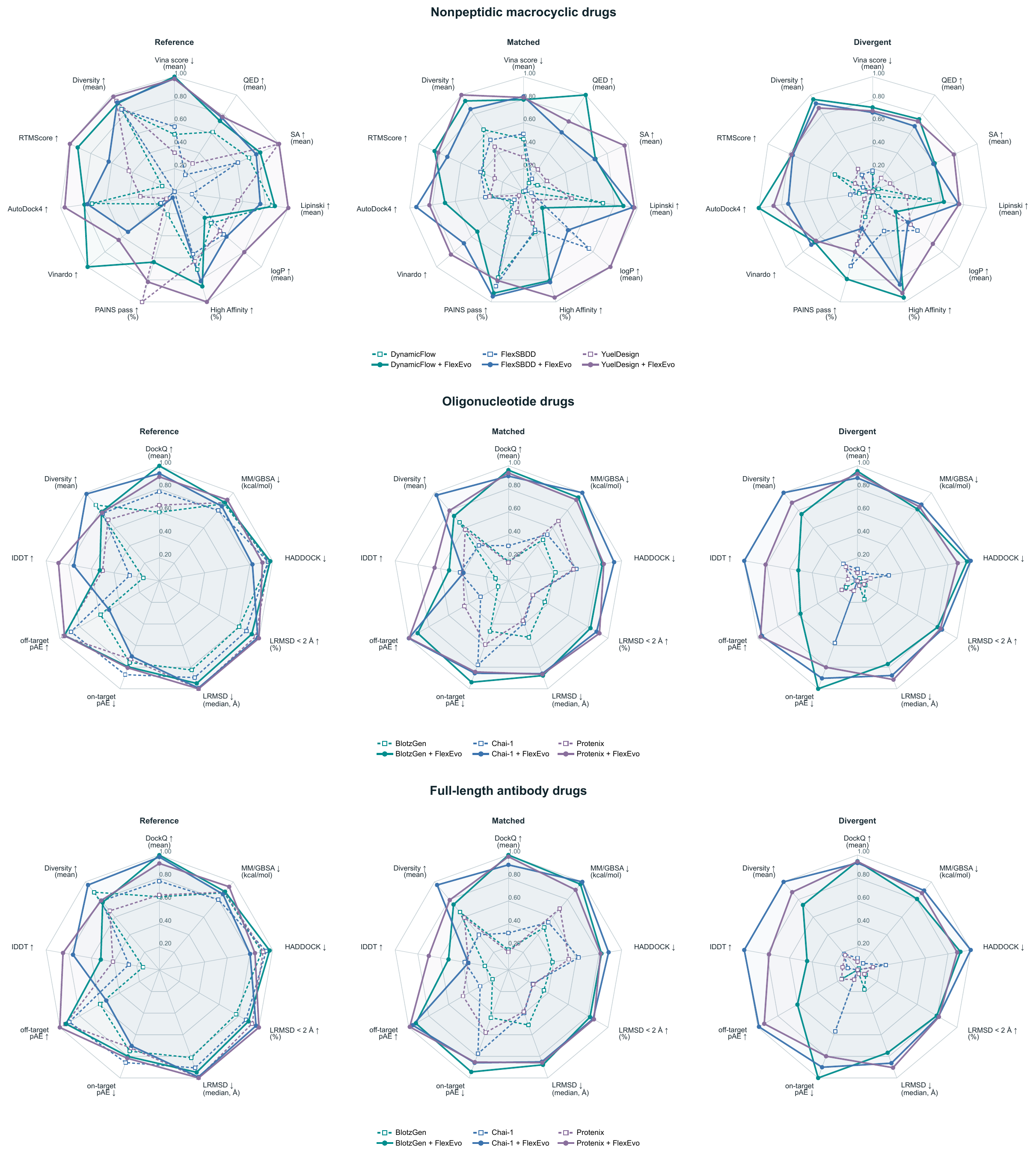}
    \caption{
    \textbf{Performance comparisons for nonpeptidic macrocycles, oligonucleotides and full-length antibodies.}
    Rows show nonpeptidic macrocyclic drugs, oligonucleotide drugs and full-length antibody drugs, respectively.
    Columns correspond to reference, matched and divergent conformations, from left to right.
    DynamicFlow, FlexSBDD and YuelDesign are compared with their FlexEvo variants for macrocycles; BoltzGen, Chai1 and Protenix are compared with their FlexEvo variants for oligonucleotides and antibodies.
    Colors identify model families within each row.
    Dashed lines with open squares indicate base models, whereas solid lines with filled circles indicate the corresponding FlexEvo variants.
    Each metric is min--max normalized to $[0,1]$ across all 18 method--conformation combinations within the same drug modality.
    Metrics marked with downward arrows are reverse-scaled, so outward values follow the indicated preferred direction.
    Normalization bounds are shared across the three columns within each row but differ between modalities.
    Means are used where indicated; LRMSD success rates below $2\,\text{\AA}$ and median LRMSD are shown separately.
    The logP axis displays increasing hydrophobicity according to the stated plotting convention.
    Performance comparisons follow the individual metric axes.
    }
    \label{fig:radar_macrocycle_oligonucleotide_antibody}
\end{figure}

\begin{figure}[!htbp]
    \centering
    \includegraphics[width=\linewidth,height=0.70\textheight,keepaspectratio]{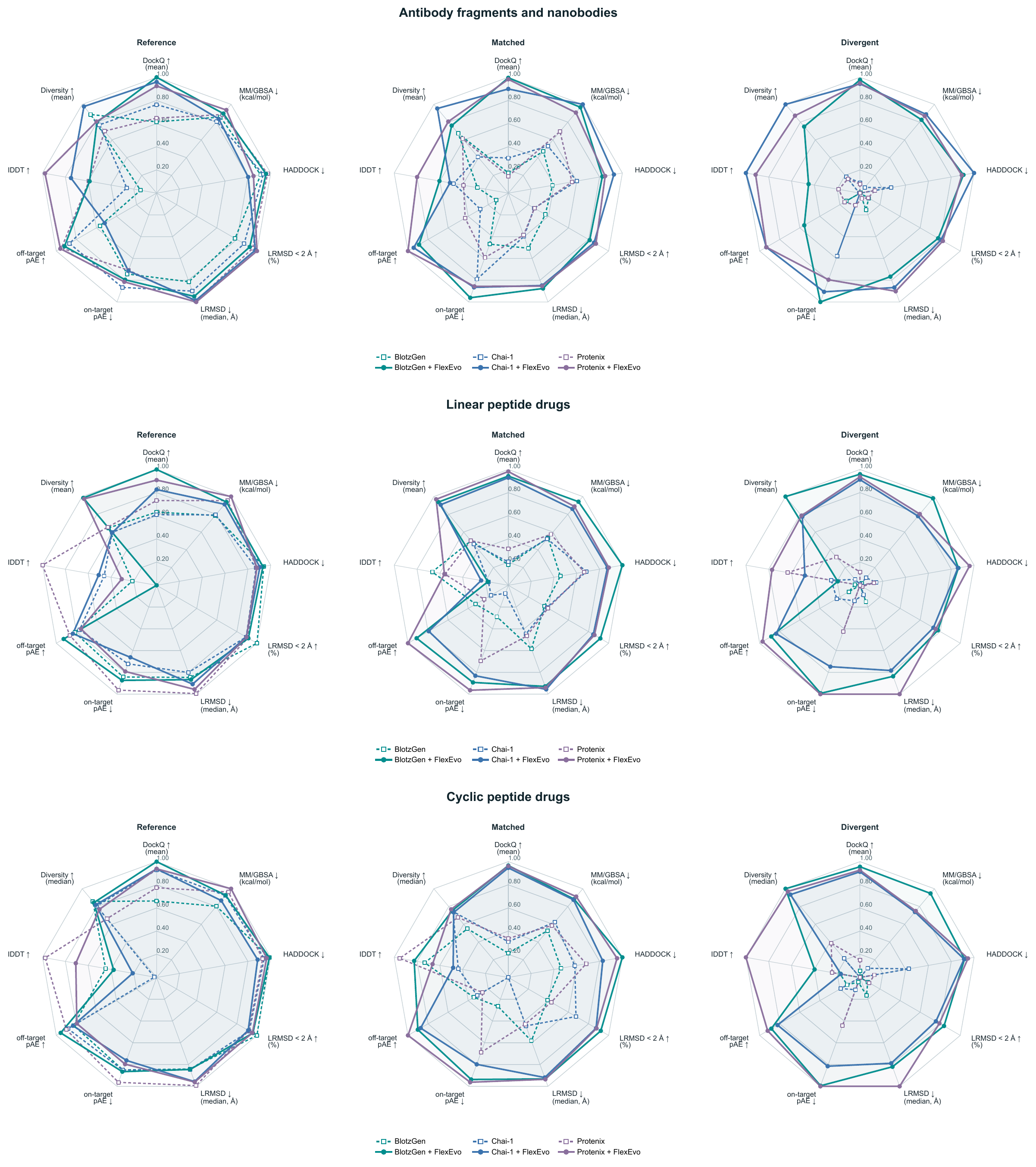}
    \caption{
    \textbf{Performance comparisons for antibody fragments, nanobodies and peptide drugs.}
    Rows show antibody fragments and nanobodies, linear peptide drugs and cyclic peptide drugs, respectively.
    Columns correspond to reference, matched and divergent conformations, from left to right.
    Each panel compares BoltzGen, Chai1 and Protenix with their corresponding FlexEvo variants.
    Colors identify model families within each row.
    Dashed lines with open squares indicate base models, whereas solid lines with filled circles indicate the corresponding FlexEvo variants.
    Each metric is min--max normalized to $[0,1]$ across all 18 method--conformation combinations within the same drug modality.
    Metrics marked with downward arrows are reverse-scaled, so outward values follow the indicated preferred direction.
    Normalization bounds are shared across the three columns within each row but differ between modalities.
    Means are used where indicated; LRMSD success rates below $2\,\text{\AA}$ and median LRMSD are shown separately.
    Diversity is represented by the median for cyclic peptides and by the mean for the other two modalities.
    Performance comparisons follow the individual metric axes.
    }
    \label{fig:radar_antibody_fragments_peptides}
\end{figure}

\begin{figure}[!htbp]
    \centering
    \includegraphics[width=\linewidth,height=0.70\textheight,keepaspectratio]{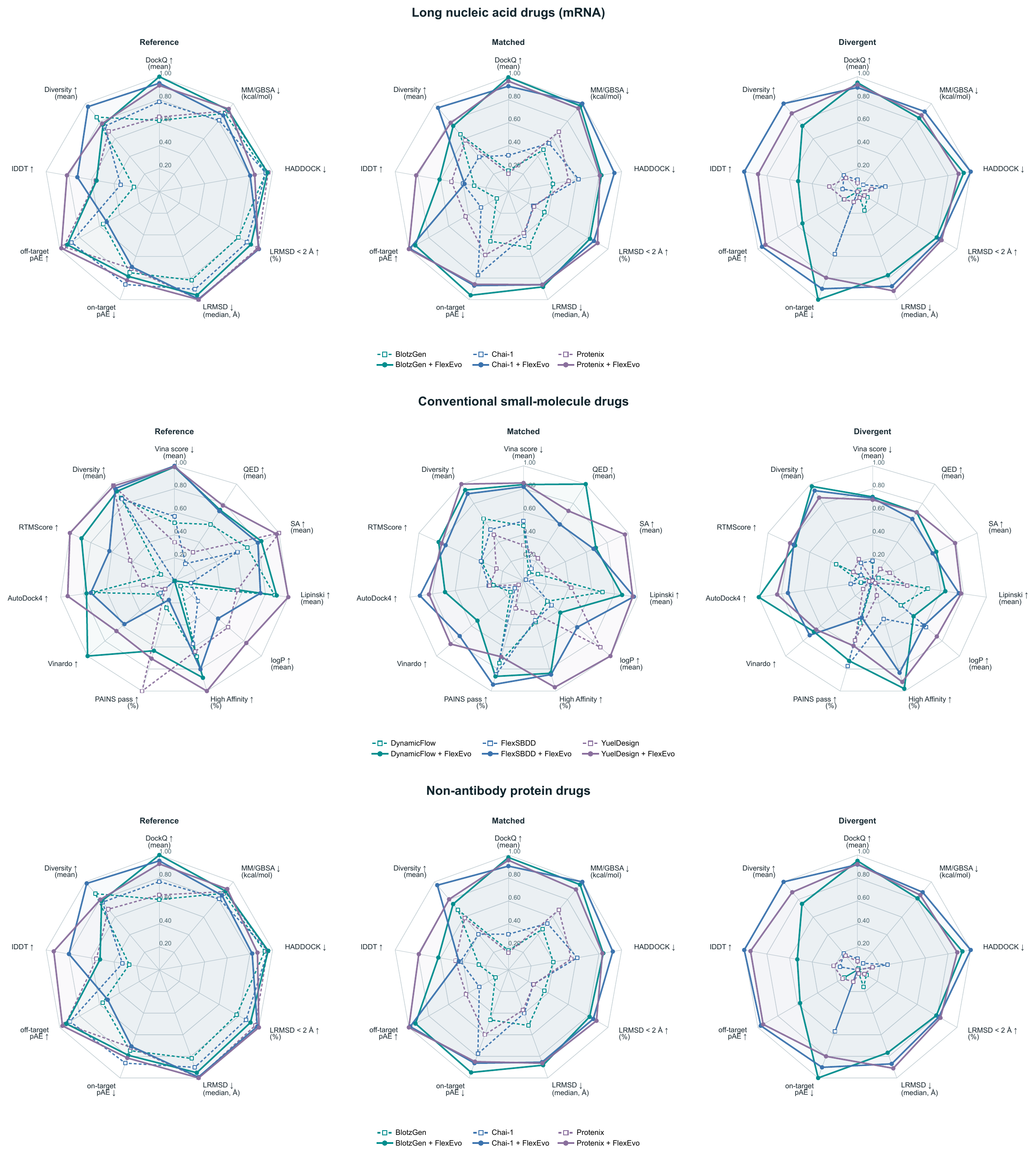}
    \caption{
    \textbf{Performance comparisons for long nucleic acids, small molecules and non-antibody proteins.}
    Rows show long nucleic acid drugs such as mRNA, conventional small-molecule drugs and non-antibody protein drugs, respectively.
    Columns correspond to reference, matched and divergent conformations, from left to right.
    BoltzGen, Chai1 and Protenix are compared with their FlexEvo variants for nucleic acids and proteins; DynamicFlow, FlexSBDD and YuelDesign are compared with their FlexEvo variants for small molecules.
    Colors identify model families within each row.
    Dashed lines with open squares indicate base models, whereas solid lines with filled circles indicate the corresponding FlexEvo variants.
    Each metric is min--max normalized to $[0,1]$ across all 18 method--conformation combinations within the same drug modality.
    Metrics marked with downward arrows are reverse-scaled, so outward values follow the indicated preferred direction.
    Normalization bounds are shared across the three columns within each row but differ between modalities.
    Means are used where indicated; LRMSD success rates below $2\,\text{\AA}$ and median LRMSD are shown separately.
    The logP axis displays increasing hydrophobicity according to the stated plotting convention.
    Performance comparisons follow the individual metric axes.
    }
    \label{fig:radar_nucleic_acids_small_molecules_proteins}
\end{figure}
\FloatBarrier
\subsection{On-target and off-target score diagnostics}
\label{app:interaction-diagnostics}
Figures~\ref{fig:interaction_scores_group_1}--\ref{fig:interaction_scores_group_3}
summarize on/off-target pAE diagnostics using the score definitions
in App.~\ref{app:evaluation-metrics}.
Highlighted candidates have on-target pAE below 5 and minimum
off-target pAE above 10 across the evaluated off-target panel.
Annotated percentages report the fraction of plotted candidates
meeting both criteria.
These model-confidence diagnostics summarize the on/off-target
contrast within the evaluated panel.

\begin{figure}[!htbp]
    \centering
    \includegraphics[width=1\linewidth,height=0.70\textheight,keepaspectratio]{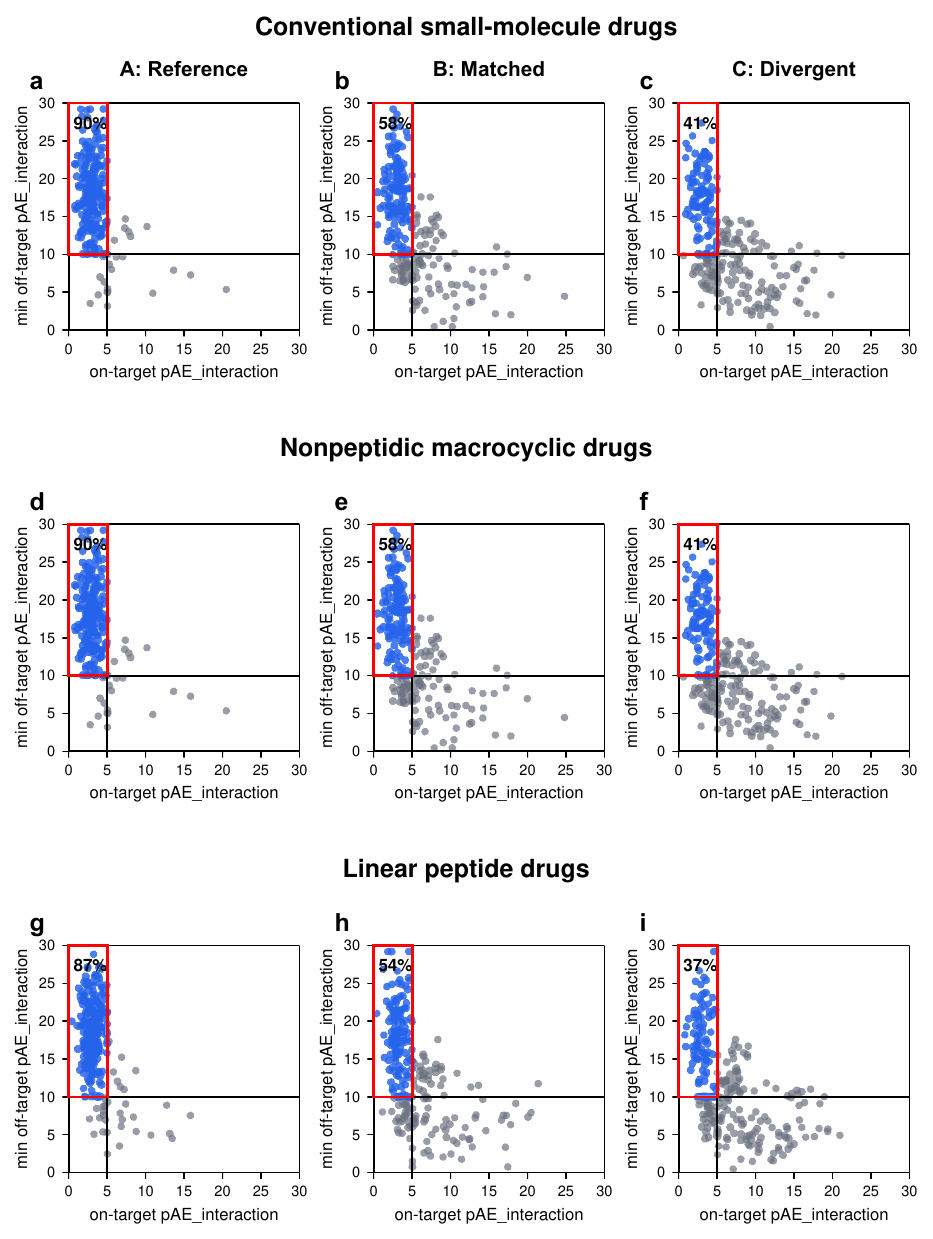}
    \caption{On-target and off-target interaction scores for small-molecule, macrocyclic and linear peptide drug candidates. Rows show \textbf{a}, conventional small-molecule drugs; \textbf{b}, nonpeptidic macrocyclic drugs; and \textbf{c}, linear peptide drugs. Columns retain the original A--C designations. Each point represents a candidate, plotted by its on-target \texttt{pAE\_interaction} score and the minimum score across the evaluated off-targets. Reference lines mark an on-target score of 5 and a minimum off-target score of 10. Red boxes highlight the region with on-target scores below 5 and minimum off-target scores above 10. Blue points fall within this region, whereas grey points fall outside it. Annotated percentages indicate the proportion of candidates within the highlighted region.}
    \label{fig:interaction_scores_group_1}
\end{figure}

\begin{figure}[!htbp]
    \centering
    \includegraphics[width=1\linewidth,height=0.70\textheight,keepaspectratio]{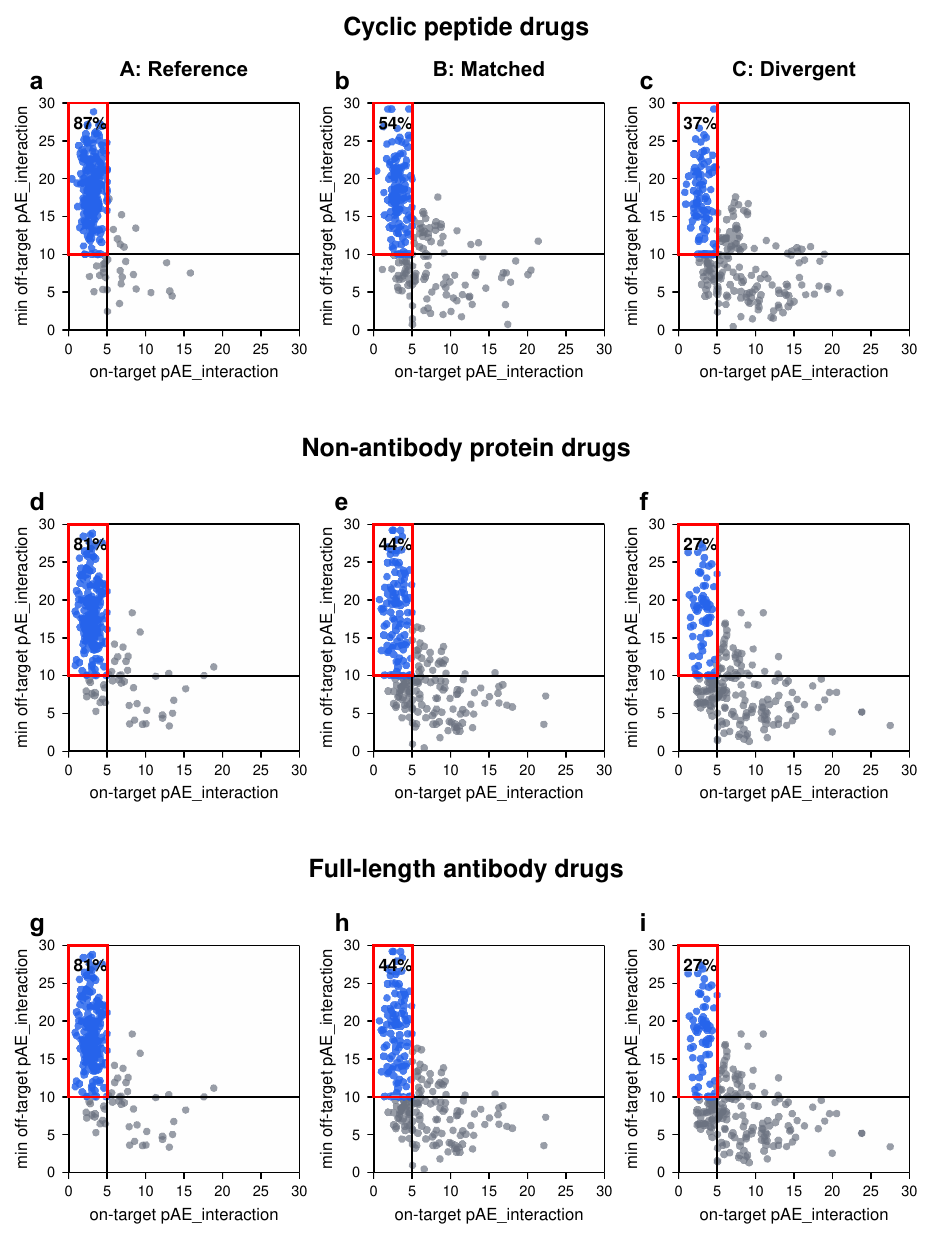}
    \caption{On-target and off-target interaction scores for cyclic peptide, protein and full-length antibody drug candidates. Rows show \textbf{a}, cyclic peptide drugs; \textbf{b}, non-antibody protein drugs; and \textbf{c}, full-length antibody drugs. Columns retain the original A--C designations. Each point represents a candidate, plotted by its on-target \texttt{pAE\_interaction} score and the minimum score across the evaluated off-targets. Reference lines mark an on-target score of 5 and a minimum off-target score of 10. Red boxes highlight the region with on-target scores below 5 and minimum off-target scores above 10. Blue points fall within this region, whereas grey points fall outside it. Annotated percentages indicate the proportion of candidates within the highlighted region.}
    \label{fig:interaction_scores_group_2}
\end{figure}

\begin{figure}[!htbp]
    \centering
    \includegraphics[width=1\linewidth,height=0.70\textheight,keepaspectratio]{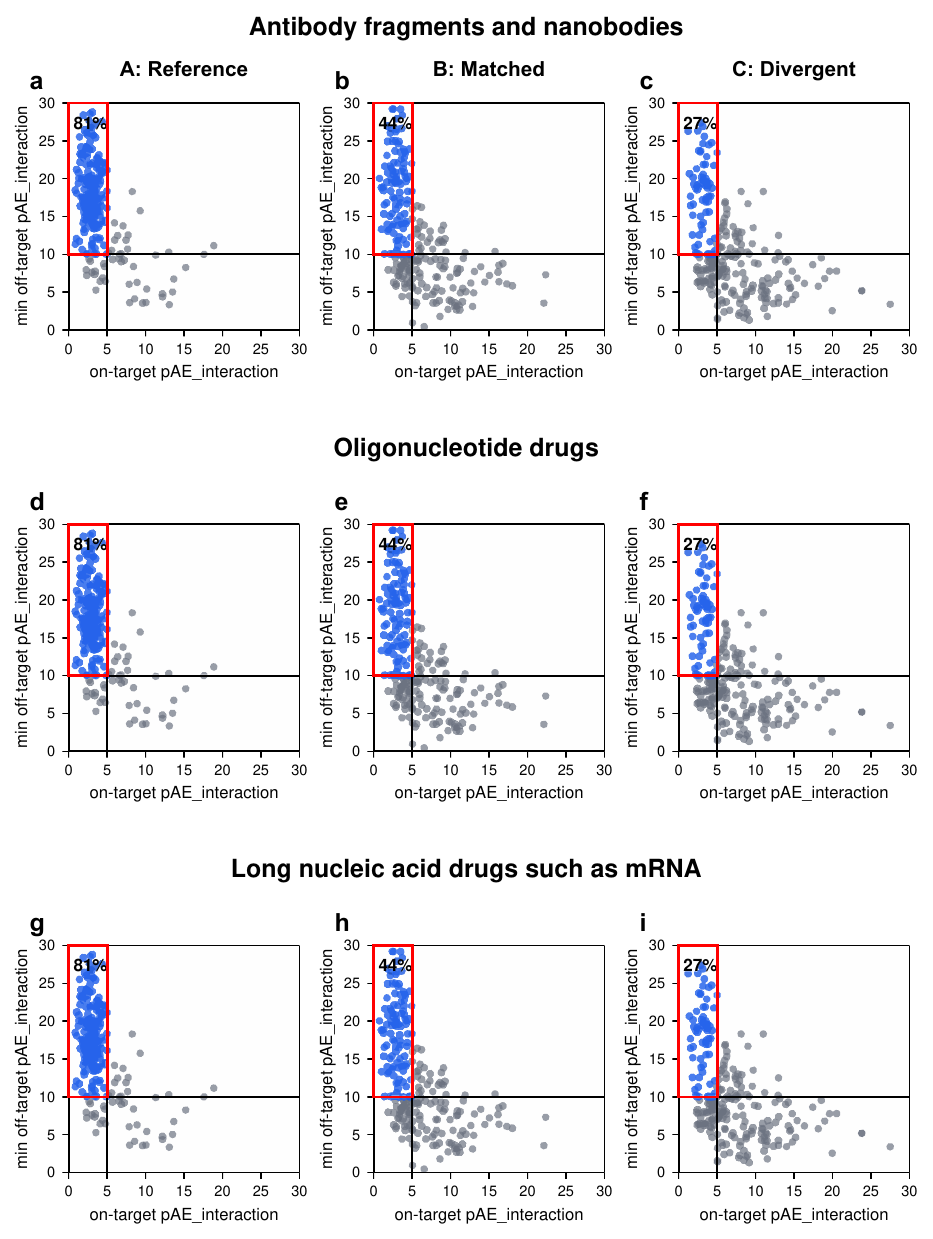}
    \caption{On-target and off-target interaction scores for antibody fragment and nucleic acid drug candidates. Rows show \textbf{a}, antibody fragments and nanobodies; \textbf{b}, oligonucleotide drugs; and \textbf{c}, long nucleic acid drugs, such as mRNA. Columns retain the original A--C designations. Each point represents a candidate, plotted by its on-target \texttt{pAE\_interaction} score and the minimum score across the evaluated off-targets. Reference lines mark an on-target score of 5 and a minimum off-target score of 10. Red boxes highlight the region with on-target scores below 5 and minimum off-target scores above 10. Blue points fall within this region, whereas grey points fall outside it. Annotated percentages indicate the proportion of candidates within the highlighted region.}
    \label{fig:interaction_scores_group_3}
\end{figure}
\FloatBarrier
\subsection{Optimization trajectories across binder categories}
\label{app:synthetic-trajectories}
Figures~\ref{fig:time_cost_small_molecules_linear_peptides}--\ref{fig:time_cost_antibody_fragments_nucleic_acids}
report the recorded optimization trajectories across the nine binder
categories. The horizontal axes show cumulative \textsc{FlexEvo}
optimization time per sample. The curves summarize metric evolution
during adaptation, including local fluctuations and late-stage
stabilization. Shaded bands summarize variability across the recorded
experimental runs. Figure~\ref{fig:optimization_dynamics} presents
representative adaptation dynamics in the main paper, while
Appendix~\ref{app:runtime-profiles} reports the computational cost analysis.

\begin{figure}[!htbp]
    \centering
    \includegraphics[width=1\linewidth,height=0.70\textheight,keepaspectratio]{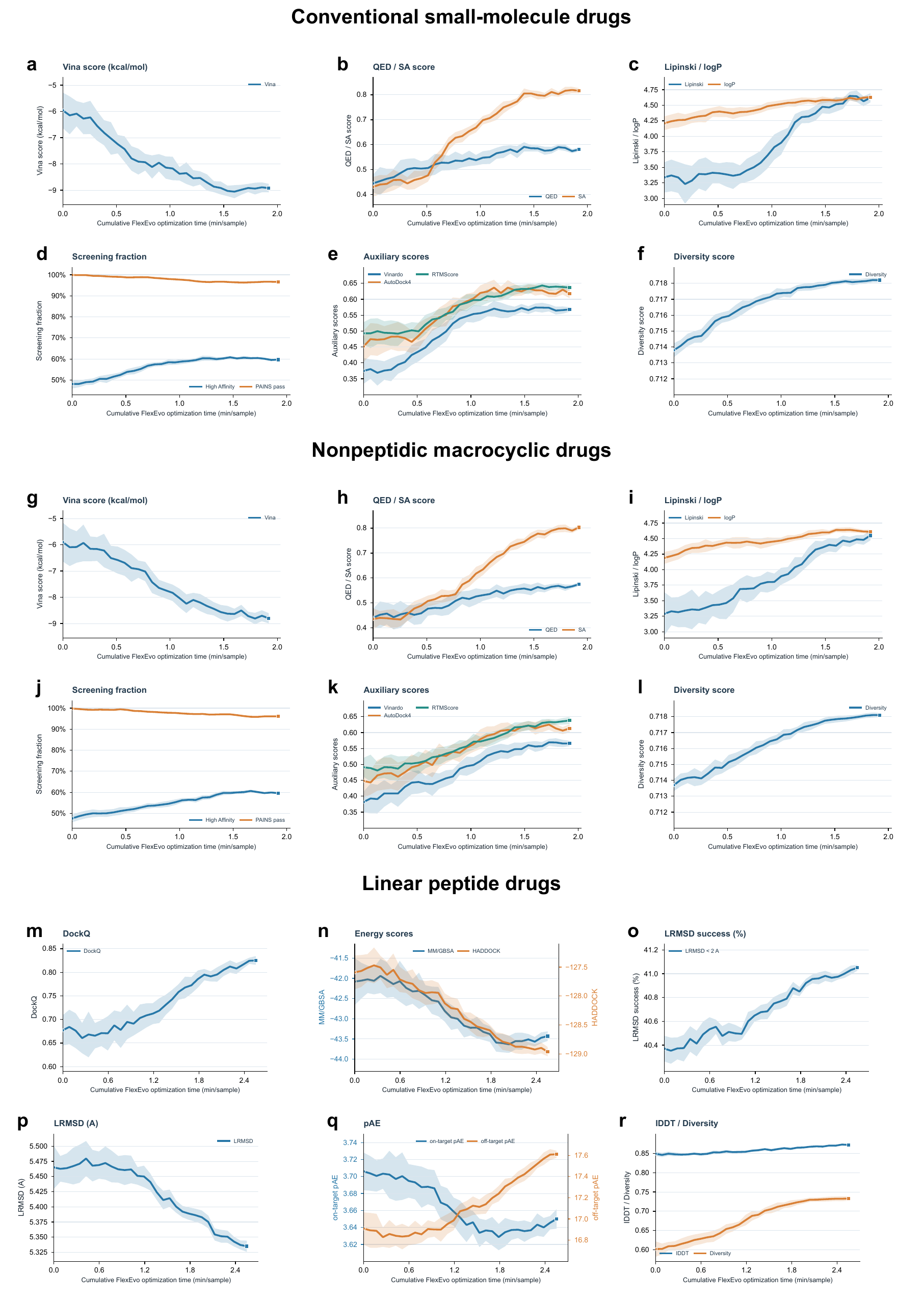}
    \caption{
Optimization trajectories across conventional small-molecule drugs,
nonpeptidic macrocyclic drugs, and linear peptide drugs.
Panels a--f, g--l, and m--r correspond to the three categories,
respectively. For the two small-molecule categories, panels a and g show
Vina scores; b and h, QED and SA scores; c and i, Lipinski scores and
logP; d and j, high-affinity and PAINS-pass fractions; e and k, auxiliary
scoring functions; and f and l, diversity. For linear peptides, panels
m--r show DockQ, MM/GBSA and HADDOCK scores, LRMSD success rate, LRMSD,
on-target and off-target pAE, and IDDT and diversity, respectively.
The horizontal axes indicate cumulative FlexEvo optimization time per
sample. The recorded trajectories show metric evolution during adaptation,
including local fluctuations and late-stage stabilization.
Shaded bands summarize variability across the recorded experimental runs.
}
\label{fig:time_cost_small_molecules_linear_peptides}
\end{figure}

\begin{figure}[!htbp]
    \centering
    \includegraphics[width=1\linewidth,height=0.70\textheight,keepaspectratio]{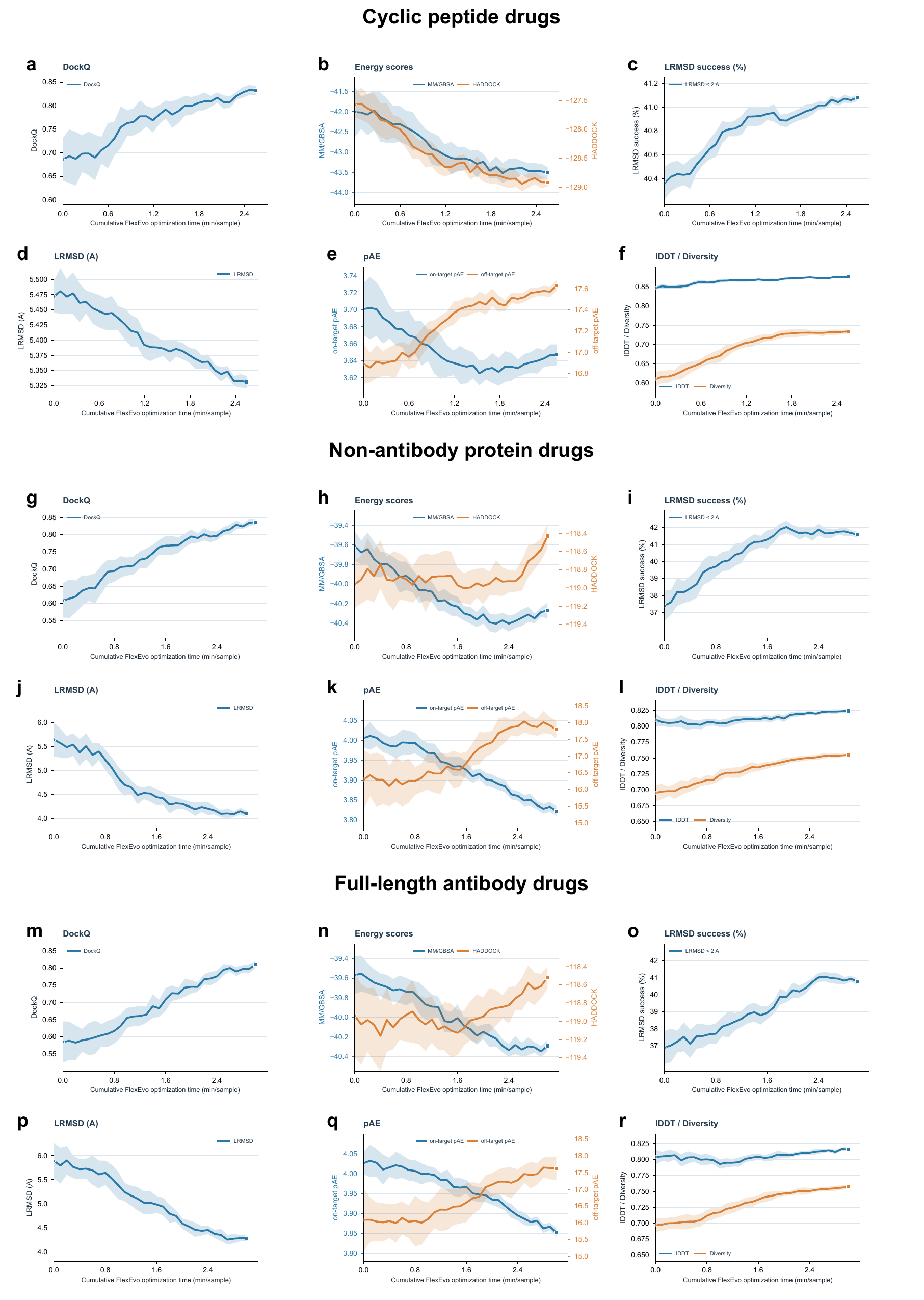}
    \caption{
Optimization trajectories across cyclic peptide drugs, non-antibody
protein drugs, and full-length antibody drugs. Panels a--f, g--l, and
m--r correspond to the three categories, respectively. Within each
category, the six panels report DockQ, MM/GBSA and HADDOCK scores, LRMSD success rate, LRMSD, on-target and off-target pAE, and IDDT and diversity, in that order. The horizontal axes indicate cumulative FlexEvo optimization time per sample. The recorded trajectories show increasing DockQ and LRMSD success rates, decreasing LRMSD and on-target pAE, and increasing diversity over optimization, with local fluctuations and late-stage stabilization. Differences in turning points and convergence profiles illustrate category-dependent optimization behavior. Energy-based metrics are interpreted according to their metric-specific preferred directions. Panels displaying paired metrics use separate left and right axes, with colors matched to the corresponding legends. 
}
\label{fig:time_cost_cyclic_peptides_proteins_antibodies}
\end{figure}

\begin{figure}[!htbp]
    \centering
    \includegraphics[width=1\linewidth,height=0.70\textheight,keepaspectratio]{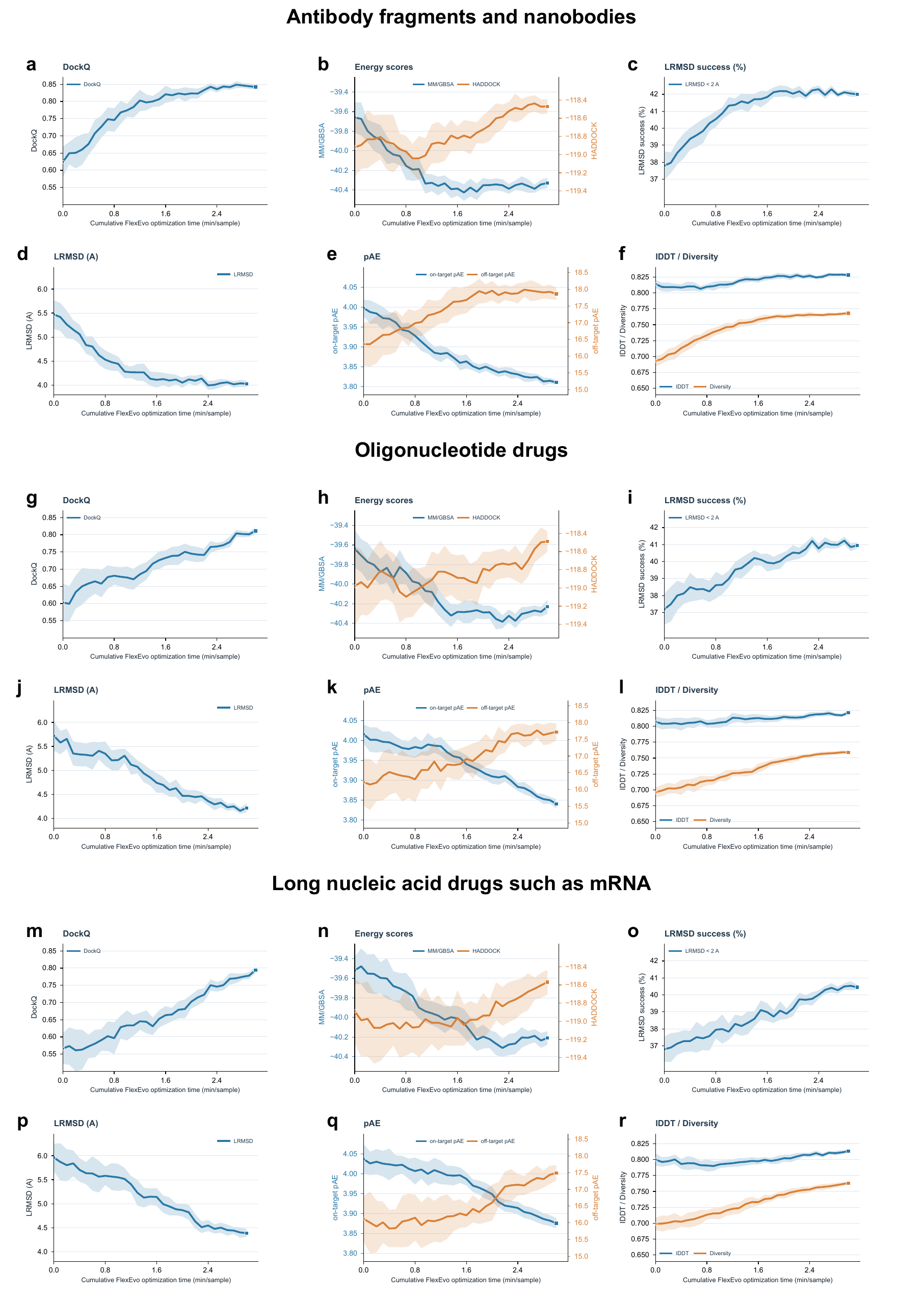}
    \caption{Optimization trajectories across antibody fragments and nanobodies, oligonucleotide drugs, and long nucleic acid drugs such as mRNA. Panels a--f, g--l, and m--r correspond to the three categories, respectively. Within each category, the six panels report DockQ, MM/GBSA and HADDOCK scores, LRMSD success rate, LRMSD, on-target and off-target pAE, and IDDT and diversity, in that order. The horizontal axes indicate cumulative FlexEvo optimization time per sample. The recorded trajectories show higher DockQ and LRMSD success rates, lower LRMSD and on-target pAE, and increased diversity over optimization, together with local fluctuations and category-dependent convergence profiles. Energy and pAE panels use separate left and right axes, with colors identifying the corresponding metrics. 
}
\label{fig:time_cost_antibody_fragments_nucleic_acids}
\end{figure}
\FloatBarrier

\section{Deterministic score bounds and conditional robustness of FlexEvo}
\label{app:flexevo-theory-v30}

This appendix analyzes the coordinate-level scores and selection rules
defined in Appendix~\ref{app:flexevo}. In particular, the
contact and clash kernels reproduce Eq.~\eqref{eq:contact_kernel_app},
and the objective coefficients reproduce
Eqs.~\eqref{eq:penalty_app} and~\eqref{eq:app-candidate-objectives}. The results concern a fixed candidate
pose and its geometric contact and clash descriptors.
Adaptation uses a single observed conformation \(P^0\) and synthetic
perturbation probes derived from it.
An unseen target enters only through an explicit proximity hypothesis,
or through radii measured after that target is supplied.
Held-out natural conformations are not arguments of the objectives, and
are reserved for evaluation.

Throughout, \(R\geq 1\) is the number of represented target residues,
the genome support satisfies
\(\mathcal R_G\subseteq\{1,\ldots,R\}\), and consequently
\(0\leq\Omega(G):=|\mathcal R_G|/R\leq 1\).
Scales satisfy
\(0<z_{\min}\leq z_r\leq z_{\max}<\infty\), and every role extent
\(b_\ell\) is positive and finite.
Set \(b_{\max}=\max_\ell b_\ell\) and write
\([v]_+=\max\{v,0\}\),
\(C(v)=\min\{1,\max\{0,v\}\}\).
Every candidate used in a score formula has \(n=n(x)\geq 1\) finite
coordinates and a finite edit history. Candidate records include that
history; equal coordinates alone need not imply equal objective values.
Every target atom set is finite and nonempty.
The contact, clash, and objective formulas below are not applied to empty
candidate records. If no adapted candidate is available, the documented
genome-level evaluator returns \((0,0)\); this return does not define any
candidate-level expression containing \(1/n\).
We use \(0\leq\rho,\lambda\leq 1\), \(0<\tau<1\), and
\(\varepsilon_q>0\) throughout.
The analysis uses exact real arithmetic: finite-precision implementations
must handle threshold endpoints and non-finite values consistently.
Distances below are numerical values in angstroms.
Throughout every target comparison, the candidate coordinates and their
number are fixed. The bounds for $q$ and $\chi$ do not transfer the
entire genome fitness or the output of a newly run adaptation procedure.
All sums over candidate coordinates run from \(1\) to \(n\), and
all probe extrema and sums run from \(0\) to \(K-1\).

The reference configuration uses
\begin{equation}
\label{eq:fev30-reference}
\begin{aligned}
&(z_{\min},z_{\max})=(0.55,1.75),&
&(b_{\mathrm{stable}},b_{\mathrm{flexible}},b_{\mathrm{forbidden}})
=(3.0,4.5,2.9),\\
&(\rho,\tau,\varepsilon_q)=(0.60,0.75,10^{-9}),&
&(\lambda,\eta)=(0.70,0.10),\\
&(N_G,g_{\max},n_o)=(8,4,2).
\end{aligned}
\end{equation}
The profile bounds below hold for every \(\rho\in[0,1]\), and the
aggregation inequality holds for every \(\lambda\in[0,1]\). Numerical
objective ranges use the stated reference coefficients.

\subsection{Geometry of the spatial prior}

For a covered residue \(r\), let \(L_r=b_{\ell_r}z_r>0\) and
\(Q_r^\top Q_r=I_3\). Its FlexBox is
\begin{equation}
\label{eq:fev30-box}
\mathcal B_r(G)
=
\bigl\{
y\in\mathbb R^3:
\|Q_r^\top(y-c_r)\|_\infty\leq L_r/2
\bigr\}.
\end{equation}

\begin{fevproposition}[Box geometry and support]
\label{prop:fev30-geometry}
Each box is nonempty, compact, and convex, with volume \(L_r^3\),
Euclidean diameter \(\sqrt{3}\,L_r\), and enclosing radius
\(\sqrt{3}\,L_r/2\) about \(c_r\).
For a common center and frame, boxes nest in the order of the products
\(b_{\ell_r}z_r\).
Moreover, with \(\Omega(G)=|\mathcal R_G|/R\),
\begin{equation}
\label{eq:fev30-volume}
\operatorname{Vol}\!\Bigl(
\bigcup_{r\in\mathcal R_G}\mathcal B_r(G)
\Bigr)
\leq
|\mathcal R_G|\,(b_{\max}z_{\max})^3
=
R\,\Omega(G)\,(b_{\max}z_{\max})^3.
\end{equation}
\end{fevproposition}

\begin{proof}
The box is the image of \([-L_r/2,L_r/2]^3\) under the affine isometry
\(u\mapsto c_r+Q_ru\), which proves nonemptiness, compactness, convexity,
volume, and diameter.
Every local coordinate satisfies
\(\|u\|_2\leq\sqrt{3}\,\|u\|_\infty\leq\sqrt{3}\,L_r/2\).
Comparing \(\infty\)-norm bounds proves nesting.
Subadditivity of volume and \(L_r\leq b_{\max}z_{\max}\) give
\eqref{eq:fev30-volume}.
\end{proof}

\begin{fevremark}[Editing prior, not a hard pose constraint]
\label{rem:fev30-box-scope}
Role names alone do not order boxes when scales differ; the products
\(b_{\ell_r}z_r\) do.
The volume bound controls the edit prior and ignores overlaps, so it is
generally loose.
A repaired candidate coordinate need not lie in any box.
\end{fevremark}

\subsection{Fixed-pose scores and finite-probe bounds}

Fix a genome \(G\) and probes \((\widetilde P_G^m)_{m=0}^{K-1}\) with
\(\widetilde P_G^0=P^0\) and integer \(K=M+1\geq 1\).
The reference construction has \(K=7\); distinct indices may share the
same geometry.
Candidate coordinates remain fixed across probes: the profile measures
the response of one candidate, not a separately re-optimized pose for
each probe.

For a finite nonempty target atom set \(\mathcal Y(P)\) in one fixed
scoring frame, set
\(d_i(P)=\min_{y\in\mathcal Y(P)}\|p_i(x)-y\|_2\) and
\begin{align}
\label{eq:fev30-kernels}
h_q(d)
&=
\mathbf{1}_{\{2\leq d\leq 4.5\}}
+0.35\,\mathbf{1}_{\{1.2\leq d<2\}}
-1.5\,\mathbf{1}_{\{d<1.15\}},
\\
h_\chi(d)
&=
\mathbf{1}_{\{d<1.15\}}
+0.35\,\mathbf{1}_{\{1.15\leq d<1.45\}},
\nonumber\\
q(x;P)
&=
C\Bigl(n^{-1}\sum_i h_q\bigl(d_i(P)\bigr)\Bigr),
\qquad
\chi(x;P)
=
n^{-1}\sum_i h_\chi\bigl(d_i(P)\bigr).
\nonumber
\end{align}
The indicators inside each kernel are pairwise disjoint, so
\(h_q\in\{-1.5,0,0.35,1\}\) and \(h_\chi\in\{0,0.35,1\}\).
In particular the clash mid-piece is half-open:
\(h_\chi(1.45)=0\), while \(h_\chi(d)=0.35\) for every
\(d\in[1.15,1.45)\).
Contact clipping is applied after averaging; clash is not clipped, but
its average still lies in \([0,1]\).

Write \(q_m=q(x;\widetilde P_G^m)\), \(\chi_m=\chi(x;\widetilde P_G^m)\),
and
\begin{align}
\label{eq:fev30-profile}
q_{\min}&=\min_m q_m,&
q_{\max}&=\max_m q_m,&
\bar q&=K^{-1}\sum_m q_m,&
\Delta_q&=q_{\max}-q_{\min},
\nonumber\\
q_{\mathrm{rob}}
&=
\rho q_{\min}+(1-\rho)\bar q,
&
\chi_{\max}&=\max_m\chi_m,
\\
\kappa_q
&=
\begin{cases}
K^{-1}\sum_m\mathbf{1}_{\{q_m\geq\tau q_0\}},
& q_0>\varepsilon_q,\\
0,
& q_0\leq\varepsilon_q.
\end{cases}
\nonumber
\end{align}

\begin{fevproposition}[Finite-probe contact bound]
\label{prop:fev30-profile}
One has
\(0\leq q_{\min}\leq q_{\mathrm{rob}}\leq\bar q\leq q_{\max}\leq 1\)
and \(0\leq\chi_m\leq\chi_{\max}\leq 1\).
Define
\begin{equation}
\label{eq:fev30-Lprof}
\begin{aligned}
L_\kappa
&=
\begin{cases}
\tau q_0,
& q_0>\varepsilon_q\text{ and }\kappa_q=1,\\
0,
& \text{otherwise},
\end{cases}
\\
L_{\mathrm{prof}}
&=
\max\Bigl\{
0,\ 
q_0-\Delta_q,\ 
L_\kappa,\ 
q_{\mathrm{rob}}-(1-\rho)\tfrac{K-1}{K}\Delta_q
\Bigr\}.
\end{aligned}
\end{equation}
Then every probe satisfies \(q_m\geq q_{\min}\geq L_{\mathrm{prof}}\).
The bound involving \(q_{\mathrm{rob}}\) and \(\Delta_q\) is sharp over
arbitrary tuples in \([0,1]^K\).
\end{fevproposition}

\begin{proof}
Clipping and the clash image give the score ranges.
The average lies between the minimum and maximum, and
\(q_{\mathrm{rob}}\) is a convex combination of the minimum and average.
Since \(q_0\leq q_{\max}\), one has \(q_{\min}\geq q_0-\Delta_q\).
When \(\kappa_q=1\) and \(q_0>\varepsilon_q\), every consistency indicator
equals one, so \(q_{\min}\geq\tau q_0\).
At least one score equals \(q_{\min}\) and all others are at most
\(q_{\min}+\Delta_q\), whence
\[
\bar q\leq q_{\min}+\tfrac{K-1}{K}\Delta_q,
\qquad
q_{\mathrm{rob}}\leq q_{\min}+(1-\rho)\tfrac{K-1}{K}\Delta_q.
\]
Rearrangement gives the remaining bound.
Equality holds for one score equal to \(q_{\min}\) and the other \(K-1\)
equal to \(q_{\min}+\Delta_q\), whenever these values lie in \([0,1]\).
\end{proof}

For \(K=7\) and \(\rho=0.60\), the last bound is
\(q_{\min}\geq q_{\mathrm{rob}}-(12/35)\Delta_q\).
This inequality gives a lower bound on aggregate contact quality shared
by all seven probes for the fixed candidate. It connects the optimized
robust summary and contact variation to worst-probe contact retention.
Transfer to natural target geometries is characterized by
Theorem~\ref{thm:fev30-transfer} under its stated neighborhood conditions.

\begin{fevcorollary}[Bound from the robust summary alone]
\label{cor:fev30-summary-only}
If only \(q_{\mathrm{rob}}\), \(K\), and \(\rho\) are retained, then
\begin{equation}
\label{eq:fev30-summary-only}
q_{\min}
\geq
L_K
:=
\Biggl[
\frac{Kq_{\mathrm{rob}}-(1-\rho)(K-1)}{1+\rho(K-1)}
\Biggr]_+.
\end{equation}
This bound is sharp over arbitrary real-valued tuples in \([0,1]^K\).
Sharpness here does not assert realizability of an arbitrary tuple by
the specified seven geometric probes or by one molecular structure.
For a complete profile it is no stronger than
\(q_{\mathrm{rob}}-(1-\rho)(K-1)\Delta_q/K\) after clipping at zero.
\end{fevcorollary}

\begin{proof}
At least one score equals \(q_{\min}\) and the remaining scores are at
most one, so \(K\bar q\leq q_{\min}+K-1\) and
\[
Kq_{\mathrm{rob}}
\leq
\bigl(1+\rho(K-1)\bigr)q_{\min}+(1-\rho)(K-1).
\]
The denominator of \eqref{eq:fev30-summary-only} is positive because
\(\rho\geq 0\) and \(K\geq 1\).
If the numerator is nonnegative, equality is attained by
\((L_K,1,\ldots,1)\).
If it is negative, then \(K>1\) and \(\rho<1\); the tuple
\((0,t,\ldots,t)\) with
\(t=Kq_{\mathrm{rob}}/[(1-\rho)(K-1)]\in[0,1)\) attains equality.

For the comparison, write \(a_K=(1-\rho)(K-1)/K\).
Then \(a_K\in[0,1)\) and
\[
1-a_K=\frac{1+\rho(K-1)}{K},
\qquad
\frac{q_{\mathrm{rob}}-a_K}{1-a_K}
=
\frac{Kq_{\mathrm{rob}}-(1-\rho)(K-1)}{1+\rho(K-1)},
\]
so \(L_K=[(q_{\mathrm{rob}}-a_K)/(1-a_K)]_+\).
Proposition~\ref{prop:fev30-profile} gives
\(q_{\mathrm{rob}}\leq q_{\min}+a_K\Delta_q
=q_{\max}-(1-a_K)\Delta_q\leq 1-(1-a_K)\Delta_q\), hence
\[
\bigl(q_{\mathrm{rob}}-a_K\Delta_q\bigr)
-\frac{q_{\mathrm{rob}}-a_K}{1-a_K}
=
\frac{a_K\bigl(1-q_{\mathrm{rob}}-(1-a_K)\Delta_q\bigr)}{1-a_K}
\geq 0.
\]
In particular, whenever \((q_{\mathrm{rob}}-a_K)/(1-a_K)\geq0\), one has
\(q_{\mathrm{rob}}-a_K\Delta_q\geq(q_{\mathrm{rob}}-a_K)/(1-a_K)\geq 0\).
If \((q_{\mathrm{rob}}-a_K)/(1-a_K)<0\), then \(L_K=0\), while the profile term is at
least zero after clipping.
Positive parts therefore preserve the ordering.
\end{proof}

\subsection{Computable bounds under target perturbations}

All compared targets are expressed in one fixed scoring frame.
Any superposition is chosen by a target-based rule independent of the
candidate being scored. All targets in a comparison are placed in this
one reference frame before the distances below are computed; separate
pairwise best-fit distances are not substituted into these formulas.
The atom sets are exactly those used in the nearest-target score.
Hausdorff distance ignores atom identities, which is sufficient here
because these two score kernels also depend only on nearest distances.
For finite nonempty sets,
\begin{equation}
\label{eq:fev30-hausdorff}
d_{\mathrm H}(P,P')
=
\max\Biggl\{
\max_{y\in\mathcal Y(P)}\min_{y'\in\mathcal Y(P')}\|y-y'\|_2,\;
\max_{y'\in\mathcal Y(P')}\min_{y\in\mathcal Y(P)}\|y-y'\|_2
\Biggr\}.
\end{equation}

\begin{fevlemma}[Nearest-target distance stability]
\label{lem:fev30-distance}
For every represented coordinate,
\(|d_i(P')-d_i(P)|\leq d_{\mathrm H}(P,P')\).
\end{fevlemma}

\begin{proof}
Let \(y\in\mathcal Y(P)\) attain \(d_i(P)\) and choose
\(y'\in\mathcal Y(P')\) with \(\|y-y'\|_2\leq d_{\mathrm H}(P,P')\).
The triangle inequality gives
\(d_i(P')\leq d_i(P)+d_{\mathrm H}(P,P')\).
Exchanging the targets proves the claim.
\end{proof}

For \(\varepsilon\geq 0\), set
\(I_i(P,\varepsilon)=[\max\{0,d_i(P)-\varepsilon\},\,d_i(P)+\varepsilon]\).
Let \(C_q=C\) and \(C_\chi(v)=v\). For \(s\in\{q,\chi\}\),
\begin{align}
\label{eq:fev30-envelopes}
L_s(x;P,\varepsilon)
&=
C_s\Bigl(
n^{-1}\sum_i\min_{d\in I_i(P,\varepsilon)}h_s(d)
\Bigr),
\\
U_s(x;P,\varepsilon)
&=
C_s\Bigl(
n^{-1}\sum_i\max_{d\in I_i(P,\varepsilon)}h_s(d)
\Bigr).
\nonumber
\end{align}
Each kernel maps a nonempty interval to a nonempty finite set of values,
so its minimum and maximum are attained despite its discontinuities.

\begin{fevtheorem}[Coordinatewise interval enclosure]
\label{thm:fev30-envelopes}
If \(d_{\mathrm H}(P,P')\leq\varepsilon\), then
\begin{equation}
\label{eq:fev30-enclosure}
L_s(x;P,\varepsilon)
\leq s(x;P')
\leq U_s(x;P,\varepsilon),
\qquad s\in\{q,\chi\}.
\end{equation}
The same enclosure holds for the reference score \(s(x;P)\), because
\(d_i(P)\in I_i(P,\varepsilon)\).
The endpoints are exact extrema over the product relaxation
\((d_1',\ldots,d_n')\in\prod_i I_i(P,\varepsilon)\).
Either endpoint may fail to be attainable by a single target geometry
inside the Hausdorff neighborhood.
The enclosure does not assert attainability of every intermediate score,
or simultaneous attainability of contact and clash endpoints.
\end{fevtheorem}

\begin{proof}
Lemma~\ref{lem:fev30-distance} places each \(d_i(P')\) in \(I_i(P,\varepsilon)\).
The reference distances already lie in those intervals.
Averaging and using that each \(C_s\) is nondecreasing yields
\eqref{eq:fev30-enclosure} for both \(P'\) and \(P\).
In the product relaxation, each coordinate can attain its kernel minimum
or maximum independently, so the envelope endpoints are attained there.
A realizable target couples the nearest-target distances, so attainment
need not extend to the geometric neighborhood.
\end{proof}

\begin{fevlemma}[Kernel oscillations]
\label{lem:fev30-omega}
Let
\(\mathcal T_q=\{1.15,1.20,2.00,4.50\}\) and
\(\mathcal T_\chi=\{1.15,1.45\}\), and define
\[
\omega_s(\varepsilon)
=
\sup\bigl\{
|h_s(d)-h_s(d')|:
d,d'\geq 0,\ |d-d'|\leq\varepsilon
\bigr\}.
\]
Then \(\omega_s(0)=0\), and for \(\varepsilon>0\)
\begin{equation}
\label{eq:fev30-oscillations}
\omega_q(\varepsilon)
=
\begin{cases}
1.50,& 0<\varepsilon\leq 0.05,\\
1.85,& 0.05<\varepsilon\leq 0.85,\\
2.50,& \varepsilon>0.85,
\end{cases}
\qquad
\omega_\chi(\varepsilon)
=
\begin{cases}
0.65,& 0<\varepsilon\leq 0.30,\\
1.00,& \varepsilon>0.30.
\end{cases}
\end{equation}
\end{fevlemma}

\begin{proof}
The ordered contact piece values are \(-1.5,0,0.35,1,0\).
The clash piece values are \(1,0.35,0\).

For any \(\varepsilon>0\), choose
\(t=\min\{\varepsilon/2,0.01\}>0\).
The pair \((1.15-t,1.15)\) consists of nonnegative distances separated
by at most \(\varepsilon\).
It gives contact difference \(1.50\) and clash difference \(0.65\).

A contact difference greater than \(1.50\) must use \(-1.5\) together
with \(0.35\) or \(1\).
The former requires \(d<1.15\) and \(d'\geq 1.20\), hence
\(|d-d'|>0.05\).
For \(\varepsilon>0.05\), take
\(t=\min\{(\varepsilon-0.05)/2,0.01\}\).
Then \((1.15-t,1.20)\) attains difference \(1.85\) within the budget.
The difference \(2.50\) requires \(d<1.15\) and \(d'\geq 2\),
hence separation strictly greater than \(0.85\).
For \(\varepsilon>0.85\), take
\(t=\min\{(\varepsilon-0.85)/2,0.01\}\).
The pair \((1.15-t,2)\) attains \(2.50\) within the budget.
No larger difference occurs among the contact kernel values.

A clash difference greater than \(0.65\) must be \(1\).
It requires \(d<1.15\) and \(d'\geq 1.45\), hence separation
strictly greater than \(0.30\).
For \(\varepsilon>0.30\), take
\(t=\min\{(\varepsilon-0.30)/2,0.01\}\).
The pair \((1.15-t,1.45)\) attains difference \(1\).
These strict separations also establish the three endpoint conventions
in \eqref{eq:fev30-oscillations}.
When \(\varepsilon=0\), the two arguments coincide and both oscillations
are zero.
\end{proof}

\begin{fevproposition}[Centered modulus enclosure]
\label{prop:fev30-margins}
Define the sensitive fraction
\begin{equation}
\label{eq:fev30-nu}
\nu_s(x;P,\varepsilon)
=
n^{-1}
\bigl|
\{i:\operatorname{dist}(d_i(P),\mathcal T_s)\leq\varepsilon\}
\bigr|
\end{equation}
and \(\delta_s=\min\{1,\omega_s(\varepsilon)\,\nu_s(x;P,\varepsilon)\}\).
Then
\begin{equation}
\label{eq:fev30-centered-enclosure}
\bigl[L_s(x;P,\varepsilon),U_s(x;P,\varepsilon)\bigr]
\subseteq
\bigl[\max\{0,s(x;P)-\delta_s\},\,
\min\{1,s(x;P)+\delta_s\}\bigr].
\end{equation}
In particular, \(\nu_s=0\) implies exact score invariance throughout the
declared neighborhood.
\end{fevproposition}

\begin{proof}
If a coordinate is not counted by \(\nu_s\), its interval contains no
threshold and its kernel contribution is constant.
For every other coordinate, any value on \(I_i(P,\varepsilon)\) differs
from the value at \(d_i(P)\) by at most \(\omega_s(\varepsilon)\), because
every point of that interval is at most \(\varepsilon\) from \(d_i(P)\).
Averaging bounds the change of the unclipped mean by
\(\omega_s(\varepsilon)\nu_s\).
Both \(C_q\) and \(C_\chi\) are \(1\)-Lipschitz, and the resulting scores
lie in \([0,1]\).
Hence every score arising from the product relaxation---including the
envelope endpoints---lies in the centered interval of half-width
\(\delta_s\) about \(s(x;P)\).
\end{proof}

\begin{fevcorollary}[Realized change, centered deviation, and envelope width]
\label{cor:fev30-width}
Write \(s_0=s(x;P)\) and suppress \((x;P,\varepsilon)\) in \(L_s,U_s\).
Define the centered deviation
\(D_s^{\mathrm{ctr}}=\max\{s_0-L_s,\,U_s-s_0\}\).
If \(d_{\mathrm H}(P,P')\leq\varepsilon\), then
\begin{equation}
\label{eq:fev30-modulus}
|s(x;P')-s_0|
\leq
D_s^{\mathrm{ctr}}
\leq
\min\{U_s-L_s,\,\delta_s\}
=
\min\bigl\{U_s-L_s,\,
\min\{1,\omega_s(\varepsilon)\,\nu_s(x;P,\varepsilon)\}\bigr\}.
\end{equation}
Also \(D_s^{\mathrm{ctr}}\leq U_s-L_s\leq 2D_s^{\mathrm{ctr}}\).
The full envelope width satisfies
\begin{equation}
\label{eq:fev30-envelope-width}
U_s-L_s
\leq
\min\Bigl\{1,\,n^{-1}\sum_i(M_i-m_i)\Bigr\}
\leq
\min\{1,\omega_s(2\varepsilon)\,\nu_s(x;P,\varepsilon)\},
\end{equation}
where \(m_i=\min_{I_i}h_s\) and \(M_i=\max_{I_i}h_s\).
The factor \(2\varepsilon\) in \eqref{eq:fev30-envelope-width} cannot in
general be replaced by \(\varepsilon\).
Because \(s_0\in[L_s,U_s]\), one always has
\(D_s^{\mathrm{ctr}}\leq U_s-L_s\).
The quantities that are incomparable as upper bounds on
\(|s(x;P')-s_0|\) are the envelope width \(U_s-L_s\) and the modulus
\(\delta_s\): either may be strictly smaller, depending on
\((x,P,\varepsilon)\).
\end{fevcorollary}

\begin{proof}
Theorem~\ref{thm:fev30-envelopes} places both \(s_0\) and \(s(x;P')\) in
\([L_s,U_s]\), so the absolute deviation from \(s_0\) is at most
\(D_s^{\mathrm{ctr}}\), which cannot exceed the width \(U_s-L_s\).
Proposition~\ref{prop:fev30-margins} gives
\(D_s^{\mathrm{ctr}}\leq\delta_s\).
Both envelope endpoints are attained in the product relaxation, so
\(D_s^{\mathrm{ctr}}\) is the exact maximum absolute deviation from
\(s_0\) on that domain.

Since \(U_s-L_s=(U_s-s_0)+(s_0-L_s)\), the width is at most
\(2D_s^{\mathrm{ctr}}\).
Both \(C_s\) are nondecreasing and \(1\)-Lipschitz.
The attained envelope endpoints lie in \([0,1]\), including for the
identity map \(C_\chi\), because all clash kernel values lie there.
This proves the first comparison in \eqref{eq:fev30-envelope-width}.
Any two points of \(I_i\) are at most \(2\varepsilon\) apart, so
\(M_i-m_i\leq\omega_s(2\varepsilon)\), and the difference vanishes when
the interval contains no threshold.
Averaging yields the second comparison.

\paragraph{Why \(2\varepsilon\) is necessary.}
Place one candidate coordinate at the origin and one target atom at
distance \(1.30\), and take \(\varepsilon=0.16\).
The clash interval \(I=[1.14,1.46]\) meets all three constancy pieces of
\(h_\chi\), so \(L_\chi=0\), \(U_\chi=1\), and \(U_\chi-L_\chi=1\),
whereas
\(\omega_\chi(0.16)\nu_\chi=0.65=\delta_\chi\).
Here \(s_0=0.35\) and \(D_\chi^{\mathrm{ctr}}=0.65\).
Thus \(U-L=1>\delta=D^{\mathrm{ctr}}\): the modulus is tighter than the
envelope width as a bound on \(|s'-s_0|\).

\paragraph{A threshold endpoint.}
With \(s=\chi\), \(d_1(P)=1.15\), and \(\varepsilon=0.30\),
the closed interval is \([0.85,1.45]\). Since \(h_\chi(1.45)=0\),
its enclosure is \([0,1]\), whereas \(\omega_\chi(0.30)=0.65\).
Here too the width and the centered deviation are different quantities.

\paragraph{Envelope tighter than the modulus.}
Place \(d_1(P)=3.25\) and \(\varepsilon=1.25\) with \(s=q\).
Then \(I_1=[2.00,4.50]\), so \(h_q\equiv 1\), \(U_q-L_q=0\), while
\(\omega_q(1.25)=2.5\) and \(\nu_q=1\), hence \(\delta_q=1\).
Thus \(U-L=0<\delta=1\): the envelope width is tighter.
\end{proof}

\begin{fevproposition}[No global Lipschitz bound]
\label{prop:fev30-nolip}
There is no finite constant \(L\) such that
\(|q(x;P')-q(x;P)|\leq L\,d_{\mathrm H}(P,P')\)
for every one-coordinate candidate and every pair of nonempty finite
targets.
The same statement holds for \(\chi\).
\end{fevproposition}

\begin{proof}
Place the single coordinate at the origin and all target atoms on
the positive first coordinate axis.
Singleton targets at distances \(4.5\) and \(4.5+t\) (\(t>0\)) have
Hausdorff distance \(t\) and contact scores \(1\) and \(0\).
Singleton targets at distances \(1.15-t\) and \(1.15\)
(\(0<t<1.15\)) give clash difference \(0.65\).
In both cases the ratio of score difference to \(t\) is unbounded as
\(t\downarrow 0\).
\end{proof}

\begin{fevremark}[Probe displacements and useful radii]
\label{rem:fev30-amplitude}
With a unit radial direction and a tangential direction that is either
zero or a perpendicular unit vector, the flexible amplitude is \(1.65\).
The radial, unit-tangential, and unit-tangential diagonal lengths are
\(1.65\), \(1.2375\), and \(0.9075\sqrt{2}\), respectively.
With zero tangential direction, the last two lengths become \(0\) and
\(0.9075\).
The non-flexible fallback amplitude is \(0.879\).
Thus each defined mode displaces each target atom by at most \(1.65\).
Matching each atom to its displaced copy in both directions gives
\(d_{\mathrm H}(P^0,\widetilde P_G^m)\leq 1.65\).
This bound concerns the generated probes and establishes no coverage of
natural target conformations.

At a large radius, an enclosure may reduce to the codomain \([0,1]\).
An interval enclosure can nevertheless improve on the centered modulus
bound.
For example, if every \(d_i(P)=2.85\) and \(\varepsilon=1.65\), then
every distance interval is \([1.20,4.50]\), giving
\([L_q,U_q]=[0.35,1]\).
Here \(\omega_q(\varepsilon)\nu_q=2.50\), so the centered modulus
bound alone gives \([0,1]\).
\end{fevremark}

\subsection{Transfer from probe neighborhoods}

\begin{fevdefinition}[Declared probe neighborhoods]
\label{def:fev30-neighborhoods}
For a fixed genome and candidate, finite radii
\(\varepsilon_m\geq 0\) are prescribed. Over the class of finite
nonempty target atom sets in the fixed scoring frame, define
\[
\mathcal U_m
=
\{P:d_{\mathrm H}(P,\widetilde P_G^m)\leq\varepsilon_m\},
\qquad
L_s^m=L_s(x;\widetilde P_G^m,\varepsilon_m),
\quad
U_s^m=U_s(x;\widetilde P_G^m,\varepsilon_m).
\]
\end{fevdefinition}

\begin{fevtheorem}[Union and intersection enclosures]
\label{thm:fev30-transfer}
For the neighborhoods in Definition~\ref{def:fev30-neighborhoods}
and \(s\in\{q,\chi\}\):
\begin{enumerate}
\item if \(P^\star\in\bigcup_m\mathcal U_m\), then
\begin{equation}
\label{eq:fev30-union}
\min_m L_s^m
\leq s(x;P^\star)
\leq\max_m U_s^m;
\end{equation}
\item if \(P^\star\in\bigcap_{m\in J}\mathcal U_m\) for
\(\emptyset\neq J\subseteq\{0,\ldots,K-1\}\), then
\begin{equation}
\label{eq:fev30-intersection}
\max_{m\in J}L_s^m
\leq s(x;P^\star)
\leq\min_{m\in J}U_s^m.
\end{equation}
\end{enumerate}
Membership in an unspecified single neighborhood supports only the
union statement.
If the active index set
\(J_\star=\{m:P^\star\in\mathcal U_m\}\) is known and nonempty, item~2 with
\(J=J_\star\) is valid and at least as tight as item~1.
A union bound restricted to the known active indices is also valid,
but is generally weaker than their intersection bound.
\end{fevtheorem}

\begin{proof}
Theorem~\ref{thm:fev30-envelopes} places the score in every interval whose
neighborhood contains \(P^\star\).
A point of one interval lies between the global minimum lower endpoint
and the global maximum upper endpoint.
A point of every interval indexed by \(J\) lies in their intersection.
\end{proof}

\begin{fevcorollary}[Profile form of a union lower bound]
\label{cor:fev30-profile-transfer}
Let \(E_q=\max_m(q_m-L_q^m)\) and
\(E_\chi=\max_m(U_\chi^m-\chi_m)\).
Both are nonnegative.
For every \(P^\star\in\bigcup_m\mathcal U_m\),
\begin{equation}
\label{eq:fev30-profile-transfer}
q(x;P^\star)\geq[L_{\mathrm{prof}}-E_q]_+,
\qquad
\chi(x;P^\star)\leq\min\{1,\,\chi_{\max}+E_\chi\}.
\end{equation}
The direct bounds in \eqref{eq:fev30-union} are at least as tight.
\end{fevcorollary}

\begin{proof}
Since each radius is nonnegative, its center belongs to its own
neighborhood. The enclosure theorem therefore gives
\(L_q^m\leq q_m\) and \(\chi_m\leq U_\chi^m\), proving nonnegativity.
Thus \(L_q^m\geq q_m-E_q\geq L_{\mathrm{prof}}-E_q\) and
\(U_\chi^m\leq\chi_m+E_\chi\leq\chi_{\max}+E_\chi\).
Apply \eqref{eq:fev30-union}.
\end{proof}

\begin{fevremark}[Prospective versus retrospective radii]
\label{rem:fev30-retrospective}
Prospective bounds prescribe the radii independently of the held-out
target; coverage remains a separate hypothesis.
Setting \(\varepsilon_m=d_{\mathrm H}(P^\star,\widetilde P_G^m)\) after
\(P^\star\) is supplied yields an a~posteriori intersection enclosure.
It is not an input of \(f_{\mathrm{pres}}\) or \(f_{\mathrm{rob}}\).
For an unobserved natural target, coverage needs an additional hypothesis
or an independently justified model of target variation.
For an observed target, neighborhood membership can instead be checked
directly.
\end{fevremark}

\begin{fevtheorem}[Combining valid probes before score aggregation]
\label{thm:fev30-joint-intervals}
Let $\emptyset\ne J\subseteq\{0,\ldots,K-1\}$ index probe
neighborhoods that all contain
the same target $P^\star$. For each candidate coordinate, define
\begin{align}
\label{eq:fev30-joint-distance}
\underline d_i^J
&=\max_{m\in J}[d_i(\widetilde P_G^m)-\varepsilon_m]_+,
&
\overline d_i^J
&=\min_{m\in J}\{d_i(\widetilde P_G^m)+\varepsilon_m\},
\\
I_i^J&=[\underline d_i^J,\overline d_i^J].\nonumber
\end{align}
Every interval $I_i^J$ is nonempty. For $s\in\{q,\chi\}$, set
\begin{align}
\label{eq:fev30-joint-score}
L_s^J
&=C_s\!\left(\frac1n\sum_i\min_{d\in I_i^J}h_s(d)\right),
&
U_s^J
&=C_s\!\left(\frac1n\sum_i\max_{d\in I_i^J}h_s(d)\right).
\end{align}
Then
\begin{equation}
\label{eq:fev30-joint-nesting}
\max_{m\in J}L_s^m
\leq L_s^J
\leq s(x;P^\star)
\leq U_s^J
\leq\min_{m\in J}U_s^m.
\end{equation}
The two new endpoints are exact over the Cartesian-product relaxation
$\prod_i I_i^J$, but need not be attained by a common target geometry.
\end{fevtheorem}

\begin{proof}
For each $m\in J$, Lemma~\ref{lem:fev30-distance} places $d_i(P^\star)$
inside $I_i(\widetilde P_G^m,\varepsilon_m)$.
It therefore belongs to their intersection, which is exactly $I_i^J$,
so every $I_i^J$ is nonempty. For each coordinate,
\[
\min_{d\in I_i^J}h_s(d)
\leq h_s\bigl(d_i(P^\star)\bigr)
\leq\max_{d\in I_i^J}h_s(d).
\]
Averaging and applying the nondecreasing map $C_s$ proves
$L_s^J\leq s(x;P^\star)\leq U_s^J$ directly. This argument does not
require the intersected intervals to have a common center or radius.
For every $m\in J$, the inclusion
$I_i^J\subseteq I_i(\widetilde P_G^m,\varepsilon_m)$ gives
\[
\min_{I_i^J}h_s\geq\min_{I_i(\widetilde P_G^m,\varepsilon_m)}h_s,
\qquad
\max_{I_i^J}h_s\leq\max_{I_i(\widetilde P_G^m,\varepsilon_m)}h_s.
\]
Average these inequalities and apply the monotone map $C_s$.
Taking the largest per-probe lower bound and smallest per-probe upper
bound proves the outside inequalities in \eqref{eq:fev30-joint-nesting}.
Independent attainment of the scalar extrema proves exactness on the
relaxed domain, as in Theorem~\ref{thm:fev30-envelopes}.
\end{proof}

This refinement can be strict. Consider two candidate coordinates with
probe distance vectors $(2.4,4.4)$ and $(2.0,4.0)$.
Use radii $0.3$ and $0.5$, respectively.
The separate contact enclosures are $[0.5,1]$ and $[0.675,1]$.
The intersected coordinate intervals are $[2.1,2.5]$ and $[4.1,4.5]$,
where both contact contributions equal one.
Consequently, the refined contact enclosure is $[1,1]$.
The example is geometrically consistent: on one axis, place candidates
at $0,100$, the two probe atom sets at $\{2.4,104.4\}$ and
$\{2.0,104.0\}$, and the common target at $\{2.2,104.2\}$.
Its Hausdorff distance from each probe is $0.2$.
This is an analytical score example, not a claim of chemical validity.

The refinement requires simultaneous membership for every index in $J$.
It cannot replace the union bound when only one unspecified covering
probe is known to exist. An empty $I_i^J$ disproves simultaneous
membership, whereas nonempty coordinate intervals alone do not prove
that any common target geometry exists.

\subsection{Objective ranges}

Write \(\mathcal S_G\) and \(\mathcal F_G\) for the stable and flexible
residue sets of \(G\), and
\(d_r(x)=\min_i\|p_i(x)-c_r\|_2\).
For nonempty \(\mathcal S_G\),
\begin{equation}
\label{eq:fev30-anchor}
a(x;G)
=
|\mathcal S_G|^{-1}
\sum_{r\in\mathcal S_G}
\exp\!\bigl(-(d_r(x)-3.1)^2/4.5\bigr);
\end{equation}
otherwise \(a=0\).
Hence \(a\in[0,1]\).
For \(n\geq 2\), the methods appendix sets
\begin{equation}
\label{eq:fev30-geometry-scores}
\begin{aligned}
g(x)
&=
C\Bigl(
1-0.16\sum_i\mathbf{1}_{\{d_i^{\mathrm{self}}<0.55\}}
-0.03[n-12]_+
\Bigr),
\\
c(x)
&=
n^{-1}\sum_i\mathbf{1}_{\{
\exists j\neq i:\ 0.65\leq\|p_i-p_j\|_2\leq 2.70\}},
\end{aligned}
\end{equation}
with \(d_i^{\mathrm{self}}=\min_{j\neq i}\|p_i-p_j\|_2\), and with the
conventions \(g=c=1\) for \(n=1\) and \(g=c=0\) for an empty candidate.
Thus \(g,c\in[0,1]\).
Also \(\kappa_q,\Delta_q\in[0,1]\) by \eqref{eq:fev30-profile}.

Let \(t_{ir}=\|Q_r^\top(p_i-c_r)\|_\infty/(b_{\ell_r}z_r)\) and
\(w_r^{\mathrm{box}}=\min\{1,\max\{0.1,z_r\}\}\).
The methods penalty maps are
\begin{align}
\label{eq:fev30-penalties}
r_{\mathrm{flex}}(x;G)
&=
\frac{
0.85\sum_{r\in\mathcal F_G}[1-d_r(x)/3]_+
+0.05\sum_{r\in\mathcal F_G}w_r^{\mathrm{box}}
\sum_i\mathbf{1}_{\{t_{ir}\leq 0.55\}}
}{\max\{1,|\mathcal F_G|\}},
\nonumber\\
\pi_{\mathrm{forb}}(x;G)
&=
n^{-1}\sum_i\sum_{\substack{r\in\mathcal R_G:\\ \ell_r=\mathrm{forbidden}}}
\Bigl(
\mathbf{1}_{\{t_{ir}\leq 0.5\}}
+0.2\,\mathbf{1}_{\{0.5<t_{ir}<0.7\}}
\Bigr),
\\
e(x)
&=
\sum_{o\in\mathcal H_x}w_o,
\qquad
\Pi
=
0.95r_{\mathrm{flex}}+1.25\pi_{\mathrm{forb}}
+0.70\chi_{\max}+0.40e,
\nonumber
\end{align}
where \(\mathcal H_x\) is a finite indexed sequence of edit records
(repeated operations are counted separately), and every record
weight satisfies \(0\leq w_o\leq 0.06\).
All summands are nonnegative.
The reference candidate objectives are
\begin{align}
\label{eq:fev30-objectives}
f_{\mathrm{pres}}
&=
1.50a+q_0+0.80q_{\mathrm{rob}}-1.40\chi_{\max},
\nonumber\\
f_{\mathrm{rob}}
&=
0.90g+0.80c+0.75\kappa_q-0.75\Delta_q-\Pi.
\end{align}
In particular, \(q_{\mathrm{rob}}\) appears inside \(f_{\mathrm{pres}}\) by
definition: the preservation objective already mixes input-state contact
with probe-profile contact.

\begin{fevproposition}[Range and aggregation bounds]
\label{prop:fev30-ranges}
Every scored candidate with \(n(x)\geq 1\) satisfies
\[
-1.40\leq f_{\mathrm{pres}}\leq 3.30,
\qquad
-0.75-\Pi\leq f_{\mathrm{rob}}\leq 2.45.
\]
These endpoint inequalities are obtained by independent component
extrema and need not be attained jointly by one molecular candidate.
For any fixed integers \(N\geq 1\) and \(H\geq 0\), if
\(n(x)\leq N\) and the edit history has at most \(H\) records, then
\(\Pi\leq B(N,H,R)\) with
\begin{equation}
\label{eq:fev30-budget}
B(N,H,R)
=
0.95(0.85+0.05N)+1.25R+0.70+0.024H.
\end{equation}
For a finite nonempty adapted pool \(\mathcal X_G\), define
\[
M_j=\max_{x\in\mathcal X_G} f_j(x;G),\qquad
\mu_j=|\mathcal X_G|^{-1}\sum_{x\in\mathcal X_G} f_j(x;G).
\]
Whenever the standard aggregation
\(F_j=\lambda M_j+(1-\lambda)\mu_j-\eta_j\Omega(G)\) is used,
\begin{equation}
\label{eq:fev30-aggregate}
\mu_j-\eta_j\Omega(G)
\leq F_j
\leq M_j-\eta_j\Omega(G).
\end{equation}
If \emph{every} \(x\in\mathcal X_G\) satisfies \(n(x)\leq N\) and
\(|\mathcal H_x|\leq H\), then, with
\(\eta_{\mathrm{pres}}=0\) and \(\eta_{\mathrm{rob}}=0.10\),
\[
-1.40\leq F_{\mathrm{pres}}\leq 3.30,
\qquad
-0.85-B(N,H,R)\leq F_{\mathrm{rob}}\leq 2.45.
\]
\end{fevproposition}

\begin{proof}
Substitute the component ranges into \eqref{eq:fev30-objectives}:
\(a,q_0,q_{\mathrm{rob}},g,c,\kappa_q,\Delta_q,\chi_{\max}\in[0,1]\)
and \(\Pi\geq 0\).

For the budget, if \(\mathcal F_G=\emptyset\) then \(r_{\mathrm{flex}}=0\).
Otherwise each proximity term satisfies \([1-d_r/3]_+\leq 1\) and each
\(w_r^{\mathrm{box}}\leq 1\), while
\(\sum_i\mathbf{1}_{\{t_{ir}\leq 0.55\}}\leq n\leq N\).
Hence the numerator of \(r_{\mathrm{flex}}\) is at most
\(0.85|\mathcal F_G|+0.05|\mathcal F_G|N\), and dividing by
\(|\mathcal F_G|\) yields \(r_{\mathrm{flex}}\leq 0.85+0.05N\).
For each coordinate and each forbidden residue, the two occupancy
indicators in \eqref{eq:fev30-penalties} are disjoint, so their
contribution is at most one; summing over at most \(R\) forbidden
residues and averaging over \(n\) coordinates gives
\(\pi_{\mathrm{forb}}\leq R\).
Finally \(\chi_{\max}\leq 1\) and \(e\leq 0.06H\), proving
\eqref{eq:fev30-budget}.

Since \(\mu_j\leq M_j\), their convex combination lies between them;
subtracting \(\eta_j\Omega(G)\) gives \eqref{eq:fev30-aggregate}.
The uniform assumptions on the pool give
\(f_{\mathrm{rob}}(x;G)\geq -0.75-B(N,H,R)\) for every candidate.
The lower bound on \(F_{\mathrm{rob}}\) follows together with
\(\eta_{\mathrm{rob}}\Omega(G)\leq 0.10\).
\end{proof}

\begin{fevremark}[Coordinate and history bounds are separate assumptions]
\label{rem:fev30-size-bounds}
The integers $N$ and $H$ in Proposition~\ref{prop:fev30-ranges} bound
coordinates per candidate and records per edit history. They are not the
inner or outer population sizes $N_x$ and $N_G$.
In particular, the maximum inner population of twelve does not imply
$n(x)\leq12$. A fixed number of adaptation rounds alone also does not
specify $N$ or $H$ without bounds on initialization and the records and
coordinates produced by each invoked operation.
For a realized finite nonempty pool of finite candidate records, one
may take $N=\max_x n(x)$ and $H=\max_x|\mathcal H_x|$; this gives an
a posteriori bound, not a uniform bound for every possible input.
\end{fevremark}

\begin{fevremark}[Equal-aggregate support comparison]
\label{rem:fev30-support}
If two standard-aggregation evaluations share identical \(M_j\) and \(\mu_j\) for both
objectives, then
\(F_{\mathrm{rob}}(G)-F_{\mathrm{rob}}(G')
=0.10\bigl(\Omega(G')-\Omega(G)\bigr)\).
In general, changing support also changes scores and probes, so this
identity is not a general preference for smaller support.
The coordinatewise maxima \(M_{\mathrm{pres}}\) and \(M_{\mathrm{rob}}\)
need not be attained by the same candidate.
The documented zero-return branch that assigns
\((F_{\mathrm{pres}},F_{\mathrm{rob}})=(0,0)\)
when boxes or adapted candidates are missing satisfies the numerical
ranges but is excluded from \eqref{eq:fev30-aggregate}; it is not a
feasibility penalty.
\end{fevremark}

\subsection{Selection: what is and is not preserved}

Encode each genome by residue roles, with an additional symbol for
absence.
The residue-role distance \(D_G\) is the Hamming distance between these
encodings: it is a metric on encodings and a pseudometric on full genomes
(scales are ignored).

For maximized objective pairs, record \(u\) dominates \(v\) when both
coordinates of \(u\) are at least those of \(v\), with at least one
strict inequality. Nondominated fronts are formed by repeatedly
removing the current nondominated records.
The outer selector visits increasing front rank, then decreasing
crowding, followed by objective sum and identifier. The counterexamples
below have equal first-front sums and specify the identifier order, so
they do not depend on the direction of objective-sum tie-breaking.
Crowding conventions are assumed to produce a defined total traversal
order, including fronts with a constant objective coordinate.
The counterexample uses the standard boundary-crowding convention:
both members of a two-record front have infinite crowding.
It defers a record whose role distance to an already accepted record is
at most two, then backfills deferred records.

\begin{fevproposition}[Outer survivor facts and a dominance failure]
\label{prop:fev30-selection}
For a finite input population \(U\) of uniquely identified records with
finite real objective pairs, and integer capacity \(N_G\geq 1\), the
selected set is a subset of \(U\) of size \(\min\{N_G,|U|\}\).
For nonempty \(U\), at least one selected record lies on the first front.
Distinct records accepted before backfill have role distance at least
three; this separation need not survive backfill.
For every integer capacity \(N_G=N\geq2\) and every target size
\(R\geq3(N-1)+2\), there exists an input on which the selector retains
a dominated record while excluding one of its dominators.
Thus dominance preservation is not guaranteed uniformly over target
sizes and scored input populations. The displayed condition on \(R\)
is sufficient for the construction below; it is not asserted to be
necessary. At capacity one, the sole survivor, when present, lies on
the first front.
\end{fevproposition}

\begin{proof}
Until capacity is reached, each visited record is accepted or deferred.
If the first pass finishes below capacity, every input record has been
accepted or deferred, so backfill fills capacity or exhausts the remainder.
The first visited record of nonempty \(U\) lies on the first front and is
accepted.
Each first-pass acceptance has integer distance greater than two from
all prior acceptances. To see why backfill can violate this separation,
take capacity two and two records whose role encodings differ at exactly
one residue. Whichever record is visited first is accepted; the other
is deferred and then restored by backfill. Both survive with distance one.

Fix integers \(N\geq2\) and \(R\geq R_0:=3(N-1)+2\).
For an operator-level counterexample at capacity \(N_G=N\),
take \(N+1\) records with objective pairs
\[
C=(2,0),\qquad A=(1,1),\qquad
B_j=(t_j,1-t_j),\quad t_j=\tfrac14+\tfrac{j}{2N},\quad 1\leq j<N.
\]
The first front is \(\{C,A\}\); both have boundary crowding and sum two.
Let identifiers place \(C\) first. The \(B_j\) are pairwise nondominated;
\(A\) dominates each of them, whereas \(C\) dominates none.
Thus they form the second front.

On the first \(R_0\) residue positions, let the role encoding of \(C\)
be stable everywhere except at position \(R_0\), which is flexible.
Let \(A\) change only the first position to flexible. Partition the first
\(3(N-1)\) positions into disjoint triples \(J_j\), and let \(B_j\)
change exactly the triple \(J_j\) to flexible. At every remaining
position \(R_0+1,\ldots,R\), give all records the same stable role.
Assign every record any common admissible scale vector. Every record
has both a stable and a flexible role. Then
\[
D_G(C,A)=1,\qquad D_G(C,B_j)=3,\qquad
D_G(B_j,B_k)=6\quad(j\ne k).
\]
The selector accepts \(C\), defers \(A\), and accepts all \(B_j\),
in any within-front order, filling capacity before backfill.
It therefore excludes a dominator of every selected \(B_j\).
The reference capacity \(N_G=8\) is covered whenever \(R\geq23\).
This disproves an unconditional guarantee for the selection operator
on arbitrary scored records; it does not assert molecular realizability
of the specified pairs and encodings.
\end{proof}

\begin{fevremark}[Why the target-size qualification matters]
\label{rem:fev30-small-target}
For \(R\leq2\), any two role encodings have \(D_G\leq2\).
For a nonempty input and positive capacity, the first record is accepted;
all subsequent records reached in the first pass are deferred.
The final selected set is therefore the first \(\min\{N_G,|U|\}\)
records of the original traversal; no backfill is needed when the first
acceptance already fills capacity. Every strict dominator of a record has a lower
front rank and occurs earlier in that traversal. Consequently, on such
small targets the selector cannot retain a record while excluding one
of its dominators. This observation does not guarantee retention of all
nondominated records when their number exceeds capacity.
The negative statement in Proposition~\ref{prop:fev30-selection} is
therefore a statement about specified sufficiently large target sizes,
not about every fixed \(R\geq1\).
\end{fevremark}

\begin{fevremark}[Inner selection and main-text wording]
\label{rem:fev30-inner}
For inner capacity \(N_x=N\geq2\), use the same objective pairs
and centroids \(c_C=(0,0,0)\), \(c_A=(0.5,0,0)\), and
\(c_{B_j}=(2j,0,0)\). With the \(1\)~\AA\ centroid niche,
\(C\) is accepted, \(A\) is deferred, and all \(B_j\) are accepted.
This includes every capacity allowed by the reference inner rule
\(N_x=\max\{4,\min\{12,2|\mathcal X_0|\}\}\).
Here too the objective pairs and centroids are inputs to the selection
operator; their simultaneous realization by the molecular evaluator is
not asserted.
Nondominated rank is a selection criterion; unconditional preservation
of nondominated records is not a property of the documented rule.
The documented rule therefore uses nondominated sorting with crowding
and niching, without guaranteeing retention of every nondominated record
or of every dominator of a retained record.
\end{fevremark}

\begin{fevremark}[Evaluation count]
\label{rem:fev30-count}
If initialization evaluates \(N_G\) genomes and each of
\(g_{\max}\) completed generations evaluates \(n_o\) offspring per
parent in a population of size \(N_G\), there are
\(N_G(1+g_{\max}n_o)\) scheduled genome evaluations.
The reference values \((N_G,g_{\max},n_o)=(8,4,2)\) give \(72\).
This count assumes the stated schedule is completed without early
termination or extra evaluations.
It counts repeated genomes, not distinct genomes, low-level score calls,
or wall-clock time.
Finite scheduling implies termination only when each invoked subroutine
also terminates and all required scoring and feedback operations are
defined. For an empty parent pool, the reference implementation uses
infinite nearest-coordinate distances and zero edit-use counts, yielding
zero feedback without evaluating a minimum over an empty set.
\end{fevremark}

\subsection{Established properties and their role in adaptation}

The analysis establishes geometric control of the editing prior,
deterministic finite-probe score bounds, and target-transfer enclosures
under the stated Hausdorff-neighborhood conditions.
Proposition~\ref{prop:fev30-geometry} bounds the spatial extent of the
FlexBox prior. Proposition~\ref{prop:fev30-profile} converts contact
retention and variation into a lower bound shared by all evaluated probes.
For $K=7$ and $\rho=0.60$, this includes
\[
q_{\min}\geq q_{\mathrm{rob}}-\frac{12}{35}\Delta_q.
\]
Theorem~\ref{thm:fev30-transfer} extends the geometric score enclosures
to target conformations satisfying its explicit neighborhood conditions.

These results complement the molecular structural safeguards in
Appendix~\ref{app:adaptation}. Permitted rigid and annotated torsional
transformations preserve the supplied molecular graph and covalent bond
lengths, while the acceptance gate enforces coordinate and bond-length
consistency before survivor selection. The geometric descriptors quantify
spatial support and target contacts, and the category-specific validation
specification addresses chemical and physical structure checks.
Together, these components give an explicit account of structural
preservation, perturbation scoring, and conditional cross-state score
control within the implemented adaptation procedure.

\end{document}